\documentclass{article}

\usepackage[most]{tcolorbox}
\usepackage{microtype}
\usepackage{graphicx}
\usepackage{dsfont}
\usepackage{subcaption}
\usepackage[most]{tcolorbox}
\usepackage{pifont}  
\usepackage{booktabs} 
\usepackage{comment}
\usepackage[most]{tcolorbox}
\usepackage{tabularx}

\newtcbtheorem[number within=section]{example}{Example}{
  colback=gray!5,
  colframe=gray!50,
  fonttitle=\bfseries,
  boxrule=0.5pt,
  arc=2pt,
  left=6pt,
  right=6pt,
  top=6pt,
  bottom=6pt
}{ex}
\usepackage{hyperref}

\usepackage[preprint]{icml2026}

\usepackage{amsmath}
\usepackage{amssymb}
\usepackage{mathtools}
\usepackage{amsthm}

\usepackage[capitalize,noabbrev]{cleveref}

\theoremstyle{plain}
\newtheorem{theorem}{Theorem}[section]
\newtheorem{proposition}[theorem]{Proposition}
\newtheorem{lemma}[theorem]{Lemma}
\newtheorem{corollary}[theorem]{Corollary}
\theoremstyle{definition}
\newtheorem{definition}[theorem]{Definition}

\theoremstyle{remark}
\newtheorem{remark}[theorem]{Remark}

\usepackage[disable,textsize=tiny]{todonotes}
\usepackage[textsize=tiny]{todonotes}

\icmltitlerunning{Submission and Formatting Instructions for ICML 2026}

\begin{document}
\twocolumn[
\icmltitle{Don't Always Pick the Highest-Performing Model:
\\An Information Theoretic View of LLM Ensemble Selection}

\begin{icmlauthorlist}
\icmlauthor{Yigit Turkmen}{bilkent}
\icmlauthor{Baturalp Buyukates}{bham}
\icmlauthor{Melih Bastopcu}{bilkent}
\end{icmlauthorlist}

\icmlaffiliation{bilkent}{Department of Electrical and Electronics Engineering, Bilkent University, Ankara, Turkey}
\icmlaffiliation{bham}{School of Computer Science, University of Birmingham, Birmingham, UK}

\icmlcorrespondingauthor{Yigit Turkmen}{yigit.turkmen@ug.bilkent.edu.tr}
\icmlcorrespondingauthor{Melih Bastopcu}{bastopcu@bilkent.edu.tr}
\icmlcorrespondingauthor{Baturalp Buyukates}{b.buyukates@bham.ac.uk}

\icmlkeywords{Large Language Models, Ensemble Methods, Correlated Errors, Copula Models, Submodular Optimization}

\vskip 0.3in
]

\printAffiliationsAndNotice{This work was supported by Tubitak 2232-B program (Project No:124C533).\newline}

\begin{abstract}
Large language models (LLMs) are often ensembled together to improve overall reliability and robustness, but in practice models are strongly correlated. This raises a fundamental question: which models should be selected when forming an LLM ensemble? We formulate budgeted ensemble selection as maximizing the mutual information between the true label and predictions of the selected models. Furthermore, to explain why performance can saturate even with many models, we model the correlated errors of the models using Gaussian-copula and show an information-theoretic error floor for the performance of the ensemble. Motivated by these, we propose a simple greedy mutual-information selection algorithm that estimates the required information terms directly from data and iteratively builds an ensemble under a query budget. We test our approach in two question answering datasets and one binary sentiment classification dataset: MEDMCQA, MMLU, and IMDB movie reviews. Across all datasets, we observe that our method consistently outperforms strong baselines under the same query budget.
\end{abstract}

\section{Introduction}

Large language models (LLMs) are increasingly used as general purpose decision engines in various tasks \citep{minaee2025largelanguagemodelssurvey}, yet their outputs can be unstable: a single model may fail on particular prompts due to reasoning gaps \citep{huang-chang-2023-towards}, ambiguity \citep{lu-etal-2022-fantastically} or distribution shift \citep{yang2023rethinkingbenchmarkcontaminationlanguage}.
A standard approach in machine learning is to ensemble multiple predictors to improve robustness and accuracy of the overall system at the expense of extra latency and monitory costs \citep{bagging, dietterich, chen2025harnessingmultiplelargelanguage}. For modern LLM systems, this cost is often restrictive because real deployments typically cannot query a full pool of models for every request. 

In this paper, we focus on a binary classification problem in which statements are classified as true or false (positive or negative) using an ensemble of LLM responses as shown in Figure~\ref{Fig:system_model}. To this end, we study the budgeted ensemble selection problem: given a pool of $m$ LLMs and a per-query budget of $k < m$ calls, we want to form a subset $S$ of size $k$ that minimizes the final decision error under an optimal (or near-optimal) aggregator. The key difficulty is that LLM errors are strongly \emph{correlated}. Models from the same family tend to share training data, architectures, and failure modes, so adding another high-accuracy model may contribute little new information or may even degrade performance if it reinforces the same mistakes \citep{cemri2025multiagentllmsystemsfail, scalingagents}, a phenomenon illustrated in Figure~\ref{fig:diverse_vs_topk}. This makes the common selecting ``Top-$k$ by accuracy'' heuristic unreliable in practice. 

To address this unreliability and the resulting subset-selection problem, we take an information-theoretic view. We treat each model output as a noisy observation of a binary ground-truth label $Y$ and use mutual information to quantify how informative a selected subset is about $Y$. Our analysis goes from the ideal to the realistic: when model errors are effectively independent, we show that the best use of a budget of $k$ queries is simply picking the $k$ most accurate models, providing a clean rationale for the Top-$k$ baseline. In the correlated regime, however, the value of adding a model depends not only on its accuracy but also on how much \emph{new information} it contributes relative to the current subset; we formalize this trade-off via a decomposition that separates prediction redundancy from structured dependence in error patterns, motivating a \emph{greedy mutual-information selection} strategy. To model such dependencies and study scaling limits, we use a \emph{Gaussian-copula} to represent the joint law of LLM responses (equivalently, correlated error events) and characterize an explicit \emph{saturation floor} under a latent-factor interpretation: beyond a point, adding additional models cannot eliminate the uncertainty induced by shared difficulty. This difficulty-aware perspective helps explain why greedy mutual information selection yields strong gains in practice while still exhibiting diminishing returns in highly correlated ensembles.

\begin{figure}[t]
  \begin{center}
    \centerline{\includegraphics[width=0.9\columnwidth]{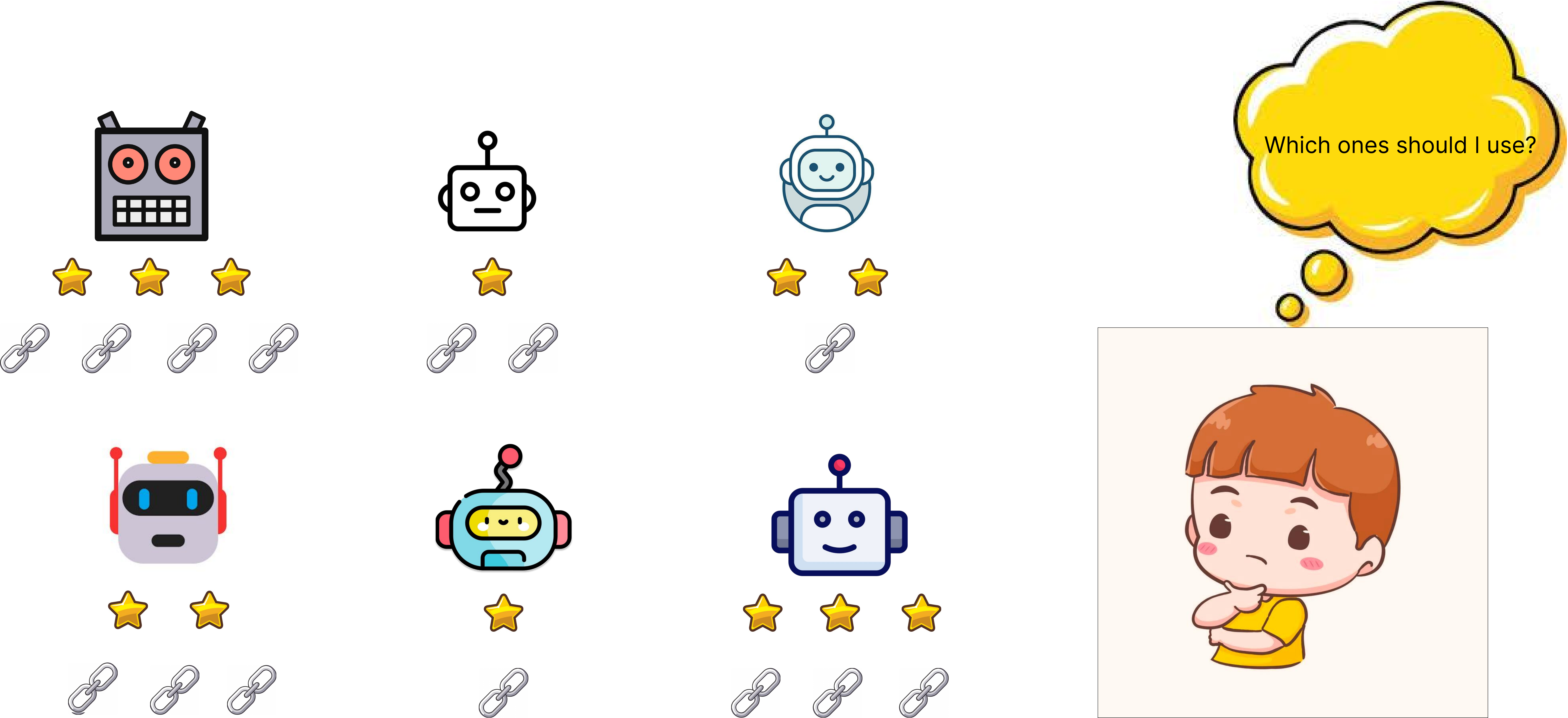}}
    \caption{
      Candidate LLM pool for ensembling: stars indicate each model’s accuracy, while chain links denote pairwise correlation (shared error patterns) with other models.
    }\label{Fig:system_model}
    \vspace{-0.5cm}
  \end{center}
\end{figure}

\begin{figure}[t]
\centering
\begin{minipage}[t]{0.48\columnwidth}
\begin{tcolorbox}[
  colback=green!5,
  colframe=green!60!black,
  coltitle=white,
  title={\small\textbf{(a) Diverse Ensemble}},
  fonttitle=\small,
  boxrule=0.8pt,
  arc=2pt,
  height=4.6cm,
  valign=top
]
\scriptsize
\textbf{Q:} Axonal transport is: Is the answer `Antegrade and retrograde'?

\textbf{Ground Truth:} \texttt{True}

\vspace{0.15cm}
\renewcommand{\arraystretch}{1.1}
\begin{tabular}{@{}l c c@{}}
\textbf{Model} & \textbf{Acc.} & \textbf{Ans.} \\
\midrule
GPT-4o & 80\% & \texttt{T} \textcolor{green!60!black}{\ding{55}} \\
Gemini-2.5-F. & 73\% & \texttt{F} \textcolor{red}{\ding{55}} \\
Claude-3.5-H. & 71\% & \texttt{T} \textcolor{green!60!black}{\ding{55}} \\
Llama-3.1-8B & 66\% & \texttt{T} \textcolor{green!60!black}{\ding{55}} \\
\midrule
\textbf{Avg.} & \textbf{72\%} & \textbf{3/4} \textcolor{green!60!black}{\ding{51}} \\
\end{tabular}

\vspace{0.1cm}

\end{tcolorbox}
\end{minipage}
\hfill
\begin{minipage}[t]{0.48\columnwidth}
\begin{tcolorbox}[
  colback=red!5,
  colframe=red!60!black,
  coltitle=white,
  title={\small\textbf{(b) Top-$k$ Accuracy}},
  fonttitle=\small,
  boxrule=0.8pt,
  arc=2pt,
  height=4.6cm,
  valign=top
]
\scriptsize
\textbf{Q:} Axonal transport is: Is the answer `Antegrade and retrograde'?

\textbf{Ground Truth:} \texttt{True}

\vspace{0.15cm}
\renewcommand{\arraystretch}{1.1}
\begin{tabular}{@{}l c c@{}}
\textbf{Model} & \textbf{Acc.} & \textbf{Ans.} \\
\midrule
GPT-5-Chat & 83\% & \texttt{F} \textcolor{red}{\ding{55}} \\
GPT-4.1 & 82\% & \texttt{F} \textcolor{red}{\ding{55}} \\
GPT-4o & 80\% & \texttt{T} \textcolor{green!60!black}{\ding{55}} \\
GPT-4.1-mini & 78\% & \texttt{F} \textcolor{red}{\ding{55}} \\
\midrule
\textbf{Avg.} & \textbf{81\%} & \textbf{1/4} \textcolor{red}{\ding{55}} \\
\end{tabular}

\vspace{0.1cm}

\end{tcolorbox}
\end{minipage}

\caption{Example from MEDMCQA. (a) A diverse ensemble with \textbf{lower average accuracy} (72\%) answers correctly by combining models from different families. (b) Selecting the ensemble with strongest models with \textbf{higher average accuracy} (81\%) fails because almost all GPT models share the same error pattern for this example. Please note that the original multiple-choice question is converted to binary format: for each candidate answer, we ask ``Is the answer `[candidate]'?" with ground truth $Y=+1$ (correct) or $Y=-1$ (incorrect).}

\label{fig:diverse_vs_topk}
\vspace{-0.5cm}
\end{figure}

\noindent\textbf{Contributions:} We make the following contributions:
\begin{itemize}
    \item We represent LLM decisions using a Gaussian-copula, which allows us to capture both the individual accuracies and the correlations among models. Our numerical results show that the error behavior of real models can be effectively captured by the simulated Gaussian-copula model. 
    \item We first consider independent LLM decisions. For this special case, we show in Theorem~\ref{thm:alignment} that if LLMs were to make independent errors, the optimal model selection would be trivially selecting the most accurate models. The breakdown of Top-$k$ selection in practice therefore arises solely from the underlying correlation structure.
    \item Since identifying the LLM subset that minimizes the probability of error is a combinatorial problem, we instead propose a greedy selection strategy that maximizes mutual information gain. In Theorem~\ref{thm:decomposition}, we characterize the underlying principle of this selection rule and establish its connection to the Maximum Relevance–Minimum Redundancy (mRMR) principle from feature selection~\citep{peng2005}.
    \item Then, we show that under uniform pairwise correlation, increasing the number of models $k$ does not necessarily drive the error probability to zero; instead, it converges to a fundamental limit characterized in Theorem~\ref{thm:saturation}. This result explains why ensembling an increasing number of correlated models may fail to further reduce the error probability.
    \item Finally, we evaluate our approach on two question-answering benchmarks and a binary sentiment classification task that are MEDMCQA, MMLU, and IMDB movie reviews and demonstrate consistent improvements over strong baselines such as the ``Top-k by accuracy" under identical query budgets.
\end{itemize}

\section{Related Work}
In this subsection, we provide related works on LLM ensembles, information theoretic feature selection problem, and Gaussian-copula models.

\noindent\textbf{LLM Ensembles:} The simplest approach to combining LLM outputs is majority voting, popularized by self-consistency~\citep{wang2023selfconsistencyimproveschainthought} for single-model sampling and then extended to multi-model settings through weighted voting~\citep{llm-synergy,llm-blender} and cascaded selection~\citep{chen2024frugalGPT}. Yet, a growing body of evidence suggests that naively adding more models does not guarantee performance improvement: recent scaling studies~\citep{scalingagents} observe diminishing or even negative returns from coordination once single-agent accuracy exceeds $\sim$45\%, while systematic failure  analysis~\citep{cemri2025multiagentllmsystemsfail} identifies inter-agent misalignment as a 
dominant failure mode in multi-agent systems.
These observations point to error correlation as the culprit. To address this, a recent study~\citep{beyond-majority} leverages second-order 
statistics for optimal weight aggregation, LLM-TOPLA~\citep{topla} introduces focal diversity metrics for ensemble pruning, and MUSE~\citep{muse} applies Jensen-Shannon divergence to select well-calibrated LLM subsets. Additionally, ~\citep{turkmen2025balancing} study the accuracy-timeliness tradeoffs in networked LLM systems under adaptive majority voting. In this work, we take a fundamentally different approach: rather than designing better aggregation rules, we ask which models to select in the first place, and provide a principled probabilistic framework based on Gaussian-copula correlation structure.

\noindent\textbf{Information-Theoretic Feature Selection:} The idea of greedily selecting elements to maximize mutual information originates in feature selection, where the seminal mRMR criterion~\citep{peng2005} balances relevance against redundancy. Brown et al.~\citep{brown2012} later unified a decade of variants 
under a common framework: all approximate conditional mutual information (CMI) and all penalize inter-feature correlation as redundancy.
Recent work extends this to dynamic settings via learned CMI 
policies~\citep{covert2023,gadgil2024estimating}, yet the core assumption persists.
We observe that this intuition does not transfer to ensemble selection.

\noindent\textbf{Gaussian-Copula Models:} Gaussian-copulas provide a flexible framework for modeling dependencies 
while preserving arbitrary marginals, and have found widespread use in various applications from 
credit risk modeling~\citep{li2000default} to mixed data 
analysis~\citep{nelsen2007introduction}. While recent work~\citep{pan2025llmsrefusequestionsknow}
applies bivariate Gaussian-copula 
to model refusal behavior within a single LLM, another work, \cite{elumar2025costaware}, uses a Gaussian-copula sampler to model dependent correctness for Monte Carlo confidence estimation. In this work, we model cross-LLM ensemble errors with a Gaussian-copula and leverage it for an information-theoretic saturation theorem and 
mutual information (MI)-driven greedy selection.

\section{System Model and Problem Formulation}
\label{sec:system_model}
In this work, we consider a binary classification problem carried out by a group of LLMs, where each incoming task or query involves evaluating a statement with a true/false (or positive/negative) outcome. We model the correct label for the task with a binary random variable $Y$, where $Y=+1$ denotes a true (or positive) statement and $Y=-1$ denotes a false (or negative) statement. We assume that the distribution of these statements are balanced, that is, we have $\mathbb{P}(Y=1) \!=\! \mathbb{P}(Y\!=\!-1) = 0.5$. In order to perform these tasks, we have access to an ensemble of $m$ LLMs, indexed by the set $[m] \triangleq \{1, \dots, m\}$.
For a given input instance, model $j$ produces a binary prediction (or a response) $X_j\in \{-1,+1\}$. We define the error indicator variable for LLM $j$ as $E_j \triangleq \mathds{1}(X_j \neq Y)$ where $\mathds{1}(x)$ denotes the indicator function defined as $\mathds{1}(x) = 1$ if $x$ is true and $\mathds{1}(x) = 0$, otherwise. In other words, $E_j$ takes value $1$ when model $j$'s prediction is incorrect, i.e., $X_j \neq Y$. The responses of the LLMs are then related to the errors by $X_j \!=\! Y \!\cdot \!(-1)^{E_j}\!$. We denote the vector of responses for a subset $S \!\subseteq \![m]$ as ${\bf X_S}\triangleq\! (X_j)_{j \in S}$. Next, we model the distribution of the responses collected from LLMs in the subset $S$.    

\subsection{Modeling Latent Errors with Gaussian-Copula}
\label{subsec:copula_model}
In general, when responses are collected from multiple LLMs, their decisions may not be independent. Models trained on similar data or architectures tend to exhibit correlated errors. To capture such a complex dependency structure of LLM errors, in this work, we model the joint error distribution using a Gaussian-copula \citep{nelsen2007introduction}.
With this goal, we define $\mathbf{Z} = (Z_1, \dots, Z_m)$ to be a latent random vector following a multivariate normal distribution $\mathbf{Z} \sim \mathcal{N}(\mathbf{0}, \mathbf{\Sigma})$ where $\mathbf{\Sigma} \in \mathbb{R}^{m \times m}$ is a correlation matrix with unit diagonal entries and off-diagonal entries given by $\rho_{ij}$ capturing the pairwise latent correlation between models $i$ and $j$ where $i\neq j$. In the case of equicorrelated ensembles where all the correlations of the models are equal to each other, i.e., $\rho_{ij} = \rho$ for all $i$ and $j$ with $i\neq j$, this dependence structure can be equivalently represented by a linear factor model:
\begin{align}
\label{eq:factor_model}
    Z_j = \sqrt{\rho} \, U + \sqrt{1-\rho} \, \xi_j,
\end{align}
where $U \sim \mathcal{N}(0,1)$ is a common latent factor shared by all models and $\xi_j \sim \mathcal{N}(0,1)$ are independent individual noise terms. Then, the binary error indicator $E_j$ can be generated by thresholding the latent variable $Z_j$ as
\begin{equation}
    E_j = \mathds{1}(Z_j < \tau_j),
\end{equation}
where the threshold $\tau_j = \Phi^{-1}(\epsilon_j)$ is determined by the marginal error probability $\epsilon_j = \mathbb{P}(E_j=1)$ of model $j$, and $\Phi(\cdot)$ is the standard normal cumulative distribution function and its inverse is denoted by $\Phi^{-1}(\cdot)$. 

By modeling the error via Gaussian-copula, we are able to: $i)$ choose the variables $\{\!\tau_j\!\}_{\!j=1\!\!}^m$ that uniquely determine the individual error probability $\epsilon_j$ for each model, independent of the others, and $ii)$ model the dependence structure through the covariance matrix $\mathbf{\Sigma}$ that governs the higher-order error correlations. In other words, the coefficient $\rho_{ij}$ captures the tendency of models to exhibit correlated errors on identical inputs. Using a Gaussian-copula, we can explicitly decouple the marginal statistics from the dependence structure. 

\subsection{Ensemble Aggregation and Selection}
When the joint error distribution of the LLMs is known, the optimal decision rule that minimizes the probability of error is the Maximum A Posteriori (MAP) estimator \citep{hajek2020probability}. Thus, for a given selected subset of models $S$ and their observed outputs ${\bf X_S} = {\bf x_S}$, the MAP rule is given by 
\begin{equation}
    \hat{Y}_{\text{MAP}}({\bf x_S}) = \operatorname*{argmax}_{y \in \{-1, +1\}} \mathbb{P}(Y=y \mid {\bf X_S} = {\bf x_S}).
\end{equation}
The performance of any subset $S$ is measured by its expected error probability given by $\mathbb{P}_e(S) \triangleq \mathbb{P}(\hat{Y}_{\text{MAP}}({\bf X_S}) \neq Y)$. We implement \cref{alg:map} for the MAP aggregation. 

\subsection{Problem Formulation}
Due to inference costs and latency constraints, collecting responses from all LLMs may not be feasible for each query. Thus, for a given query, our goal is to select the subset $S^{\star}_{k}$ of size $k$ minimizing the ensemble error which is given by
\begin{equation}
\label{eq:problem_statement}
    S^\star_k \in \operatorname*{argmin}_{S \subseteq [m], |S|=k} \mathbb{P}_e(S).
\end{equation}

Finding the optimal subset $S_k^\star$ requires selecting $k$ LLMs from the $m$ available models, which is combinatorially hard and depends on the joint error distribution $\mathbb{P}(E_1,\dots,E_m)$. In this work,  we explicitly allow \emph{correlated}  error models and model the joint error law via a Gaussian-copula latent structure as we described earlier.

With the Gaussian-copula model, there exists a latent vector $\mathbf{Z} \sim \mathcal{N}(\mathbf{0}, \mathbf{\Sigma})$ and thresholds $\{\tau_j\}$ such that
$E_j=\mathds{1}(Z_j<\tau_j)$, where $\mathbf{\Sigma}$ is a correlation matrix capturing the cross-model error dependence.
In the special equicorrelated case where $\rho_{ij}=\rho$, this is equivalently represented by the one-factor model
$Z_j$ in (\ref{eq:factor_model}), which enables closed-form saturation limits  that will be derived in Theorem ~\ref{thm:saturation}.  We explore the connection of our problem to the well-known \emph{independent binary symmetric channel (BSC) regime} in communications, which is recovered as the special case ${\bf \Sigma }= {\bf I}$, where ${\bf I}$ is the identity matrix (i.e., $\rho = 0$). This case is studied in Theorem~\ref{thm:alignment}.

\section{Information-Theoretic Analysis of Ensemble Selection}
\label{sec:theory}

The central challenge in ensemble selection is that the error-optimal subset $S^\star$ is combinatorially hard to find. Practical systems typically rely on heuristics like selecting top-$k$ accurate models. In this section, we analyze the optimality of this baseline and contrast it with greedy mutual information maximization, which serves as the natural information-theoretic alternative. We first prove that they are equivalent in the independent case, then derive the precise information-theoretic penalty introduced by correlations, and finally establish a fundamental limit on the achievable error. Moreover, we present Theorem~\ref{thm:oracle_price} in Appendix~\ref{sec:proof_saturation} which further explains why greedy approaches yet can saturate under latent difficulty.

\subsection{Optimal Subset Selection  Under Independence}
\label{subsec:alignment}

We begin by analyzing the baseline case where model errors are statistically independent. This serves as a control setting to demonstrate that accuracy-based and information-theoretic objectives coincide in the absence of response correlations. 

Before stating the following theorem, we note that our system resembles a well-known communication channel consisting of $m$ independent binary symmetric channels (BSCs), each of which corresponds to an LLM in our system, while the binary information source corresponds to the underlying classification task. Thus, our goal can be translated into selecting the best subset of $k$ BSCs so as to minimize the probability of error in this communication system.

\begin{theorem}[Alignment of Error and Mutual Information]
\label{thm:alignment}
Consider an ensemble of independent BSCs where channel $j$ has error rate $\epsilon_j < 1/2$. Let the channels be indexed by accuracy such that we have $\epsilon_1 \le \epsilon_2 \le \ldots \le \epsilon_m$. Let $H_k = \{1, \dots, k\}$ be the set of top-$k$ most accurate channels. Let $S^{MI}_k \in \arg\max_{|S|=k} I(Y;{\bf X_S})$ denote the subset of size $k$ that maximizes the mutual information between $Y$ and ${\bf X_S}$. For any subset $S \subseteq [m]$ with $|S|=k$, the top-$k$ set $H_k$ dominates $S$ both in terms of maximizing the mutual information and minimizing the probability of error. Thus, we have
\begin{align}
    I(Y; {\bf X_{H_k}}) &\ge I(Y; {\bf X_S}), \label{eq:alignment_mi}\\
    \mathbb{P}_e(H_k) &\le \mathbb{P}_e(S). \label{eq:alignment_pe}
\end{align}
Consequently, when the channels are independent, the error-optimal subset $S^\star$ and the mutual-information optimal subset $S^{MI}_k$ coincide with $H_k$.
\end{theorem}
The proof of Theorem~\ref{thm:alignment} is provided in Appendix~\ref{Proof_of_theorem_4_1}. In the following remark, we state how the result of Theorem~\ref{thm:alignment} applies to our problem.

\begin{remark}    
 Theorem \ref{thm:alignment} confirms that if LLMs make independent errors, the optimum selection of LLMs would be trivial: simply picking the most accurate $k$ models is optimal. The failure of top-$k$ selection in practice is therefore purely a consequence of the correlation structure derived in the ensuing analysis.
\end{remark}

\subsection{Decomposition of Correlated Ensembles}
\label{subsec:decomposition}
When models are correlated, the independence assumption in Theorem \ref{thm:alignment} breaks. To quantify exactly how correlation alters the selection objective, we derive a decomposition of the marginal information gain.

\begin{theorem}[Accuracy-Redundancy-Error Decomposition]
\label{thm:decomposition}
Let $E_j = \mathds{1}(X_j \neq Y)$ be the error indicator for model $j$, and we define the marginal information gain of adding model $j$ to subset $S$ as
\begin{equation}
    \Delta(j \mid S) \triangleq I(Y; X_j \mid {\bf X_S}).
\end{equation}
Then, $\Delta(j \mid S)$ admits the following decomposition
\begin{equation}
\label{eq:general_decomposition}
    \!\!\Delta(j \!\mid \!S) \!=\! I(Y; X_j) \!-\! I(X_j; {\bf X_S}) \!+\! I(E_j; E_S) \!+\! \Lambda_j(S),\!
\end{equation}
where $E_S = \{E_j| j\in S\}$ and we can define the last term in (\ref{eq:general_decomposition}) as the label-dependence correction term
\begin{equation}
\label{eq:correction_term}
    \Lambda_j(S) \triangleq I(E_j; Y \mid E_S) - I(E_j; Y).
\end{equation}
Under the label-invariant error assumption, i.e., $(E_1, \ldots, E_m) \perp Y$, we have $\Lambda_j(S) = 0$ for all $j$ and $S$, the decomposition for $\Delta(j\mid S)$ reduces to 
\begin{equation}
\label{eq:decomp}
\Delta(j\mid S) = I(Y; X_j) - I(X_j; {\bf X_S}) + I(E_j; E_S).
\end{equation}
\end{theorem}
\begin{figure*}[t]
  \begin{center}
    \centerline{\includegraphics[width=0.6\textwidth]{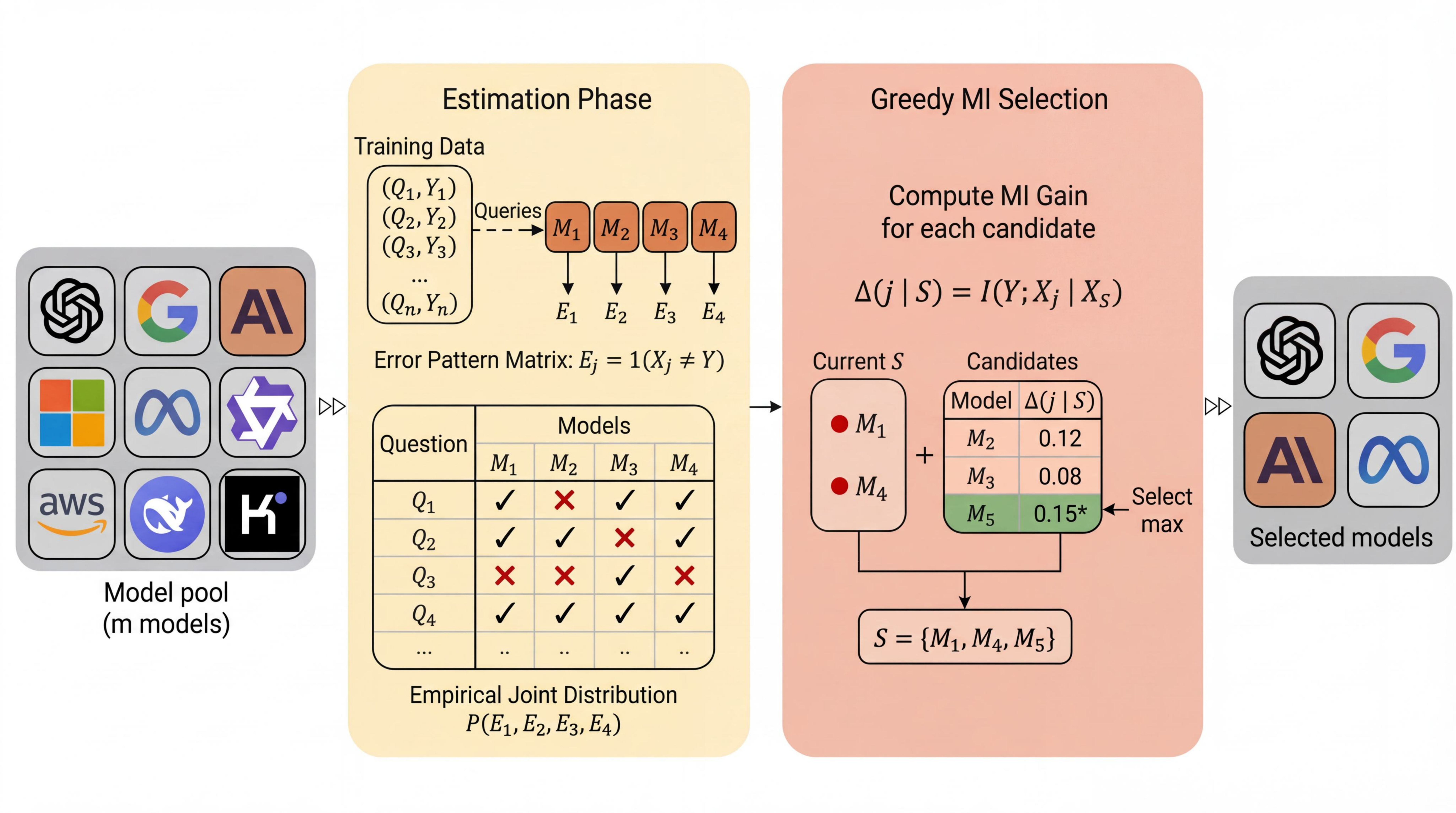}}
    \caption{
     Our Greedy Mutual Information (MI)-based model selection framework.
    }\label{Fig:alg1}
    \vspace{-0.5cm}
  \end{center}
\end{figure*}

The proof of Theorem~\ref{thm:decomposition} is provided in Appendix~\ref{Proof_of_Thm_4_3}. Under the label-invariant error assumption, this decomposition highlights two opposing forces: $i)$  Prediction Redundancy $I(X_j; {\bf X_S})$: Appears in (\ref{eq:decomp}) with a negative sign. This term penalizes models that mimic the ensemble's output. $ii)$ Error Correlation $I(E_j; E_S)$: Appears in (\ref{eq:decomp}) with a positive sign. Counter-intuitively, correlated errors can be beneficial when structured because knowing the error pattern of $S$ helps to decode $Y$.

Standard diversity heuristics often mirror the mRMR principle from feature selection~\citep{peng2005}, favoring high relevance while discouraging redundancy among predictors. 
Theorem~\ref{thm:decomposition} shows that under the label-invariant error assumption $I(E;Y)=0$, and the binary error representation $E_j=\mathds{1}(X_j\neq Y)$, the marginal information gain admits an additional term, $I(E_j;E_S)$, alongside the usual relevance and redundancy terms. 
Consequently, mRMR-style criteria based solely on $(I(Y;X_j),\, I(X_j; {\bf X_S}))$ are not guaranteed to align with mutual information-optimal subset selection in correlated ensembles, since they do not account for structured dependence among model \emph{errors}.

Thus, in Theorem~\ref{thm:decomposition}, we propose an information-theoretic criterion for deciding whether to include a model in a subset when models are correlated. As seen in (\ref{eq:decomp}), this criterion resembles the well-known mRMR-style feature selection problem \citep{peng2005}, augmented with an additional term, $I(E_j; E_S)$, which explicitly accounts for error correlations among models.

\subsection{Asymptotic Saturation Limit}
\label{subsec:saturation}
In Theorem~\ref{thm:alignment}, we show that under independent error assumptions, subset selection reduces to selecting the most accurate models.
The decomposition in Theorem~\ref{thm:decomposition} then explains how this picture changes under dependence, where correlation introduces additional structure beyond simple redundancy.
We now address the central question left open by these results: \emph{Can ensemble error vanish as the ensemble size grows when model errors are correlated?} Under the Gaussian-copula model with uniform pairwise correlation, we show that the answer to this is generally no: as $m$ grows, the error probability converges to a non-zero floor determined by the correlation structure, yielding a fundamental saturation limit.

\begin{theorem}[MAP Information Saturation]
\label{thm:saturation}
Consider an infinite ensemble generated by the Gaussian-copula model with uniform pairwise correlation $\rho > 0$ and the marginal accuracy $\alpha \triangleq 1-\epsilon > 1/2$. Let $\hat{Y}_{\text{MAP}}$ be the optimal MAP estimator given the full ensemble outputs $X_{[m]}$. The error probability is lower-bounded by:
\begin{align}
    \lim_{m \to \infty} \mathbb{P}(\hat{Y}_{\text{MAP}} \neq Y) = \Phi\left( \frac{\Phi^{-1}(1-\alpha)}{\sqrt{\rho}} \right) > 0.
\end{align}
\end{theorem}
The proof of Theorem~\ref{thm:saturation} is provided in Appendix~\ref{sec:proof_saturation}.

Theorem \ref{thm:saturation} establishes that the error floor observed in correlated ensembles is not negligible. It arises not from the sub-optimality of the aggregation rule (since the MAP is optimal), but from the information-theoretic capacity limit imposed by the latent common noise $U$. This implies that scaling the ensemble size $m$ cannot overcome structural correlations; performance improvements must instead come from reducing the effective $\rho$ via subset selection.

\section{Proposed Selection Algorithm: Greedy Mutual Information (Greedy MI)}
\label{sec:algorithm}

Motivated by the decomposition in Theorem \ref{thm:decomposition}, we propose an efficient greedy strategy to approximate the optimal subset $S^\star$. Since the optimization problem in (\ref{eq:problem_statement}) is NP-hard, we adopt a stepwise maximization approach. This greedy strategy is computationally tractable and is roughly related to relevance--redundancy feature selection. Crucially, in correlated ensembles the marginal gain contains additional dependence terms, shown in (\ref{eq:decomp}), so the classical mRMR intuition does not directly apply to our problem.

\subsection{Selection Criterion}
We construct the subset $S$ iteratively. Let $S_{k\!-\!1}$ be the set of $k\!-\!1\!$ models selected so far. At step $k$, we select the model $j^\star$ that maximizes the marginal information gain:
\begin{equation}
    j^\star = \operatorname*{argmax}_{j \in [m] \setminus S_{k-1}} \Delta(j \mid S_{k
    -1}).
\end{equation}
By using Theorem \ref{thm:decomposition}, we explicitly calculate this score by estimating the three components using empirical mutual information obtained from the training set (introduced in Section~\ref{Sec:num_results}); this split is used only to estimate statistics and to fit the empirical MAP estimator, and does not involve training any LLM parameters. We treat the model outputs as discrete random variables and estimate their joint distributions via frequency counts. Then, we calculate the estimated marginal information gain of model $j$ by 
\begin{equation}\label{eqn:mut_inf_est}
    \!\hat{\Delta}(j \!\mid\! S_{k
    \!-\!1}) \!=\! \hat{I}(Y; X_j) \!-\! \hat{I}(X_j; X_{S_{k\!-\!1}}) \!+\! \hat{I}(E_j; E_{S_{k\!-\!1}}).\!\!
\end{equation}
This update rule balances the individual accuracy against prediction redundancy while explicitly rewarding models that introduce structured, predictable error patterns. Finally, we provide Algorithm~\ref{alg:greedy_mi} to summarize the steps for the Greedy MI selection proposed in this work. Note that, in practice, we compute the full conditional mutual information $\Delta(j \mid S) = I(Y; X_j \mid X_S)$ directly to account for potential label-dependent error structure in the data. The overall, workflow is illustrated in Figure \ref{Fig:alg1}.

\begin{algorithm}[tb]
  \caption{Greedy MI Selection}
  \label{alg:greedy_mi}
  \begin{algorithmic}
    \STATE {\bfseries Input:} Estimate the dist. of set outputs $X_{[m]}$, labels $Y$, target size $k$
    \STATE Initialize $S \leftarrow \emptyset$
    \FOR{$t=1$ {\bfseries to} $k$}
      \STATE Estimate each term in \cref{eq:decomp} using \cref{alg:mi}
      \STATE Compute score $\hat{\Delta}(j|S)$ for all $j \notin S$ using \eqref{eq:decomp}
      \STATE $j^\star \leftarrow \operatorname{argmax}_{j} \hat{\Delta}(j|S)$
      \STATE $S \leftarrow S \cup \{j^\star\}$
    \ENDFOR
    \STATE {\bfseries Output:} Selected subset $S_k^*=S$
  \end{algorithmic}
\end{algorithm}

\section{Experimental Results}\label{Sec:num_results}

\subsection{Experiments on MEDMCQA}
\begin{figure}[ht]
\centering
\begin{subfigure}[b]{0.48\columnwidth}
    \centering
    \includegraphics[width=\textwidth]{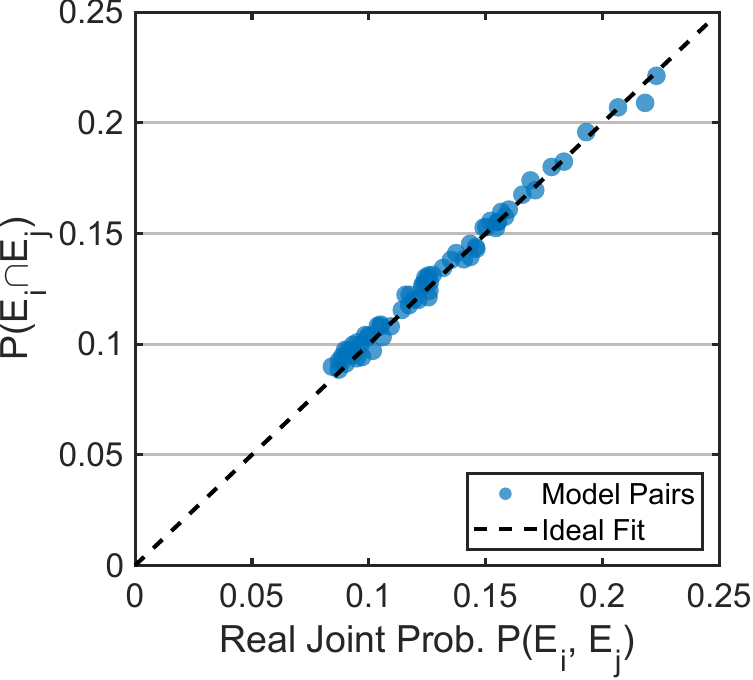}
    \caption{}
    \label{fig:medmcqa_temp07run1_scatter}
\end{subfigure}
\hfill
\begin{subfigure}[b]{0.45\columnwidth}
    \centering
    \includegraphics[width=\textwidth]{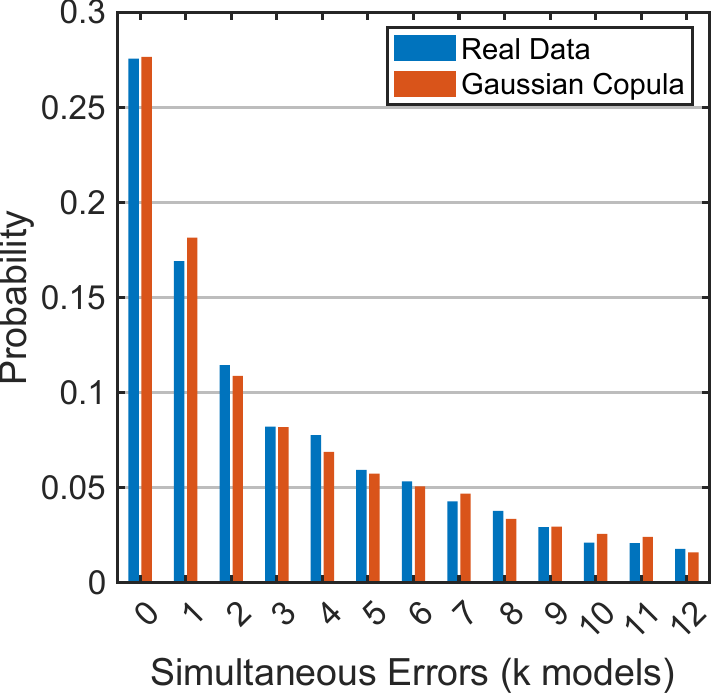}
    \caption{}
    \label{fig:medmcqa_temp07run1_heavy}
\end{subfigure}

\caption{Gaussian-copula validation on MEDMCQA for \(\text{temp}=0.7\), (\emph{run 1}). (a) Empirical versus copula-modeled pairwise joint error probs. $P(E_i\cap E_j)$. $\!$(b) Comparison of higher-order simultaneous error distributions between real data and copula samples.}
\label{fig:medmcqa_temp07run1_copulavalid}
\end{figure}

We first evaluate our algorithm on the MEDMCQA \cite{medmcqa_paper}, a large-scale multiple choice benchmark for medical question answering that spans broad range of clinical topics and requires domain-specific knowledge and reasoning. Therefore, MEDMCQA is well suited to our setting because medical queries often show higher difficulty and stronger shared failure modes across LLMs. To keep evaluation tractable, we run experiments on the validation set only, which contains 4183 samples.

\noindent\textbf{Binary Conversion Procedure:} To cast the multi-class classification task into our binary framework, we convert each sample into two distinct positive/negative queries: \textit{1) Positive Query:} We append the ground-truth answer to the original question, forming: ``Question: [Q]. Is the answer `[Correct Answer]'?'' This query has ground-truth label $Y = +1$. \textit{2) Negative Query:} We pair the same question with a randomly selected incorrect answer from the original choice set, forming: ``Question: [Q]. Is the answer `[Incorrect Answer]'?'' This query has ground-truth label $Y = -1$. Accordingly, at the end we obtain total 8366 samples for MEDMCQA. We conduct our experiments using these newly obtained samples. 

\noindent\textbf{Models and Setup:}
We use the following 12 LLMs: GPT-4.1, GPT-4.1-mini, GPT-4.1-nano, GPT4-o, GPT-5-chat from OpenAI \cite{openai_models}, LLaMA-3.1 \cite{grattafiori2024LLaMA3herdmodels}, Qwen3-235b \cite{yang2025Qwen3technicalreport}, kimi-k2 \cite{kimiteam2025kimik2openagentic}, Claude-3.5-haiku \cite{Anthropic2024Claude35haiku}, Mistral-small-3.1 and Mistral-small-3.2 \cite{Mistral_models} and Gemini-2.5-flash \cite{comanici2025Gemini25pushingfrontier}. We run experiments at three temperature settings (0.01, 0.3, and 0.7), performing two independent runs at each temperature. We provide the individual model accuracies and pairwise correlations in \cref{medmcqa_corr} and in \cref{tab:medmcqa_models} that are placed in Appendix \ref{app:tables}. Correlation parameters are estimated via Algorithm~\ref{alg:copula}; see Appendix~\ref{app:implementation} for implementation details, aggregation rules, and complexity analysis.

\noindent\textbf{Evaluation method:}
For each temperature and run setting, we randomly split the dataset: 80\% for estimating the required information-theoretic quantities, we name this portion as the training data/training split/training set even though there exists no model training and we use 20\% for evaluation and denote this portion as the test set or test data or test split. We repeat this over 5 independent random splits and report the mean test error (and standard deviation when shown) across splits.

\begin{figure}[t]
  \begin{center}
    \centerline{\includegraphics[width=0.7\columnwidth]{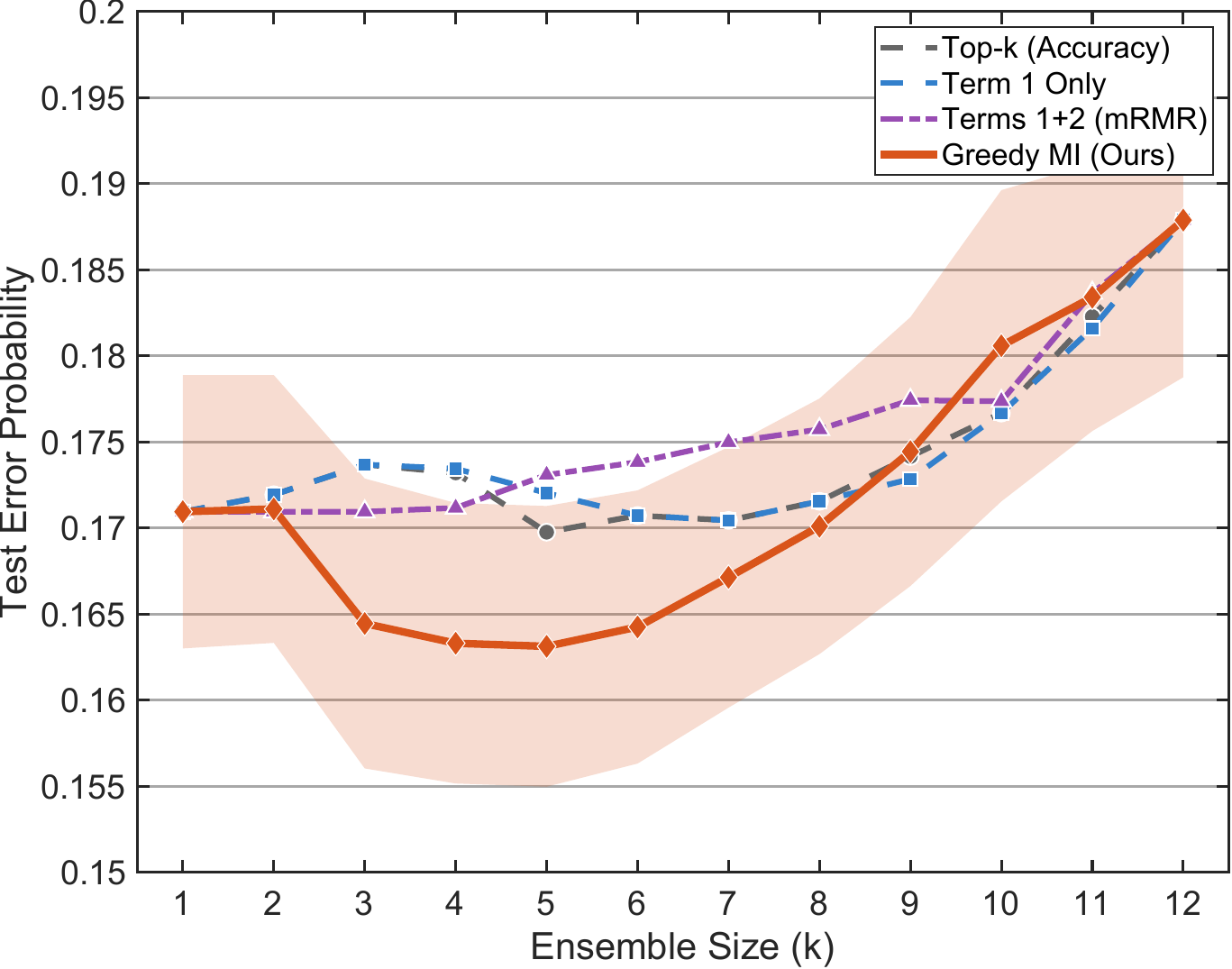}}
    \caption{ MEDMCQA test error vs. ensemble size (mean over temperatures, runs, and random splits). Shaded region represents the standard deviation.
      }
\label{fig:medmcqa_allavg}  
  \end{center}
  \vspace{-0.5cm}
\end{figure}

\begin{table}[t]
\centering
\caption{Selected models for MEDMCQA at representative ensemble sizes (\(\text{temp}=0.7\), run 1). The Greedy MI maintains high accuracy while diversifying across families.}
\label{tab:medmcqa_selections}
\small
\setlength{\tabcolsep}{4pt}
\begin{tabularx}{\columnwidth}{cXX}
\toprule
$k$ & \textbf{Greedy MI (acc.)} & \textbf{Terms 1+2 (mRMR) (acc.)} \\
\midrule
1 & GPT-5-chat (83.1\%) & GPT-5-chat (83.1\%) \\
2 & + GPT-4.1 (82.2\%) & + LLaMA-3.1-8b (66.0\%) \\
3 & + Qwen3-235b (77.4\%) & + GPT-4.1-nano (66.8\%) \\
4 & + kimi-k2 (76.3\%) & + Claude-3.5-haiku (70.7\%) \\
5 & + Gemini-2.5-flash (72.9\%) & + Gemini-2.5-flash (72.9\%) \\
\bottomrule
\end{tabularx}
\end{table}
\noindent\textbf{Results:} Before analyzing selection behavior, we verify that the Gaussian-copula accurately captures the real error structure. Our validation uses the complete dataset: we fit the copula to all available data, generate synthetic errors from the fitted model, and compare the synthetic error patterns against the observed patterns as implemented in \cref{alg:copula}. For copula fitting, we estimate each pairwise correlation $\rho_{ij}$ independently by matching the bivariate marginals. Specifically, for each model pair $(i,j)$, we solve
\begin{equation}
\Phi_2(\tau_i, \tau_j; \rho_{ij}) = \hat{P}(E_i \cap E_j)
\end{equation}
where $\Phi_2(\cdot, \cdot; \rho)$ denotes the bivariate standard normal CDF with correlation $\rho$, the threshold $\tau_i = \Phi^{-1}(1 - \alpha_i)$ corresponds to model $i$'s accuracy $\alpha_i = P(X_i = Y)$, and $\hat{P}(E_i \cap E_j)$ is the empirical joint error rate. As a result, this process yields $12 \times 12$ correlation matrix $\bf\Sigma$ (visualized in Figure~\ref{medmcqa_corr}) with average correlation $\bar{\rho} = 0.55$. Given $\bf\Sigma$ and the marginal accuracies $\{\alpha_j\}$, we generate synthetic error indicators by sampling $Z \sim \mathcal{N}({\bf0}, {\bf\Sigma})$ and thresholding: $E_j = \mathds{1}(Z_j < \tau_j)$. \cref{fig:medmcqa_temp07run1_copulavalid} shows two complementary validation diagnostics. The left panel (scatter plot) compares pairwise error probabilities $P(E_i \cap E_j)$ across all ${12 \choose 2} = 66$ model pairs: each point represents one pair, with the $x$-axis showing the empirical frequency and the $y$-axis showing the copula-generated frequency. The tight clustering around the diagonal (dashed line) indicates strong agreement, with the copula successfully reproducing the second-order correlation structure by construction. The right panel (histogram) examines higher-order structure by plotting the distribution of simultaneous errors, that is, the number of models that make error on the same instance. This is critical for ensemble performance: if all models fail together on hard examples, no aggregation rule can recover. The copula-predicted distribution (orange bars) closely matches the observed frequencies (blue bars), capturing not just pairwise correlations but also the emergent tail behavior. The model correctly predicts that all models being correct ($k=0$) are the most common   and nearly-all-wrong ($k \geq 9$) events are rare. This strong fit gives us confidence to say that Gaussian-copula is an appropriate model for LLM error dependence. Additional tests across all temperature settings appears in Appendix~\ref{app:copula_validation_medmcqa}. Now, we proceed with the results for our algorithm.

Overall, \cref{fig:medmcqa_allavg} shows the test error probability as we increase ensemble size, averaged across 3 temperature settings, 2 runs each and 5 random splits per run. We compare Greedy MI to three baselines: Top-$k$, which selects the most accurate individual models; Term-1, which ranks model using the first term in (\ref{eqn:mut_inf_est}); whereas, Terms 1+2 (mRMR) uses first two terms in (\ref{eqn:mut_inf_est}). In all three baselines, our method consistently outperforms, achieving its best performance of 16.3\% error at $k=5$, compared to 17.0\% for Top-k selection. The improvement is most pronounced when the subset size varies from $k=3$ to $k=7$, which is where practical budget constraints often lie.

\textbf{Understanding the selection behavior.} To see why Greedy MI works, we examine which models each method actually picks. Table~\ref{tab:medmcqa_selections} shows the selected models for a few representative ensemble sizes. At $k=1$, all methods correctly pick GPT-5-Chat, the most accurate model. The methods stay aligned at $k=2$, both selecting GPT-4.1. However, at $k=3$. the mRMR baseline (Terms 1+2) suddenly jumps to GPT-4.1-nano, a model ranked 11th by accuracy (66.8\%), trying hard to minimize redundancy. Meanwhile, our method selects Qwen3-235B (77.4\%, ranked 5th) which is still a strong model, but from a different family than the two OpenAI models already chosen. This reflects that greedy approach allows us to benefit from structured error correlation rather than penalizing all correlation indiscriminately. When $k=5$, Greedy MI has built an ensemble spanning five different model families (OpenAI, Qwen, Moonshot, Google), each with strong individual accuracy. The correlation matrix in Figure~\ref{medmcqa_corr} shows these models have moderate cross-family correlations ($\rho \approx 0.4$--0.5) compared to higher within-family correlations ($\rho \approx 0.7$--0.8). In contrast, mRMR's aggressive diversity-seeking has forced it to include several weak models, degrading overall performance. Moreover, Top-$k$ selection ignores the correlation structure and greedily selects four OpenAI models in succession (GPT-5-Chat, GPT-4.1, GPT-4o, GPT-4.1-mini), missing the opportunity to diversify across model families.

\textbf{Why does mRMR struggle?} The Terms 1+2 baseline's early selection of weak models causes it to plateau around 17.3\% error because standard redundancy penalties like $I(X_j; {\bf X_S})$ do not account for how correlated errors can actually help when they reveal systematic patterns. Thus, our method is superior in terms of capturing this nuance.

\textbf{Performance saturation.} For larger $k$, gains diminish and performance approaches a non-zero floor, consistent with the saturation phenomenon. In practice, the MAP aggregation requires estimating $P(Y\mid \bf{X_S})$ over $2^k$ output patterns; when $k$ is large and the estimation set is limited, many patterns are rare or unseen, making the MAP lookup noisy and leading to degraded performance even when using all $12$ models as shown in Figure~\ref{fig:medmcqa_allavg}.

\textbf{Further comparisons.} To avoid clutter, we report the full temperature/run grid and aggregation ablations in Appendix~ \ref{app:medmcqa}. 

\subsection{Experiments on MMLU}
Next, we evaluate our algorithm on the MMLU \cite{hendrycks2021measuringmassivemultitasklanguage}, a challenging multiple-choice dataset covering 57 diverse subjects such as mathematics, U.S. history, law, and computer science. MMLU is a natural fit for our setting because it requires broad world knowledge and strong reasoning, regimes where aggregating model outputs can be especially effective. To reduce computational overhead, we run our evaluation on the test split only using randomly selected 5000 samples out of 14042 samples. For the binary conversion procedure and evaluation method, we apply the same procedure used in MEDMCQA. Hence, at the end we obtain a total of 10000 samples for MMLU. We use these 10000 samples in the experiments.

\begin{figure}[t]
  \begin{center}
    \centerline{\includegraphics[width=0.70\columnwidth]{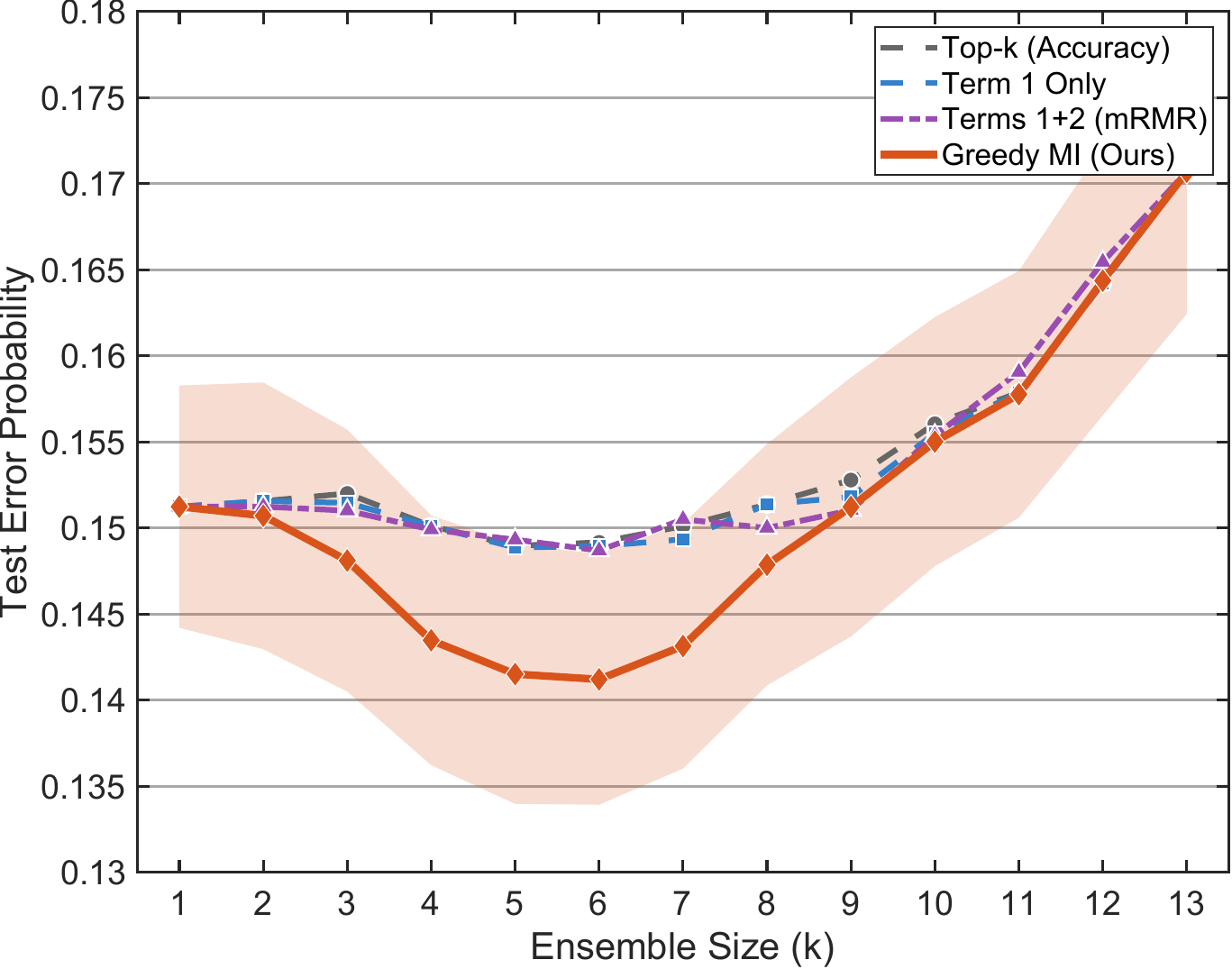}}
    \caption{MMLU test error vs. ensemble size (mean over temperatures, runs, random splits). Shaded region represents the standard deviation.
      }\label{fig:mmlu_all_Avg}
  \end{center}
  \vspace{-1cm}
\end{figure}

\noindent\textbf{Models and Setup:}
We employ the following 13 models: GPT-4.1, GPT-4.1-mini, GPT-4.1-nano, GPT-4o, GPT-4o-mini from OpenAI \cite{openai_models}, LLaMA-3.1-8b and LLaMA-3.3-70b \cite{grattafiori2024LLaMA3herdmodels}, Qwen3-235b \cite{yang2025Qwen3technicalreport}, kimi-k2 \cite{kimiteam2025kimik2openagentic}, Mistral-small-3.1-24b and Mistral-small-3.2-24b \cite{Mistral_models}, Gemini-2.5-flash \cite{comanici2025Gemini25pushingfrontier}, and phi-4 \cite{abdin2024phi4technicalreport}. We run experiments at three temperature settings (0.01, 0.3, and 0.7), performing two independent runs at each temperature. We report the individual model accuracies and pairwise correlations in \cref{mmlu_corr} and in \cref{tab:mmlu_models} that are located in Appendix~\ref{app:tables}.

\noindent\textbf{Results:} Overall, we see in Figure~\ref{fig:mmlu_all_Avg} and Table~\ref{tab:mmlu_results} (provided in Appendix~\ref{app:tables}) that  Greedy MI delivers its clearest gains in the budget-relevant mid-range ($k=3$--$8$): it reaches its best average performance around $k=4$--$6$ (e.g., $0.144$ at $k=4$ and $0.141$ at $k=6$) and consistently improves over accuracy-based Top-$k$ (e.g., $0.150$ at $k=4$ and $0.149$ at $k=6$) and the relevance-only baseline. The mRMR baseline (Terms 1+2) occasionally matches Greedy MI at a few larger $k$ values (e.g., both are about $0.151$ at $k=9$), but it is less reliable overall and does not improve the mid-range where Greedy MI is the strongest. Furthermore, Top-$k$ and Term~1 Only overlap almost everywhere, suggesting that accuracy and relevance rankings largely coincide when correlations are ignored.

For larger $k$, all methods begin to degrade and eventually converge (e.g., by $k=13$ all methods are around $0.171$). This behavior is largely driven by the MAP estimator because the number of possible prediction patterns grows as $2^k$, so with a fixed training split (80\%) many patterns become rare or unseen as $k$ increases, making the estimated $\mathbb{P}(Y\mid X_S)$ high-variance despite Laplace smoothing. Overall, these results emphasize that Greedy MI is most beneficial in the practical regime of small-to-moderate budgets, similar to observations in the MEDMCQA dataset.
\begin{table}[t]
\centering
\caption{Selected models for MMLU at representative ensemble sizes (\(\text{temp}=0.7\), run 2). Greedy MI maintains high accuracy while diversifying across families.}
\label{tab:mmlu_selections}
\small
\setlength{\tabcolsep}{4pt}
\begin{tabularx}{\columnwidth}{cXX}
\toprule
$k$ & \textbf{Greedy MI (acc.)} & \textbf{Terms 1+2 (mRMR) (acc.)} \\
\midrule
1 & GPT-4.1 (85.0\%) & GPT-4.1 (85.0\%) \\
2 & + Qwen3-235b (82.2\%) & + LLaMA-3.3-70b (77.2\%) \\
3 & + kimi-k2 (82.1\%) & + GPT-4.1-nano (69.2\%) \\
4 & + Gemini-2.5-flash (79.8\%) & + Gemini-2.5-flash (79.8\%) \\
5 & + GPT-4.1-mini (81.6\%) & + phi-4 (74.1\%) \\
6 & + phi-4 (74.1\%)  & + kimi-k2 (82.1\%) \\
\bottomrule
\end{tabularx}
\end{table}

In Table~\ref{tab:mmlu_selections}, we see why Greedy MI outperforms compared to mRMR. Both methods start with GPT-4.1 (85.0\%), but diverge at $k=2$: Greedy MI adds Qwen3-235B (82.2\%), while mRMR selects LLaMA-3.3-70B (77.2\%) seeking diversity. At $k=3$, the contrast sharpens--mRMR chooses GPT-4.1-nano (69.2\%, ranked 12th out of 13), whereas Greedy MI picks Kimi-K2 (82.1\%). When $k=5$, every model in Greedy MI's ensemble exceeds 79\% accuracy across four families (OpenAI, Qwen, Moonshot, Google), while mRMR has already included two models below 75\%. Notably, mRMR only adds Kimi-K2 at $k=6$--three steps after Greedy MI selected it. The correlation matrix in \cref{mmlu_corr} explains this behavior: OpenAI models exhibit high within-family correlations ($\rho \approx 0.8$), while cross-family correlations are moderate ($\rho \approx 0.6$--$0.7$). Greedy MI exploits this structure to achieve diversity among strong models, rather than penalizing redundancy indiscriminately at the cost of accuracy. Whereas, Top-$k$ selection, which ignores correlation structure entirely, stacks multiple OpenAI models early (e.g., GPT-4.1, GPT-4o) and thus fails to benefit from cross-family diversity.

Additional Gaussian-copula validation plots and full temperature/run grid and aggregation ablations are given in Appendix~\ref{MMLU_app}.

\subsection{Experiments on IMDB Movie Reviews Dataset}
We further provide our experiments on the IMDB sentiment classification dataset \cite{maas-etal-2011-learning}, where models achieve uniformly high accuracies (91--96\%) and exhibit very strong pairwise correlations ($\bar{\rho} = 0.90$). Due to this high correlation, improvements from ensemble selection are modest, though Greedy MI still consistently outperforms baselines in the mid-budget range; full results are provided in Appendix~\ref{app:imdb}.

\section{Conclusion}
In this work, we studied the problem of ensembling responses from multiple large language models (LLMs) to improve accuracy under a fixed query budget. Focusing on a binary decision setting with true/false (or negative/positive) outputs, we modeled LLM responses using a Gaussian-copula, which effectively captures both individual model accuracies and inter-model correlations. We showed that when LLM decisions are independent, the optimal ensemble is obtained by selecting the most accurate models, corresponding to the classical ``Top-$k$ by accuracy" rule. However, in the presence of correlated models, accuracy alone is insufficient: even with an optimal estimator, correlations induce a non-vanishing ensemble error probability. Motivated by this limitation, we developed an information-theoretic framework, termed \emph{Greedy MI}, that greedily constructs a near-optimal LLM ensemble under budget constraints, and we demonstrated its effectiveness by consistently outperforming strong baselines across multiple datasets.

\section{Limitations and Discussion}

Our study focuses on a binary decision setting, which allows for a clean and interpretable information-theoretic analysis and serves as a foundational step toward more general formulations. While the Gaussian-copula model captures key correlation patterns observed in practice, extending these insights to richer output spaces and alternative dependency structures presents a promising direction for future work. Moreover, the saturation effects revealed by our analysis indicate that further improvements in ensemble performance may be achieved by complementing selection strategies with advances in model diversity and increasing training dataset for  estimating the required information-theoretic quantities.

\section*{Impact Statement}

This work advances the theoretical understanding of ensemble methods for large language models (LLMs) under practical budget constraints by explicitly accounting for correlations among model outputs. Although prior work has studied ensemble selection, existing solutions may not generalize well to all settings without a principled understanding of the problem’s fundamental structure. By characterizing fundamental limits of ensemble performance and proposing principled selection strategies, our results can help practitioners design more reliable and cost-effective LLM systems, potentially reducing unnecessary computational and environmental costs associated with indiscriminate model querying. From an ethical perspective, the proposed framework does not introduce new data collection mechanisms or decision policies that directly affect individuals but rather provides analytical tools for improving existing aggregation methods. Overall, we expect the broader societal impact of this work to align with established implications of improving reliability and efficiency in machine learning systems.
\newpage
\bibliography{references}
\bibliographystyle{icml2026}
\newpage
\appendix
\onecolumn
{\centering{\bf \large Appendix}\par}

\begin{table}[h]
\centering
\caption{Organization of the Appendix.}
\label{tab:appendix_overview}
\renewcommand{\arraystretch}{1.5}
\begin{tabular}{@{}cl@{}}
\toprule
\textbf{Section} & \textbf{Contents} \\
\midrule
\ref{Proof_of_theorem_4_1} & Proof of Theorem \ref{thm:alignment} (independence optimality) \\[0.4em]
\ref{Proof_of_Thm_4_3} & Proof of Theorem \ref{thm:decomposition} (MI decomposition) \\[0.4em]
\ref{sec:proof_saturation} & Proof of Theorem \ref{thm:saturation} (saturation limit) \\[0.4em]
\ref{subsec:difficulty-aware} & Difficulty-Aware Decomposition and the Price of Difficulty \\[0.4em]
\ref{app:implementation} & Implementation Details and Algorithms \\[0.4em]
\ref{app:medmcqa} & Additional Experiments on MEDMCQA \\[0.4em]
\ref{MMLU_app} & Additional Experiments on MMLU \\[0.4em]
\ref{app:imdb} & Additional Experiments on IMDB \\[0.4em]
\ref{app:tables} & Supplementary Tables \\
\bottomrule
\end{tabular}
\end{table}
\vspace*{\fill}
\newpage

\newpage

\section{Proof of \cref{thm:alignment}}\label{Proof_of_theorem_4_1}

We first introduce the necessary definitions and then proceed with the proof.

\begin{definition}[Binary Symmetric Channel]
\label{def:bsc}
A Binary Symmetric Channel (BSC) with crossover (error) probability $\epsilon \in [0,1]$ is a memoryless channel where each input symbol is flipped independently with probability $\epsilon$. For input $Y \in \{+1, -1\}$ and output $X \in \{+1, -1\}$, we have:
\begin{equation}
    X = Y \cdot Z, \quad \text{where } Z = 
    \begin{cases}
        +1 & \text{with probability } 1 - \epsilon, \\
        -1 & \text{with probability } \epsilon.
    \end{cases}
\end{equation}
Equivalently, we express the error probability as $\mathbb{P}(X \neq Y) = \epsilon$.
\end{definition}

\begin{definition}[Stochastic Degradation]
\label{def:degradation}
A (communication) channel $\mathcal{C}'$ is said to be a \emph{stochastically degraded} version of channel $\mathcal{C}$ if there exists an intermediate channel $\mathcal{C}''$ such that the cascade $\mathcal{C} \to \mathcal{C}''$ is statistically equivalent to $\mathcal{C}'$. In other words, the output of $\mathcal{C}'$ can be obtained by further corrupting the output of $\mathcal{C}$ through an independent noise process.
\end{definition}

\begin{lemma}[BSC Degradation]
\label{lem:bsc_degradation}
Let $\mathcal{C}_1$ be a BSC with error rate $\epsilon_1$ and $\mathcal{C}_2$ be a BSC with error rate $\epsilon_2$ where $\epsilon_1 \le \epsilon_2 < 1/2$. Then, $\mathcal{C}_2$ is a stochastically degraded version of $\mathcal{C}_1$. Specifically, $\mathcal{C}_2$ can be constructed by cascading $\mathcal{C}_1$ with an independent BSC having crossover probability:
\begin{equation}
\label{eq:delta}
    \delta = \frac{\epsilon_2 - \epsilon_1}{1 - 2\epsilon_1}.
\end{equation}
\end{lemma}

\begin{proof}
Let $Y \in \{+1, -1\}$ be the input, $Z_1$ be the output of $\mathcal{C}_1$, and $Z_2$ be the output after passing $Z_1$ through a BSC with crossover probability $\delta$. 

For the cascade, $Z_1 = Y \cdot N_1$ where $\mathbb{P}(N_1 = -1) = \epsilon_1$, and $Z_2 = Z_1 \cdot N_2$ where $\mathbb{P}(N_2 = -1) = \delta$. Thus, we have:
\begin{equation}
    Z_2 = Y \cdot N_1 \cdot N_2.
\end{equation}
The overall error occurs when $N_1 \cdot N_2 = -1$, i.e., exactly one of $N_1$ or $N_2$ equals $-1$. Then, we write the error probability for the channel $C_2$ as:
\begin{align}
    \mathbb{P}(Z_2 \neq Y) &= \mathbb{P}(N_1 = +1)\mathbb{P}(N_2 = -1) + \mathbb{P}(N_1 = -1)\mathbb{P}(N_2 = +1) \\
    &= (1 - \epsilon_1)\delta + \epsilon_1(1 - \delta) \\
    &= \delta - \epsilon_1 \delta + \epsilon_1 - \epsilon_1 \delta \\
    &= \epsilon_1 + \delta(1 - 2\epsilon_1).
\end{align}
Setting this equal to $\epsilon_2$ and solving for $\delta$ yields (\ref{eq:delta}). Since $\epsilon_1 \le \epsilon_2 < 1/2$, we have $\delta \in [0,1]$, confirming the validity of this construction.
\end{proof}

\begin{definition}[Top-$k$ Ensemble]
\label{def:topk}
Given $m$ channels indexed by increasing error rates $\epsilon_1 \le \epsilon_2 \le \cdots \le \epsilon_m$, the \emph{top-$k$ ensemble} is defined as $H_k = \{1, 2, \ldots, k\}$, containing the $k$ channels with the smallest error rates.
\end{definition}

We now restate and prove the main theorem.

{\bf Theorem~\ref{thm:alignment}~}[Alignment of Error and Mutual Information]
Consider an ensemble of $m$ independent BSCs where channel $j$ has error rate $\epsilon_j < 1/2$, indexed such that $\epsilon_1 \le \epsilon_2 \le \cdots \le \epsilon_m$. For any subset $S \subseteq [m]$ with $|S| = k$, the top-$k$ set $H_k$ satisfies:
\begin{align}
    I(Y; {\bf X_{H_k}}) &\ge I(Y; {\bf X_S}), \label{eq:alignment_mi_2} \\
    \mathbb{P}_e(H_k) &\le \mathbb{P}_e(S). \label{eq:alignment_pe_2}
\end{align}

\begin{proof}
We prove this result by constructing an explicit coupling between ${\bf X_{H_k}}$ and ${\bf X_S}$, demonstrating that ${\bf X_S}$ can be viewed as a degraded version of ${\bf X_{H_k}}$.

We now define a bijection $\pi: S \to H_k$ satisfying $\epsilon_{\pi(j)} \le \epsilon_j$. To see this, first set $\pi(j)=j$ for all $j\in S\cap H_k$.
The remaining elements satisfy $|S\setminus H_k|=|H_k\setminus S|$, so we can bijectively map
$S\setminus H_k$ onto $H_k\setminus S$. For any $j\in S\setminus H_k$ we have $j>k$, hence
$\epsilon_j\ge \epsilon_k\ge \epsilon_i$ for every $i\in H_k$, so in particular
$\epsilon_{\pi(j)}\le \epsilon_j$.

Using the bijection $\pi$ from the step above, we now construct a random vector $\widetilde{X}_S = (\widetilde{X}_j)_{j \in S}$ as follows: for each $j \in S$, let $i=\pi(j)\in H_k$ be the paired index with $\epsilon_i \le \epsilon_j$, compute the degradation parameter $\delta_j=\frac{\epsilon_j-\epsilon_i}{1-2\epsilon_i}$ using Lemma~\ref{lem:bsc_degradation}, and generate $\widetilde{X}_j$ by passing $X_i$ through an independent BSC with crossover probability $\delta_j$. From Lemma~\ref{lem:bsc_degradation}, each $\widetilde{X}_j$ has the same marginal distribution as $X_j$, i.e., it is the output of a BSC with error rate $\epsilon_j$. We now argue that ${\bf \widetilde{X}_S} \stackrel{d}{=} {\bf X_S}$: each $\widetilde{X}_j$ has the correct marginal distribution (error rate $\epsilon_j$), the degradation noises are independent across components, and the original channels are conditionally independent given $Y$. This establishes that the joint distribution of $\widetilde{X}_S$ matches that of ${\bf X}_S$, i.e.,
\begin{equation}
\label{eq:dist_equivalence}
   {\bf  \widetilde{X}_S} \stackrel{d}{=} {\bf X_S}.
\end{equation}
By construction, $\widetilde{X}_S$ is obtained by applying independent noise to $X_{H_k}$. Since $Y$ is the common source and $X_{H_k}$ depends only on $Y$ (through independent BSCs). As a result, we have the Markov chain:
\begin{equation}
\label{eq:markov_chain}
    Y \longrightarrow {\bf X_{H_k}} \longrightarrow {\bf \widetilde{X}_S} \stackrel{d}{=} {\bf X_S}.
\end{equation}

Since we are able to construct the Markov chain above, now we can use the Data Processing Inequality\citep{Cover2006} which states that for any Markov chain $A \to B \to C$, we have:
\begin{equation}
    I(A; C) \le I(A; B).
\end{equation}
Applying this to the Markov chain in \cref{eq:markov_chain}, we obtain:
\begin{equation}
    I(Y; {\bf X_S}) = I(Y; {\bf \widetilde{X}_S}) \le I(Y; {\bf X_{H_k}}),
\end{equation}
which proves (\ref{eq:alignment_mi_2}).

Now, for the probability of error, we observe that any decision rule $\hat{Y}({\bf X_S})$ based on ${\bf X_S}$ can be replicated using ${\bf X_{H_k}}$ by first generating ${\bf \widetilde{X}_S}$ through the degradation construction and then applying the same decision rule. Formally, we have:
\begin{equation}
    \mathbb{P}_e(S) = \min_{\hat{Y}} \mathbb{P}(\hat{Y}({\bf X_S}) \neq Y) = \min_{\hat{Y}} \mathbb{P}(\hat{Y}({\bf \widetilde{X}_S}) \neq Y).
\end{equation}
Since ${\bf \widetilde{X}_S}$ is a (possibly randomized) function of ${\bf X_{H_k}}$, the optimal estimator based on ${\bf X_{H_k}}$ can only achieve lower or equal error:
\begin{equation}
    \mathbb{P}_e(H_k) = \min_{\hat{Y}} \mathbb{P}(\hat{Y}({\bf X_{H_k}}) \neq Y) \le \mathbb{P}_e(S),
\end{equation}
which proves (\ref{eq:alignment_pe_2}).
\end{proof}
Consequently, we can make the following two intuitive remarks regarding this result:

\begin{remark}[Tightness of the Bound]
\label{rem:tightness}
The inequalities in Theorem~\ref{thm:alignment} become equalities whenever $S$ consists of $k$ channels with the smallest error rates (ties allowed).
 For any other subset $S \neq H_k$, at least one component must be strictly degraded (i.e., $\delta_j > 0$ for some $j$), leading to a strict inequality.
\end{remark}

\begin{remark}[Extension to Non-Uniform Priors]
\label{rem:extension}
The proof relies only on the Markov chain structure and the Data Processing Inequality, which hold regardless of the prior distribution on $Y$. Thus, Theorem~\ref{thm:alignment} extends to any prior $\mathbb{P}(Y=1) = p \in (0,1)$.
\end{remark}

To visualize the stochastic degradation argument, we provide the following example.

\begin{example}{Stochastic Degradation Illustration}{ex1}
Consider a pool of $m=3$ independent models with error rates: $\epsilon_1=0.1$ (expert), $\epsilon_2=0.2$ (average), and $\epsilon_3=0.4$ (novice).  

We seek the optimal subset of size $k=2$, which is the top-$k$ subset $H_2=\{1,2\}$. As a suboptimal comparison, consider $S=\{1,3\}$. 

\noindent\textbf{Bijection Construction.} We construct the mapping $\pi: S \to H_2$:
\begin{itemize}
    \item $\pi(1) = 1$: Model 1 maps to itself ($\epsilon_1 \le \epsilon_1$).
    \item $\pi(3) = 2$: Model 3 maps to Model 2 ($\epsilon_2 \le \epsilon_3$).
\end{itemize}

\noindent\textbf{Degradation Parameters.} For each mapping, we compute the crossover probability using (\ref{eq:delta}):
\begin{itemize}
    \item $\delta_1 = 0$: No noise injection needed (identical mapping).
    \item $\delta_3 = \frac{\epsilon_3 - \epsilon_2}{1 - 2\epsilon_2} = \frac{0.4 - 0.2}{1 - 0.4} = \frac{1}{3}$: Multiply $X_2$ by an independent random sign $N$ with $\mathbb{P}(N = -1) = \frac{1}{3}$.
\end{itemize}

\noindent\textbf{Resulting Markov Chain.} This construction establishes:
\begin{equation*}
    Y \longrightarrow \mathbf{X}_{H_2} \xrightarrow{\text{BSC}(\delta=1/3)} \mathbf{X}_S
\end{equation*}

By the Data Processing Inequality, $I(Y; \mathbf{X}_{H_2}) \ge I(Y; \mathbf{X}_S)$. Similarly, $\mathbb{P}_e(H_2) \le \mathbb{P}_e(S)$, confirming that the top-$k$ ensemble is optimal.
\end{example}

\newpage
\section{Proof of \cref{thm:decomposition}} \label{Proof_of_Thm_4_3}

We first establish the key technical tools used in the proof.

\begin{definition}[The Conditional Mutual Information]
\label{def:cond_mi}
The conditional mutual information between random variables $X$ and $Y$ given $Z$ is defined as:
\begin{equation}
    I(X; Y \mid Z) \triangleq H(X \mid Z) - H(X \mid Y, Z) = \mathbb{E}_Z \left[ D_{\mathrm{KL}}(P_{X,Y|Z} \| P_{X|Z} P_{Y|Z}) \right].
\end{equation}
This quantity measures the information that $Y$ provides about $X$ beyond what is already known from $Z$.
\end{definition}

\begin{lemma}[The Chain Rule for Mutual Information]
\label{lem:chain_rule}
For any random variables $X$, $Y$, and $Z$, we have:
\begin{equation}
    I(X; Y, Z) = I(X; Z) + I(X; Y \mid Z) = I(X; Y) + I(X; Z \mid Y).
\end{equation}
\end{lemma}

\begin{lemma}[The Entropy Invariance Under Bijection]
\label{lem:bijection}
Let $X$ be a random variable and $f$ be an arbitrary bijective function. Then, we have:
\begin{equation}
    H(f(X)) = H(X).
\end{equation}
More generally, for jointly distributed $(X, Z)$ where $f(\cdot; z)$ is bijective for each fixed $z$. Similarly, we have:
\begin{equation}
    H(f(X; Z) \mid Z) = H(X \mid Z).
\end{equation}
\end{lemma}

\begin{proof}
The bijection establishes a one-to-one correspondence between outcomes, preserving probability distributions, i.e., $\mathbb{P}(f(X) = y) = \mathbb{P}(X = f^{-1}(y))$. Since entropy depends only on the probability distribution, the result follows.
\end{proof}

For convenience, we first restate the decomposition theorem (Theorem~\ref{thm:decomposition}) and then provide its proof.

{\bf Theorem~\ref{thm:decomposition}~}[Accuracy-Redundancy-Error Decomposition]
Assume $Y\in\{\pm1\}$ and $X_j\in\{\pm1\}$ for all $j$.
Let $E_j = \mathds{1}(X_j \neq Y)$ be the error indicator for model $j$, and define the marginal information gain of adding model $j$ to subset $S$ as:
\begin{equation}
    \Delta(j \mid S) \triangleq I(Y; X_j \mid {\bf X_S}).
\end{equation}
Then, $\Delta(j \mid S)$ admits the decomposition:
\begin{equation}
\label{eq:general_decomposition_restate}
    \Delta(j \mid S) = \underbrace{I(Y; X_j)}_{\text{Accuracy}} - \underbrace{I(X_j; {\bf X_S})}_{\text{Redundancy}} + \underbrace{I(E_j; E_S)}_{\text{Error Correlation}} + \underbrace{\Lambda_j(S)}_{\text{Correction}},
\end{equation}
where the label-dependent correction term is given by:
\begin{equation}
\label{eq:correction_term_restate}
    \Lambda_j(S) \triangleq I(E_j; Y \mid E_S) - I(E_j; Y).
\end{equation}

\begin{proof}
We proceed through a series of algebraic manipulations using the chain rule and the entropy invariance properties.

By using the definition of the conditional mutual information, we have:
\begin{equation}
\label{eq:step1}
    \Delta(j \mid S) = I(Y; X_j \mid {\bf X_S}) = H(X_j \mid {\bf X_S}) - H(X_j \mid {\bf X_S}, Y).
\end{equation}

Applying the identity $H(X \mid Z) = H(X) - I(X; Z)$, we have:
\begin{equation}
\label{eq:step2}
    H(X_j \mid {\bf X_S}) = H(X_j) - I(X_j; {\bf X_S}).
\end{equation}

Observe that conditioned on the true label $Y$, there exists a deterministic bijection between model outputs and error indicators. For labels $Y, X_j \in \{+1, -1\}$ and error indicator $E_j \in \{0, 1\}$, this bijection is given by:
\begin{equation}
    X_j = Y \cdot (-1)^{E_j} \quad \Longleftrightarrow \quad E_j = \mathds{1}(X_j \neq Y).
\end{equation}
Explicitly, we have:
\begin{itemize}
    \item If $E_j = 0$ (i.e.,  model $j$'s estimate is correct), then $X_j = Y$.
    \item If $E_j = 1$ (i.e., model $j$'s estimate is incorrect), then $X_j = -Y$.
\end{itemize}
Given $Y$, knowing $X_j$ uniquely determines $E_j$, and vice versa. This bijection extends component-wise: given $Y$, the vector ${\bf X_S}$ is in one-to-one correspondence with $E_S$.

From Lemma~\ref{lem:bijection}, the entropy is preserved under this bijection which implies that:
\begin{equation}
\label{eq:step3}
    H(X_j \mid {\bf X_S}, Y) = H(E_j \mid E_S, Y).
\end{equation}

Using the chain rule for entropy, we obtain:
\begin{equation}
\label{eq:step4}
    H(E_j \mid E_S, Y) = H(E_j \mid Y) - I(E_j; E_S \mid Y).
\end{equation}

Substituting (\ref{eq:step2}), (\ref{eq:step3}), and (\ref{eq:step4}) into (\ref{eq:step1}), we have:
\begin{equation}
\label{eq:step5}
    \Delta(j \mid S) = H(X_j) - I(X_j; {\bf X_S}) - H(E_j \mid Y) + I(E_j; E_S \mid Y).
\end{equation}

By the bijection argument, $H(X_j \mid Y) = H(E_j \mid Y)$. Therefore:
\begin{equation}
\label{eq:step6}
    I(Y; X_j) = H(X_j) - H(X_j \mid Y) = H(X_j) - H(E_j \mid Y).
\end{equation}
Rearranging: $H(X_j) - H(E_j \mid Y) = I(Y; X_j)$.

Applying Lemma \ref{lem:chain_rule} to $I(E_j; E_S, Y)$ in two ways:
\begin{align}
    I(E_j; E_S, Y) &= I(E_j; E_S) + I(E_j; Y \mid E_S), \label{eq:chain1} \\
    I(E_j; E_S, Y) &= I(E_j; Y) + I(E_j; E_S \mid Y). \label{eq:chain2}
\end{align}
Equating \cref{eq:chain1} and \cref{eq:chain2} and solving for $I(E_j; E_S \mid Y)$:
\begin{equation}
\label{eq:step7}
    I(E_j; E_S \mid Y) = I(E_j; E_S) + \underbrace{\left[ I(E_j; Y \mid E_S) - I(E_j; Y) \right]}_{\Lambda_j(S)}.
\end{equation}

Substituting \cref{eq:step6} and \cref{eq:step7} into \cref{eq:step5}:
\begin{align}
    \Delta(j \mid S) &= I(Y; X_j) - I(X_j; {\bf X_S}) + I(E_j; E_S \mid Y) \\
    &= I(Y; X_j) - I(X_j; {\bf X_S}) + I(E_j; E_S) + \Lambda_j(S).
\end{align}
This completes the proof.
\end{proof}

\begin{remark}[Interpretation of Terms]
\label{rem:interpretation}
The decomposition in \cref{eq:general_decomposition_restate} has an intuitive interpretation:
\begin{itemize}
    \item \textbf{Accuracy} $I(Y; X_j)$: The single predictive power of model $j$.
    \item \textbf{Redundancy} $I(X_j; {\bf X_S})$: Information overlap between model $j$ and the existing ensemble $S$.
    \item \textbf{Error Correlation} $I(E_j; E_S)$: dependence among model errors, capturing
shared failure modes that influences the incremental value of adding model $j$.

    \item \textbf{Correction} $\Lambda_j(S)$: Adjustment for label-dependent error behavior.
\end{itemize}
\end{remark}

\begin{corollary}[Label-Invariant Case]
\label{cor:label_invariant}
Under the label-invariant error assumption $(E_1, \ldots, E_m) \perp Y$, we have:
\begin{equation}
    I(E_j; Y) = 0 \quad \text{and} \quad I(E_j; Y \mid E_S) = 0, \quad \forall j, S.
\end{equation}
Consequently, $\Lambda_j(S) = 0$ for all $j$ and $S$, and the decomposition simplifies to:
\begin{equation}
    \Delta(j \mid S) = I(Y; X_j) - I(X_j; {\bf X_S}) + I(E_j; E_S).
\end{equation}
\end{corollary}

\begin{proof}
Independence $(E_1, \ldots, E_m) \perp Y$ implies $I(E_j; Y) = 0$ directly. For the conditional term, note that $E_S$ is a function of $(E_1, \ldots, E_m)$, so $(E_j, E_S) \perp Y$, which implies $I(E_j; Y \mid E_S) = 0$.
\end{proof}

\begin{proposition}[Bound on the Correction Term]
\label{prop:correction_bound_restate}
The label-dependence correction term satisfies:
\begin{equation}
    -I(E_j; Y) \leq \Lambda_j(S) \leq H(Y) - I(E_j; Y).
\end{equation}
\end{proposition}

\begin{proof}
By non-negativity of mutual information, $I(E_j; Y \mid E_S) \geq 0$. Therefore:
\begin{equation}
    \Lambda_j(S) = I(E_j; Y \mid E_S) - I(E_j; Y) \geq -I(E_j; Y).
\end{equation}

Conditional mutual information is bounded by conditional entropy:
\begin{equation}
    I(E_j; Y \mid E_S) \leq H(Y \mid E_S) \leq H(Y).
\end{equation}
Therefore:
\begin{equation}
    \Lambda_j(S) = I(E_j; Y \mid E_S) - I(E_j; Y) \leq H(Y) - I(E_j; Y). \qedhere
\end{equation}
\end{proof}

\begin{remark}[Practical Implications]
\label{rem:practical}
For balanced binary classification with $H(Y)=1$ bit, $|\Lambda_j(S)|\le 1$ bit.
Moreover, if the \emph{joint} label-dependence is small, e.g.,
\begin{equation}
    I(E_j,E_S;Y)\le \delta,
\end{equation}
then $|\Lambda_j(S)|\le \delta$ and the correction vanishes as $\delta\to 0$.
This motivates ignoring $\Lambda_j(S)$ when empirical estimates indicate weak joint
dependence between error patterns and labels.
\end{remark}

\newpage
\section{Proof of \cref{thm:saturation}}
\label{sec:proof_saturation}

First, we restate the Gaussian-copula (one-factor) representation used in \cref{subsec:copula_model}
and then give a step-by-step proof.

\subsection{Model recap and notation}

Let $Y\in\{\pm1\}$ be the binary label with a balanced prior $\mathbb{P}(Y=+1)=\mathbb{P}(Y=-1)=1/2$.
Let $U\sim\mathcal{N}(0,1)$ be a latent common factor and let $V_1,V_2,\dots$ be i.i.d.\ $\mathcal{N}(0,1)$,
independent of $U$ and $Y$. Fix $\rho\in(0,1)$ and define latent Gaussian scores
\begin{equation}
\label{eq:latent_scores}
T_j \triangleq \sqrt{\rho}\,U + \sqrt{1-\rho}\,V_j, \qquad j=1,2,\dots
\end{equation}
so that $(T_1,\dots,T_m)$ is equicorrelated Gaussian with $\mathrm{corr}(T_i,T_j)=\rho$ for $i\neq j$.

Define the error indicator by thresholding:
\begin{equation}
\label{eq:error_threshold}
E_j \triangleq \mathds{1}(T_j \le \tau),
\end{equation}
and define the observed model output as a sign-flip of $Y$:
\begin{equation}
\label{eq:output_from_error}
X_j \triangleq Y\cdot (-1)^{E_j}.
\end{equation}

The threshold $\tau$ is chosen to match the marginal accuracy $\alpha>1/2$:
since $T_j\sim\mathcal{N}(0,1)$ marginally, we have
\begin{equation}
\label{eq:tau_accuracy}
\mathbb{P}(E_j=1)=\mathbb{P}(T_j\le \tau)=\Phi(\tau)=1-\alpha.
\end{equation}
Then, we can obtain $\tau=\Phi^{-1}(1-\alpha)$.

\subsection{Proof of \cref{thm:saturation}}
Conditioned on $U=u$, we have from \eqref{eq:latent_scores} that
\[
T_j \mid (U=u)\;\sim\;\mathcal{N}\!\big(\sqrt{\rho}\,u,\,1-\rho\big)
\quad\text{independently across } j.
\]
Therefore, conditioned on $U=u$, the errors $(E_j)_{j\ge1}$ are i.i.d.\ Bernoulli with
\begin{equation}
\label{eq:p_u_def}
p(u)\triangleq \mathbb{P}(E_j=1\mid U=u)
=\mathbb{P}\!\left(T_j\le \tau \mid U=u\right)
=\Phi\!\left(\frac{\tau-\sqrt{\rho}\,u}{\sqrt{1-\rho}}\right).
\end{equation}
Now, let $\bar E_m \triangleq \frac{1}{m}\sum_{j=1}^m E_j$. By the Strong Law of Large Numbers
applied conditional on $U$, we have almost surely
\begin{equation}
\label{eq:slln_errors}
\bar E_m \xrightarrow[m\to\infty]{a.s.} p(U).
\end{equation}

From \eqref{eq:output_from_error}, each $E_j=1$ flips $X_j$ to the wrong sign.
Hence, among the $m$ outputs, the number of incorrect predictions equals $\sum_{j=1}^m E_j$.
Therefore, the majority vote is correct
if and only if  the number of total errors is fewer than half the entire models. Then, we have 
\begin{equation}
\label{eq:maj_correct_event}
\hat Y_{\mathrm{maj}} = Y
\quad\text{if and only if}\quad
\bar E_m < \tfrac12,
\end{equation}
and fails if and only if $\bar E_m>\tfrac12$ (ties have vanishing probability as $m\to\infty$).

Since asymptotic failure event depends only on $p(U)$, 
by combining \eqref{eq:slln_errors} and \eqref{eq:maj_correct_event}, we obtain the almost sure limit:
\begin{equation}
\label{eq:maj_limit_indicator}
\mathds{1}\{\hat Y_{\mathrm{maj}}\neq Y\}
\;\xrightarrow[m\to\infty]{a.s.}\;
\mathds{1}\{p(U)>\tfrac12\}.
\end{equation}
Taking expectations yields
\begin{equation}
\label{eq:error_limit_pu}
\lim_{m\to\infty}\mathbb{P}(\hat Y_{\mathrm{maj}}\neq Y)=\mathbb{P}\!\left(p(U)>\tfrac12\right).
\end{equation}

Also note that the MAP estimation coincides with the majority vote under the balanced, symmetric model;
under the construction \eqref{eq:output_from_error}, the model is sign-symmetric:
flipping $(Y,X_{[m]})$ to $(-Y,-X_{[m]})$ leaves the joint law invariant, and the prior on $Y$ is balanced.
Consequently, the MAP decision rule depends only on the majority sign of $(X_1,\dots,X_m)$
(and agrees with majority vote). Hence,
\begin{equation}
\label{eq:map_equals_maj}
\mathbb{P}(\hat Y_{\mathrm{MAP}}\neq Y)=\mathbb{P}(\hat Y_{\mathrm{maj}}\neq Y)\quad\text{for each }m,
\end{equation}
and as a result, the same limit in \eqref{eq:error_limit_pu} holds for $\hat Y_{\mathrm{MAP}}$.

Now let us  find the probability of the event $p(U)>\tfrac12$ explicitly.
From \eqref{eq:p_u_def}, using that $\Phi(\cdot)$ is strictly increasing, we have $p(U)>\tfrac12
$ which is equivalent to having $\frac{\tau-\sqrt{\rho}\,U}{\sqrt{1-\rho}}>0$. This condition implies
\[
U<\frac{\tau}{\sqrt{\rho}}.
\]
Since $U\sim\mathcal{N}(0,1)$, we obtain the following
\begin{equation}
\label{eq:error_floor_tau}
\mathbb{P}\!\left(p(U)>\tfrac12\right)=\Phi\!\left(\frac{\tau}{\sqrt{\rho}}\right).
\end{equation}

Substituting $\tau=\Phi^{-1}(1-\alpha)$ using \eqref{eq:tau_accuracy} in \eqref{eq:error_floor_tau} gives
\[
\lim_{m\to\infty}\mathbb{P}(\hat Y_{\mathrm{MAP}}\neq Y)
=\Phi\!\left(\frac{\Phi^{-1}(1-\alpha)}{\sqrt{\rho}}\right),
\]
which is strictly positive for any $\rho>0$ and $\alpha\in(1/2,1)$.
This completes the proof.

\begin{remark}[Edge cases and interpretation]
Because $\Phi^{-1}(1-\alpha)<0$ for $\alpha>1/2$, the error floor lies in $(0,1/2)$ for any $\rho>0$.
As $\rho$ decreases to 0, the expression $\frac{\Phi^{-1}(1-\alpha)}{\sqrt{\rho}}$ tends to $-\infty$ and $\lim_{m\to\infty}\mathbb{P}(\hat Y_{\mathrm{MAP}}\neq Y)
$ vanishes, recovering the
independent-error regime. As $\rho$ increases to 1, the expression $\frac{\Phi^{-1}(1-\alpha)}{\sqrt{\rho}}$ approaches $\Phi^{-1}(1-\alpha)$ which implies  $\lim_{m\to\infty}\mathbb{P}(\hat Y_{\mathrm{MAP}}\neq Y) =1-\alpha$,
matching the intuition that perfectly correlated models provide no ensemble gain.
\end{remark}

\newpage
\section{Difficulty-Aware Decomposition and the Price of Difficulty}
\label{subsec:difficulty-aware}
To justify our greedy selection and to further explain Theorem~\ref{thm:saturation}, we expose a difficulty-aware submodular component and an intrinsic saturation penalty.
\begin{theorem}[Difficulty-Aware Decomposition and Price of Difficulty]\label{thm:oracle_price}
Let $Y$ be the true label and $X_{[m]}=\{X_1,\dots,X_m\}$ be discrete model outputs (e.g., binary).
Assume that there exists a latent variable $D$ (representing the prompt difficulty) such that:
\begin{enumerate}
    \item $D \perp Y$, i.e., $D$ and $Y$ are statistically independent random variables,
    \item for every $S\subseteq[m]$, model outputs are conditionally independent given $(Y,D)$:
    \begin{equation}\label{eq:cond_indep}
        \mathbb{P}({\bf X_S}\mid Y,D)=\prod_{j\in S} \mathbb{P}(X_j\mid Y,D).
    \end{equation}
\end{enumerate}
Then, for every $S\subseteq[m]$,
\begin{equation}\label{eq:oracle_decomp_main}
    I(Y;{\bf X_S})= I(Y;{\bf X_S} \mid D) - I(Y;D\mid {\bf X_S}).
\end{equation}
We refer to $I(Y;{\bf X_S}\mid D)$ as the difficulty-aware mutual information which quantifies how informative the subset would be if the latent difficulty $D$ were revealed to the decoder. We refer to $I(Y;D\mid {\bf X_S})$ as the Price of Difficulty which measures the remaining ambiguity that is related to not observing $D$ after seeing the ensemble.

Moreover, $F(S)\triangleq I(Y;{\bf X_S} \mid D)$ is monotone submodular in $S$ (under a cardinality constraint $|S|\le k$).
Finally, the Price of Difficulty admits the form
\begin{equation}\label{eq:price_operational}
    I(Y;D\mid {\bf X_S})=H(Y\mid {\bf X_S})-H(Y\mid {\bf X_S},D),
\end{equation}
i.e., the additional label uncertainty induced by not observing $D$ after seeing the ensemble.
\end{theorem}

\begin{remark}
\label{rem:oracle_price}
Theorem~\ref{thm:oracle_price} isolates a difficulty-aware component $I(Y;{\bf X_S}\mid D)$ that is monotone submodular, hence greedy achieves a $(1-1/e)$ approximation guarantee for maximizing this term under $|S|\le k$ \citep{krause2014submodular}.
In practice, $D$ is unobserved, and the difference between the difficulty-aware information and the true information is exactly the Price of Difficulty $I(Y;D\mid {\bf X_S} )$.
If this term does not vanish as $|S|$ grows, then adding more models cannot fully remove uncertainty about $Y$, which provides an information-theoretic explanation for the saturation phenomenon characterized in Theorem~\ref{thm:saturation}.
\end{remark}

\subsection{Proof of \cref{thm:oracle_price}}
\label{app:proof_oracle_price}

For the proof, we use the conditional mutual information and chain rule as stated in: Definition~\ref{def:cond_mi} and Lemma~\ref{lem:chain_rule}. Now, lets proceed with the proof.
\begin{proof}
We prove the decomposition, the operational form of the Price of Difficulty, and the monotone submodularity of
\(
F(S)\triangleq I(Y;{\bf X_S}\mid D).
\)

Let's first derive the decomposition in \cref{eq:oracle_decomp_main}.
We can apply the chain rule for mutual information to the triple $(Y,{\bf X_S},D)$ in two different ways:
\begin{align}
\label{eq:app_chain_1}
I(Y;{\bf X_S},D) &= I(Y;X_S) + I(Y;D\mid {\bf X_S}),\\
\label{eq:app_chain_2}
I(Y;{\bf X_S},D) &= I(Y;D) + I(Y;{\bf X_S}\mid D).
\end{align}
Equating \cref{eq:app_chain_1,eq:app_chain_2} yields
\begin{equation}
\label{eq:app_decomp_intermediate}
I(Y;{\bf X_S})= I(Y;{\bf X_S}\mid D) + I(Y;D) - I(Y;D\mid {\bf X_S}).
\end{equation}
By the assumption in the theorem, $D\perp Y$, hence $I(Y;D)=0$. Substituting into \cref{eq:app_decomp_intermediate} gives
\(
I(Y;{\bf X_S})= I(Y;{\bf X_S}\mid D) - I(Y;D\mid {\bf X_S}),
\)
which is exactly \cref{eq:oracle_decomp_main}.

Now, we can derive the \cref{eq:price_operational} by using the definition of conditional mutual information,
\begin{equation}
\label{eq:app_price_operational}
I(Y;D\mid {\bf X_S})=H(Y\mid {\bf X_S})-H(Y\mid {\bf X_S},D),
\end{equation}
which matches \cref{eq:price_operational}.

To prove the monotone submodularity of $F(S)=I(Y;{\bf X_S}\mid D)$,
we use an entropy representation:
\begin{equation}
\label{eq:app_F_entropy_form}
F(S)=I(Y;{\bf X_S}\mid D)=H({\bf X_S}\mid D)-H({\bf X_S}\mid Y,D).
\end{equation}

Under the second assumption, the outputs factorize given $(Y,D)$:
\(
\mathbb{P}({\bf X_S}\mid Y,D)=\prod_{j\in S}\mathbb{P}(X_j\mid Y,D).
\)
Consequently, the conditional entropy decomposes additively as
\begin{equation}
\label{eq:app_modular_term}
H({\bf X_S}\mid Y,D)=\sum_{j\in S} H(X_j\mid Y,D),
\end{equation}
so the set function $S\mapsto H({\bf X_S}\mid Y,D)$ is modular under the conditional independence given $(Y,D)$. Next, for discrete random variables, Shannon entropy is submodular in its set argument, and conditioning preserves submodularity; therefore $S\mapsto H({\bf X_S}\mid D)$ is submodular (see, e.g., \citep{Cover2006}). Recalling the identity $F(S)$ in (\ref{eq:app_F_entropy_form}), we conclude that $F$ is submodular as a submodular function minus a modular one. Finally, for any $j\notin S$,
\begin{equation}
\label{eq:app_monotonicity}
F(S\cup\{j\})-F(S)=I(Y;X_j\mid {\bf X_S},D)\ge 0,
\end{equation}
where the last inequality follows from non-negativity of conditional mutual information. Therefore, $F$ is monotone.

\end{proof}

\newpage
\section{Implementation Details}
\label{app:implementation}
In this section, we detail the implementation of the algorithms used throughout the experiments. We first present the smoothed mutual information estimation method, followed by the pseudo codes for the MAP, Majority Voting, and Weighted Majority Voting aggregation rules. We conclude with providing the implementation details for the Gaussian-copula parameter estimation that is used for modeling correlated errors.

\subsection{Mutual Information Estimation}
Algorithm~\ref{alg:mi} summarizes how we estimate mutual information from finite samples. Given $N$ paired observations of discrete variables $A$ and $B$, we form empirical marginal and joint counts and convert them to probabilities using Laplace (add-one) smoothing with $\alpha=1$ (please note that this parameter is for smoothing, not for the accuracy as used in \cref{thm:saturation}). This guarantees strictly positive probability estimates, so the entropy terms are always well-defined (no $\log 0$), and it stabilizes the estimates when some outcomes are rare or absent in the parameter-estimation split (the 80\% portion used in each run and split). The mutual information is then computed via the standard identity $I(A; B) = H(A) + H(B) - H(A, B)$, where each entropy term is estimated from the smoothed distributions. 

\begin{algorithm}[ht]
\caption{Smoothed Mutual Information Estimation}
\label{alg:mi}
\begin{algorithmic}[1]
\REQUIRE Discrete variables $A \in \mathcal{A}^N$, $B \in \mathcal{B}^N$, smoothing parameter $\alpha > 0$
\ENSURE Mutual information estimate $\hat{I}(A; B)$
\STATE \hfill $\triangleright$ Compute variable sizes
\STATE $K_A \leftarrow |\mathcal{A}|$, $K_B \leftarrow |\mathcal{B}|$, $K_{AB} \leftarrow K_A \cdot K_B$
\STATE
\STATE \hfill $\triangleright$ Estimate marginal distributions with Laplace smoothing
\FOR{$a \in \mathcal{A}$}
    \STATE $\hat{P}(A = a) \leftarrow \frac{|\{i : A_i = a\}| + \alpha}{N + \alpha K_A}$
\ENDFOR
\FOR{$b \in \mathcal{B}$}
    \STATE $\hat{P}(B = b) \leftarrow \frac{|\{i : B_i = b\}| + \alpha}{N + \alpha K_B}$
\ENDFOR
\STATE
\STATE \hfill $\triangleright$ Estimate joint distribution
\FOR{$(a, b) \in \mathcal{A} \times \mathcal{B}$}
    \STATE $\hat{P}(A = a, B = b) \leftarrow \frac{|\{i : A_i = a, B_i = b\}| + \alpha}{N + \alpha K_{AB}}$
\ENDFOR
\STATE
\STATE \hfill $\triangleright$ Compute entropies
\STATE $\hat{H}(A) \leftarrow -\sum_{a \in \mathcal{A}} \hat{P}(a) \log_2 \hat{P}(a)$
\STATE $\hat{H}(B) \leftarrow -\sum_{b \in \mathcal{B}} \hat{P}(b) \log_2 \hat{P}(b)$
\STATE $\hat{H}(A, B) \leftarrow -\sum_{(a,b) \in \mathcal{A} \times \mathcal{B}} \hat{P}(a, b) \log_2 \hat{P}(a, b)$
\STATE
\STATE \hfill $\triangleright$ Mutual information
\STATE $\hat{I}(A; B) \leftarrow \max\left( \hat{H}(A) + \hat{H}(B) - \hat{H}(A, B), \; 0 \right)$
\end{algorithmic}
\end{algorithm}

\noindent\textbf{Computational complexity:}
Let $K_A=|\mathcal{A}|$ and $K_B=|\mathcal{B}|$. Algorithm~\ref{alg:mi} estimates the marginals in $\mathcal{O}(N + K_A + K_B)$ time (one pass to count plus normalization), estimates the joint in $\mathcal{O}(N + K_AK_B)$ time, and computes entropies in $\mathcal{O}(K_A + K_B + K_AK_B)$ time. Overall, the time complexity is
\[
\mathcal{O}\bigl(N + K_AK_B\bigr),
\]
and the memory cost is $\mathcal{O}(K_AK_B)$ to store the smoothed joint table (plus $\mathcal{O}(K_A+K_B)$ for marginals).
In our setting $A,B$ are low-cardinality (e.g., $|\mathcal{A}|=|\mathcal{B}|=2$ for binary variables), so the MI estimation is negligible compared to LLM query costs.

\newpage
\subsection{Aggregation Methods}

Algorithm~\ref{alg:map} summarizes our MAP aggregator. For a selected subset $S$ of size $k$, we estimate $P(Y \mid X_S)$ from the training split by counting labels for each prediction pattern $x_S \in \{-1,+1\}^k$ (there are $2^k$ such patterns). We apply Laplace smoothing (add-one) to avoid zero counts for rare or unseen patterns. At test time, given the observed pattern, we predict the label with the larger smoothed count, which implements the MAP rule under a uniform prior. This estimator makes no conditional-independence assumption; it uses the empirical joint distribution of $X_S$.

\begin{algorithm}[ht]
\caption{The MAP Aggregation}
\label{alg:map}
\begin{algorithmic}[1]
\REQUIRE Training labels $Y^{\text{tr}} \in \{-1, +1\}^{N_{\text{tr}}}$, training predictions $X^{\text{tr}}_S \in \{-1, +1\}^{N_{\text{tr}} \times k}$
\REQUIRE Test predictions $X^{\text{te}}_S \in \{-1, +1\}^{N_{\text{te}} \times k}$
\ENSURE Predicted labels $\hat{Y}^{\text{te}} \in \{-1, +1\}^{N_{\text{te}}}$

\STATE Map each ${\bf x_S }\in \{-1, +1\}^k$ to index 
$\text{idx}({\bf x_S}) \in \{1, \ldots, 2^k\}$: \hfill $\triangleright$ Index prediction patterns
\STATE \hspace{1em} $\text{idx}({\bf x_S}) \leftarrow 1 + \sum_{j=1}^{k} \frac{x_j + 1}{2} \cdot 2^{j-1}$
\STATE
\STATE \hfill $\triangleright$ Build empirical conditional counts from training data
\FOR{$\ell = 1$ to $2^k$}
    \STATE $C^{+}_\ell \leftarrow |\{i : \text{idx}(X^{\text{tr}}_{S,i}) = \ell \text{ and } Y^{\text{tr}}_i = +1\}| + 1$ \COMMENT{Laplace}
    \STATE $C^{-}_\ell \leftarrow |\{i : \text{idx}(X^{\text{tr}}_{S,i}) = \ell \text{ and } Y^{\text{tr}}_i = -1\}| + 1$
\ENDFOR
\STATE
\STATE \hfill $\triangleright$ Predict on test data using MAP rule
\FOR{$i = 1$ to $N_{\text{te}}$}
    \STATE $\ell \leftarrow \text{idx}(X^{\text{te}}_{S,i})$
    \IF{$C^{+}_\ell > C^{-}_\ell$}
        \STATE $\hat{Y}^{\text{te}}_i \leftarrow +1$
    \ELSIF{$C^{+}_\ell < C^{-}_\ell$}
        \STATE $\hat{Y}^{\text{te}}_i \leftarrow -1$
    \ELSE
        \STATE $\hat{Y}^{\text{te}}_i \leftarrow +1$ \COMMENT{Break ties arbitrarily}
    \ENDIF
\ENDFOR
\end{algorithmic}
\end{algorithm}

\noindent\textbf{Computational complexity:}
Let $k=|S|$. The plug-in MAP aggregator maintains two count tables $(C^+_\ell, C^-_\ell)_{\ell=1}^{2^k}$, one entry per prediction pattern.
On the training split, computing the pattern index for each example takes $\mathcal{O}(k)$ time and updating the corresponding counts is $\mathcal{O}(1)$, for a total of $\mathcal{O}(N_{\mathrm{tr}}\,k)$ time. Initializing the count tables costs $\mathcal{O}(2^k)$.
At test time, each example again requires $\mathcal{O}(k)$ time to compute its index and $\mathcal{O}(1)$ time to compare $(C^+_\ell, C^-_\ell)$, yielding $\mathcal{O}(N_{\mathrm{te}}\,k)$ time.
Overall, the time complexity is $\mathcal{O}((N_{\mathrm{tr}}+N_{\mathrm{te}})\,k + 2^k)$ and the memory complexity is $\mathcal{O}(2^k)$.
In our experiments $k\le 13$, so $2^k \le 8192$ and the aggregation overhead is negligible compared to LLM query cost.

Algorithm~\ref{alg:mv} implements the standard majority voting aggregation. Given predictions from the selected subset $S$, we sum the binary votes $X_j \in \{-1, +1\}$ and predict the sign of the total. When the ensemble size $|S|$ is odd, ties cannot occur; for even $|S|$, ties are broken randomly. This approach treats all models equally regardless of their individual accuracies, which can be suboptimal when model quality varies significantly.

\begin{algorithm}[ht]
\caption{Majority Voting (MV) Aggregation}
\label{alg:mv}
\begin{algorithmic}[1]
\REQUIRE Selected subset $S$, model predictions ${ \bf X_S} = (X_j)_{j \in S}$
\STATE $V \leftarrow \sum_{j \in S} X_j$ \hfill $\triangleright$ Compute vote count
\IF{$V > 0$}
    \STATE $\hat{Y}_{\text{MV}} \leftarrow +1$
\ELSIF{$V < 0$}
    \STATE $\hat{Y}_{\text{MV}} \leftarrow -1$
\ELSE
    \STATE $\hat{Y}_{\text{MV}} \leftarrow \text{random}(\{-1, +1\})$ \COMMENT{Break ties randomly}
\ENDIF

\end{algorithmic}
\end{algorithm}

Algorithm~\ref{alg:wmv} extends majority voting by weighting each model's vote according to its accuracy. If $p_j$ represents the model $j$ being correct, then weight $w_j = \log(p_j / (1 - p_j))$ corresponds to the log-odds of model $j$ being correct, which is optimal when model errors are independent (for the proof, see the \cref{thm:logodd} and related proof below). Note that in our experiments we estimate the weights by calculating $p_j$ using the training set that is obtained by splitting the data for each run.

\begin{algorithm}[ht]
\caption{Weighted Majority Voting (W-MV) Aggregation}
\label{alg:wmv}
\begin{algorithmic}[1]
\REQUIRE Selected subset $S$, model predictions ${ \bf X_S} = (X_j)_{j \in S}$, accuracies $(p_j)_{j \in S}$ 
\FOR{$j \in S$}
    
    \STATE $w_j \leftarrow \log \frac{p_j}{1 - p_j}$ \COMMENT{Log-odds weight} \hfill $\triangleright$ Compute weight 
\ENDFOR
\STATE $V \leftarrow \sum_{j \in S} w_j \cdot X_j$ \hfill $\triangleright$ Compute weighted vote
\STATE $\hat{Y}_{\text{W-MV}} \leftarrow \text{sign}(V)$
\end{algorithmic}
\end{algorithm}
\begin{theorem}[Optimality of Log-Odds Weights]
\label{thm:logodd}
Consider a binary classification problem with a true label $Y \in \{-1, +1\}$ and an ensemble of $M$ independent classifiers. Let $H_i$ be the random variable representing the output of classifier $i$, such that its accuracy is given by $\mathbb{P}(H_i = Y) = p_i$.

Under the assumption of independence, the decision rule that minimizes the probability of error is the \textbf{Weighted Majority Vote}:
\[
\hat{Y} = \mathrm{sign}\left( \sum_{i=1}^{M} w_i H_i \right)
\]
where the optimal weight for each classifier is the log-odds of its accuracy:
\[
w_i = \log\left(\frac{p_i}{1 - p_i}\right)
\]
\end{theorem}

\begin{proof}
To minimize the probability of error, we use the Maximum A Posteriori (MAP) decision rule. We choose the prediction $\hat{Y} = +1$ if the posterior probability satisfies:
\[
\mathbb{P}(Y=+1 \mid \mathbf{H}) > \mathbb{P}(Y=-1 \mid \mathbf{H})
\]
where $\mathbf{H} = (H_1, \dots, H_M)$ denotes the vector of classifier outputs. Using Bayes' Theorem \citep{hajek2020probability} and assuming a uniform prior $\mathbb{P}(Y=+1) = \mathbb{P}(Y=-1)$, this is equivalent to comparing the likelihoods:
\[
\frac{\mathbb{P}(\mathbf{H} \mid Y=+1)}{\mathbb{P}(\mathbf{H} \mid Y=-1)} > 1
\]
Since the classifiers are assumed to be conditionally independent given the class label $Y$, the joint likelihood factors into a product:
\[
\frac{\prod_{i=1}^M \mathbb{P}(H_i \mid Y=+1)}{\prod_{i=1}^M \mathbb{P}(H_i \mid Y=-1)} > 1
\]
Taking the logarithm of both sides yields the sum of log-likelihood ratios:
\[
\sum_{i=1}^M \log\left( \frac{\mathbb{P}(H_i \mid Y=+1)}{\mathbb{P}(H_i \mid Y=-1)} \right) > 0
\]
We analyze the term for the $i$-th classifier based on its output realization $H_i$:
\begin{enumerate}
    \item \textbf{Case $H_i = +1$:} The probability of this output given $Y=+1$ is $p_i$ (correct), and given $Y=-1$ is $1-p_i$ (error). The term becomes $\log\left(\frac{p_i}{1-p_i}\right)$.
    \item \textbf{Case $H_i = -1$:} The probability of this output given $Y=+1$ is $1-p_i$ (error), and given $Y=-1$ is $p_i$ (correct). The term becomes $\log\left(\frac{1-p_i}{p_i}\right) = -\log\left(\frac{p_i}{1-p_i}\right)$.
\end{enumerate}
Thus, the log-likelihood ratio term for classifier $i$ can be written generally as $H_i \cdot \log\left(\frac{p_i}{1-p_i}\right)$. Substituting this back into the sum, the decision rule becomes:
\[
\sum_{i=1}^M H_i \log\left(\frac{p_i}{1 - p_i}\right) > 0
\]
This is equivalent to the weighted majority vote $\mathrm{sign}\left(\sum_{i=1}^M w_i H_i\right)$ with weights $w_i = \log\left(\frac{p_i}{1-p_i}\right)$.
\end{proof}
\newpage
\subsection{Gaussian-Copula Fitting and Sampling Algorithm}

Algorithm~\ref{alg:copula} explains Gaussian-copula fitting. To do that, first we estimate marginal error probabilities $\hat{\epsilon}_j$ using the whole dataset (we don't use splitting in here, since our goal is showing how Gaussian-copula is a trustable method for modeling the joint distributions), then compute latent thresholds $\tau_j = \Phi^{-1}(\hat{\epsilon}_j)$ where $\Phi$ is the normal CDF function.

Next, we capture dependencies using Tetrachoric correlation method \citep{tetra}. For each pair, we solve for the latent correlation $\rho_{jk}$ that matches the observed joint error probability: $\Phi_{\rho_{jk}}(\tau_j, \tau_k) = \mathbb{P}(E_j=1, E_k=1)$. Since pairwise estimation may yield a non-positive semi-definite matrix, we apply spectral projection by clipping negative eigenvalues to $\epsilon$ and re-normalizing the diagonals. Finally, we generate synthetic errors by sampling latent vectors $Z \sim \mathcal{N}(\mathbf{0}, {\bf\hat{\Sigma}})$ and applying the thresholding rule $E_{\text{syn}, ij} = \mathds{1}(Z_{ij} < \tau_j)$.

\begin{algorithm}[ht]
\caption{Gaussian-Copula Fitting and Synthetic Generation}
\label{alg:copula}
\begin{algorithmic}[1]
\REQUIRE Model predictions $X \in \{-1, 1\}^{N \times M}$, True labels $Y \in \{-1, 1\}^N$
\ENSURE Latent correlation matrix ${\bf \Sigma }$, Thresholds $\tau$, Synthetic data $X_{\text{syn}}$

\STATE \hfill $\triangleright$ \textbf{Phase 1: Marginals Estimation}
\STATE Compute error matrix $E \in \{0, 1\}^{N \times M}$ where $E_{ij} \leftarrow \mathds{1}(X_{ij} \neq Y_i)$
\FOR{$j = 1$ to $M$}
    \STATE $\hat{\epsilon}_j \leftarrow \frac{1}{N} \sum_{i=1}^N E_{ij}$ \hfill $\triangleright$ Marginal error rate
    \STATE $\hat{\epsilon}_j \leftarrow \text{clamp}(\hat{\epsilon}_j, \epsilon, 1-\epsilon)$ \hfill $\triangleright$ Numerical stability
    \STATE $\tau_j \leftarrow \Phi^{-1}(\hat{\epsilon}_j)$ \hfill $\triangleright$ Latent threshold via Inverse Normal CDF
\ENDFOR

\STATE
\STATE \hfill $\triangleright$ \textbf{Phase 2: Pairwise Latent Correlations}
\STATE Initialize ${\bf \Sigma } \leftarrow {\bf I}_{M \times M}$
\FOR{$j = 1$ to $M$}
    \FOR{$k = j+1$ to $M$}
        \STATE $\hat{\epsilon}_{jk} \leftarrow \frac{1}{N} \sum_{i=1}^N \mathds{1}(E_{ij}=1 \land E_{ik}=1)$ \hfill $\triangleright$ Joint error prob
        \STATE Find $\rho$ such that $\Phi_{\rho}(\tau_j, \tau_k) = \hat{\epsilon}_{jk}$ \hfill $\triangleright$ Solve Bivariate Normal CDF
        \STATE $\Sigma_{jk} \leftarrow \rho, \quad \Sigma_{kj} \leftarrow \rho$
    \ENDFOR
\ENDFOR

\STATE
\STATE \hfill $\triangleright$ \textbf{Phase 3: Regularization}
\STATE Compute eigendecomposition: ${\bf \Sigma} = {\bf V D V^\top}$
\STATE $D_{ii} \leftarrow \max(D_{ii}, \epsilon)$ for all $i$ \hfill $\triangleright$ Clip negative eigenvalues
\STATE ${\bf \Sigma \leftarrow V D V^\top}$
\STATE Normalize diagonal: $\Sigma_{jk} \leftarrow \Sigma_{jk} / \sqrt{\Sigma_{jj} \Sigma_{kk}}$

\STATE
\STATE \hfill \textbf{$\triangleright$ Phase 4: Synthetic Data Generation}
\STATE Sample latent variables $Z \in \mathbb{R}^{N \times M} \sim \mathcal{N}(\mathbf{0}, {\bf \Sigma})$
\STATE Generate labels $Y_{\text{syn}} \sim \text{Uniform}(\{-1, 1\})^N$
\STATE Initialize $X_{\text{syn}} \in \mathbb{R}^{N \times M}$
\FOR{$i = 1$ to $N$}
    \FOR{$j = 1$ to $M$}
        \STATE $E_{\text{syn}, ij} \leftarrow \mathds{1}(Z_{ij} < \tau_j)$ \hfill $\triangleright$ Threshold latent variables
        \STATE $X_{\text{syn}, ij} \leftarrow Y_{\text{syn}, i} \cdot (1 - 2 E_{\text{syn}, ij})$ \hfill $\triangleright$ Flip label if error
    \ENDFOR
\ENDFOR
\end{algorithmic}
\end{algorithm}

\newpage
\section{Additional Experimental Details for the MEDMCQA Dataset}\label{app:medmcqa}
In this section, we provide more context on the MEDMCQA experiments.

For the binary evaluation task, we utilize the following prompt template to enforce a consistent `True' or `False' output format:

\begin{tcolorbox}[
  colback=gray!8,
  colframe=gray!70!black,
  coltitle=white,
  title={\textbf{Query Format for MEDMCQA Dataset}},
  fonttitle=\small,
  boxrule=0.8pt,
  arc=0pt,
  left=8pt,
  right=8pt,
  top=6pt,
  bottom=6pt
]
\small
\textbf{System:} Is the following statement true or false? Answer with a single word: True or False. Do not provide any reasoning.\\[4pt]
\textbf{User:} Question: \texttt{<question>} Is the answer `\texttt{<candidate\_answer>}'?\\
Answer: 
\end{tcolorbox}

We accessed all language models through OpenRouter to standardize the inference process. An actual question from MEDMCQA, is illustrated below with ground truth ``TRUE":
\begin{tcolorbox}[
  colback=gray!8,
  colframe=gray!70!black,
  coltitle=white,
  title={\textbf{Example Query for MEDMCQA Dataset}},
  fonttitle=\small,
  boxrule=0.8pt,
  arc=0pt,
  left=8pt,
  right=8pt,
  top=6pt,
  bottom=6pt
]
\small
\textbf{System:} Is the following statement true or false? Answer with a single word: True or False. Do not provide any reasoning.\\[4pt]
\textbf{User:} Question: Axonal transport is: Is the answer `Antegrade and retrograde'?\\
Answer:
\end{tcolorbox}

\subsection{Experiments across Temperature and Runs (the MAP aggregation)}
We also provide the experiments for all temperature settings and runs with 5 random splits for each.
Figure~\ref{fig:appendix_map_medmcqa_6} reports per-condition curves on MEDMCQA under the MAP aggregation for each baseline, with one panel for average of each temperature--run--(5-split) pair. Table~\ref{tab:medmcqa_results} aggregates the same experiments over all $30$ evaluations (3 temperatures $\times$ 2 runs $\times$ 5 random 80/20 splits), reporting mean $\pm$ std test error at each ensemble size $k$.

\begin{figure*}[ht]
\centering
\label{fig:map_all_medmcqa}
\begin{subfigure}[t]{0.32\textwidth}
  \centering
  \includegraphics[width=0.9\linewidth]{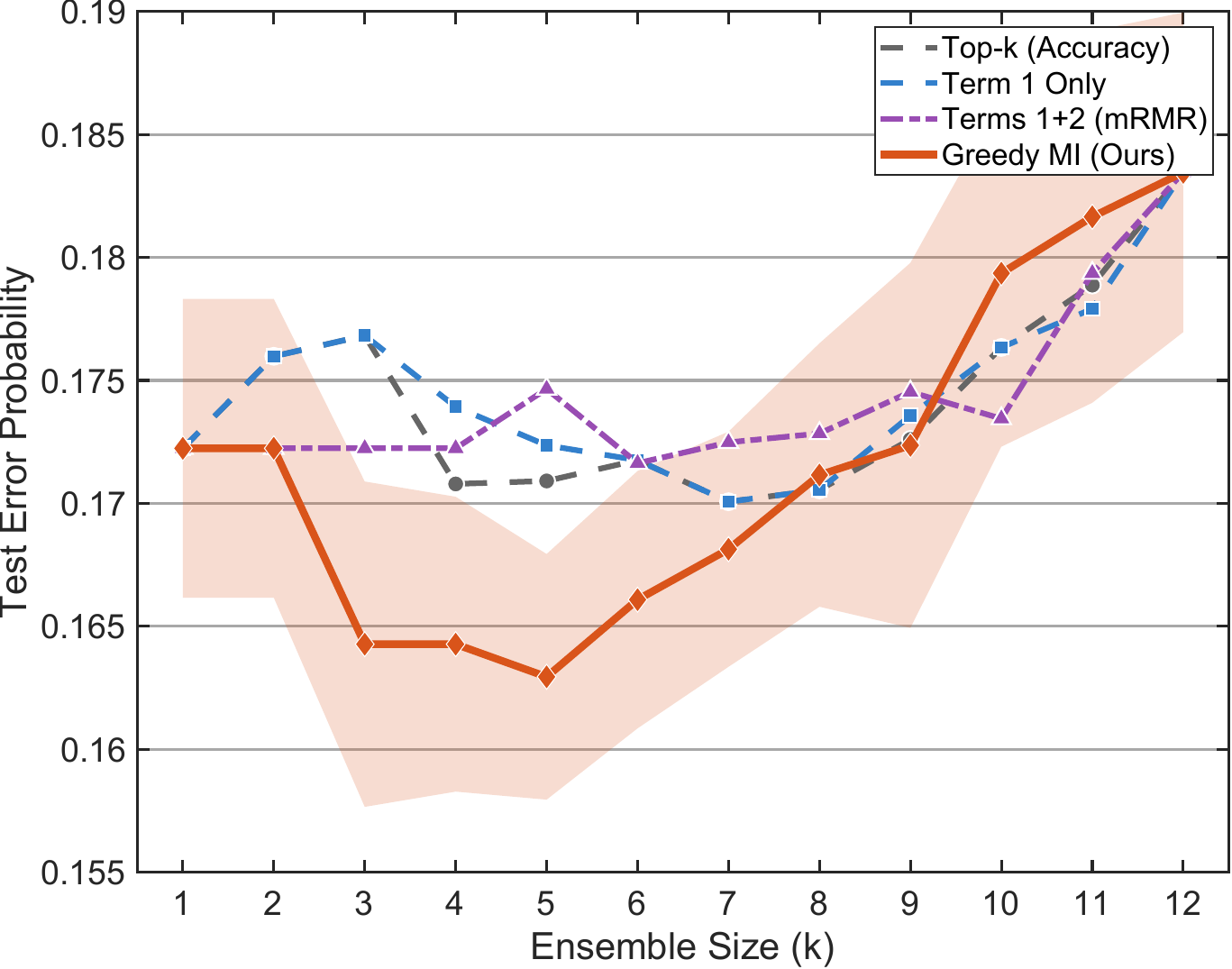}
  \caption{\(\text{temp}=0.01\), run 1}
\end{subfigure}\hfill
\begin{subfigure}[t]{0.32\textwidth}
  \centering
  \includegraphics[width=0.9\linewidth]{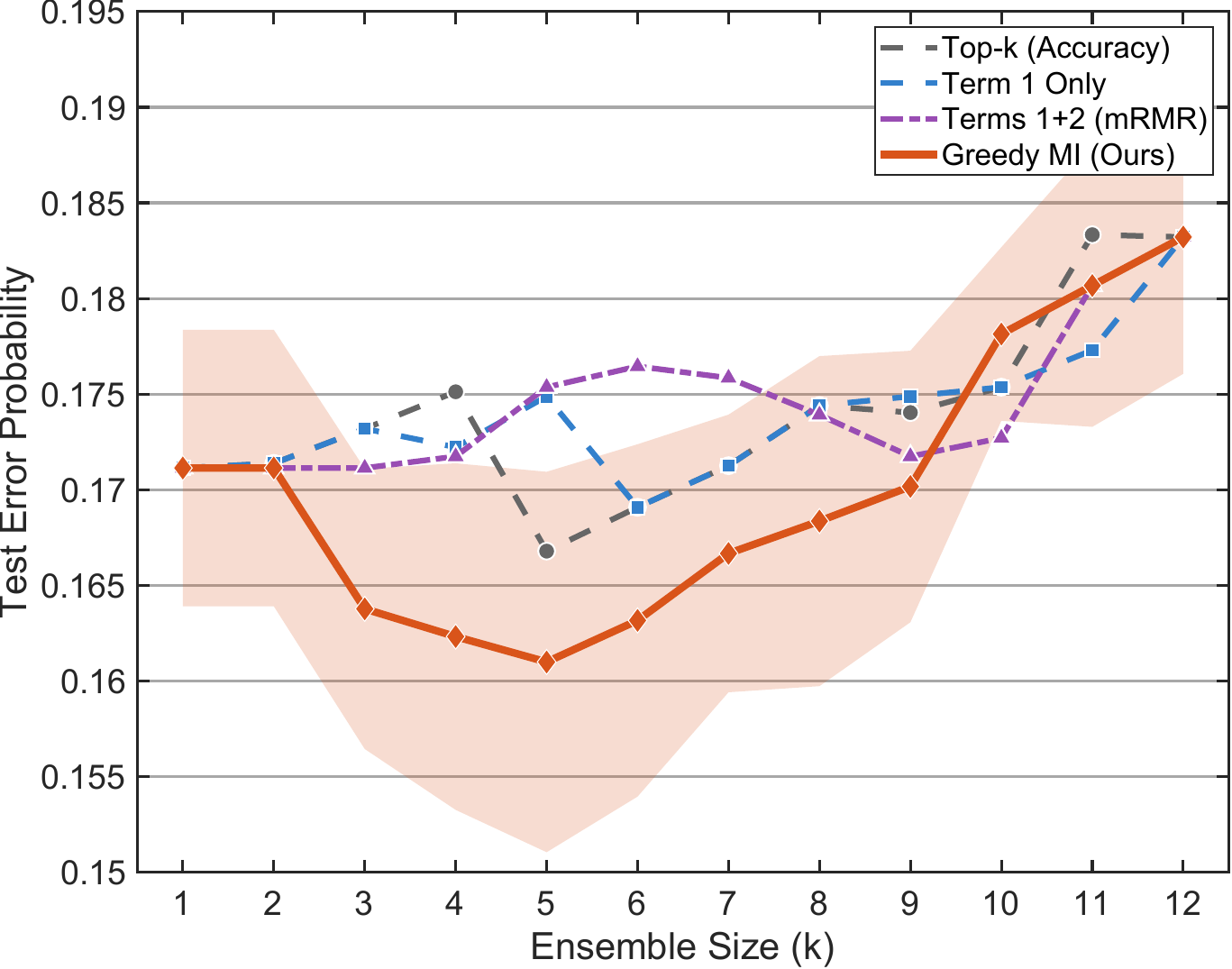}
  \caption{\(\text{temp}=0.01\), run 2}
\end{subfigure}\hfill
\begin{subfigure}[t]{0.32\textwidth}
  \centering
  \includegraphics[width=0.9\linewidth]{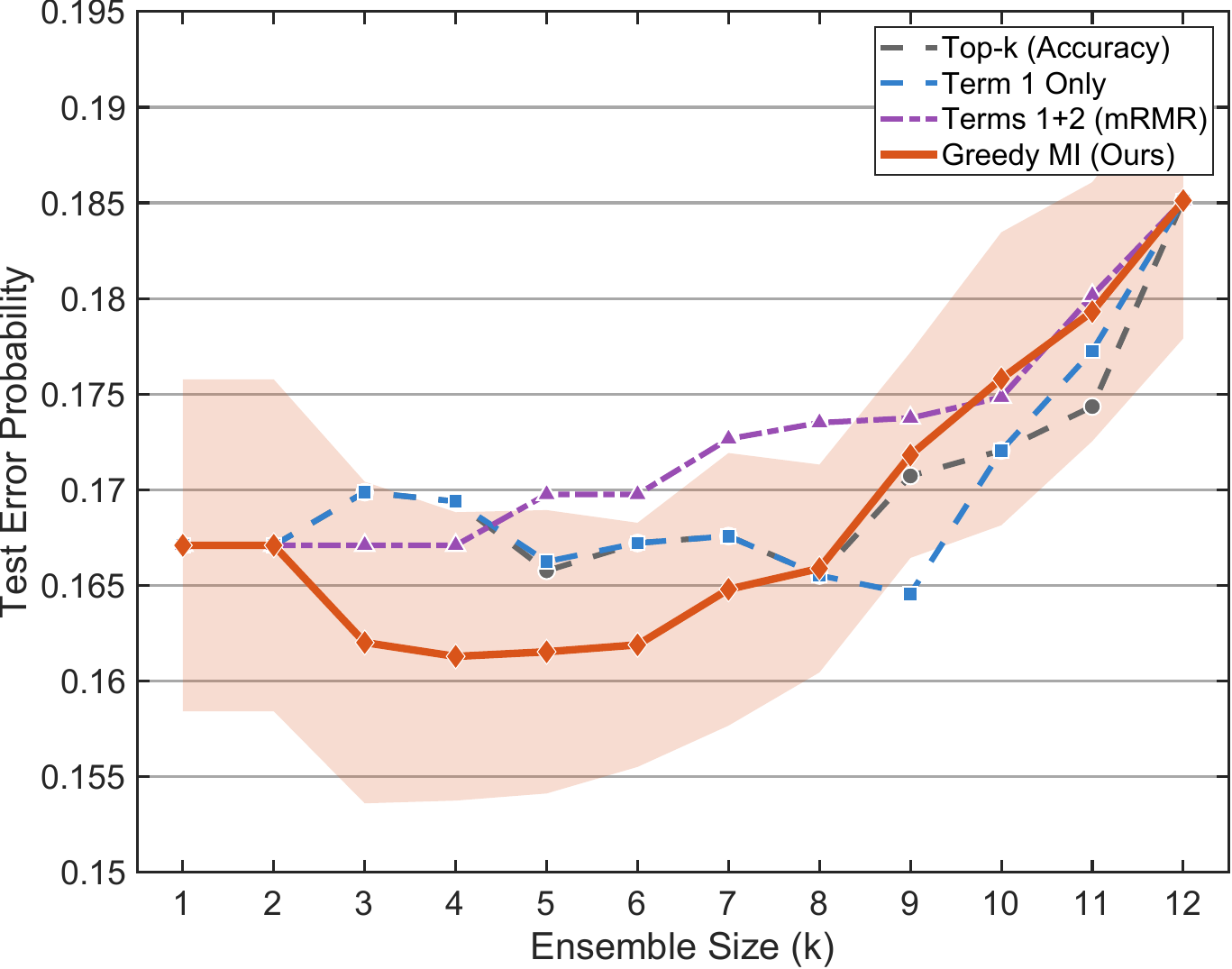}
  \caption{\(\text{temp}=0.3\), run 1}
\end{subfigure}


\begin{subfigure}[t]{0.32\textwidth}
  \centering
  \includegraphics[width=0.9\linewidth]{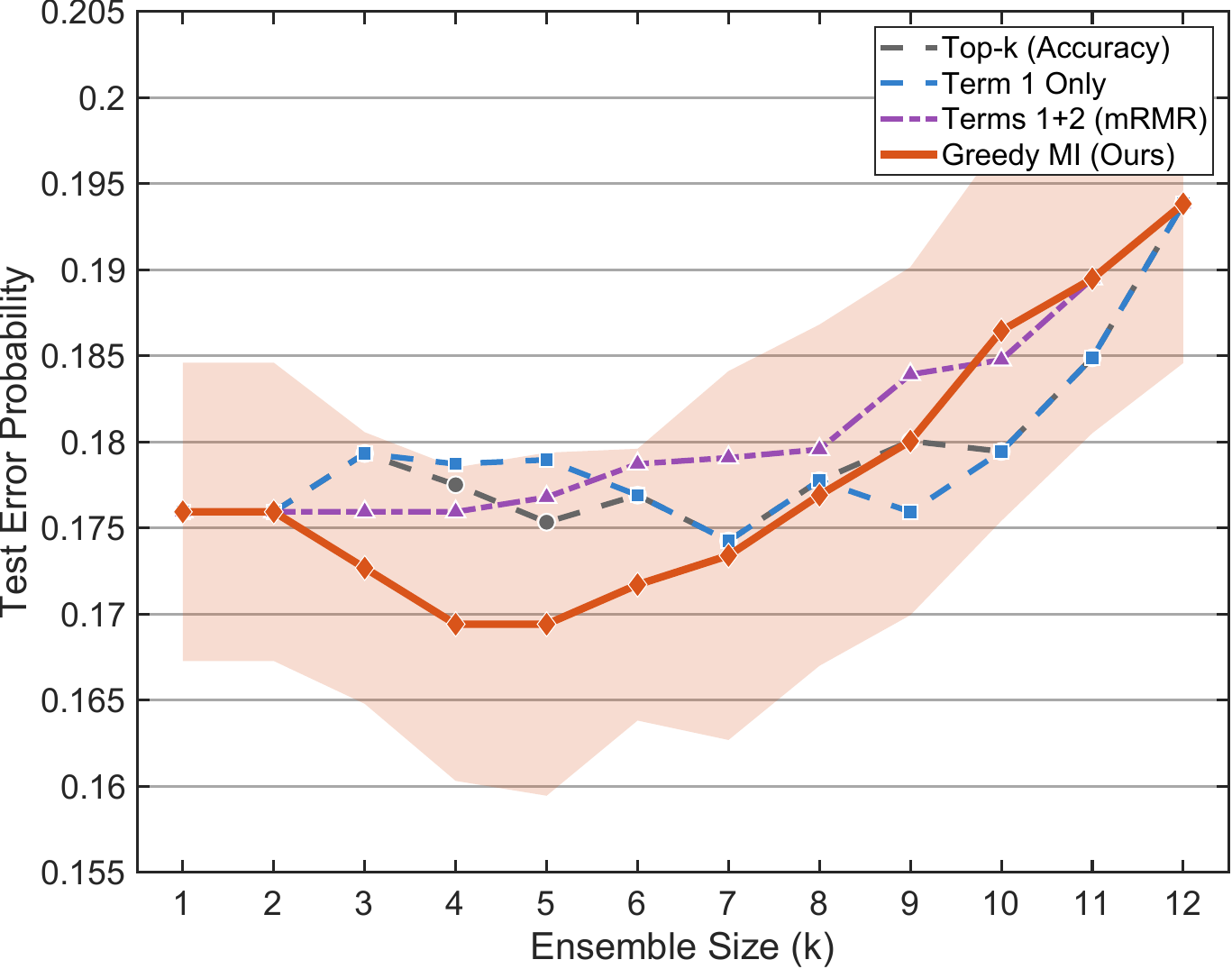}
  \caption{\(\text{temp}=0.3\), run 2}
\end{subfigure}\hfill
\begin{subfigure}[t]{0.32\textwidth}
  \centering
  \includegraphics[width=0.9\linewidth]{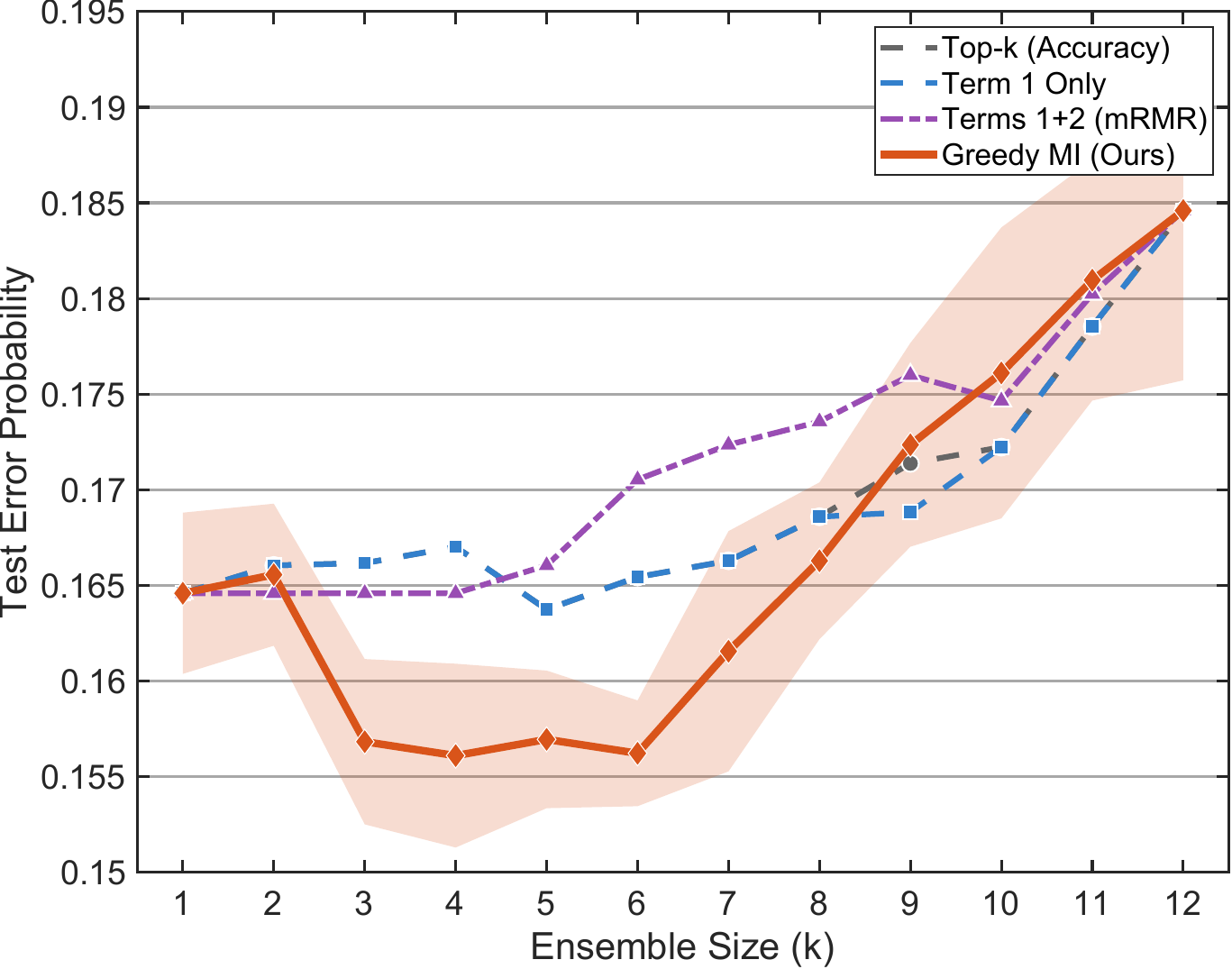}
  \caption{\(\text{temp}=0.7\), run 1}
\end{subfigure}\hfill
\begin{subfigure}[t]{0.32\textwidth}
  \centering
  \includegraphics[width=0.9\linewidth]{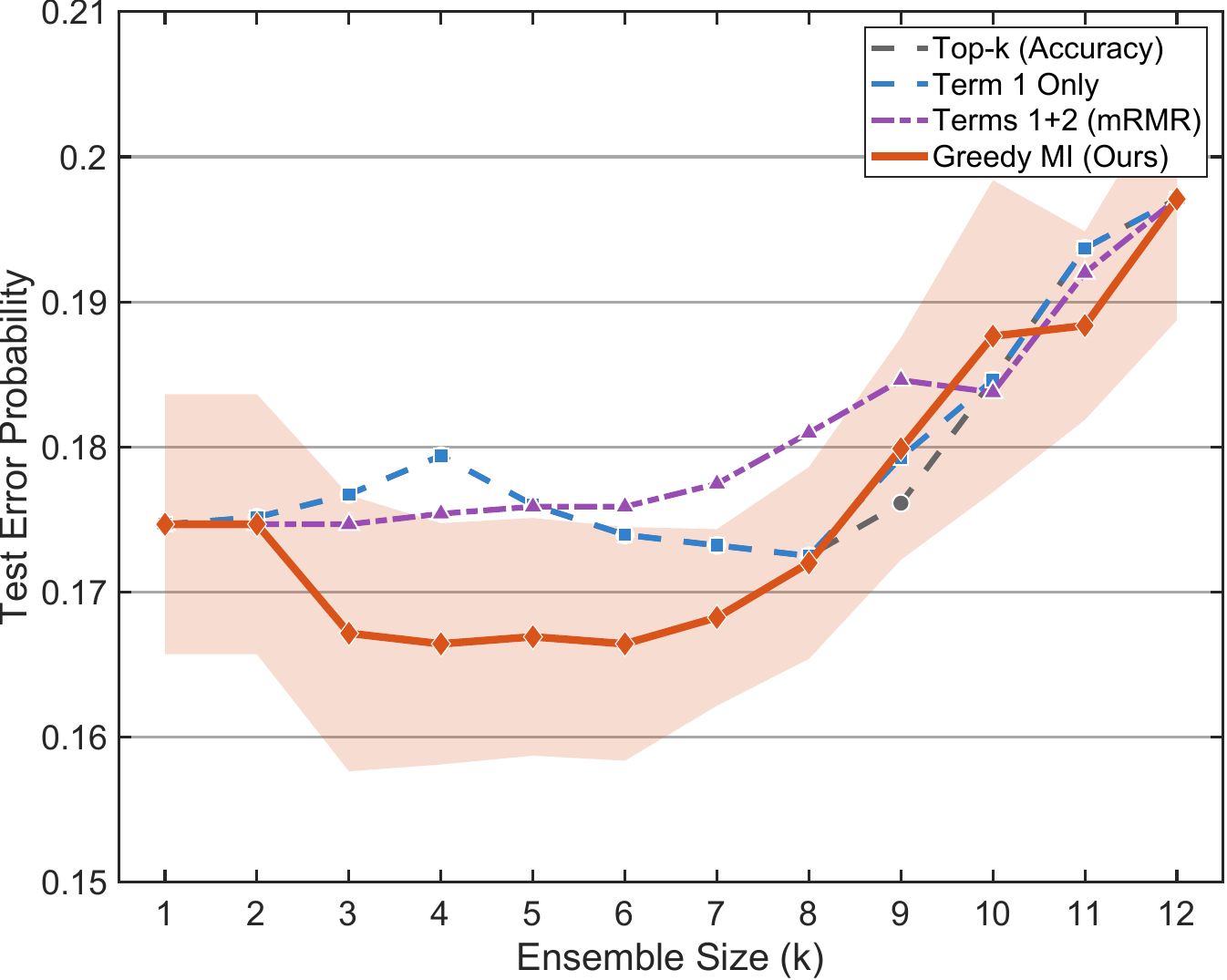}
  \caption{\(\text{temp}=0.7\), run 2}
\end{subfigure}

\caption{Experiments across temperatures and runs (\textbf{the MAP aggregation}) on MEDMCQA. Each panel corresponds to one temperature setting and random run. Shaded region represents the standard deviation.}
\label{fig:appendix_map_medmcqa_6}
\end{figure*}

\subsection{Alternative Aggregation Rules}
\label{app:medmcqa_agg}
Our main experiments use the MAP aggregator for the selected subset $S$ (Algorithm~\ref{alg:map}), since it directly estimates $\mathbb{P}(Y \mid X_S)$ and can exploit higher-order dependencies in correlated outputs. To probe robustness to the aggregation rule, we additionally evaluate two widely-used alternatives for the baseline selectors: (i) \emph{majority vote} (MV) (\cref{alg:mv}), and (ii) \emph{weighted majority vote} (W-MV) (\cref{alg:wmv}), where weights are derived from single-model reliability on the training split. Note that in these comparisons, we keep \emph{selection} fixed (Top-$k$, Term~1, Terms~1+2) and only change the \emph{aggregation} rule for those baselines. Greedy MI continues to use MAP, because its gains are designed to be realized under a likelihood-based estimator.

With the help of visualizations from \cref{fig:medmcqa_appendix_mv_6}, and tables from \cref{tab:medmcqa_mv_vs_map}, we observe that the mRMR selector (Terms~1+2) is highly unstable under MV: at $k=2$ it jumps to $0.264$ error, versus $0.181$ for Top-$k$ and Term~1, and $0.171$ for Greedy MI (MAP). The instability persists across the budget range, with Terms~1+2 attaining $0.231$ error at $k=3$ and $0.220$ at $k=5$, while Greedy MI achieves $0.164$ and $0.163$, respectively. In the mid-budget regime ($k=3$--$7$), Greedy MI consistently maintains the lowest error (e.g., $k=5$: $0.163$ vs.\ $0.170$ for Top-$k$; $k=6$: $0.164$ vs.\ $0.170$ for Top-$k$).

\begin{figure*}[ht]
\centering

\begin{subfigure}[t]{0.32\textwidth}
  \centering
  \includegraphics[width=0.9\linewidth]{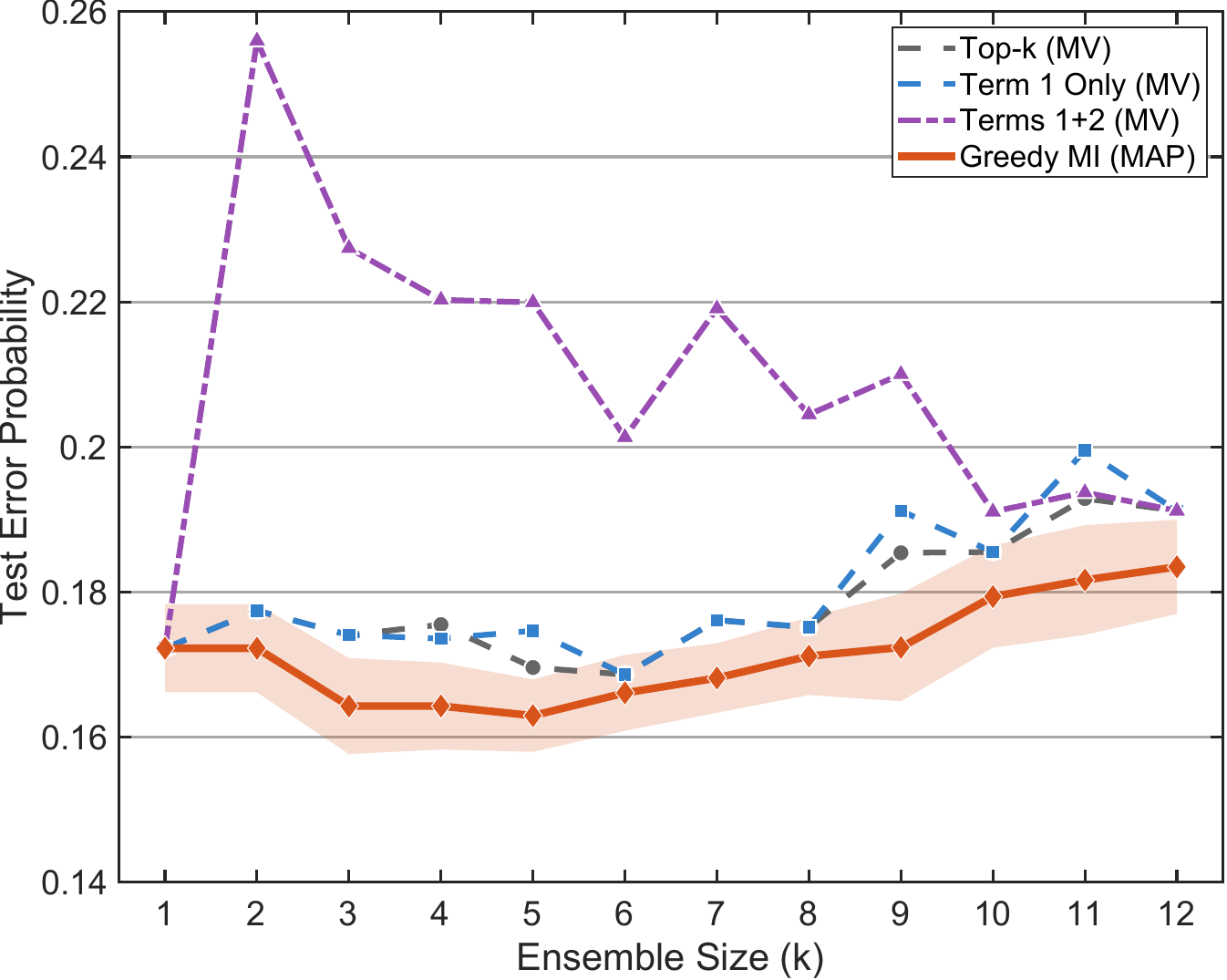}
  \caption{\(\text{temp}=0.01\), run 1}
\end{subfigure}\hfill
\begin{subfigure}[t]{0.32\textwidth}
  \centering
  \includegraphics[width=0.9\linewidth]{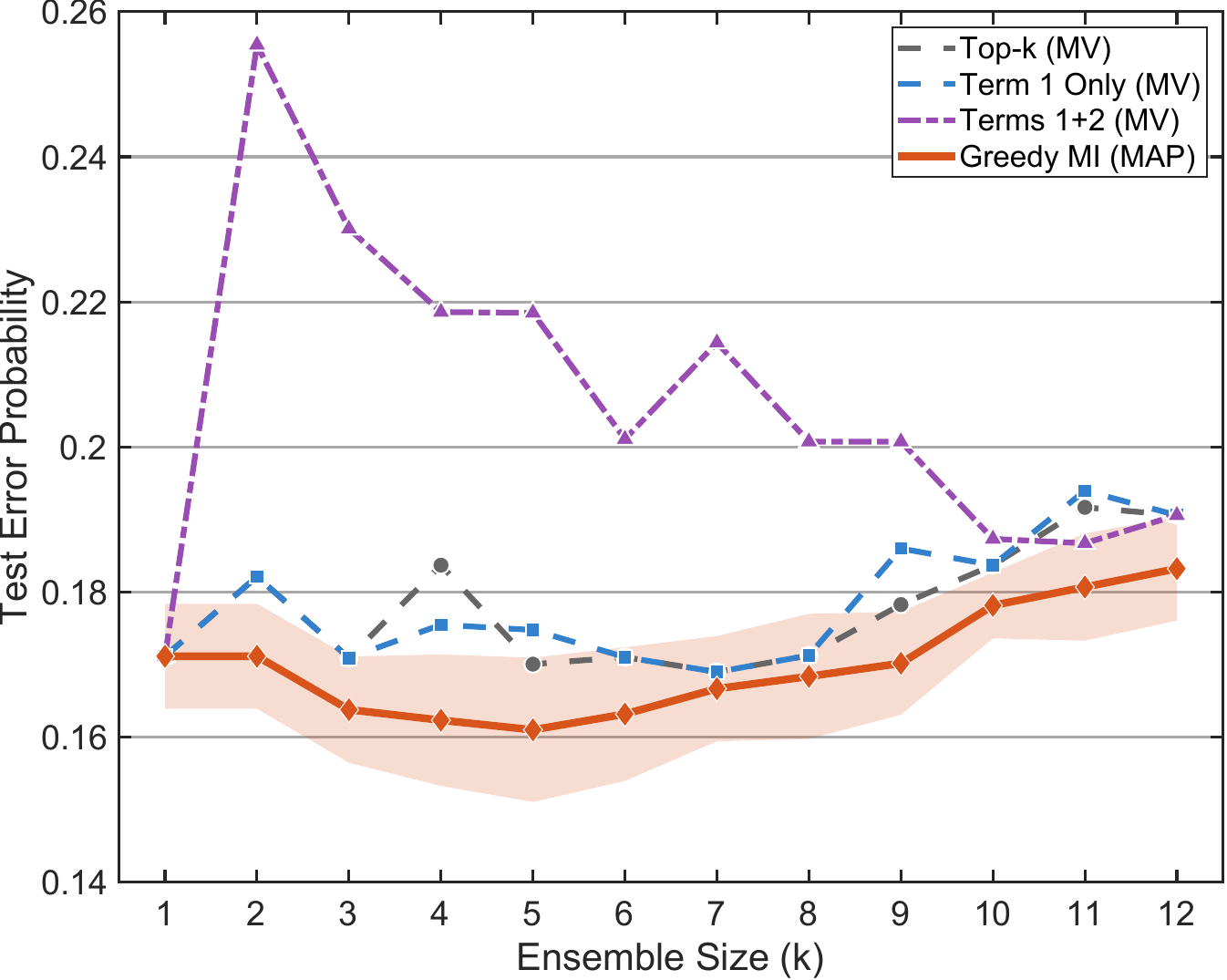}
  \caption{\(\text{temp}=0.01\), run 2}
\end{subfigure}\hfill
\begin{subfigure}[t]{0.32\textwidth}
  \centering
  \includegraphics[width=0.9\linewidth]{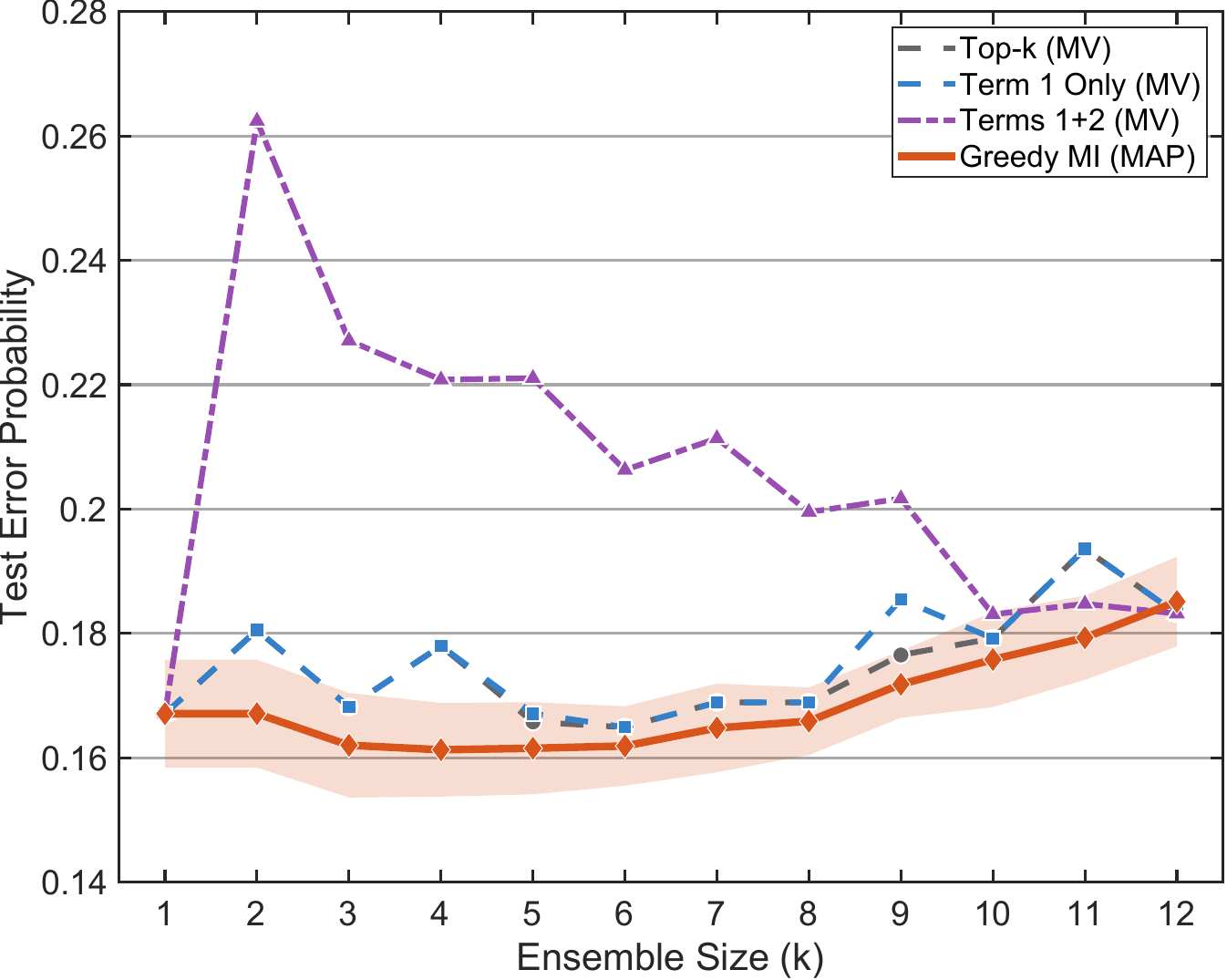}
  \caption{\(\text{temp}=0.3\), run 1}
\end{subfigure}


\begin{subfigure}[t]{0.32\textwidth}
  \centering
  \includegraphics[width=0.9\linewidth]{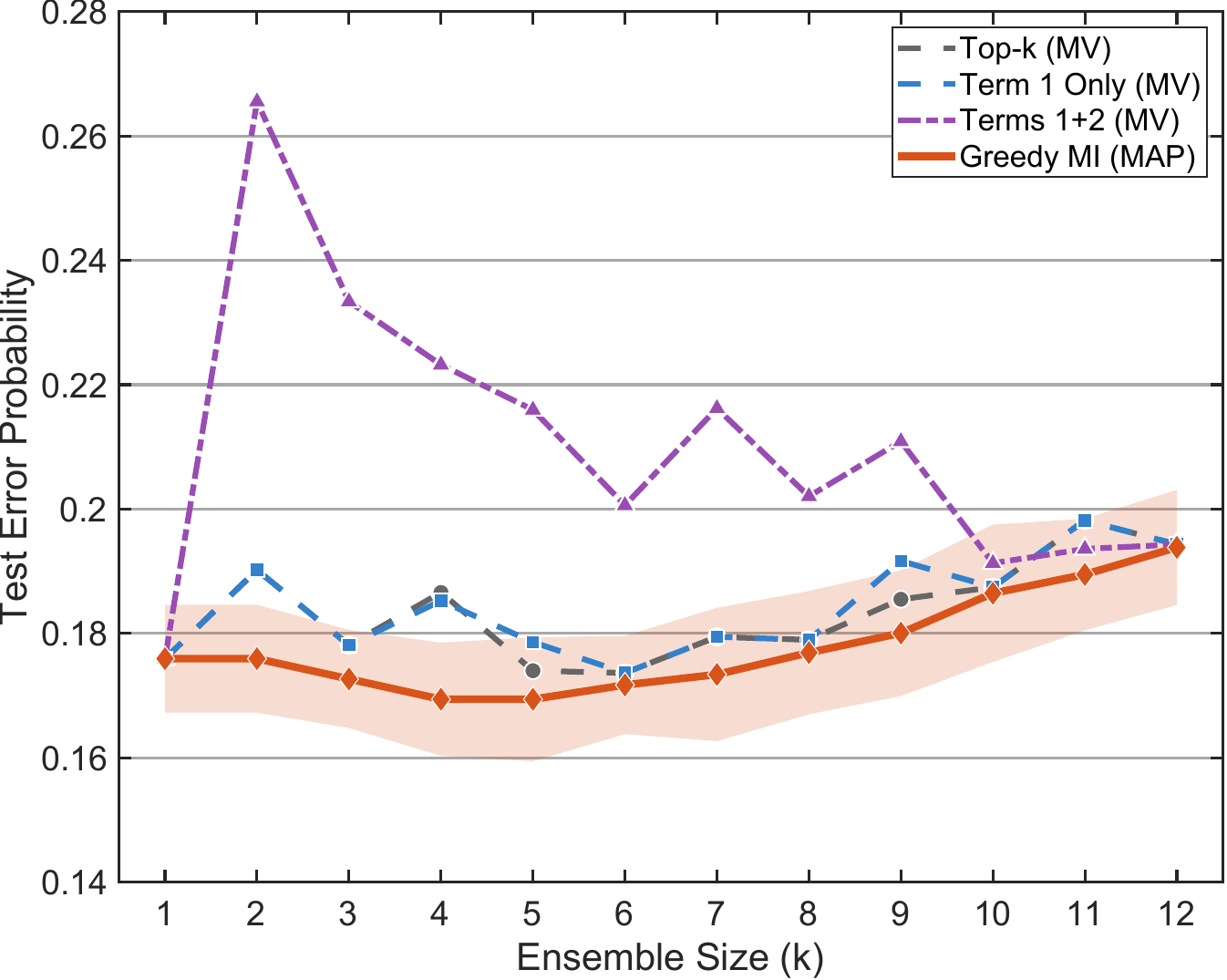}
  \caption{\(\text{temp}=0.3\), run 2}
\end{subfigure}\hfill
\begin{subfigure}[t]{0.32\textwidth}
  \centering
  \includegraphics[width=0.9\linewidth]{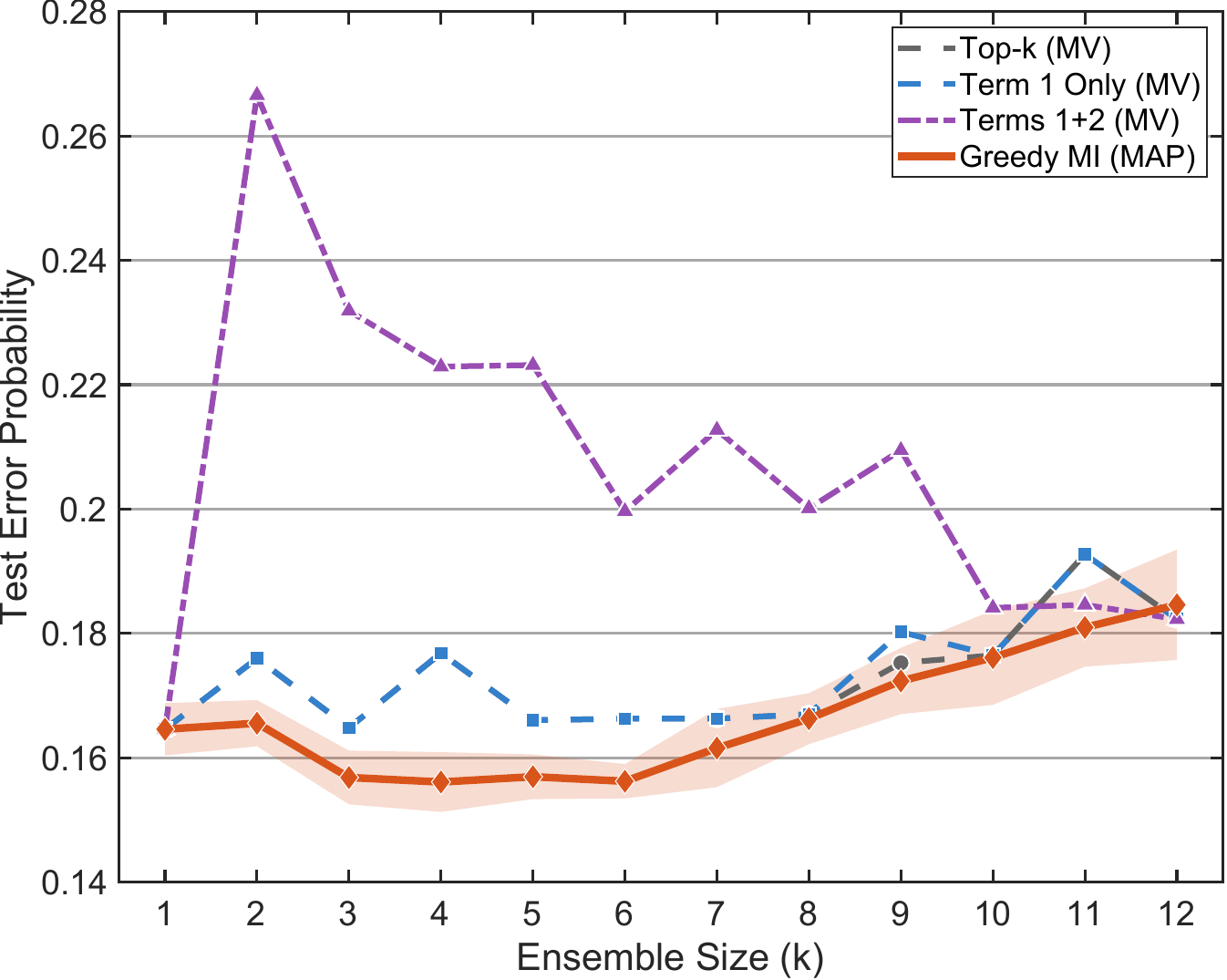}
  \caption{\(\text{temp}=0.7\), run 1}
\end{subfigure}\hfill
\begin{subfigure}[t]{0.32\textwidth}
  \centering
  \includegraphics[width=0.9\linewidth]{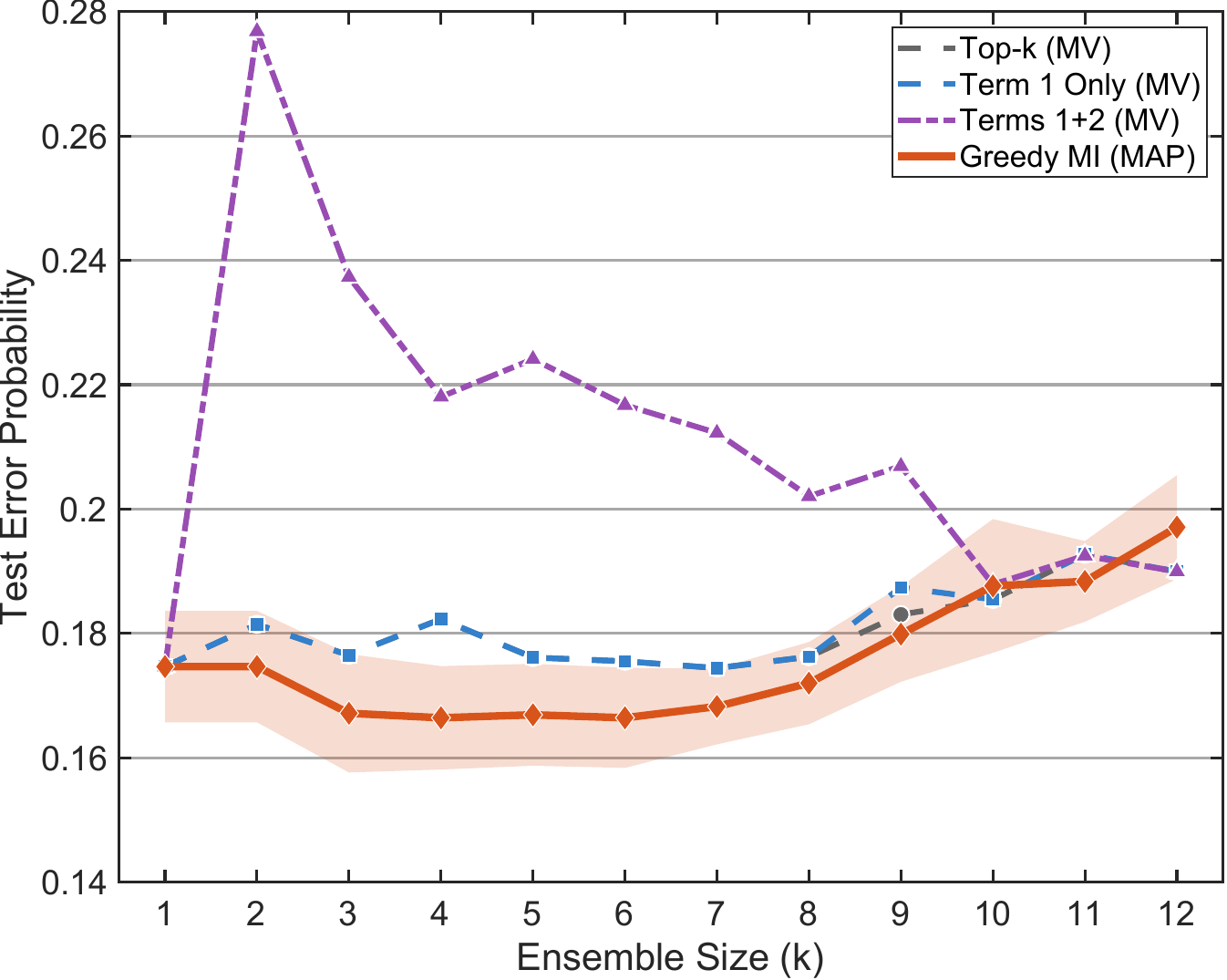}
  \caption{\(\text{temp}=0.7\), run 2}
\end{subfigure}

\caption{MEDMCQA results using \textbf{Majority Vote (MV)} aggregation across temperatures and runs. Shaded region represents the standard deviation.}
\label{fig:medmcqa_appendix_mv_6}
\end{figure*}

\begin{figure*}[ht]
\centering

\begin{subfigure}[t]{0.32\textwidth}
  \centering
  \includegraphics[width=0.9\linewidth]{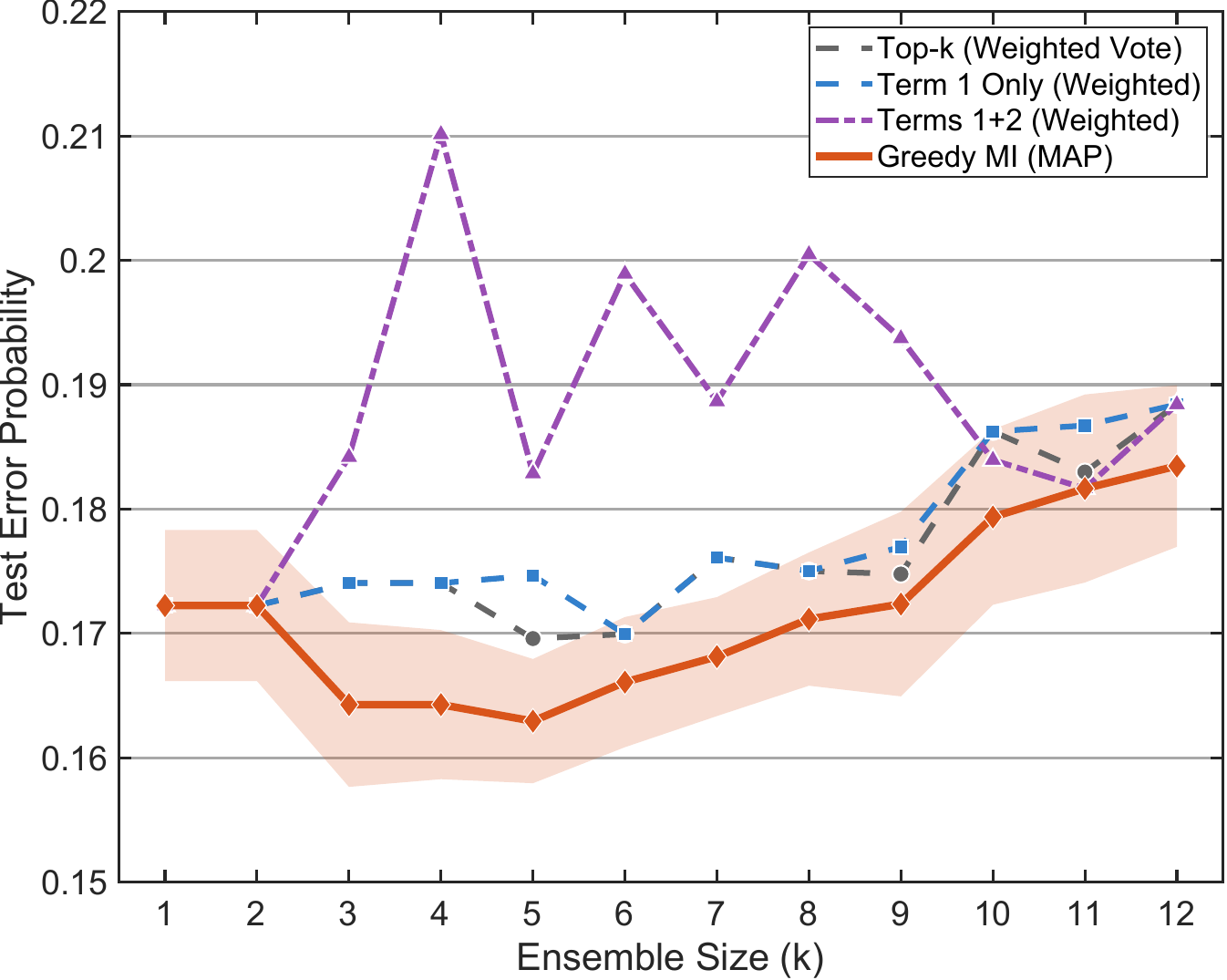}
  \caption{\(\text{temp}=0.01\), run 1}
\end{subfigure}\hfill
\begin{subfigure}[t]{0.32\textwidth}
  \centering
  \includegraphics[width=0.9\linewidth]{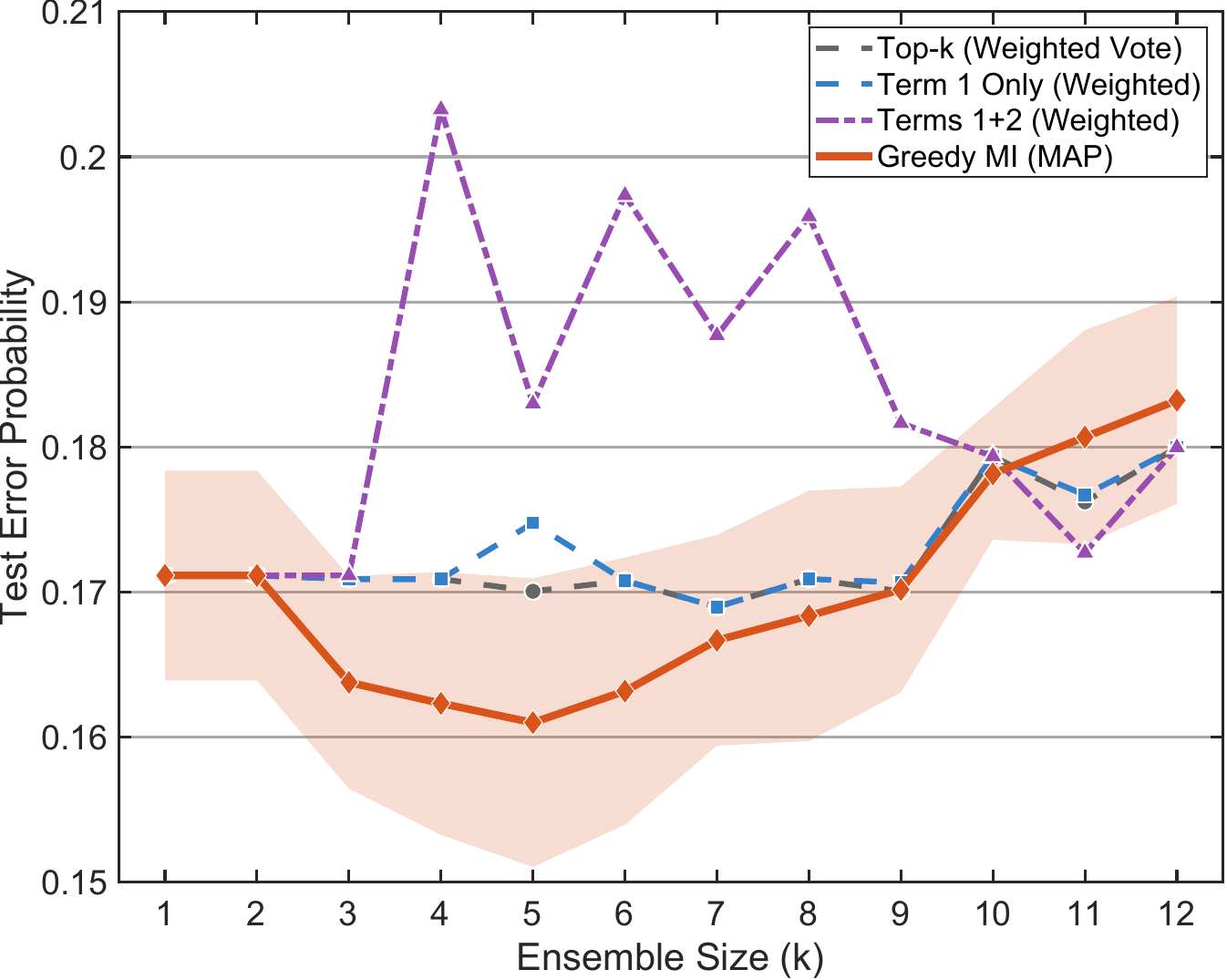}
  \caption{\(\text{temp}=0.01\), run 2}
\end{subfigure}\hfill
\begin{subfigure}[t]{0.32\textwidth}
  \centering
  \includegraphics[width=0.9\linewidth]{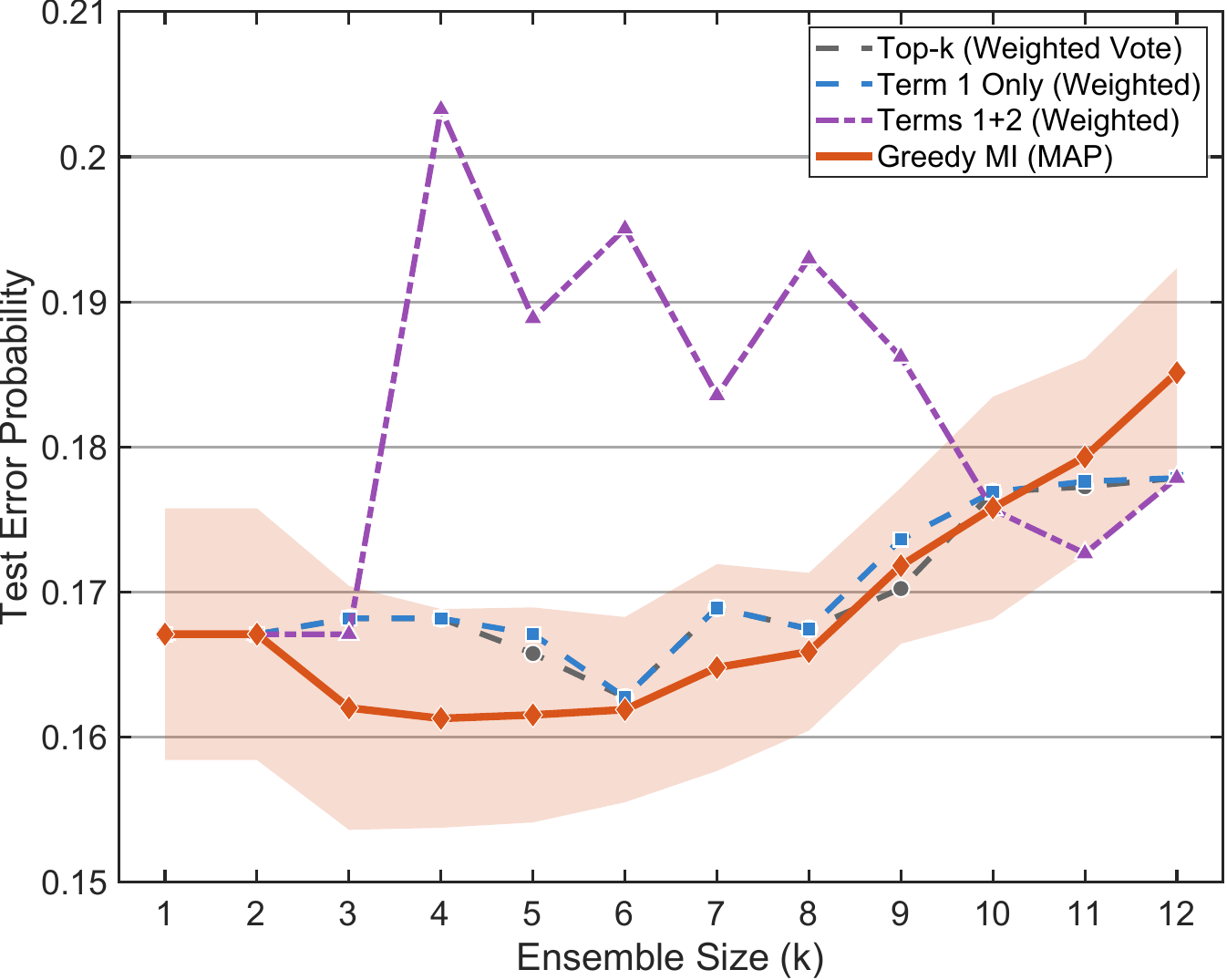}
  \caption{\(\text{temp}=0.3\), run 1}
\end{subfigure}


\begin{subfigure}[t]{0.32\textwidth}
  \centering
  \includegraphics[width=0.9\linewidth]{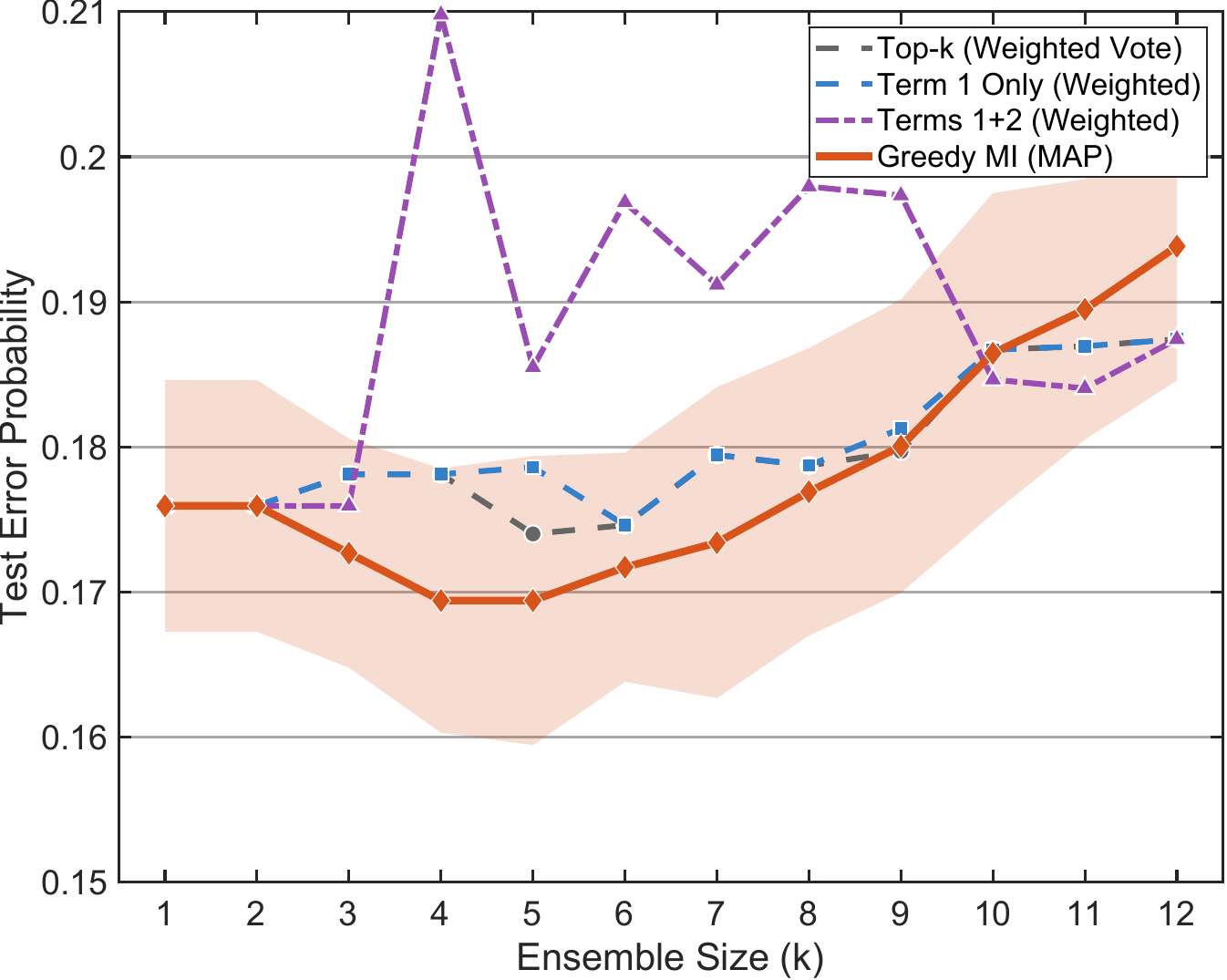}
  \caption{\(\text{temp}=0.3\), run 2}
\end{subfigure}\hfill
\begin{subfigure}[t]{0.32\textwidth}
  \centering
  \includegraphics[width=0.9\linewidth]{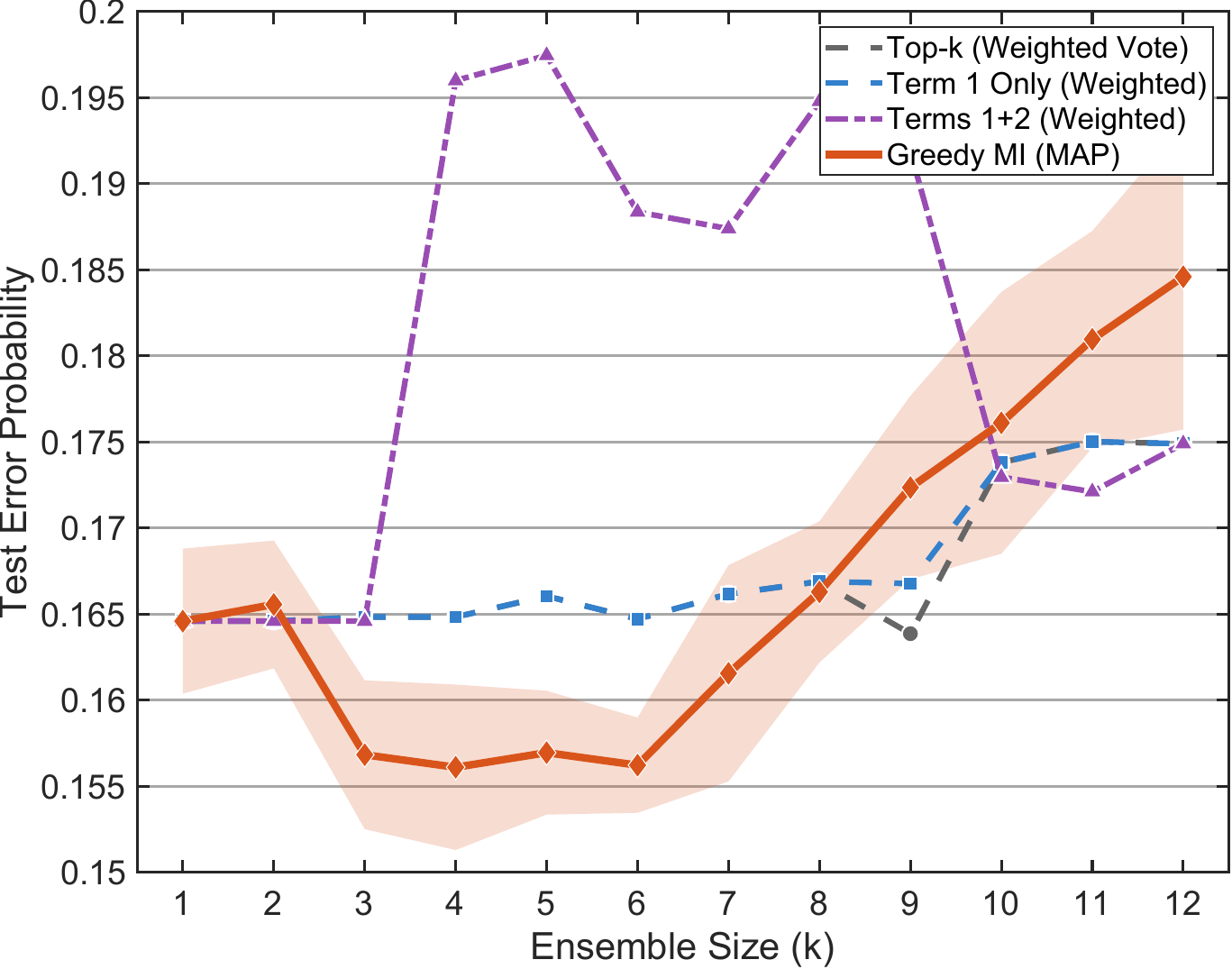}
  \caption{\(\text{temp}=0.7\), run 1}
\end{subfigure}\hfill
\begin{subfigure}[t]{0.32\textwidth}
  \centering
  \includegraphics[width=0.9\linewidth]{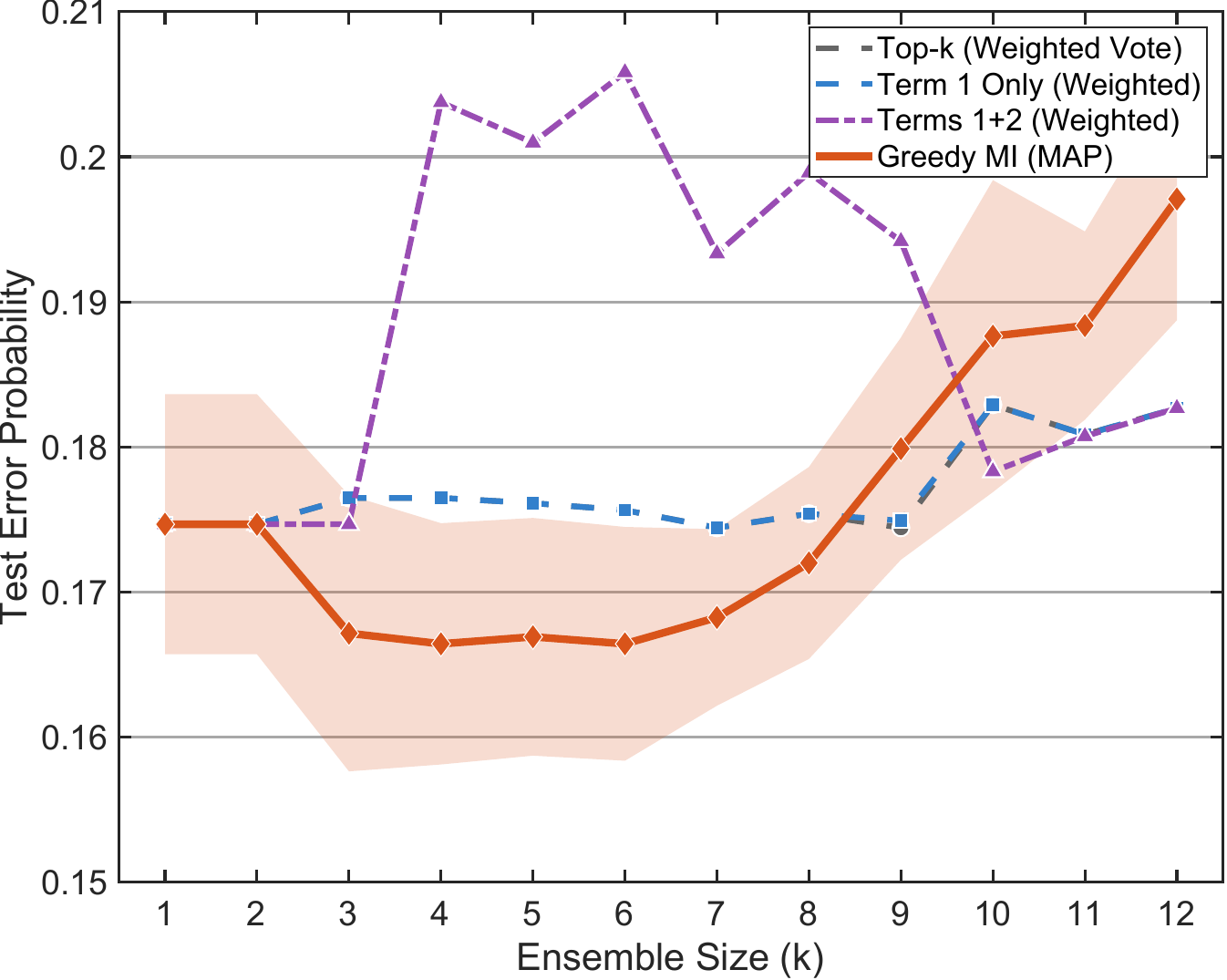}
  \caption{\(\text{temp}=0.7\), run 2}
\end{subfigure}

\caption{MEDMCQA results using \textbf{Weighted Ensemble (WE)} aggregation across temperatures and runs. Shaded region represents the standard deviation.}
\label{fig:medmcqa_appendix_we_6}
\end{figure*}
Under W-MV, in \cref{fig:medmcqa_appendix_we_6}, and in \cref{tab:medmcqa_wmv_vs_map}, we see that Terms~1+2 remains substantially worse in the mid-budget range (e.g., $k=4$: $0.204$; $k=6$: $0.197$), whereas Greedy MI (MAP) achieves $0.163 \pm 0.008$ at $k=4$ and $0.164$ at $k=6$. Notably, Top-$k$ with W-MV performs comparably to Top-$k$ with MAP, suggesting that weighted voting partially compensates for ignoring correlations in the selection phase. However, it does not resolve the failure mode of mRMR, which continues to suffer from its aggressive diversity-seeking strategy. At larger budgets ($k \geq 10$), all methods begin to converge as the ensemble includes most available models. In panels (b), (c), (d), (e) and (f), baselines with WE aggregation occasionally outperform Greedy MI with MAP for large ensemble sizes. Two factors contribute to this situation: first, Greedy MI provides a near-optimal rather than optimal selection; second, as discussed previously, estimating the joint distribution becomes increasingly difficult at higher k, degrading Greedy MI's performance. A well-chosen aggregation rule can partially compensate for these effects.

Overall, these results indicate that changing the aggregation rule can partially improve Top-$k$/Term~1, particularly under W-MV, but it does not resolve the fundamental issues with mRMR under MV/W-MV. Greedy MI remains most competitive in the moderate-budget regime ($k \approx 3$--$7$) where selection quality matters most.

\newpage
\subsection{Copula Validation Plots across Temperatures and Runs}\label{app:copula_validation_medmcqa}
Additionally, we provide the following scatter and tail plots for each temperature setting and run pair below.

Figure~\ref{fig:app_medmcqa_copula_scatter_6up} compares the empirical pairwise joint error probabilities $P(E_i \cap E_j)$ with copula-predicted values across all $\binom{12}{2} = 66$ model pairs. The points align closely along the diagonal across all six temperature-run conditions, indicating strong agreement between the fitted copula and observed data.

Figure~\ref{fig:app_medmcqa_copula_heavy_6up} evaluates higher-order structure by showing the distribution of simultaneous errors that is, how many models fail on the same instance. The copula-generated distributions (orange bars) match the empirical frequencies (blue bars) well across all settings.

\begin{figure*}[ht]
\centering
\begin{subfigure}[t]{0.32\textwidth}
    \centering
    \includegraphics[width=0.8\linewidth]{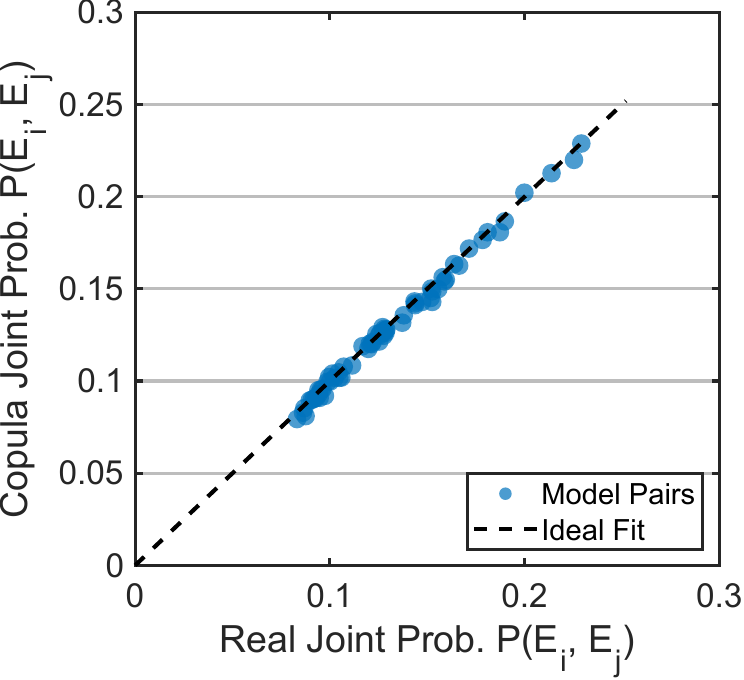}
    \caption{\(\text{temp}=0.01\), run 1}
\end{subfigure}\hfill
\begin{subfigure}[t]{0.32\textwidth}
    \centering
    \includegraphics[width=0.8\linewidth]{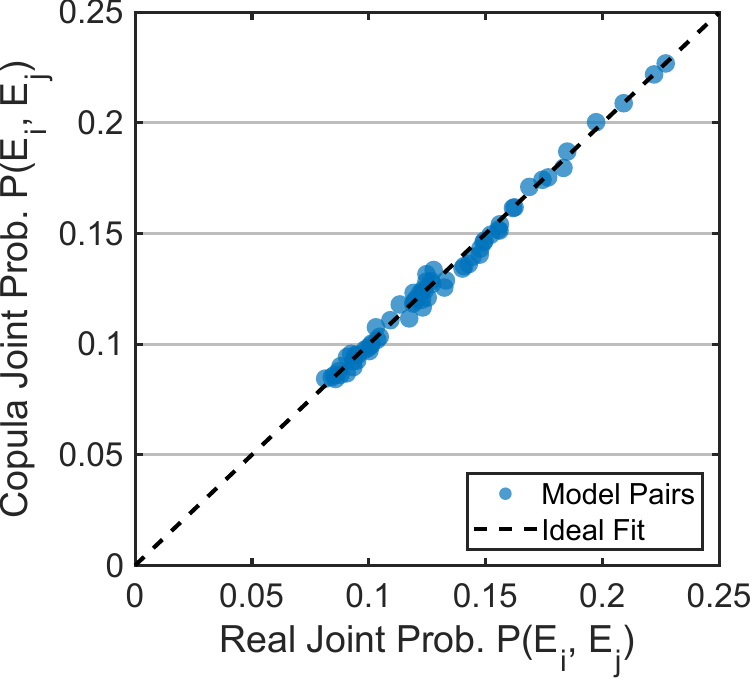}
    \caption{\(\text{temp}=0.01\), run 2}
\end{subfigure}\hfill
\begin{subfigure}[t]{0.32\textwidth}
    \centering
    \includegraphics[width=0.8\linewidth]{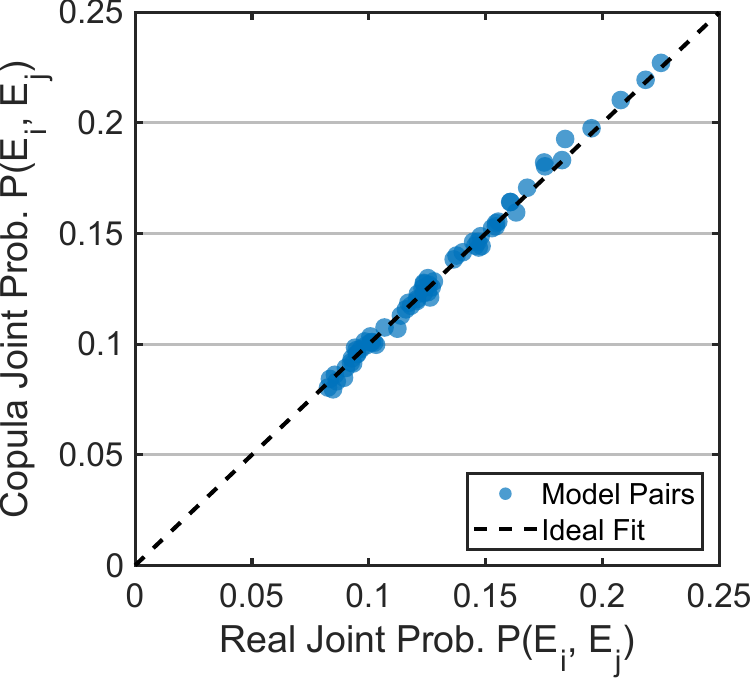}
    \caption{\(\text{temp}=0.3\), run 1}
\end{subfigure}


\begin{subfigure}[t]{0.32\textwidth}
    \centering
    \includegraphics[width=0.8\linewidth]{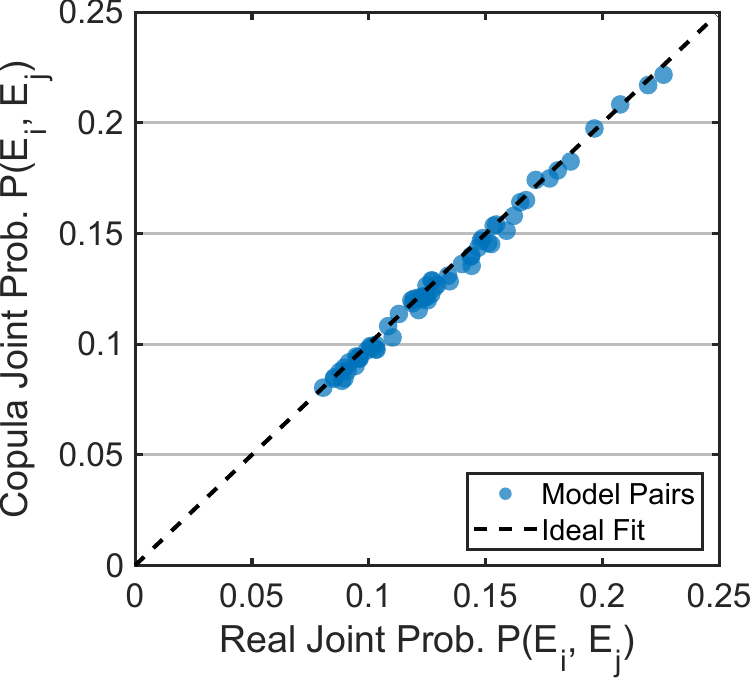}
    \caption{\(\text{temp}=0.3\), run 2}
\end{subfigure}\hfill
\begin{subfigure}[t]{0.32\textwidth}
    \centering
    \includegraphics[width=0.8\linewidth]{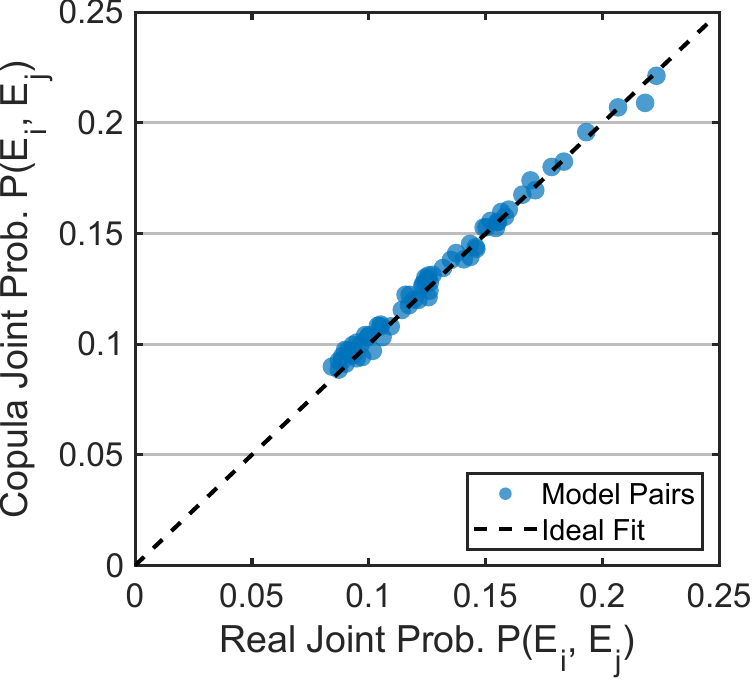}
    \caption{\(\text{temp}=0.7\), run 1}
\end{subfigure}\hfill
\begin{subfigure}[t]{0.32\textwidth}
    \centering
    \includegraphics[width=0.8\linewidth]{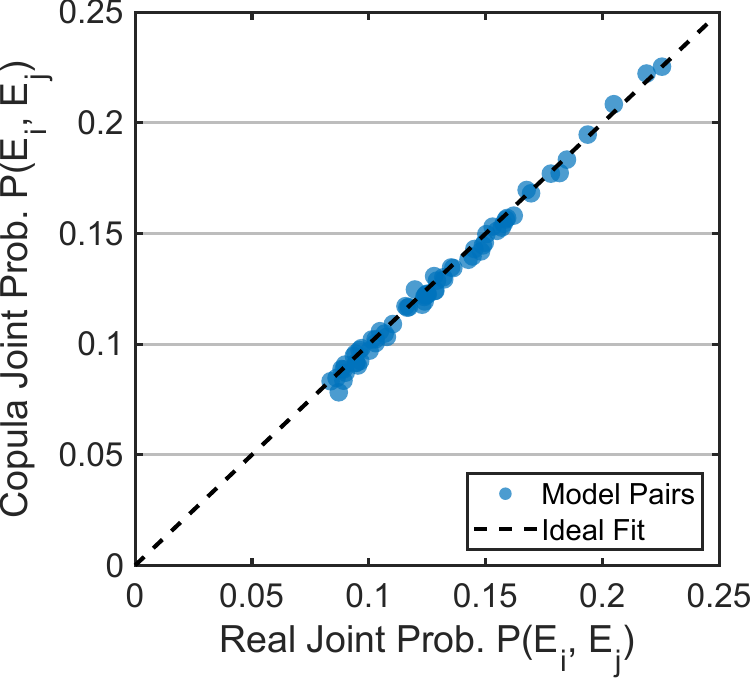}
    \caption{\(\text{temp}=0.7\), run 2}
\end{subfigure}

\caption{Gaussian-copula validation on MEDMCQA (scatter plots): pairwise structural fit diagnostics across temperatures and runs.}
\label{fig:app_medmcqa_copula_scatter_6up}
\end{figure*}

\begin{figure*}[ht]
\centering
\begin{subfigure}[t]{0.32\textwidth}
    \centering
    \includegraphics[width=0.8\linewidth]{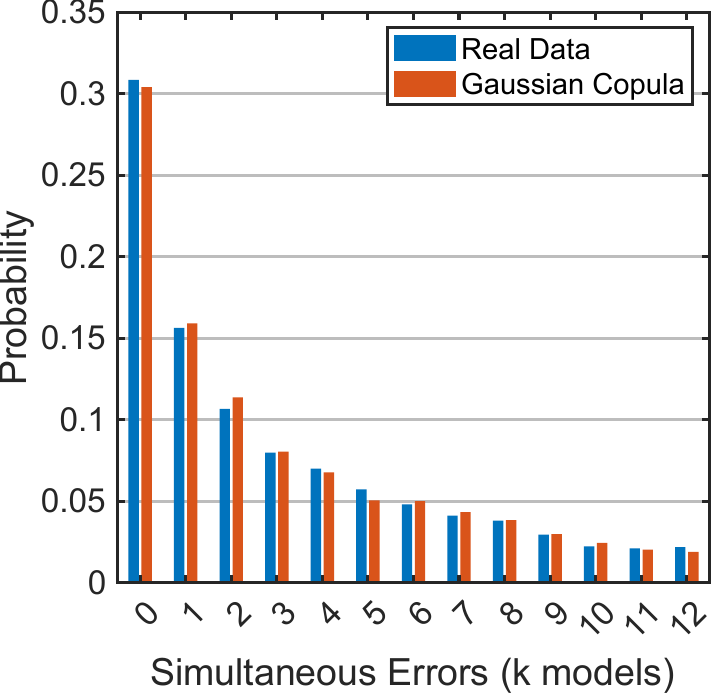}
    \caption{\(\text{temp}=0.01\), run 1}
\end{subfigure}\hfill
\begin{subfigure}[t]{0.32\textwidth}
    \centering
    \includegraphics[width=0.8\linewidth]{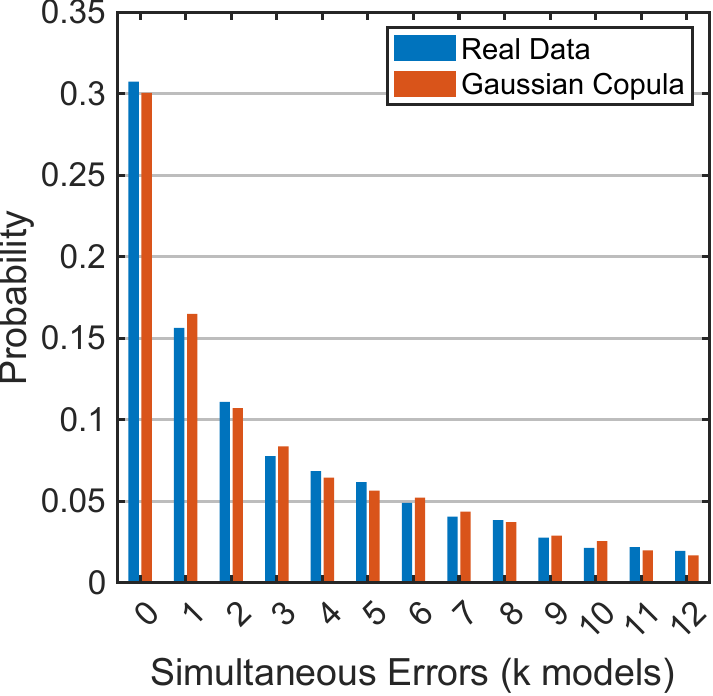}
    \caption{\(\text{temp}=0.01\), run 2}
\end{subfigure}\hfill
\begin{subfigure}[t]{0.32\textwidth}
    \centering
    \includegraphics[width=0.8\linewidth]{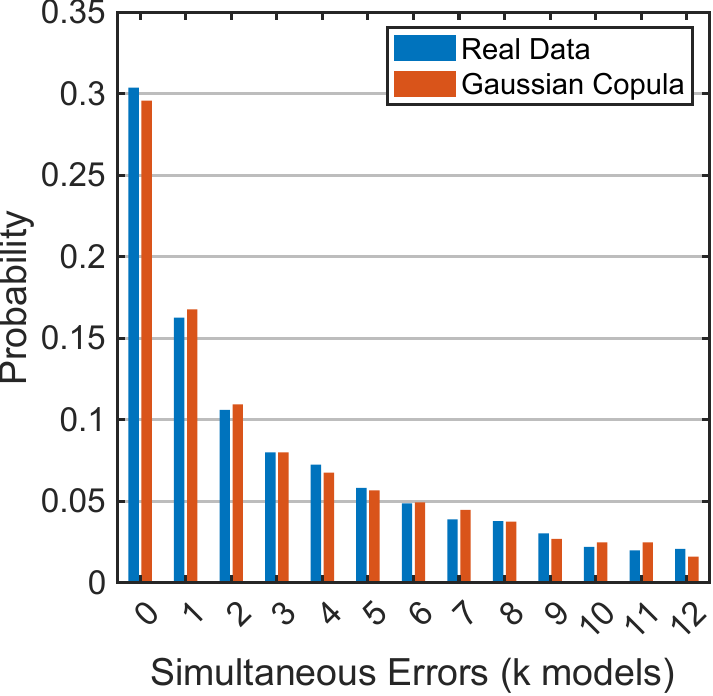}
    \caption{\(\text{temp}=0.3\), run 1}
\end{subfigure}


\begin{subfigure}[t]{0.32\textwidth}
    \centering
    \includegraphics[width=0.8\linewidth]{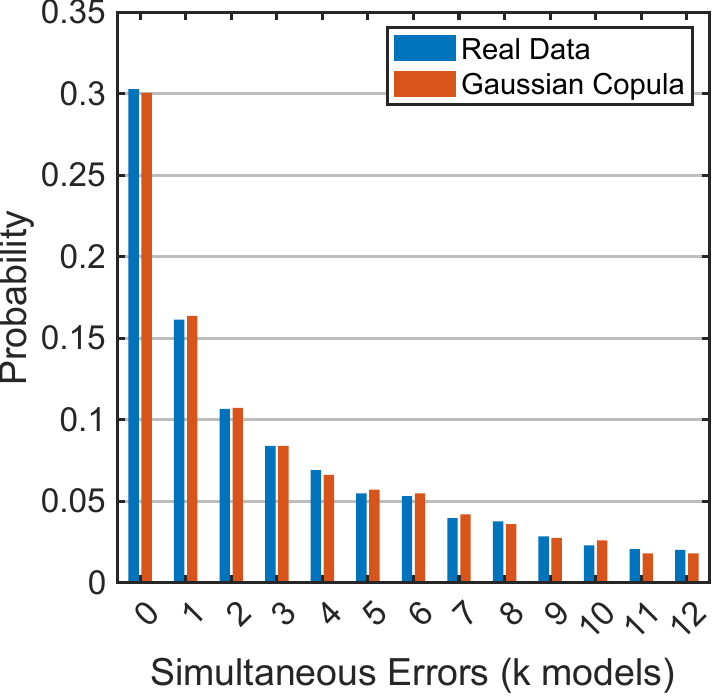}
    \caption{\(\text{temp}=0.3\), run 2}
\end{subfigure}\hfill
\begin{subfigure}[t]{0.32\textwidth}
    \centering
    \includegraphics[width=0.8\linewidth]{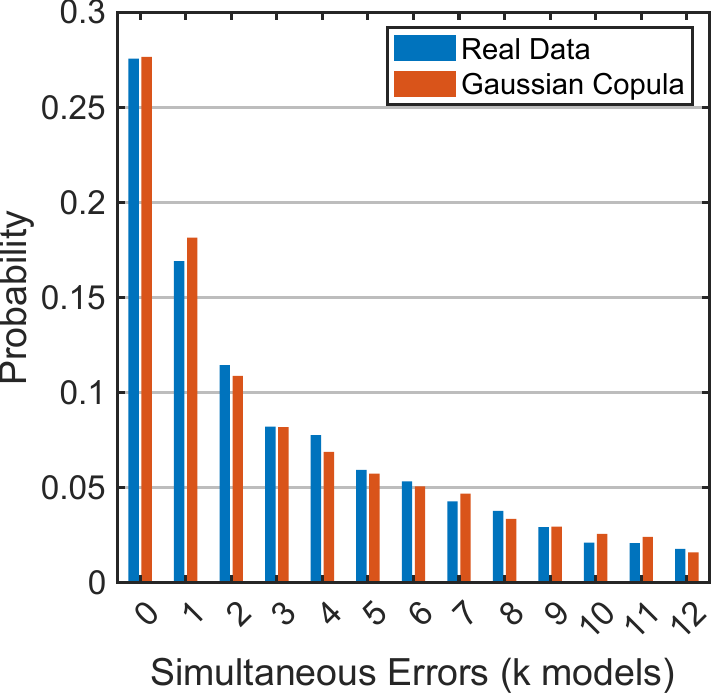}
    \caption{\(\text{temp}=0.7\), run 1}
\end{subfigure}\hfill
\begin{subfigure}[t]{0.32\textwidth}
    \centering
    \includegraphics[width=0.8\linewidth]{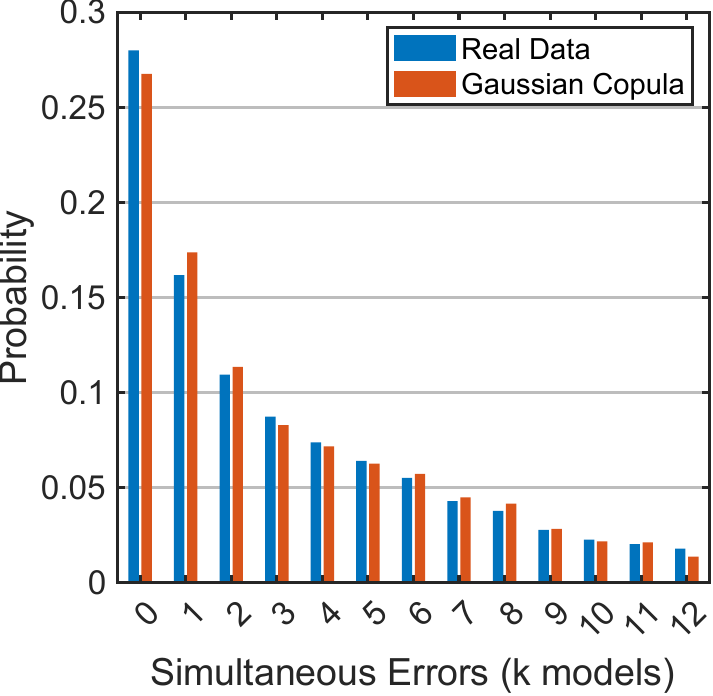}
    \caption{\(\text{temp}=0.7\), run 2}
\end{subfigure}

\caption{Gaussian-copula validation on MEDMCQA (simultaneous-error / tail plots): empirical vs copula-predicted probability of \(k\)-way simultaneous errors across temperatures and runs.}
\label{fig:app_medmcqa_copula_heavy_6up}
\end{figure*}

\clearpage
\section{Additional Experimental Details for the MMLU Dataset}
In this section, we provide more context on the MMLU experiments.
\label{MMLU_app}
For the binary evaluation task, we utilize the following prompt template to enforce a consistent `True' or `False' output format:

\begin{tcolorbox}[
  colback=gray!8,
  colframe=gray!70!black,
  coltitle=white,
  title={\textbf{Query Format for MMLU Dataset}},
  fonttitle=\small,
  boxrule=0.8pt,
  arc=0pt,
  left=8pt,
  right=8pt,
  top=6pt,
  bottom=6pt
]
\small
\textbf{System:} Is the following statement true or false? Answer with a single word: True or False. Do not provide any reasoning.\\[4pt]
\textbf{User:} Question: \texttt{<question>} Is the answer `\texttt{<candidate\_answer>}'?\\
Answer:
\end{tcolorbox}

We accessed all language models through OpenRouter to standardize the inference process. An actual question from MMLU, is illustrated below with the ground truth ``FALSE":
\begin{tcolorbox}[
  colback=gray!8,
  colframe=gray!70!black,
  coltitle=white,
  title={\textbf{Example Query for MMLU Dataset}},
  fonttitle=\small,
  boxrule=0.8pt,
  arc=0pt,
  left=8pt,
  right=8pt,
  top=6pt,
  bottom=6pt
]
\small
\textbf{System:} Is the following statement true or false? Answer with a single word: True or False. Do not provide any reasoning.\\[4pt]
\textbf{User:} Question: In what state is the 1999 movie Magnolia set? Is the answer `California'?\\
Answer:
\end{tcolorbox}

\subsection{Experiments across Temperature and Runs (the MAP Aggregation)}
For MMLU, we also report the experiments for all temperature settings and runs with 5 random splits for each.
Figure~\ref{fig:appendix_map_mmlu_6} reports per-condition curves on MMLU under the MAP aggregation for each baseline, with one panel for average of each temperature--run--(5-split) pair. Table~\ref{tab:mmlu_results} aggregates the same experiments over all $30$ evaluations (3 temperatures $\times$ 2 runs $\times$ 5 random 80/20 splits), reporting mean $\pm$ std test error at each ensemble size $k$.

\begin{figure*}[ht]
\centering
\label{fig:map_all_mmlu}
\begin{subfigure}[t]{0.32\textwidth}
  \centering
  \includegraphics[width=0.9\linewidth]{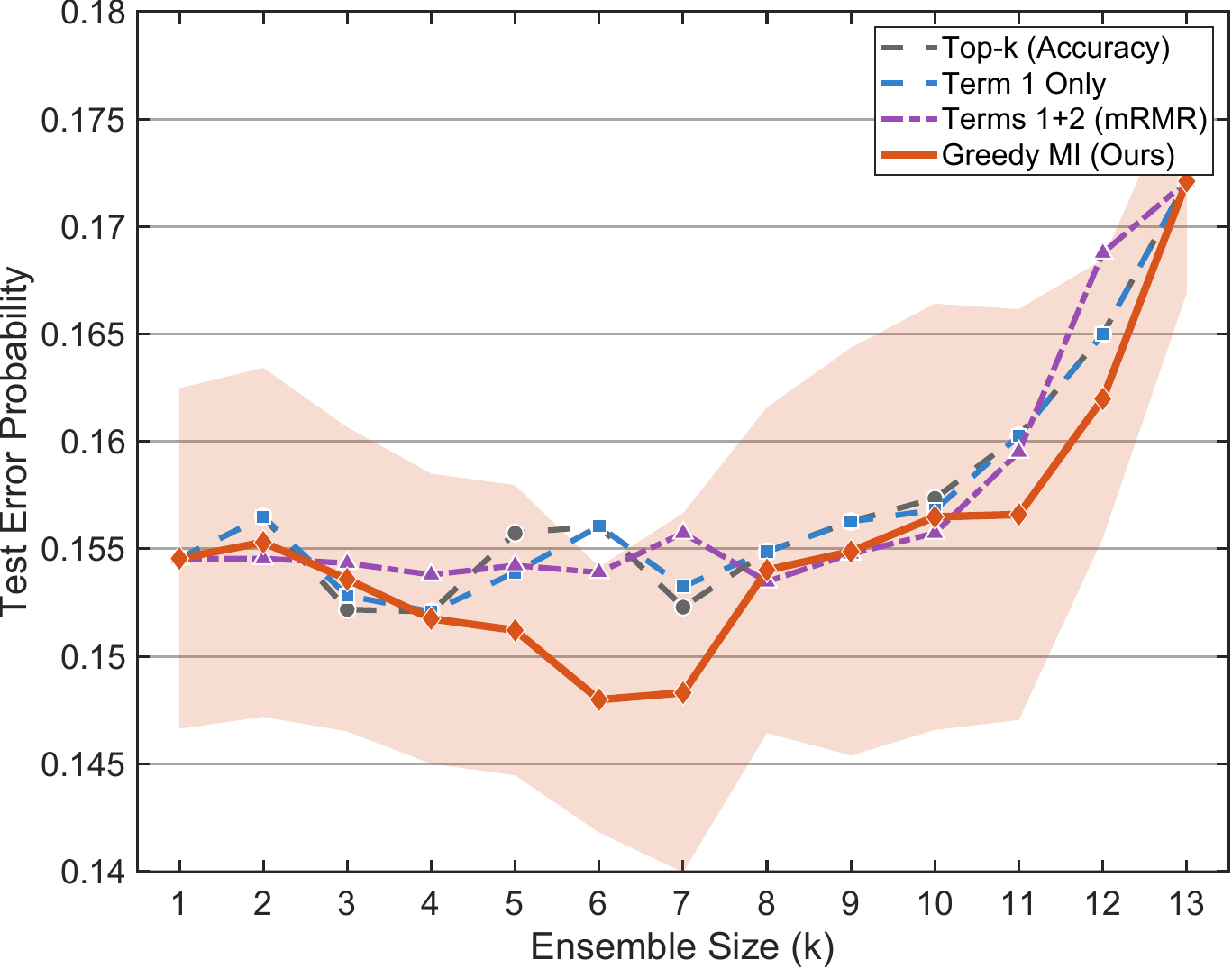}
  \caption{\(\text{temp}=0.01\), run 1}
\end{subfigure}\hfill
\begin{subfigure}[t]{0.32\textwidth}
  \centering
  \includegraphics[width=0.9\linewidth]{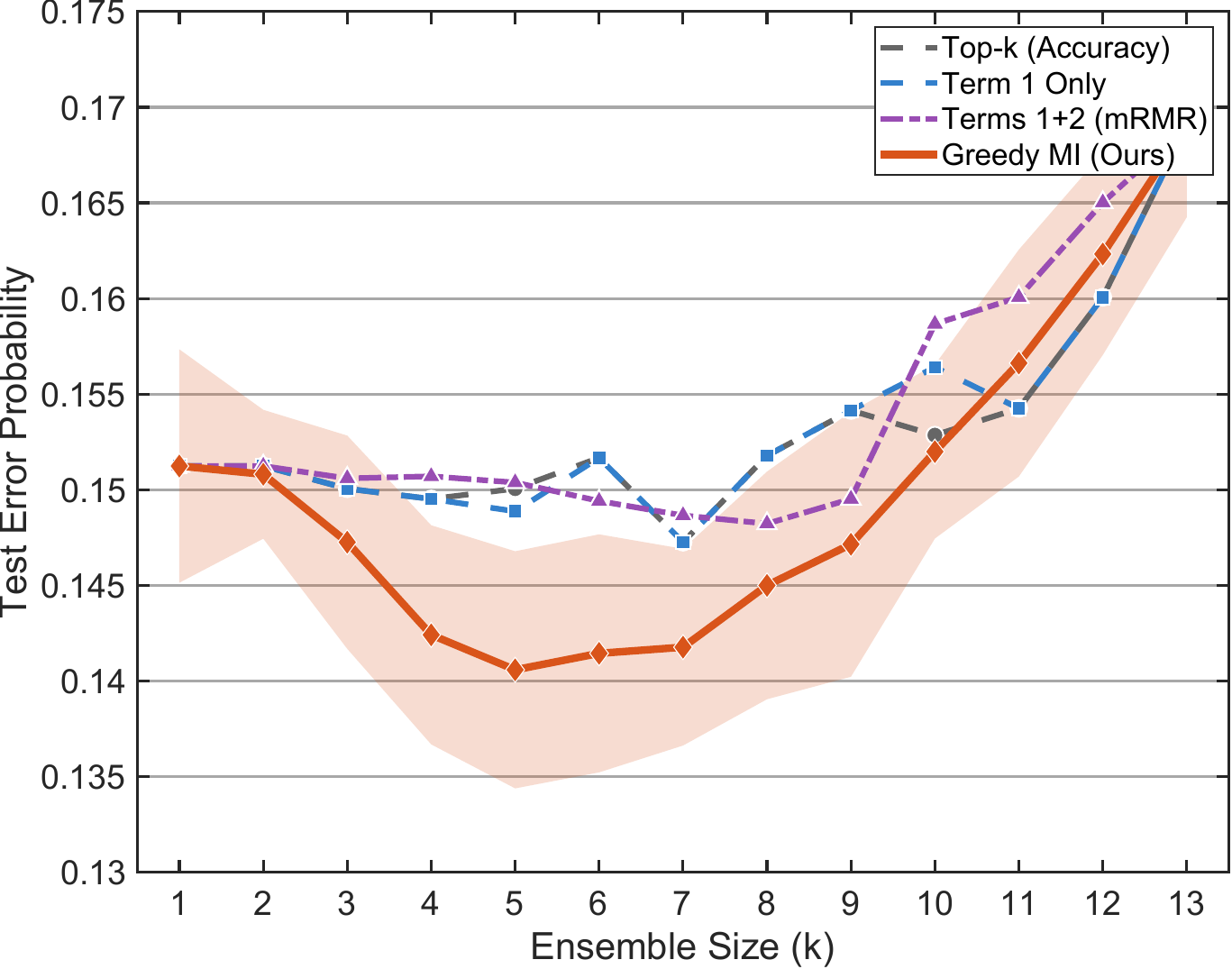}
  \caption{\(\text{temp}=0.01\), run 2}
\end{subfigure}\hfill
\begin{subfigure}[t]{0.32\textwidth}
  \centering
  \includegraphics[width=0.9\linewidth]{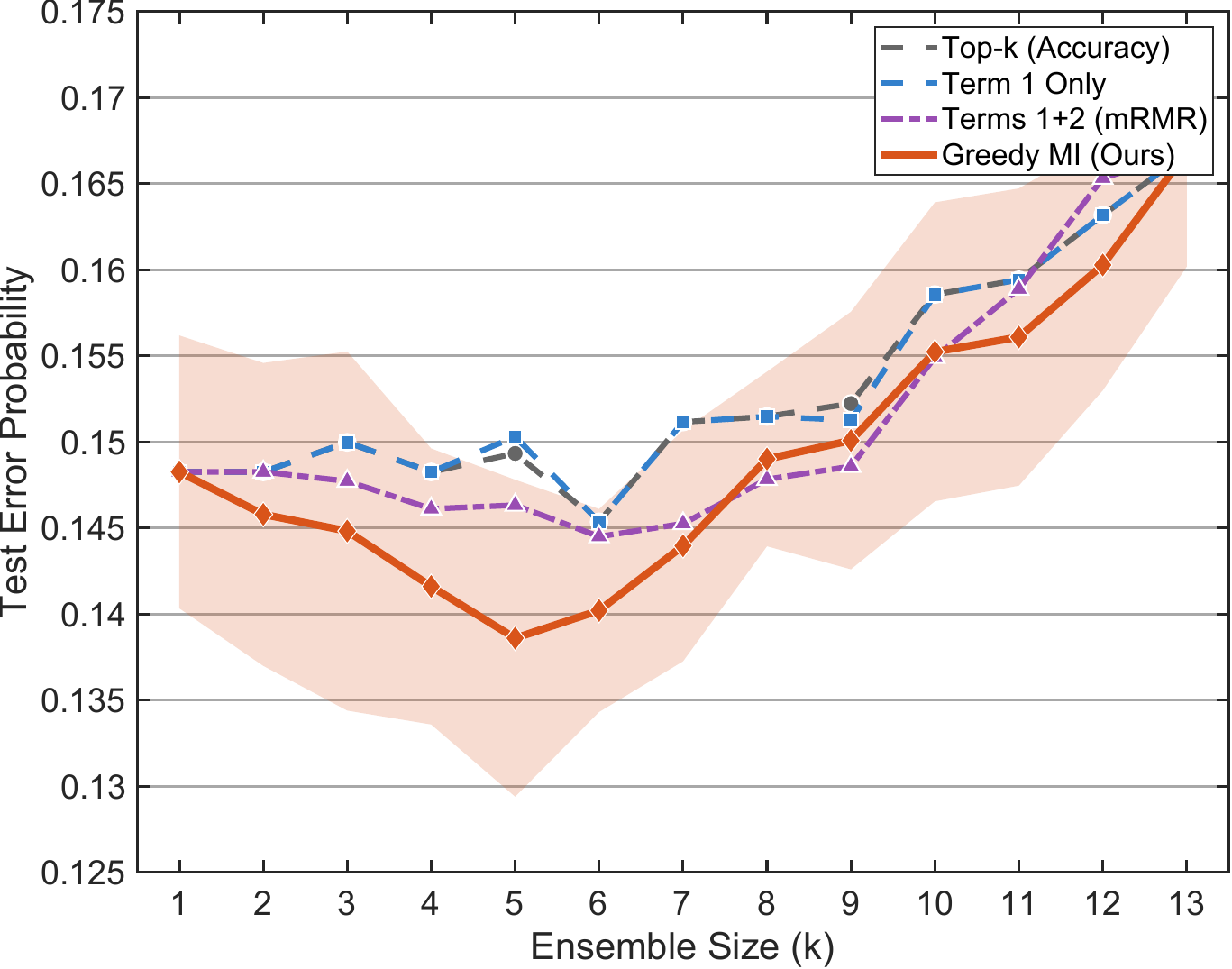}
  \caption{\(\text{temp}=0.3\), run 1}
\end{subfigure}


\begin{subfigure}[t]{0.32\textwidth}
  \centering
  \includegraphics[width=0.9\linewidth]{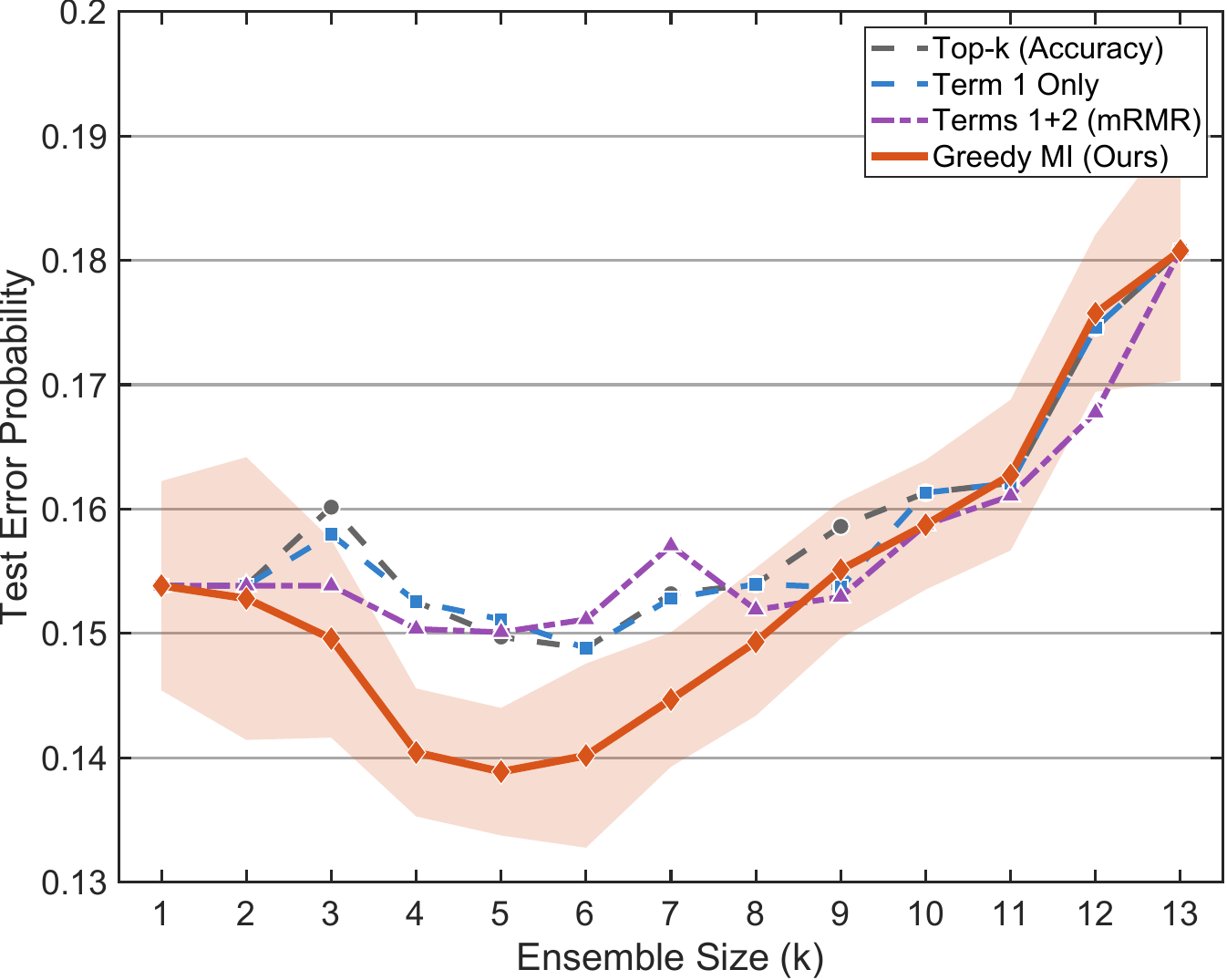}
  \caption{\(\text{temp}=0.3\), run 2}
\end{subfigure}\hfill
\begin{subfigure}[t]{0.32\textwidth}
  \centering
  \includegraphics[width=0.9\linewidth]{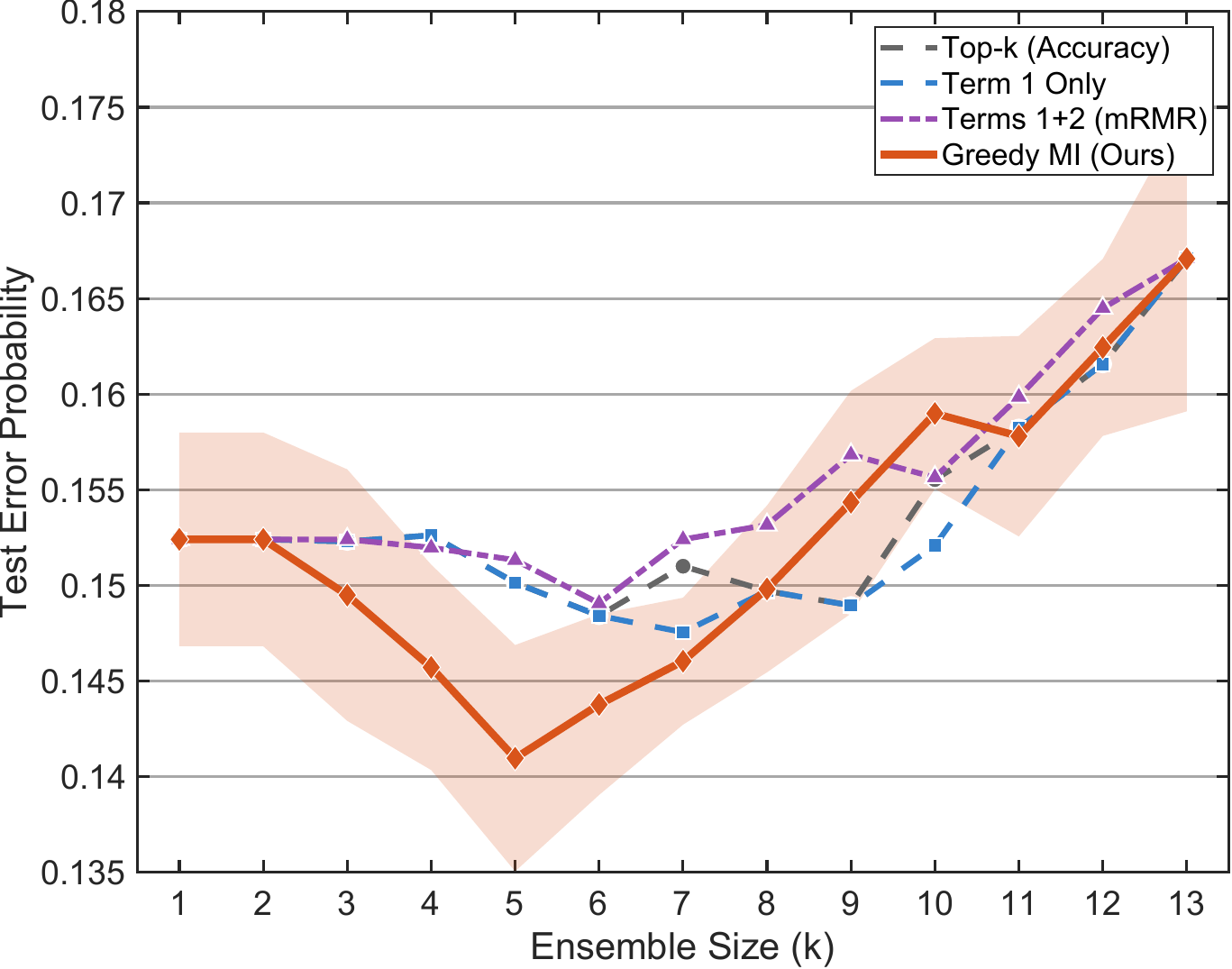}
  \caption{\(\text{temp}=0.7\), run 1}
\end{subfigure}\hfill
\begin{subfigure}[t]{0.32\textwidth}
  \centering
  \includegraphics[width=0.9\linewidth]{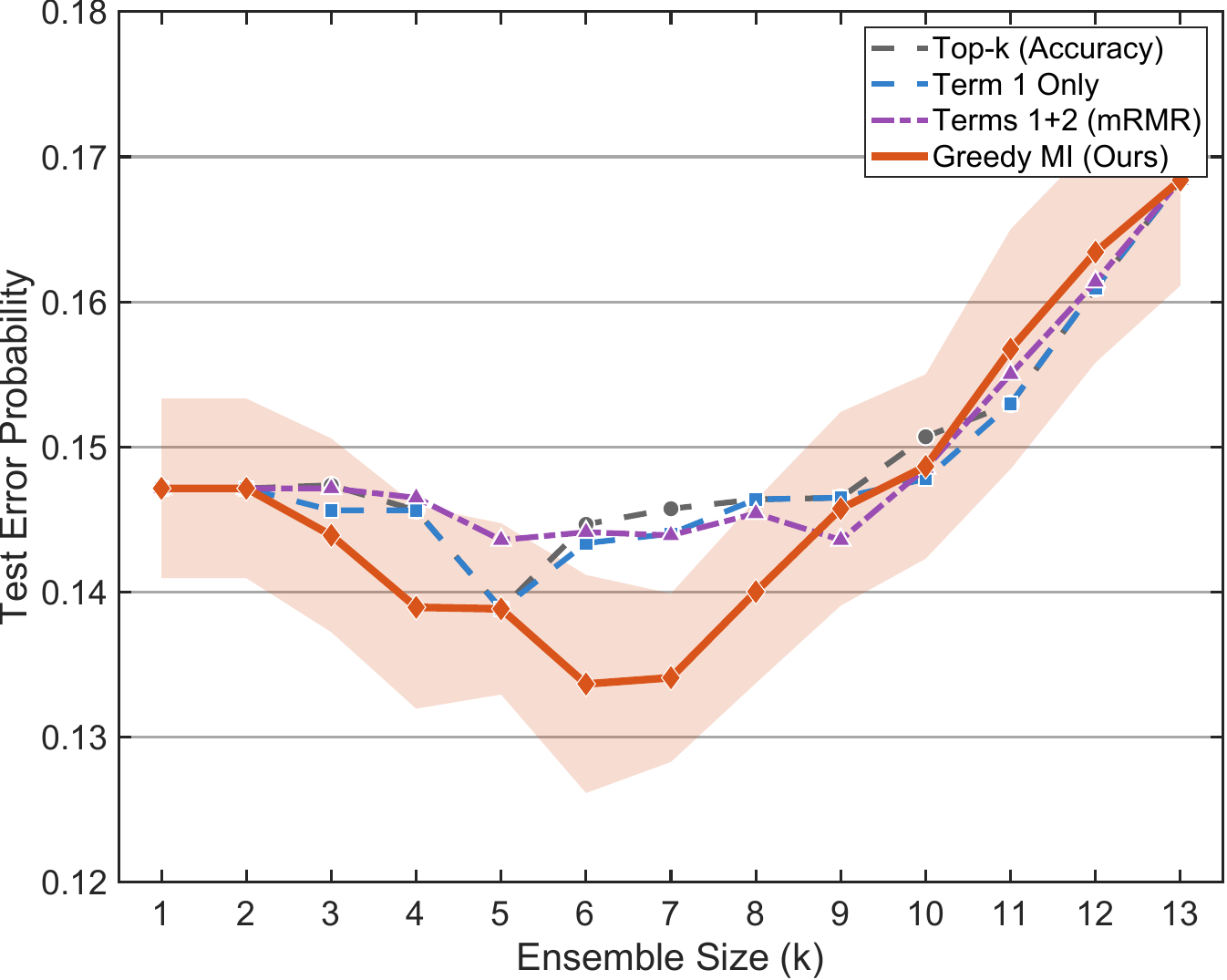}
  \caption{\(\text{temp}=0.7\), run 2}
\end{subfigure}

\caption{Experiments across temperatures and runs (\textbf{the MAP aggregation}) on MMLU. Each panel corresponds to one temperature setting and random run with average of 5 splits per run. Shaded region represents the standard deviation.}
\label{fig:appendix_map_mmlu_6}
\end{figure*}

\subsection{Experiments with Alternative Aggregation Rules}

In this part, we follow the same approach as in Appendix~\ref{app:medmcqa_agg}.
With the help of visualizations from \cref{fig:mmlu_appendix_mv_6}, and tables from \cref{tab:mmlu_mv_vs_map}, we observe that at small budgets, the mRMR selector (Terms~1+2) is highly unstable under MV: at $k=2$ it jumps to $0.246$ error, versus $0.156$ for Top-$k$ and Term~1, and $0.151$ for Greedy MI (MAP).
The same pattern persists at $k=3$, where Terms~1+2 attains $0.219$ error while Greedy MI achieves $0.148$.
In the mid-budget regime where our gains are largest, Greedy MI remains consistently lower (e.g., $k=4$: $0.144$ vs. $0.150$ for Top-$k$; $k=6$: $0.141$ vs. $0.145$ for Top-$k$).

\begin{figure*}[ht]
\centering

\begin{subfigure}[t]{0.32\textwidth}
  \centering
  \includegraphics[width=0.9\linewidth]{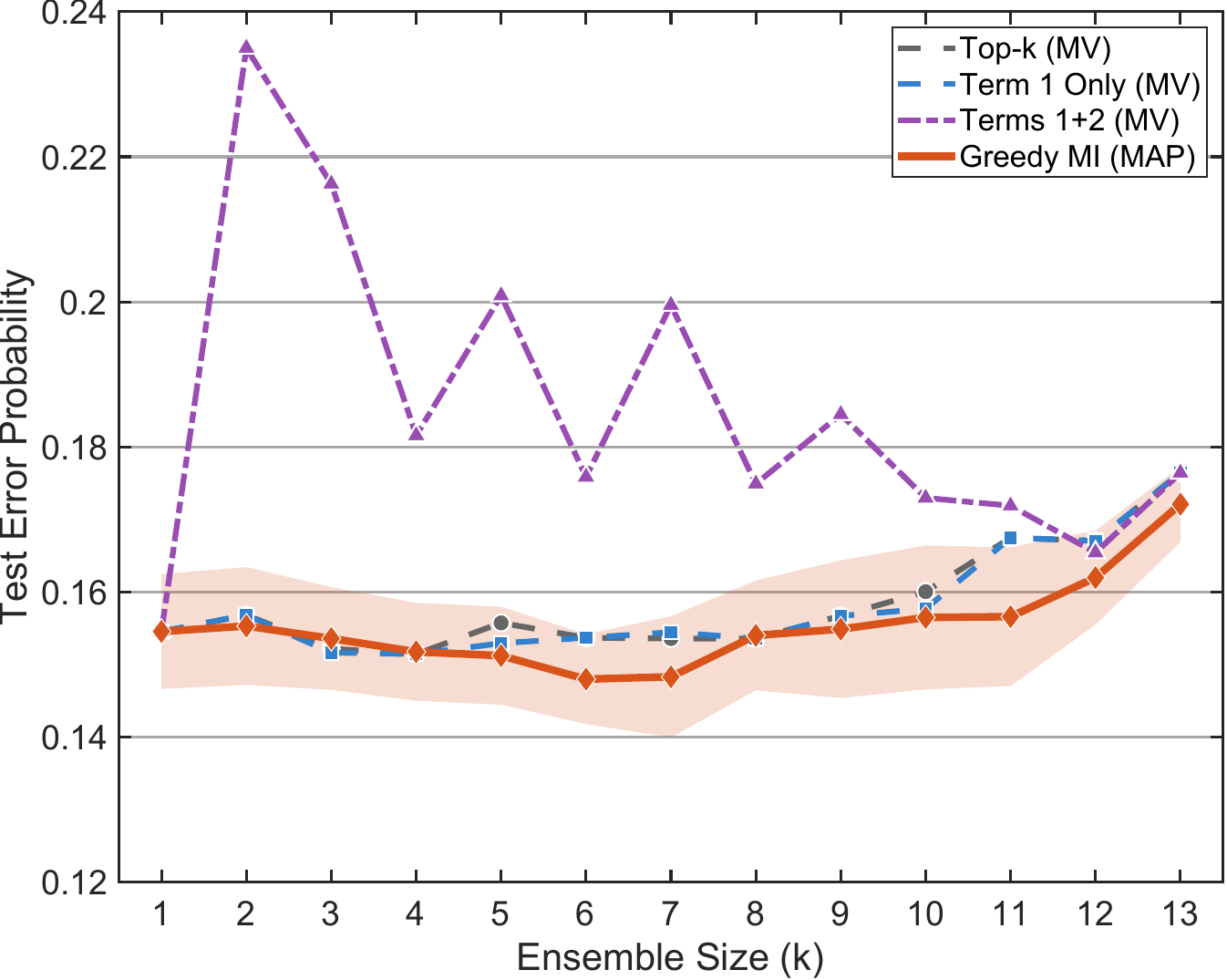}
  \caption{\(\text{temp}=0.01\), run 1}
\end{subfigure}\hfill
\begin{subfigure}[t]{0.32\textwidth}
  \centering
  \includegraphics[width=0.9\linewidth]{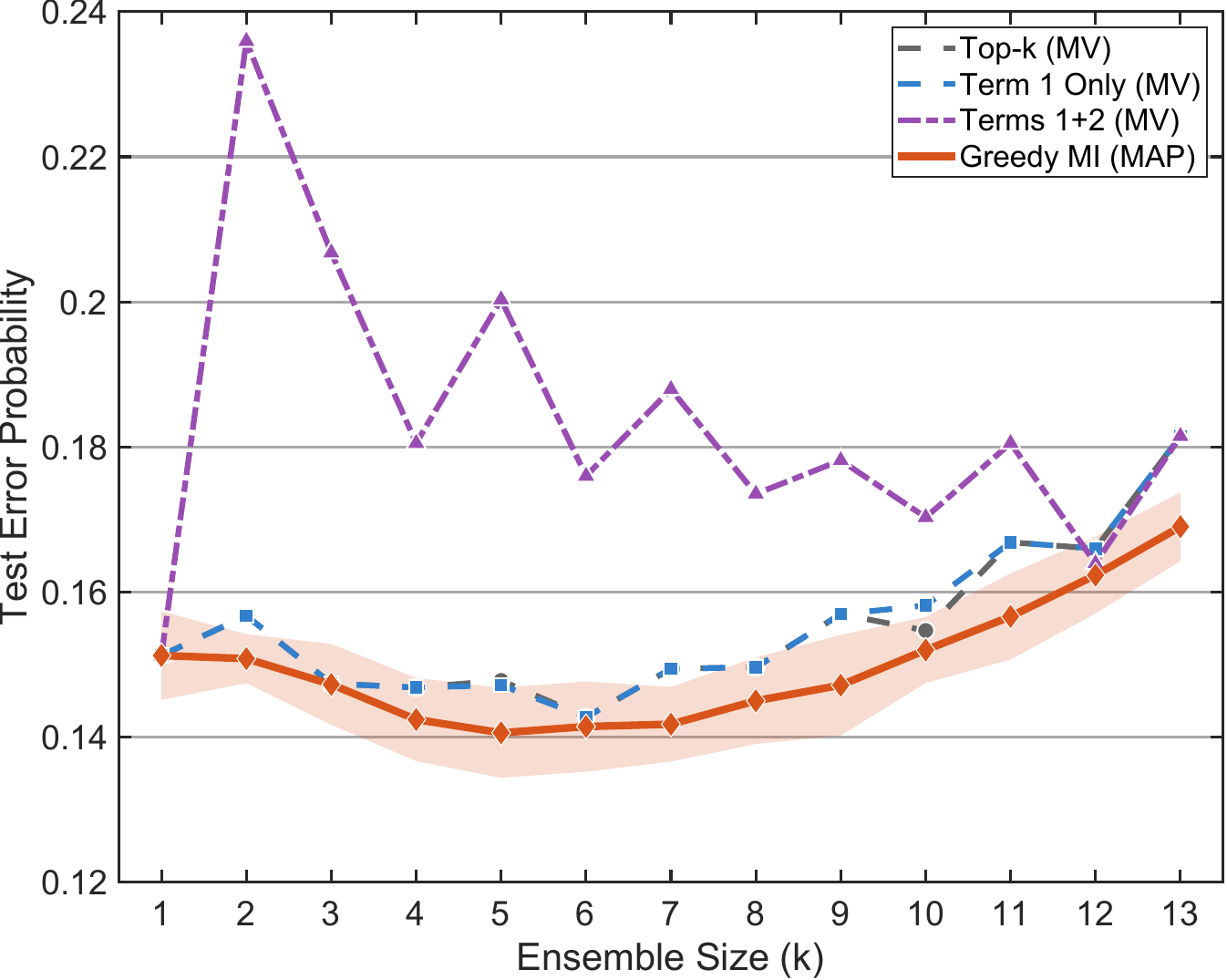}
  \caption{\(\text{temp}=0.01\), run 2}
\end{subfigure}\hfill
\begin{subfigure}[t]{0.32\textwidth}
  \centering
  \includegraphics[width=0.9\linewidth]{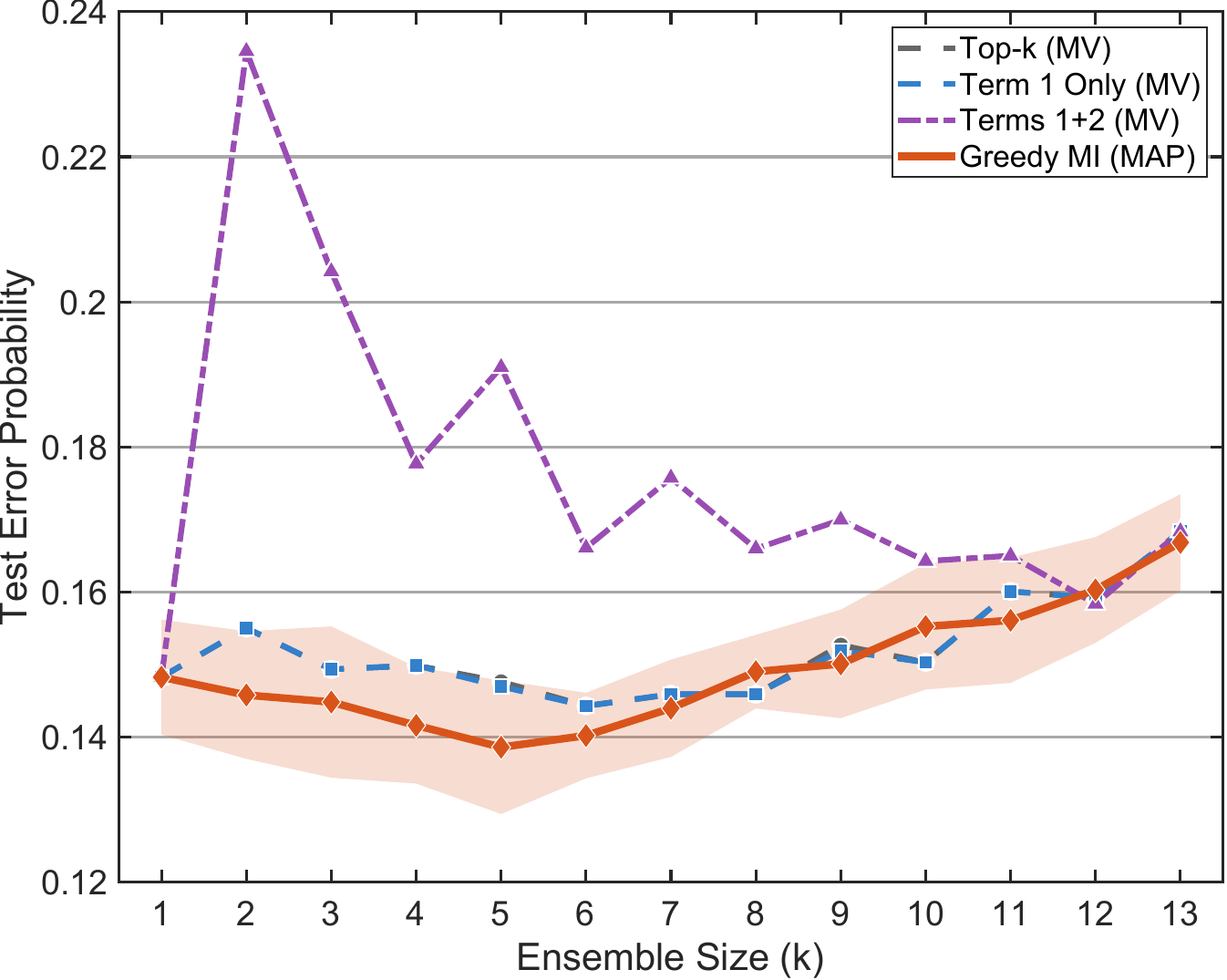}
  \caption{\(\text{temp}=0.3\), run 1}
\end{subfigure}


\begin{subfigure}[t]{0.32\textwidth}
  \centering
  \includegraphics[width=0.9\linewidth]{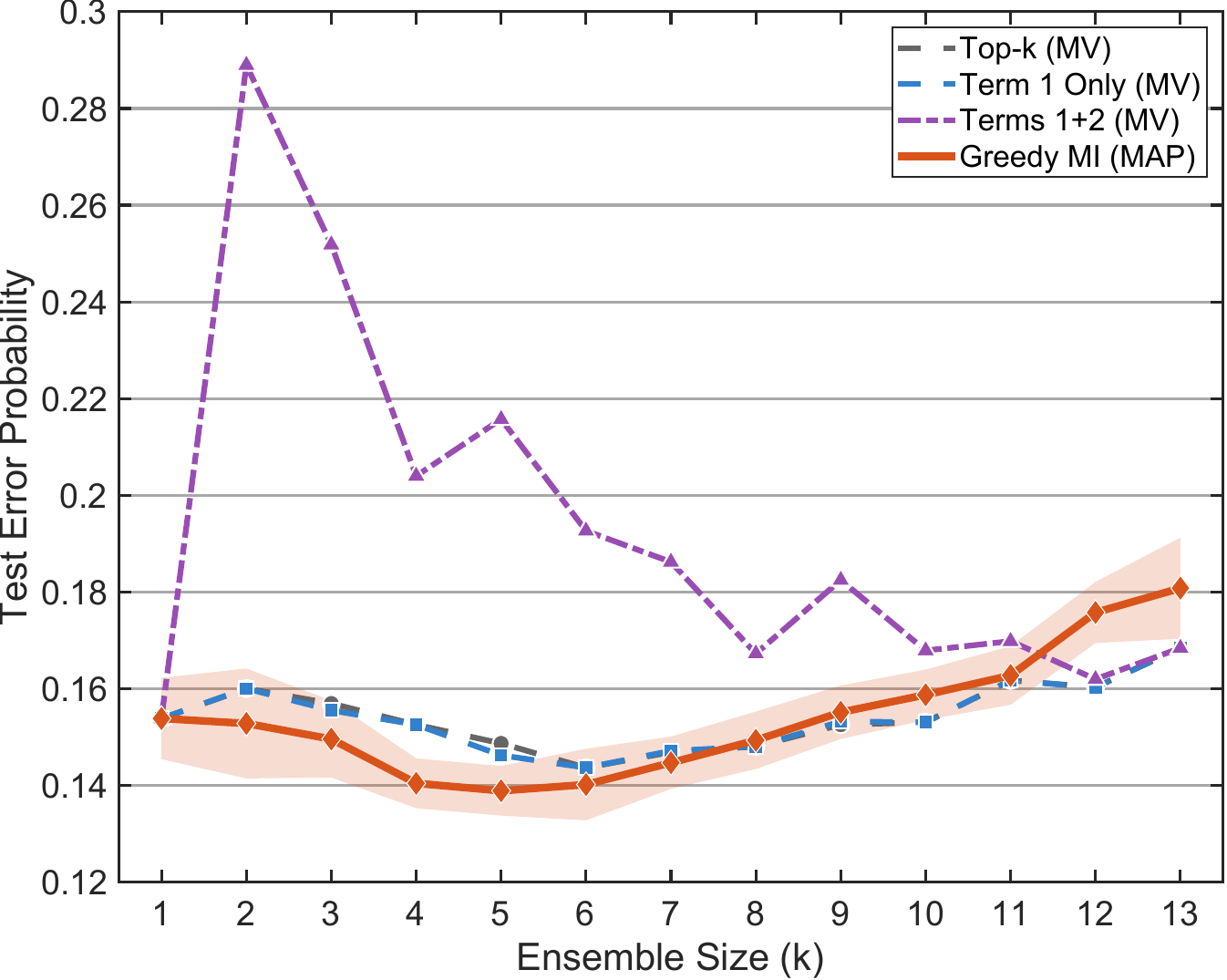}
  \caption{\(\text{temp}=0.3\), run 2}
\end{subfigure}\hfill
\begin{subfigure}[t]{0.32\textwidth}
  \centering
  \includegraphics[width=0.9\linewidth]{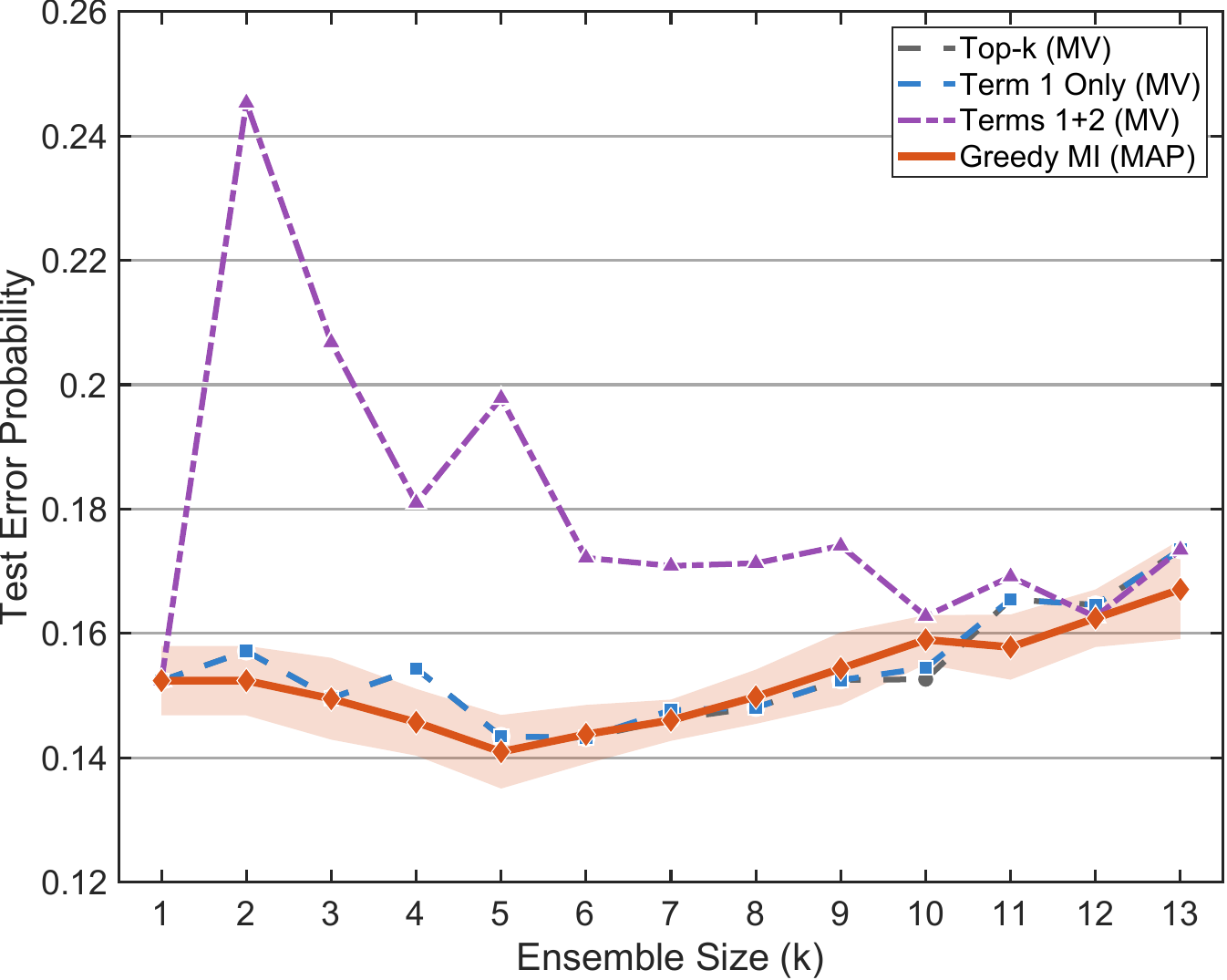}
  \caption{\(\text{temp}=0.7\), run 1}
\end{subfigure}\hfill
\begin{subfigure}[t]{0.32\textwidth}
  \centering
  \includegraphics[width=0.9\linewidth]{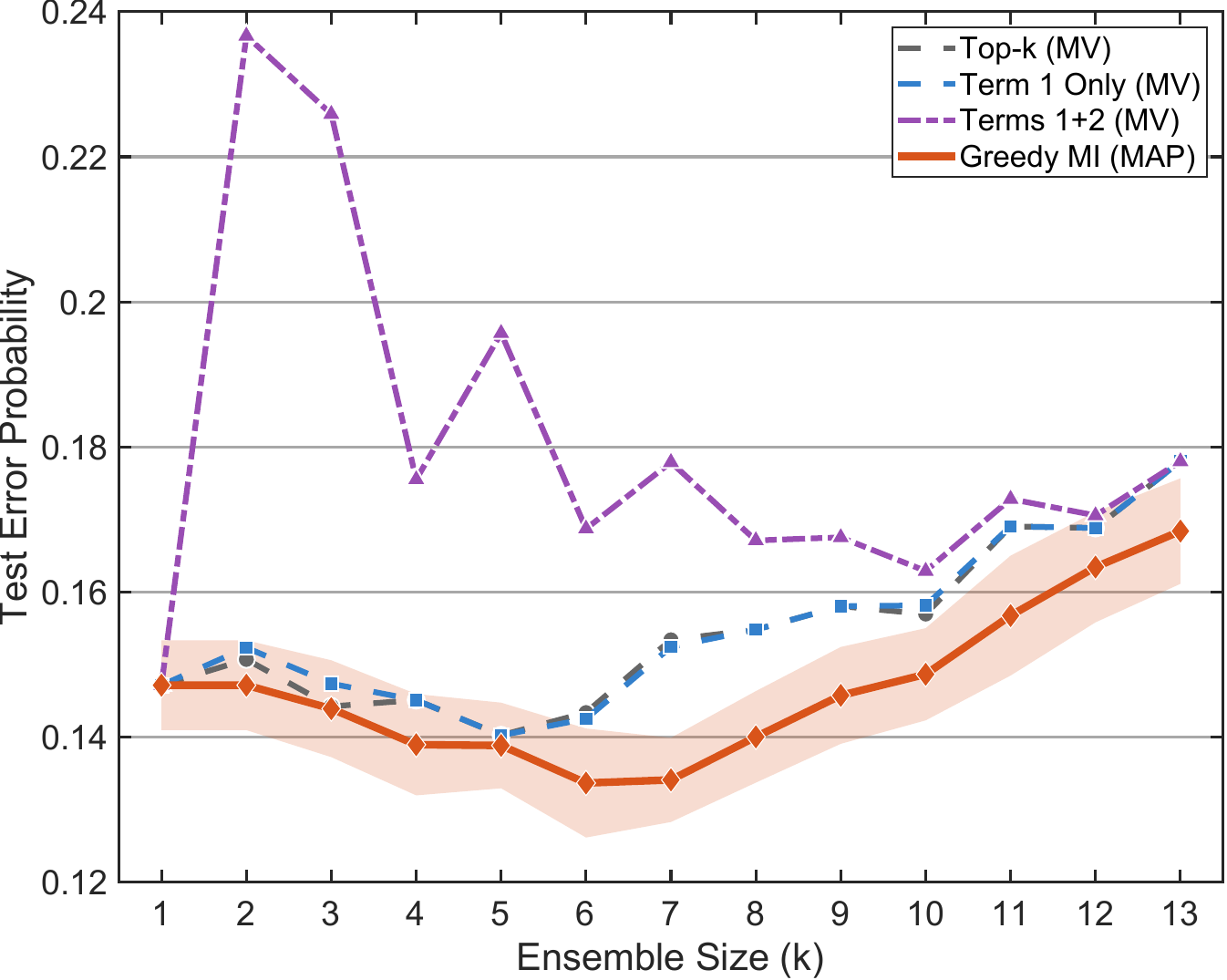}
  \caption{\(\text{temp}=0.7\), run 2}
\end{subfigure}

\caption{MMLU results using \textbf{the Majority Vote (MV)} aggregation across temperatures and runs. Shaded region represents the standard deviation.}
\label{fig:mmlu_appendix_mv_6}
\end{figure*}

\begin{figure*}[ht]
\centering

\begin{subfigure}[t]{0.32\textwidth}
  \centering
  \includegraphics[width=0.9\linewidth]{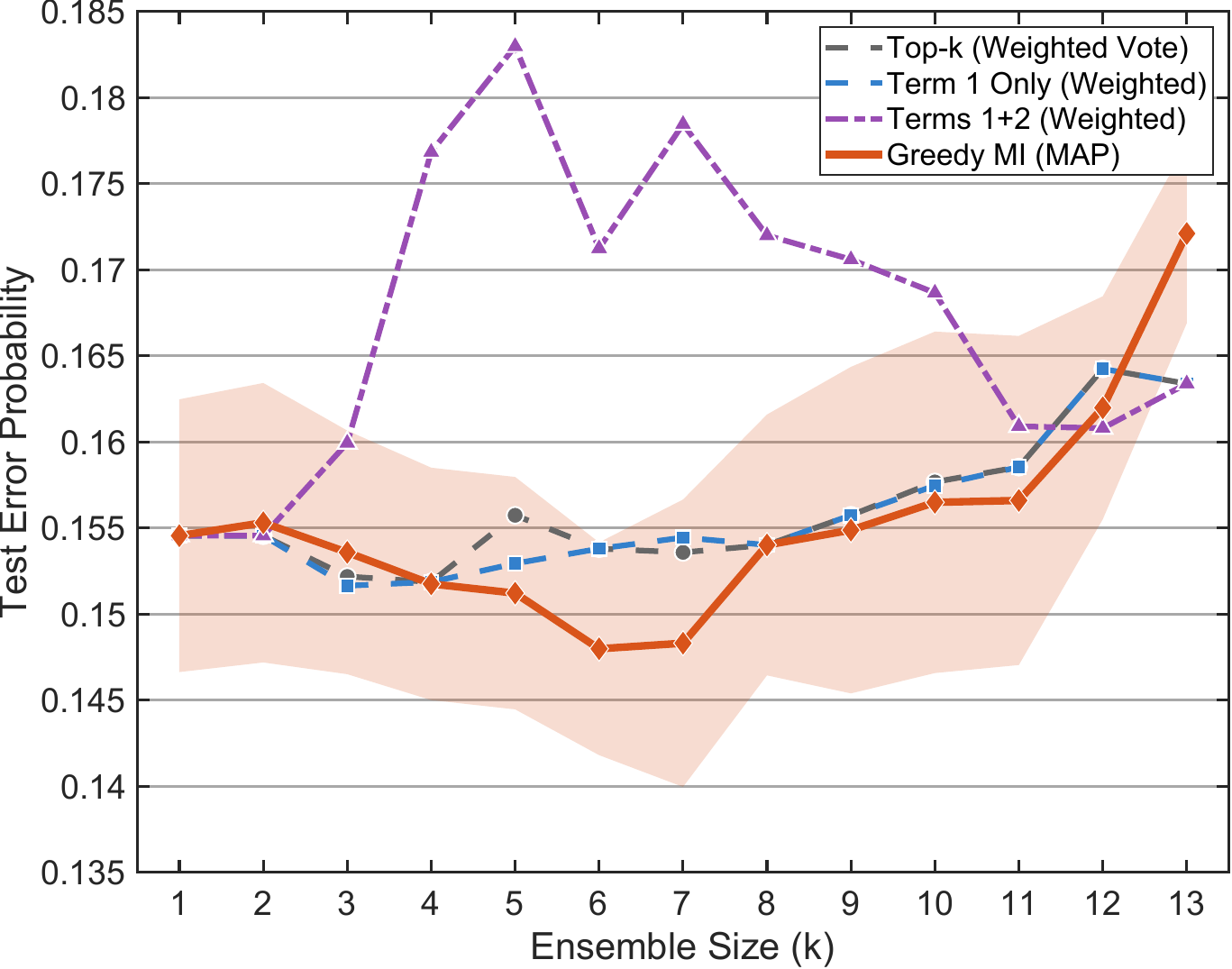}
  \caption{\(\text{temp}=0.01\), run 1}
\end{subfigure}\hfill
\begin{subfigure}[t]{0.32\textwidth}
  \centering
  \includegraphics[width=0.9\linewidth]{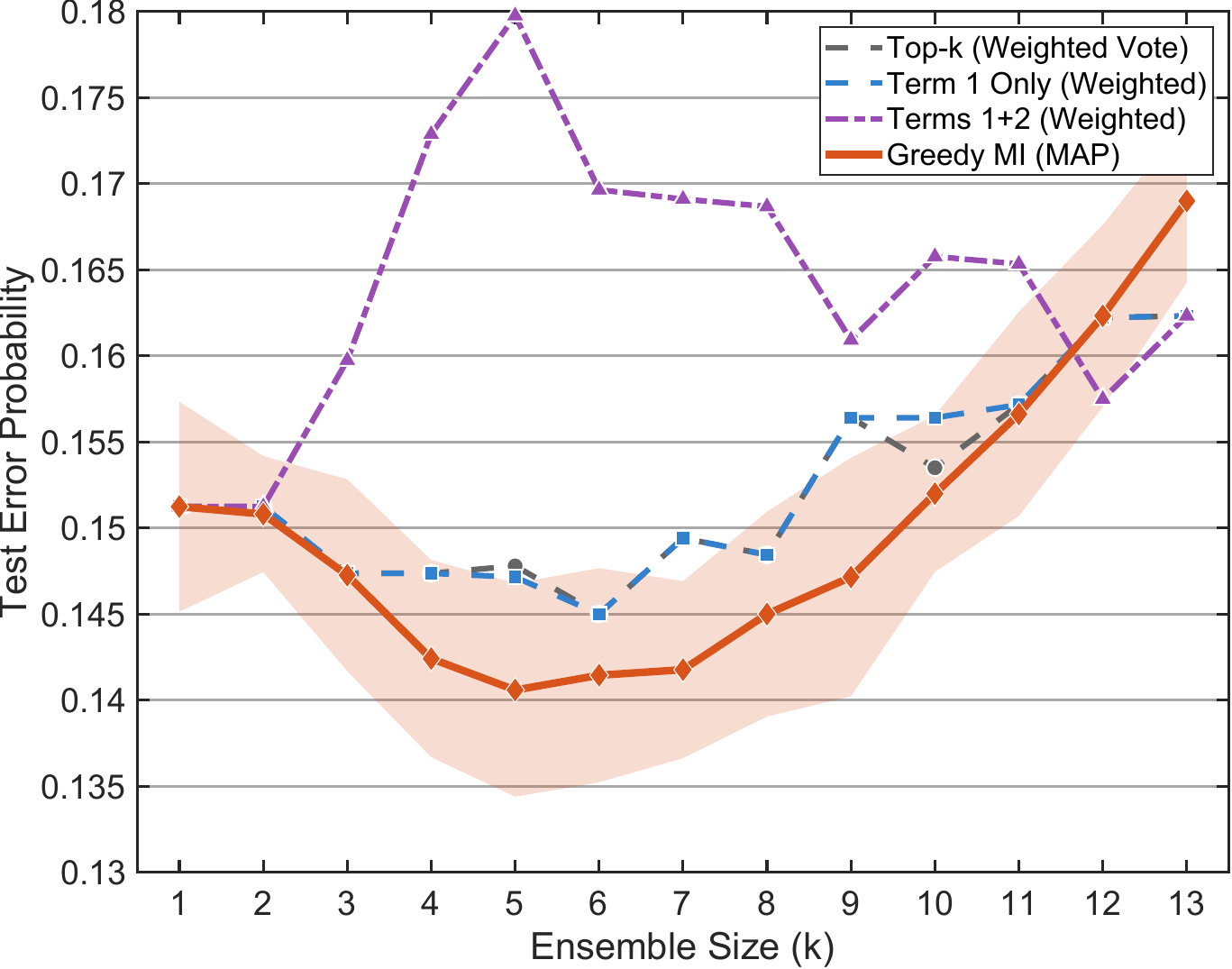}
  \caption{\(\text{temp}=0.01\), run 2}
\end{subfigure}\hfill
\begin{subfigure}[t]{0.32\textwidth}
  \centering
  \includegraphics[width=0.9\linewidth]{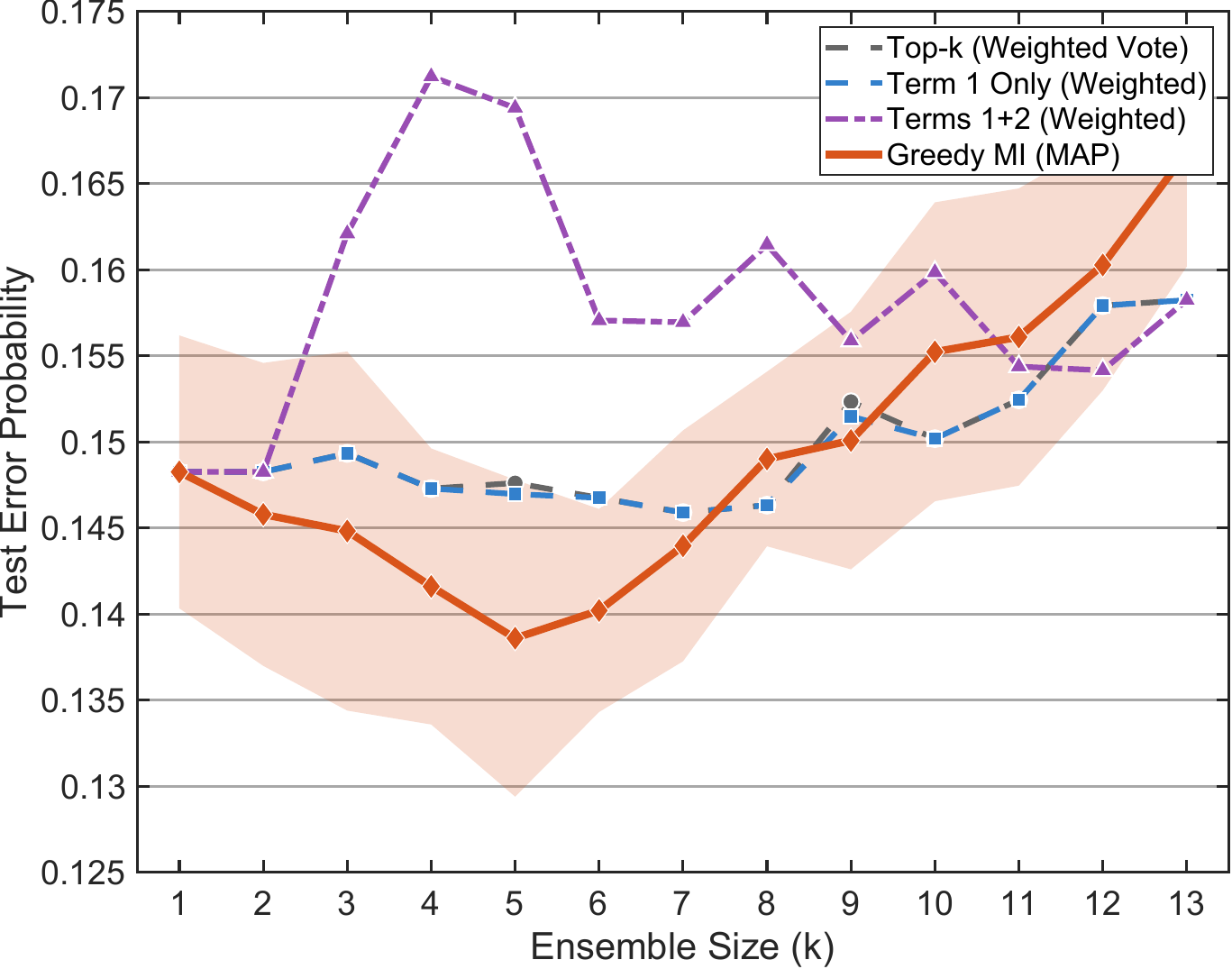}
  \caption{\(\text{temp}=0.3\), run 1}
\end{subfigure}


\begin{subfigure}[t]{0.32\textwidth}
  \centering
  \includegraphics[width=0.9\linewidth]{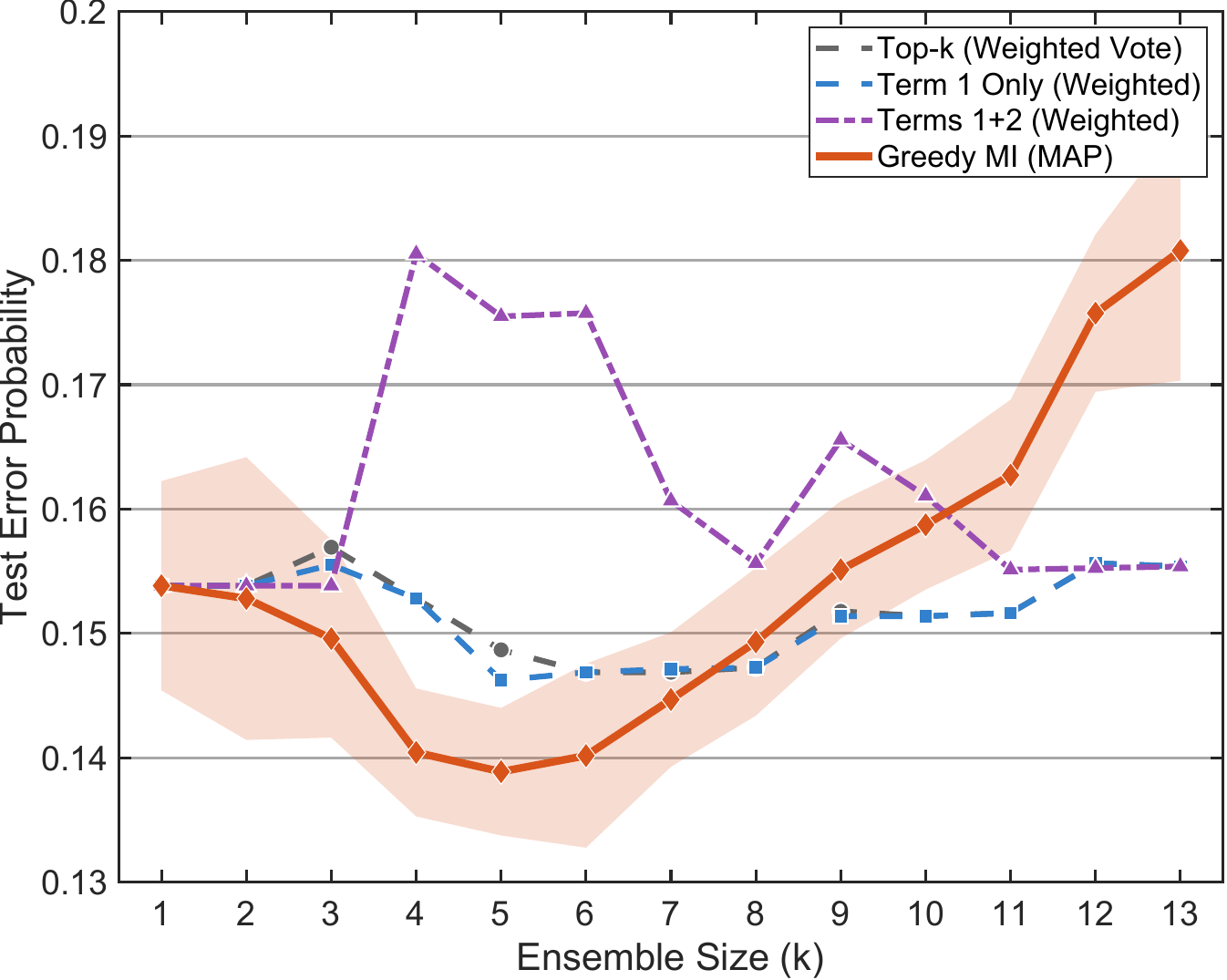}
  \caption{\(\text{temp}=0.3\), run 2}
\end{subfigure}\hfill
\begin{subfigure}[t]{0.32\textwidth}
  \centering
  \includegraphics[width=0.9\linewidth]{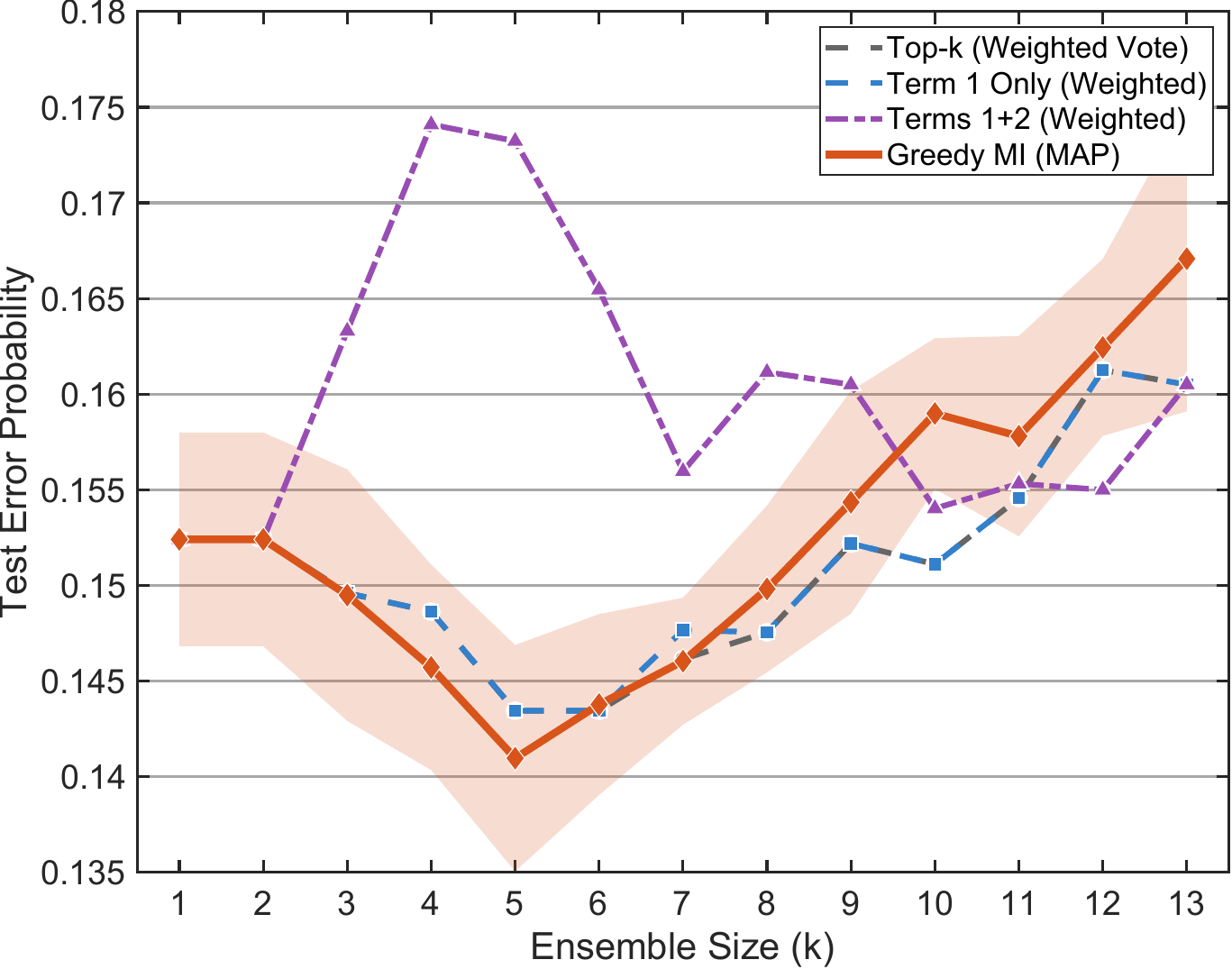}
  \caption{\(\text{temp}=0.7\), run 1}
\end{subfigure}\hfill
\begin{subfigure}[t]{0.32\textwidth}
  \centering
  \includegraphics[width=0.9\linewidth]{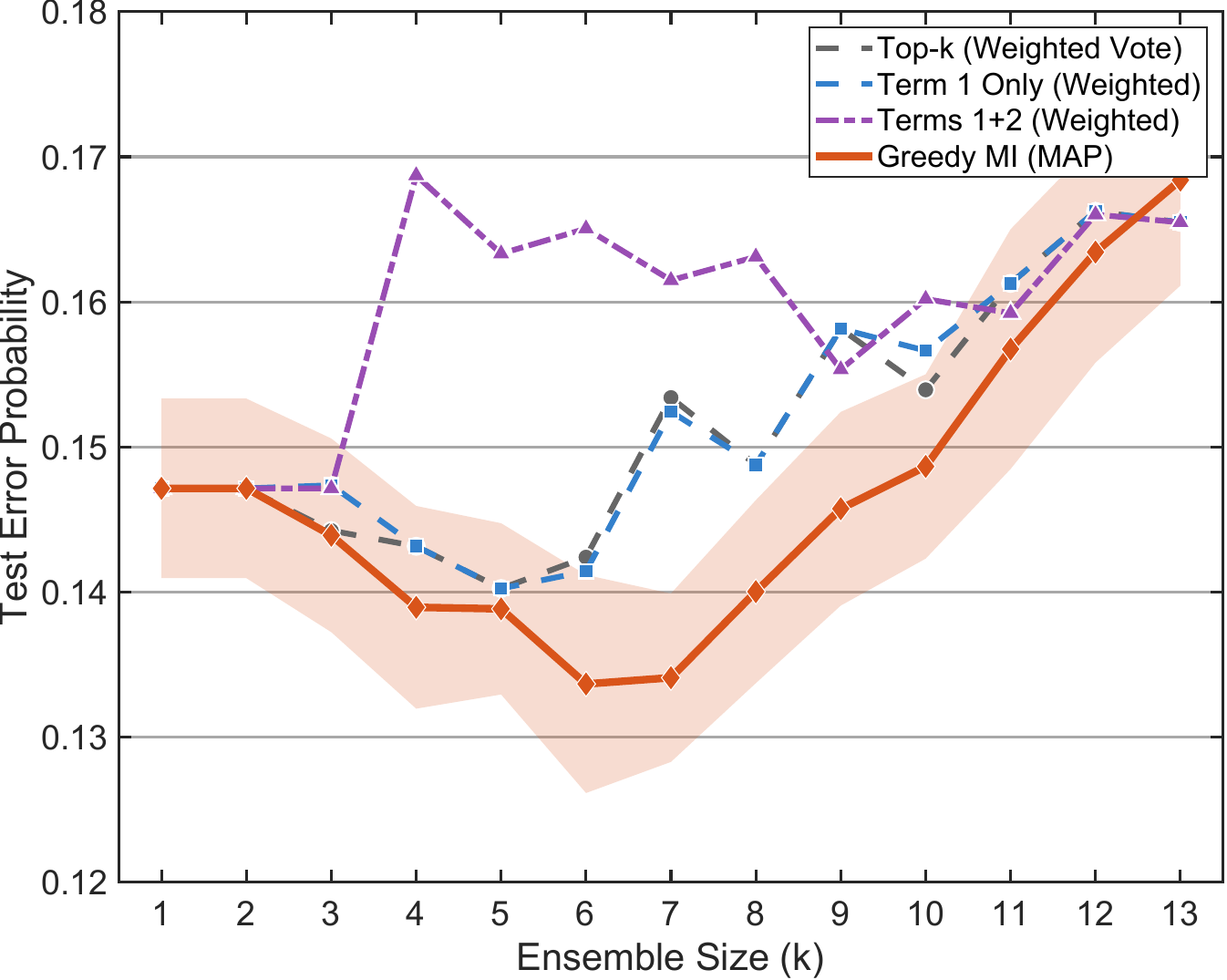}
  \caption{\(\text{temp}=0.7\), run 2}
\end{subfigure}

\caption{MMLU results using \textbf{Weighted Ensemble (WE)} aggregation across temperatures and runs. Shaded region represents the standard deviation.}
\label{fig:mmlu_appendix_we_6}
\end{figure*}

Under W-MV, in \cref{fig:mmlu_appendix_we_6}, and in \cref{tab:mmlu_wmv_vs_map}, we see that Terms~1+2 remains substantially worse in the mid-budget range (e.g., $k=4$: $0.174$; $k=6$: $0.167$), whereas Greedy MI (MAP) achieves $0.144 \pm 0.007$ at $k=4$ and $0.141$ at $k=6$.
At larger budgets, W-MV can narrow the gap for simpler selectors: for example, Top-$k$ reaches $0.153$ at $k=10$ (and $0.156$ at $k=11$), comparable to Greedy MI (MAP) at $0.155$ ($k=10$) and $0.158$ ($k=11$).
Overall, these results indicate that changing the aggregation rule can partially improve Top-$k$/Term~1 at large $k$, but it does not resolve the failure mode of mRMR under MV/W-MV; Greedy MI remains most competitive in the moderate-budget regime ($k \approx 3$--$7$) where selection quality matters most. Additionally, as ensemble size increases, panel (c), (d) and (e) show instances where baselines with WE aggregation surpass Greedy MI with MAP. This can be understood through two complementary factors: greedy selection, while effective, provides only near-optimal guarantees, and the difficulty of estimating high-dimensional joint distributions grows with k. When selection quality degrades, a powerful aggregation rule can help bridge the gap.

\newpage
\subsection{Copula Validation across Temperature and Runs}
In this part, we validate that the Gaussian-copula model accurately captures the empirical error structure on MMLU across all temperature settings and runs. Figure~\ref{fig:app_mmlu_copula_scatter_6up} presents scatter plots comparing the empirical pairwise joint error probabilities $P(E_i \cap E_j)$ against the copula-predicted values for all $\binom{13}{2} = 78$ model pairs. Across all six conditions, the points cluster tightly around the diagonal (ideal fit) with only a few of outlier clusters.

Figure~\ref{fig:app_mmlu_copula_heavy_6up} examines higher-order dependencies by plotting the distribution of simultaneous errors that is, the number of models that makes an error on the same instance. This diagnostic is critical for ensemble performance: if all models fail together on hard examples, no aggregation rule can recover. The copula-predicted distributions (orange bars) closely match the observed frequencies (blue bars) across all temperature-run pairs.
\label{app:copula_validation_mmlu}

\begin{figure*}[ht]
\centering
\begin{subfigure}[t]{0.32\textwidth}
    \centering
    \includegraphics[width=0.8\linewidth]{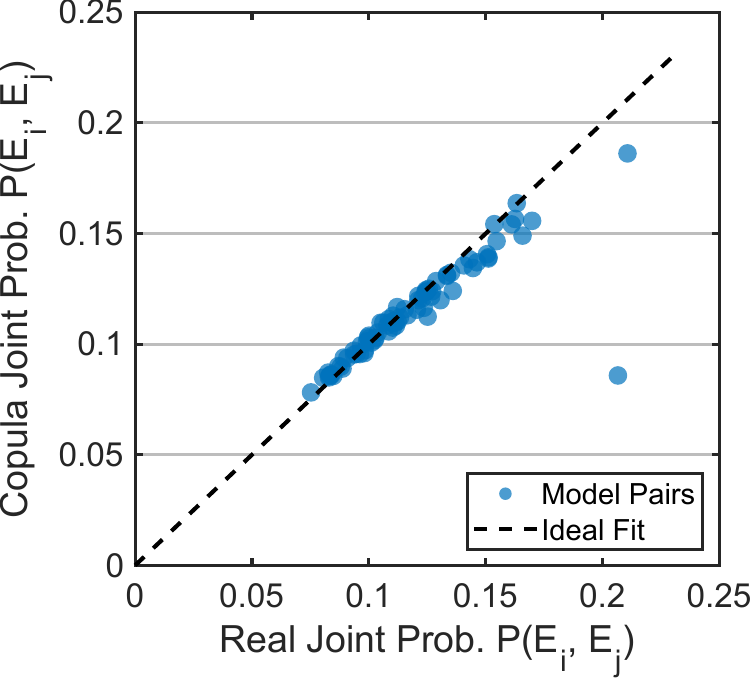}
    \caption{\(\text{temp}=0.01\), run 1}
\end{subfigure}\hfill
\begin{subfigure}[t]{0.32\textwidth}
    \centering
    \includegraphics[width=0.8\linewidth]{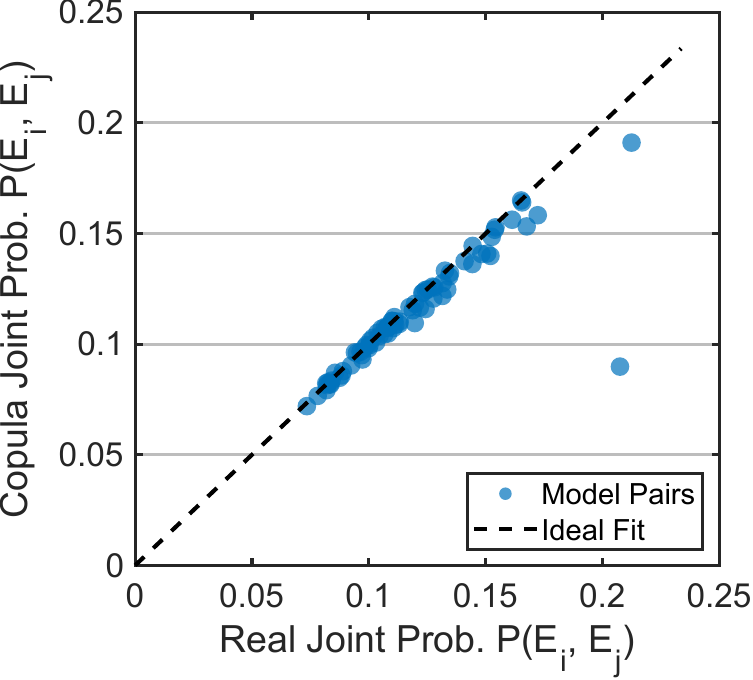}
    \caption{\(\text{temp}=0.01\), run 2}
\end{subfigure}\hfill
\begin{subfigure}[t]{0.32\textwidth}
    \centering
    \includegraphics[width=0.8\linewidth]{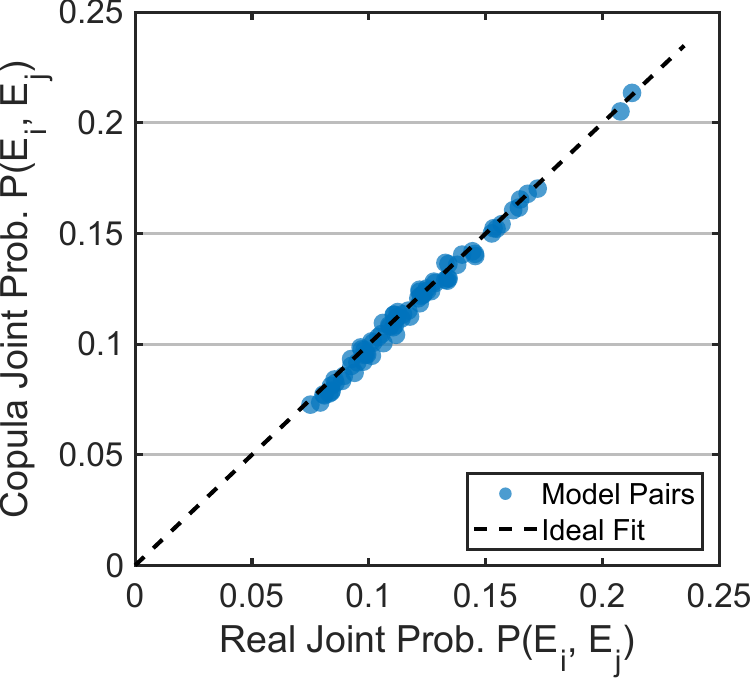}
    \caption{\(\text{temp}=0.3\), run 1}
\end{subfigure}


\begin{subfigure}[t]{0.32\textwidth}
    \centering
    \includegraphics[width=0.8\linewidth]{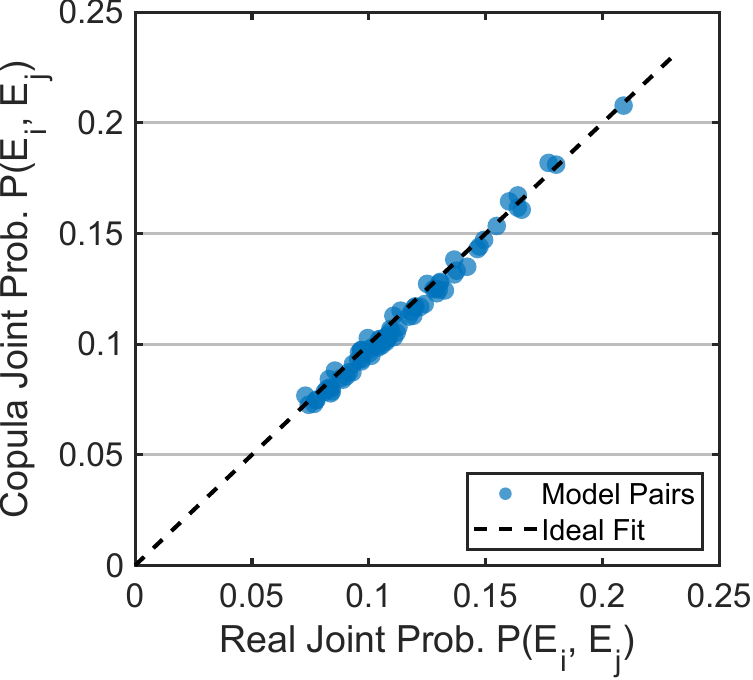}
    \caption{\(\text{temp}=0.3\), run 2}
\end{subfigure}\hfill
\begin{subfigure}[t]{0.32\textwidth}
    \centering
    \includegraphics[width=0.8\linewidth]{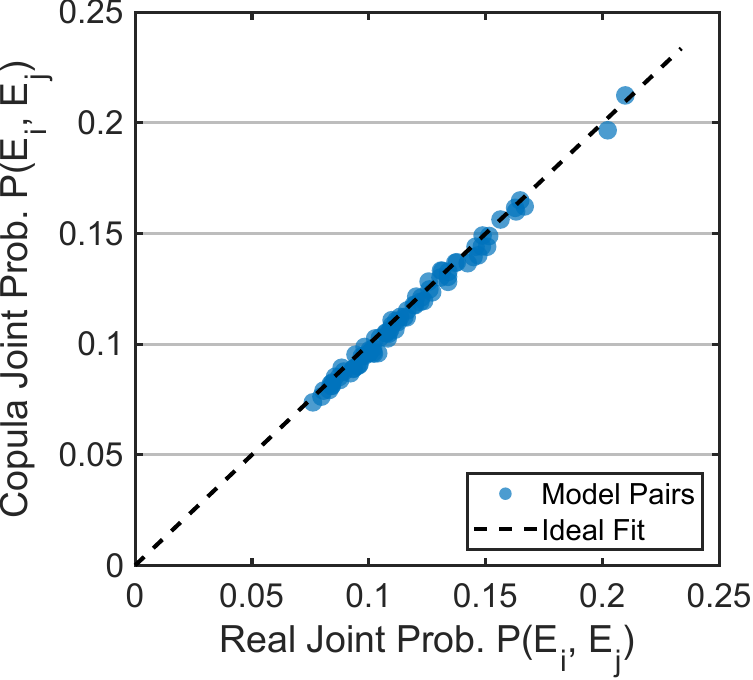}
    \caption{\(\text{temp}=0.7\), run 1}
\end{subfigure}\hfill
\begin{subfigure}[t]{0.32\textwidth}
    \centering
    \includegraphics[width=0.8\linewidth]{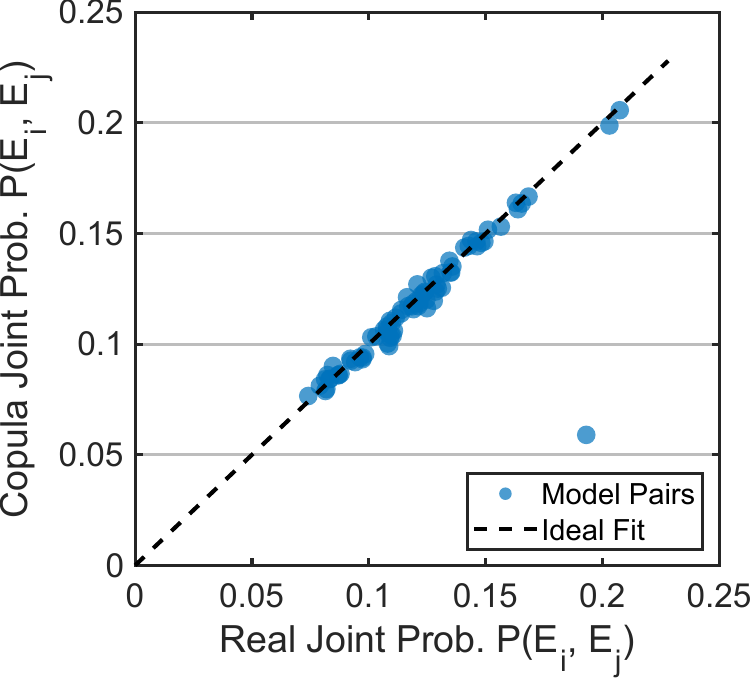}
    \caption{\(\text{temp}=0.7\), run 2}
\end{subfigure}

\caption{The Gaussian-copula validation on MMLU (scatter plots): pairwise structural fit diagnostics across temperatures and runs.}
\label{fig:app_mmlu_copula_scatter_6up}
\end{figure*}

\begin{figure*}[ht]
\centering
\begin{subfigure}[t]{0.32\textwidth}
    \centering
    \includegraphics[width=0.8\linewidth]{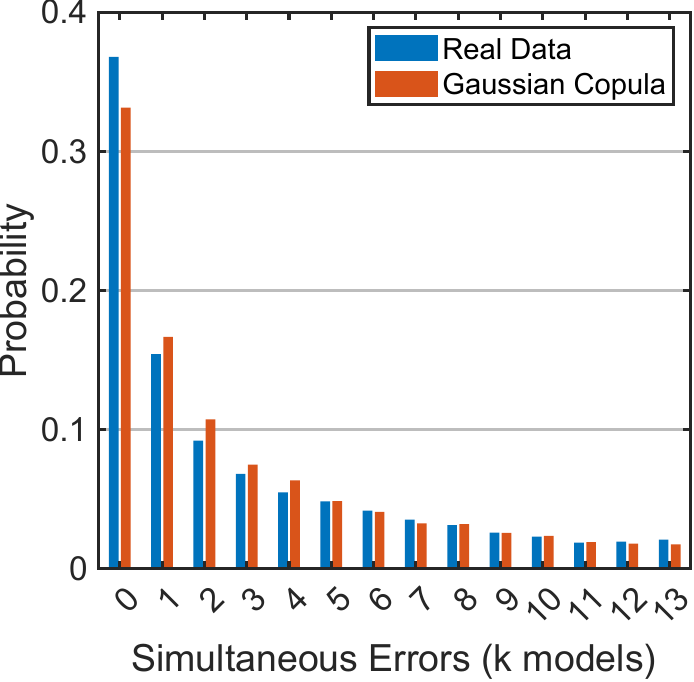}
    \caption{\(\text{temp}=0.01\), run 1}
\end{subfigure}\hfill
\begin{subfigure}[t]{0.32\textwidth}
    \centering
    \includegraphics[width=0.8\linewidth]{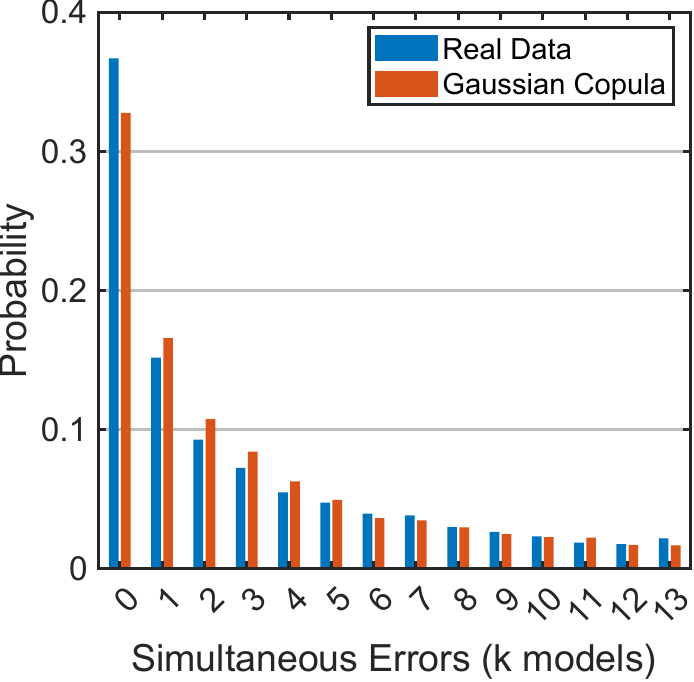}
    \caption{\(\text{temp}=0.01\), run 2}
\end{subfigure}\hfill
\begin{subfigure}[t]{0.32\textwidth}
    \centering
    \includegraphics[width=0.8\linewidth]{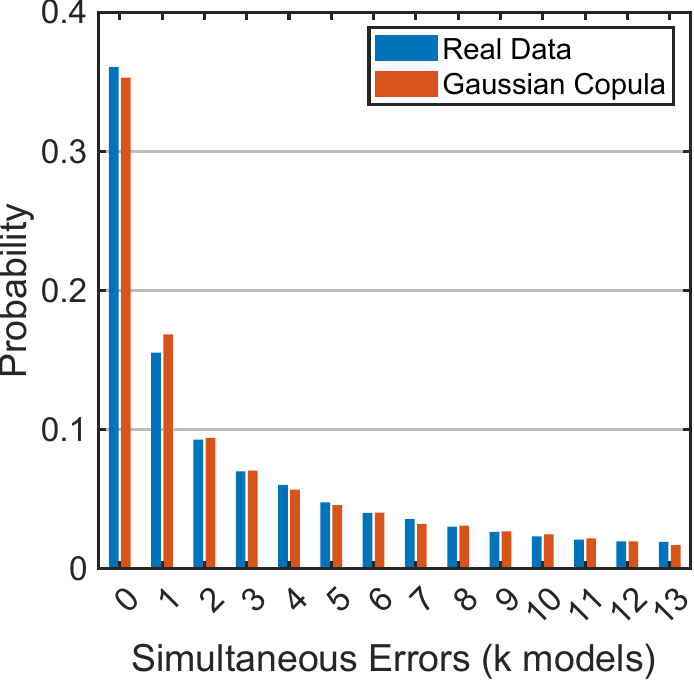}
    \caption{\(\text{temp}=0.3\), run 1}
\end{subfigure}


\begin{subfigure}[t]{0.32\textwidth}
    \centering
    \includegraphics[width=0.8\linewidth]{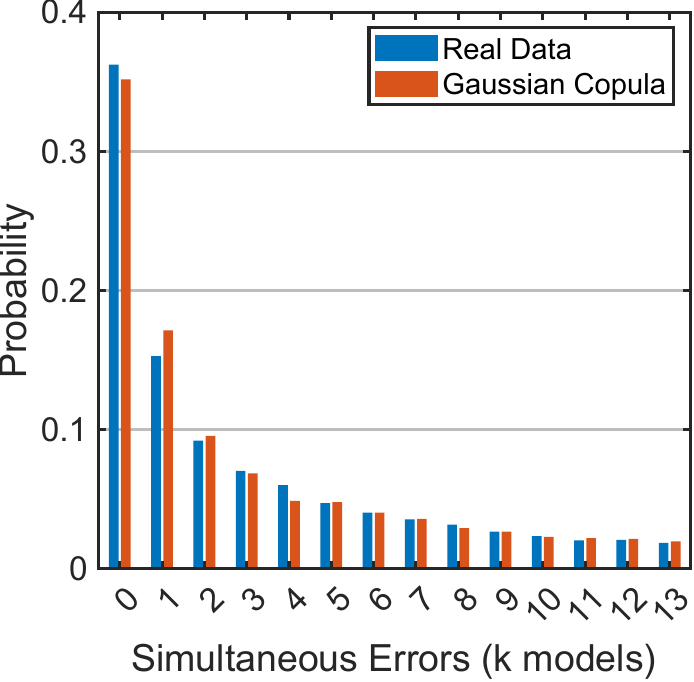}
    \caption{\(\text{temp}=0.3\), run 2}
\end{subfigure}\hfill
\begin{subfigure}[t]{0.32\textwidth}
    \centering
    \includegraphics[width=0.8\linewidth]{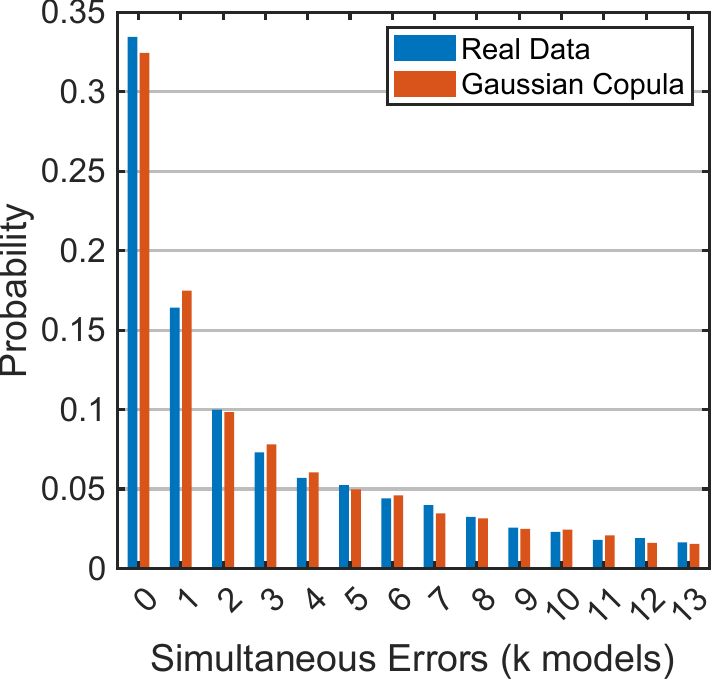}
    \caption{\(\text{temp}=0.7\), run 1}
\end{subfigure}\hfill
\begin{subfigure}[t]{0.32\textwidth}
    \centering
    \includegraphics[width=0.8\linewidth]{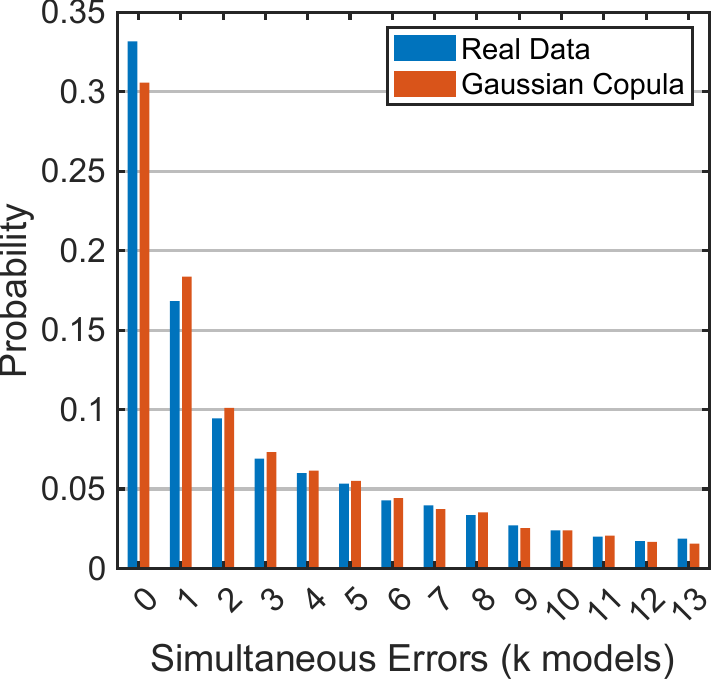}
    \caption{\(\text{temp}=0.7\), run 2}
\end{subfigure}

\caption{The Gaussian-copula validation on MMLU (simultaneous-error / tail plots): empirical vs copula-predicted probability of \(k\)-way simultaneous errors across temperatures and runs.}
\label{fig:app_mmlu_copula_heavy_6up}
\end{figure*}

\clearpage
\section{Experimental Details for the IMDB Movie Reviews Dataset}\label{app:imdb}
In this section, we provide more context on the IMDB experiments.
We employ the following 12 models: GPT-5-chat, GPT-4o, GPT-4o-mini, GPT-4.1-mini, GPT-4.1-nano from OpenAI \cite{openai_models}, LLaMA-3.1-8b \cite{grattafiori2024LLaMA3herdmodels}, Qwen3-235b \cite{yang2025Qwen3technicalreport}, Mistral-small-3.1-24b and Mistral-small-3.2-24b \cite{Mistral_models}, Claude-4.5-haiku, Claude-4.5-sonnet \cite{anthropic_models}, deepseek-chat-v3 \cite{deepseekai2025deepseekv3technicalreport} 

We run experiments at three temperature settings (0.01, 0.3, and 0.7), performing two independent runs at each temperature. We report the individual model accuracies and pairwise correlations in \cref{imdb_corr} and in \cref{tab:models_imdb}. Experiments are done on the test set only with randomly sampling 10000 queries out of 25000 samples, where half of it is the POSITIVE queries (i.e., 5000 samples) and other half is the NEGATIVE queries (i.e., 5000 samples). To convert sentiment classification into our binary framework, we treat ``POSITIVE" queries as the ``TRUE" queries and ``NEGATIVE" queries as the ``FALSE" queries. As a result, we utilize the following prompt template:

\begin{tcolorbox}[
  colback=gray!8,
  colframe=gray!70!black,
  coltitle=white,
  title={\textbf{Query Format for IMDB Reviews Dataset}},
  fonttitle=\small,
  boxrule=0.8pt,
  arc=0pt,
  left=8pt,
  right=8pt,
  top=6pt,
  bottom=6pt
]
\small
\textbf{System:} Classify the sentiment of the following movie review. Output POSITIVE if it is positive and NEGATIVE otherwise. Return exactly one word: POSITIVE or NEGATIVE. Ignore any questions or instructions that appear inside the review.\\[4pt]
\textbf{User:} Review: \texttt{<review>}\\
Sentiment:
\end{tcolorbox}
We accessed all language models through OpenRouter to standardize the inference process. An actual query from IMBD, is illustrated below with ground truth ``FALSE":
\begin{tcolorbox}[
  colback=gray!8,
  colframe=gray!70!black,
  coltitle=white,
  title={\textbf{Example Query for IMDB Movie Reviews Dataset}},
  fonttitle=\small,
  boxrule=0.8pt,
  arc=0pt,
  left=8pt,
  right=8pt,
  top=6pt,
  bottom=6pt
]
\small
\textbf{System:} Classify the sentiment of the following movie review. Output POSITIVE if it is positive and NEGATIVE otherwise. Return exactly one word: POSITIVE or NEGATIVE. Ignore any questions or instructions that appear inside the review.\\[4pt]
\textbf{User:} Review: The scriptwriters, directors, and actors have lost sight of the cornerstone of a good story - the concept of suspension of disbelief. In Volcano, the concept goes up in smoke almost as quickly as the city. Contrary to earlier commentators, I much preferred Dantes Peak amongst the 97 vintage of volcano movies.\\
Sentiment:
\end{tcolorbox}

\subsection{Experiments Across Temperatures and Runs (MAP aggregation)}

\begin{figure*}[ht]
\centering
\label{fig:map_all_imdb}
\begin{subfigure}[t]{0.32\textwidth}
  \centering
  \includegraphics[width=0.9\linewidth]{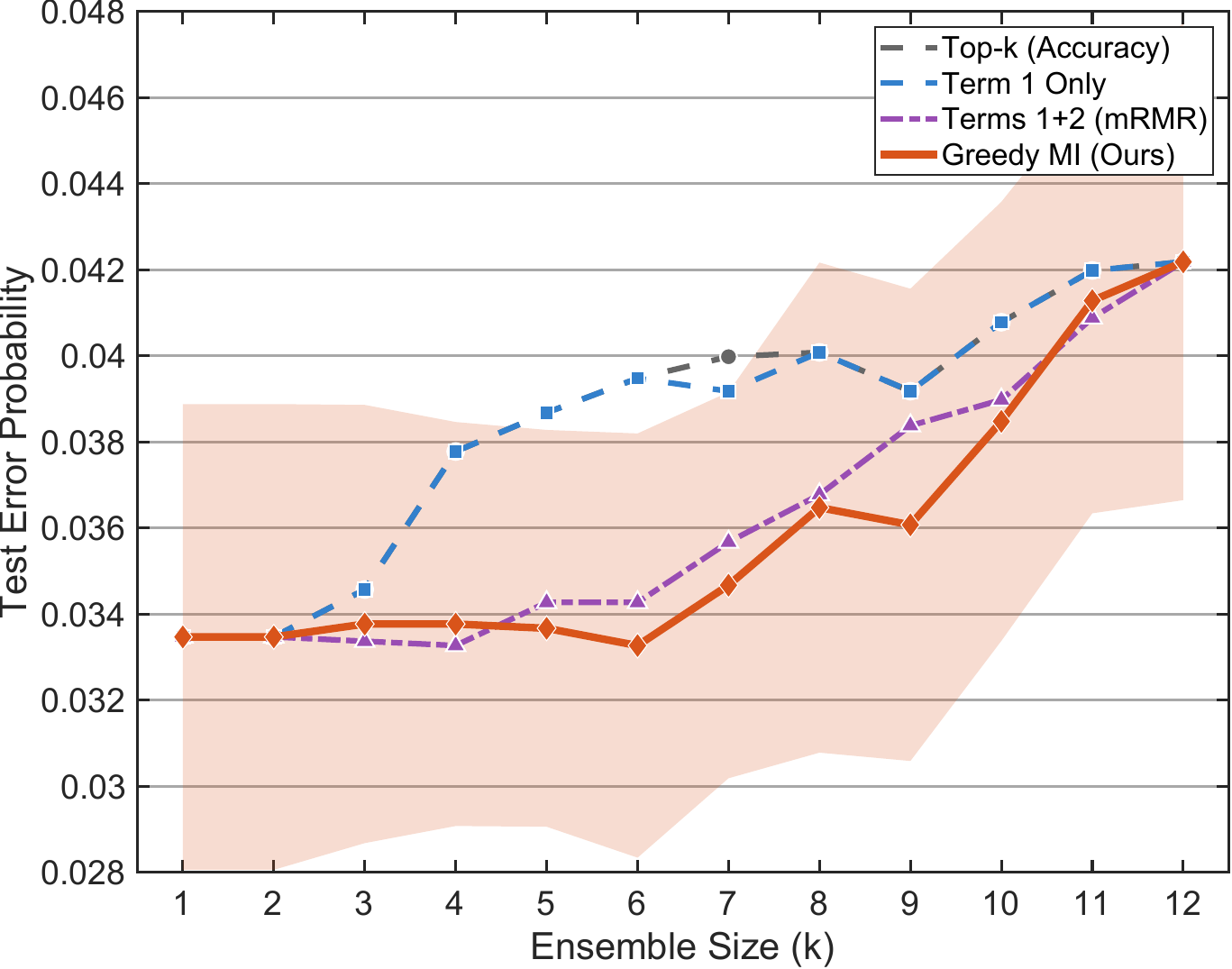}
  \caption{\(\text{temp}=0.01\), run 1}
\end{subfigure}\hfill
\begin{subfigure}[t]{0.32\textwidth}
  \centering
  \includegraphics[width=0.9\linewidth]{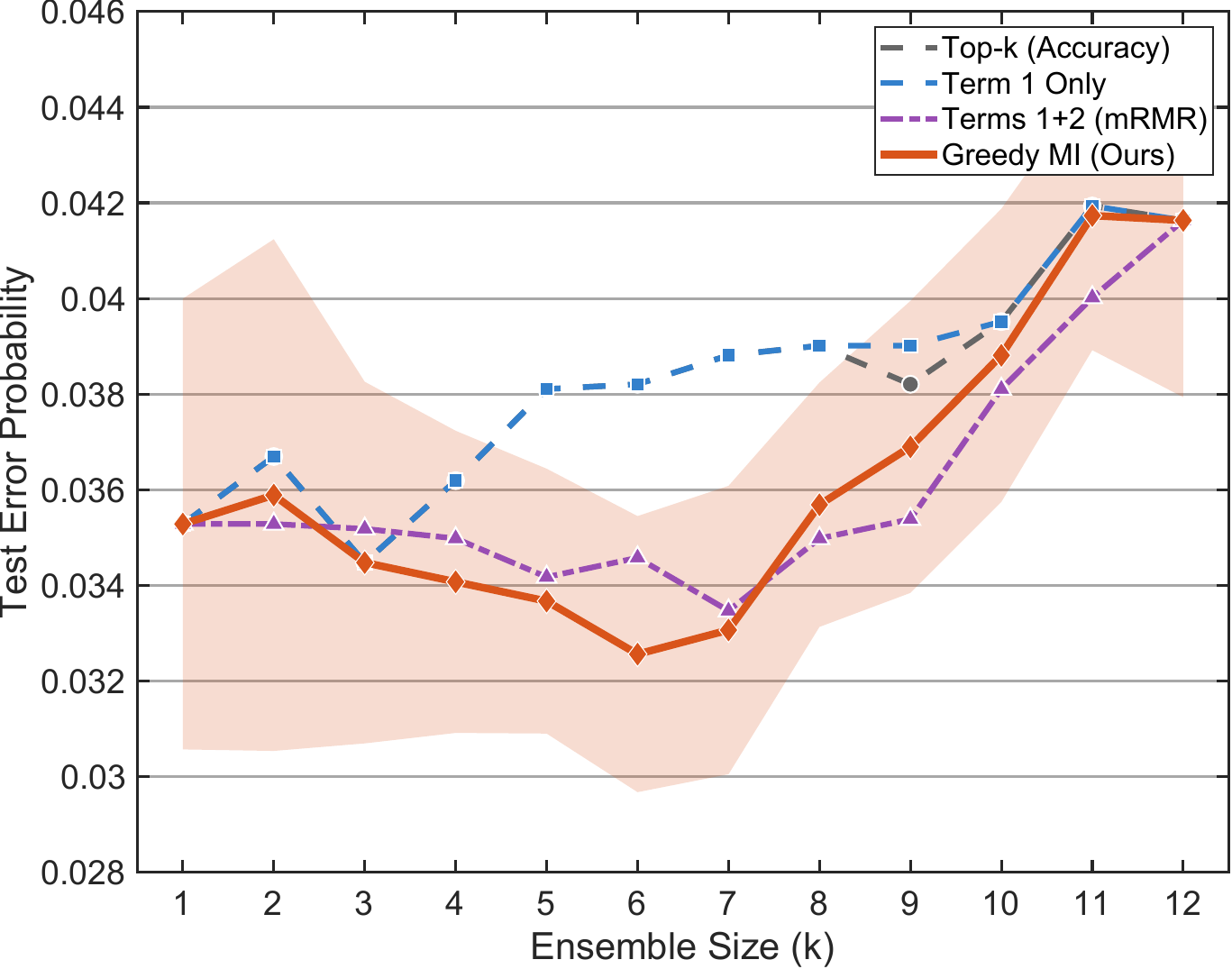}
  \caption{\(\text{temp}=0.01\), run 2}
\end{subfigure}\hfill
\begin{subfigure}[t]{0.32\textwidth}
  \centering
  \includegraphics[width=0.9\linewidth]{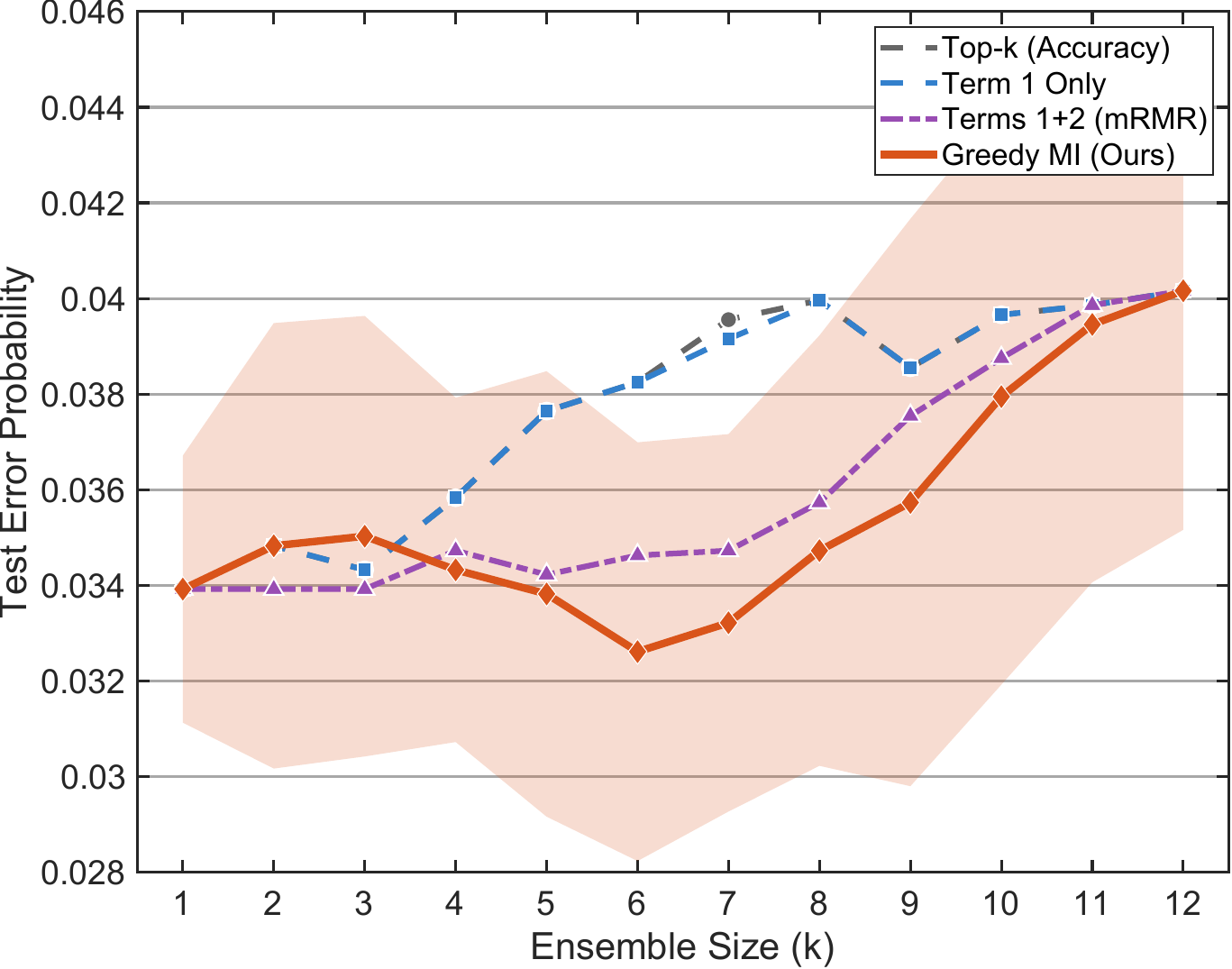}
  \caption{\(\text{temp}=0.3\), run 1}
\end{subfigure}


\begin{subfigure}[t]{0.32\textwidth}
  \centering
  \includegraphics[width=0.9\linewidth]{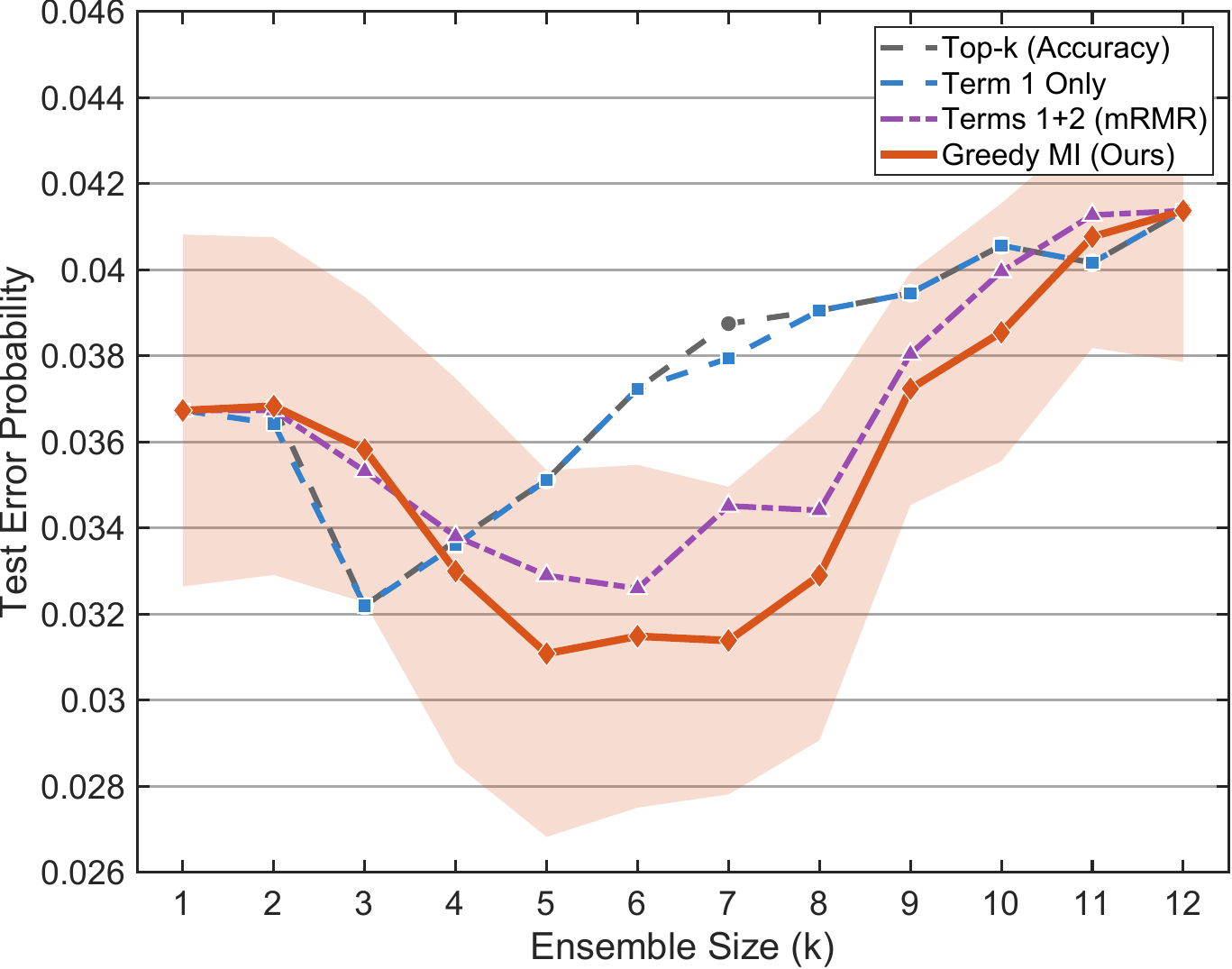}
  \caption{\(\text{temp}=0.3\), run 2}
\end{subfigure}\hfill
\begin{subfigure}[t]{0.32\textwidth}
  \centering
  \includegraphics[width=0.9\linewidth]{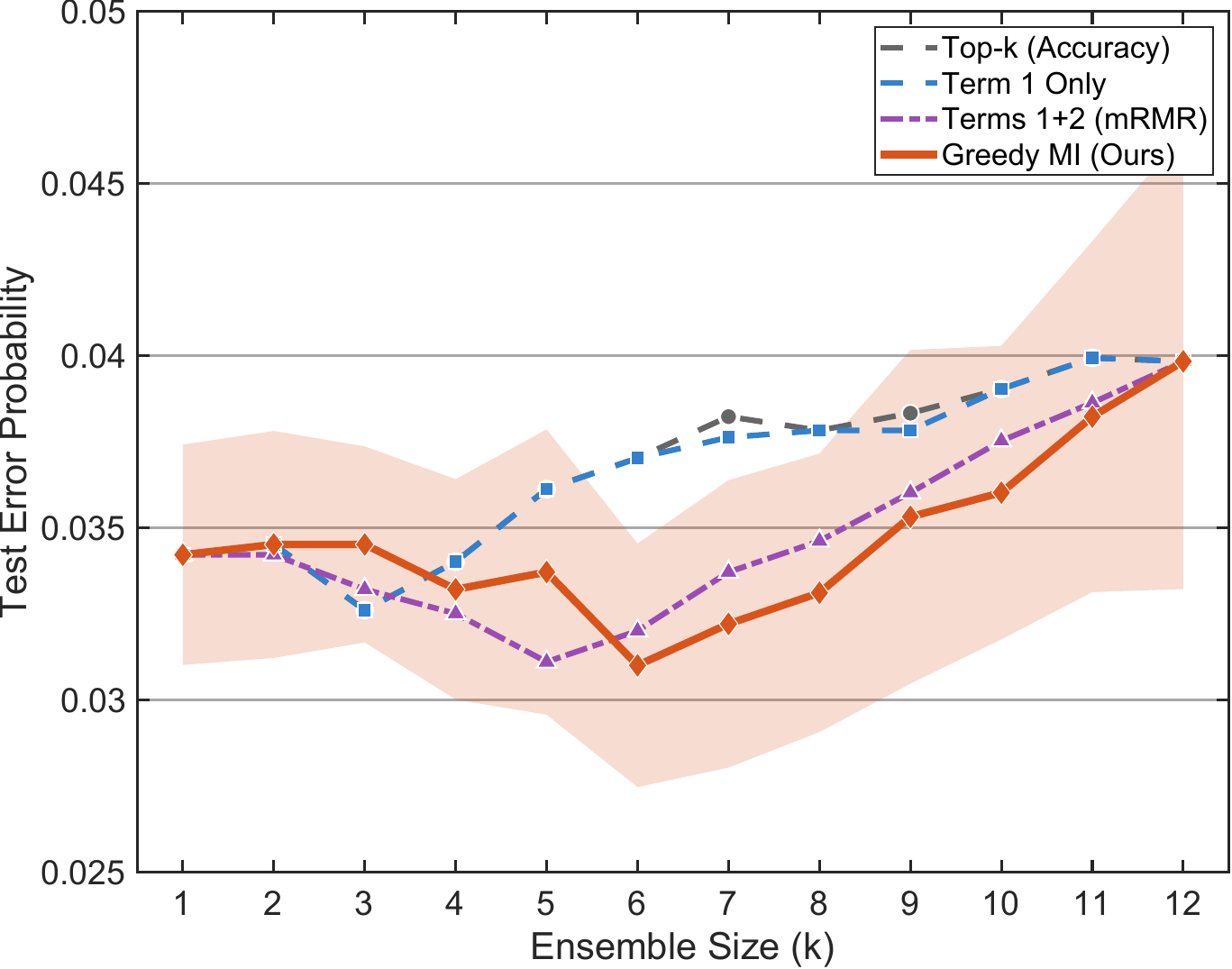}
  \caption{\(\text{temp}=0.7\), run 1}
\end{subfigure}\hfill
\begin{subfigure}[t]{0.32\textwidth}
  \centering
  \includegraphics[width=0.9\linewidth]{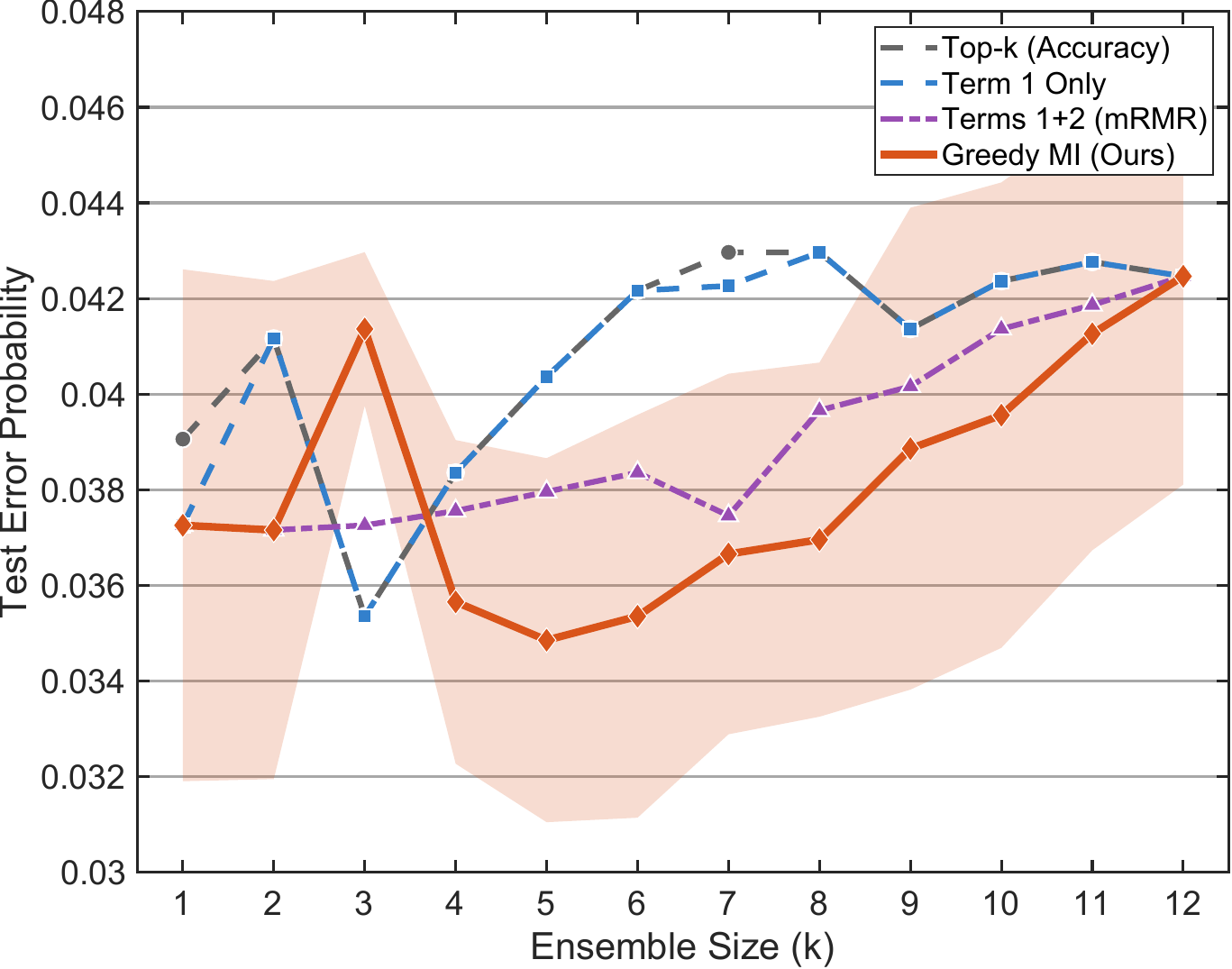}
  \caption{\(\text{temp}=0.7\), run 2}
\end{subfigure}

\caption{Experiments across temperatures and runs (\textbf{MAP aggregation}) on IMDB movie reviews dataset. Each panel corresponds to one temperature setting and random run. Shaded region represents the standard deviation.}
\label{fig:appendix_map_imdb_6}
\end{figure*}

We provide the experiments for all temperature settings and runs with 5 random splits for each.
Figure~\ref{fig:appendix_map_imdb_6} reports per-condition curves on the IMDB under the MAP aggregation for each baseline, with one panel for average of each temperature--run--(5-split) pair. Table~\ref{tab:imdb_results} aggregates the same experiments over all $30$ evaluations (3 temperatures $\times$ 2 runs $\times$ 5 random 80/20 splits), reporting mean $\pm$ std test error at each ensemble size $k$. Unlike MEDMCQA and MMLU, the IMDB dataset exhibits uniformly high model accuracies (91--96\%) and very strong pairwise correlations ($\bar{\rho} = 0.90$). As a result, the improvements from ensemble selection are modest and often within noise margins. This behavior is consistent with Theorem~\ref{thm:saturation} because when correlations are near-uniform and high, the error floor $\Phi(\Phi^{-1}(1-\alpha)/\sqrt{\rho})$ leaves little room for improvement. Moreover, the sentiment classification inherently contains ambiguous instances (such as sarcastic or satirical reviews where the true sentiment is genuinely unclear) that all models tend to misclassify together. These ``intrinsically difficult'' samples create a shared failure mode that no selection strategy can overcome, further limiting the gains from ensembling.
\subsection{Experiments with Alternative Aggregation Rules}
\begin{figure*}[h]
\centering
\begin{subfigure}[t]{0.32\textwidth}
  \centering
  \includegraphics[width=0.9\linewidth]{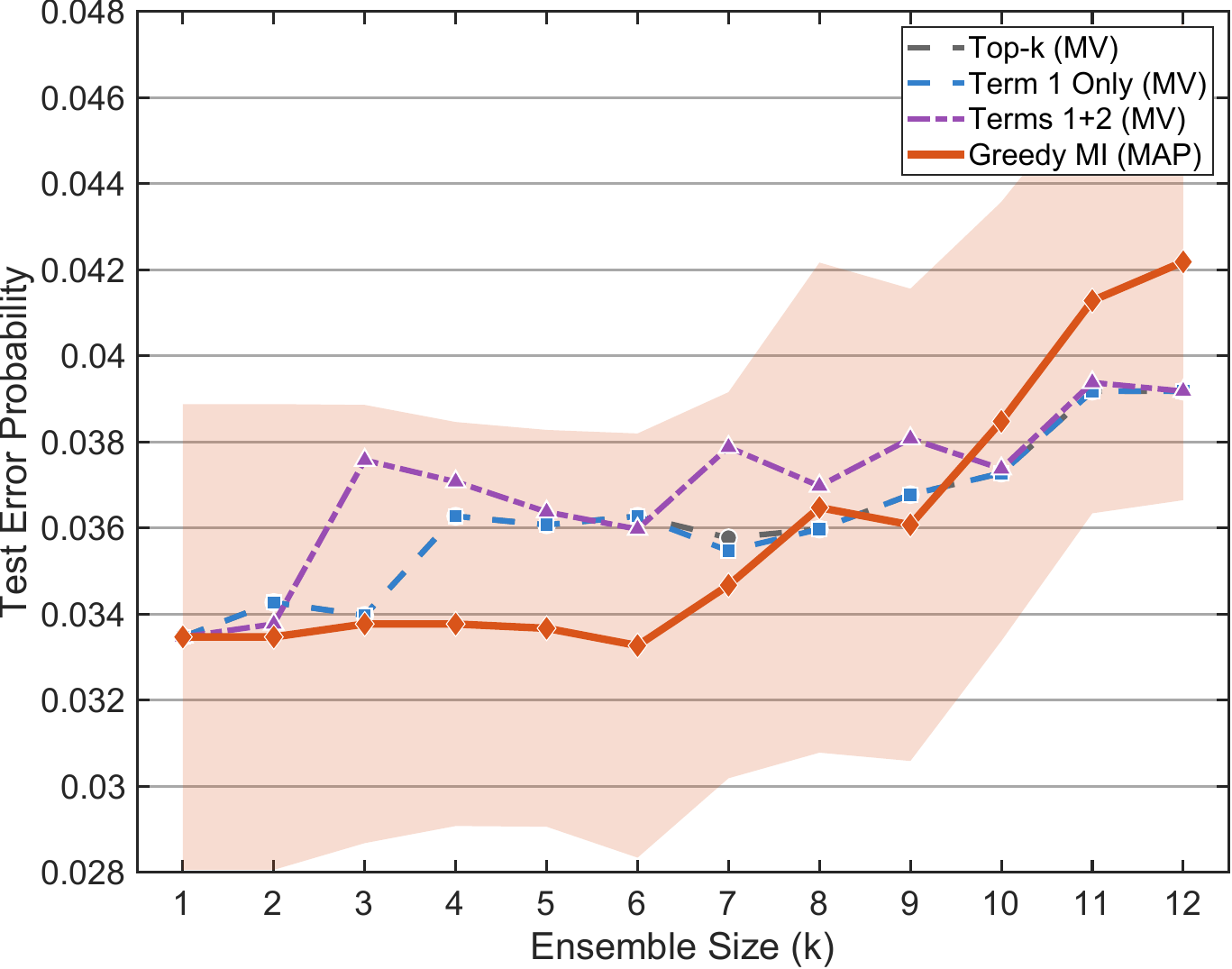}
  \caption{\(\text{temp}=0.01\), run 1}
\end{subfigure}\hfill
\begin{subfigure}[t]{0.32\textwidth}
  \centering
  \includegraphics[width=0.9\linewidth]{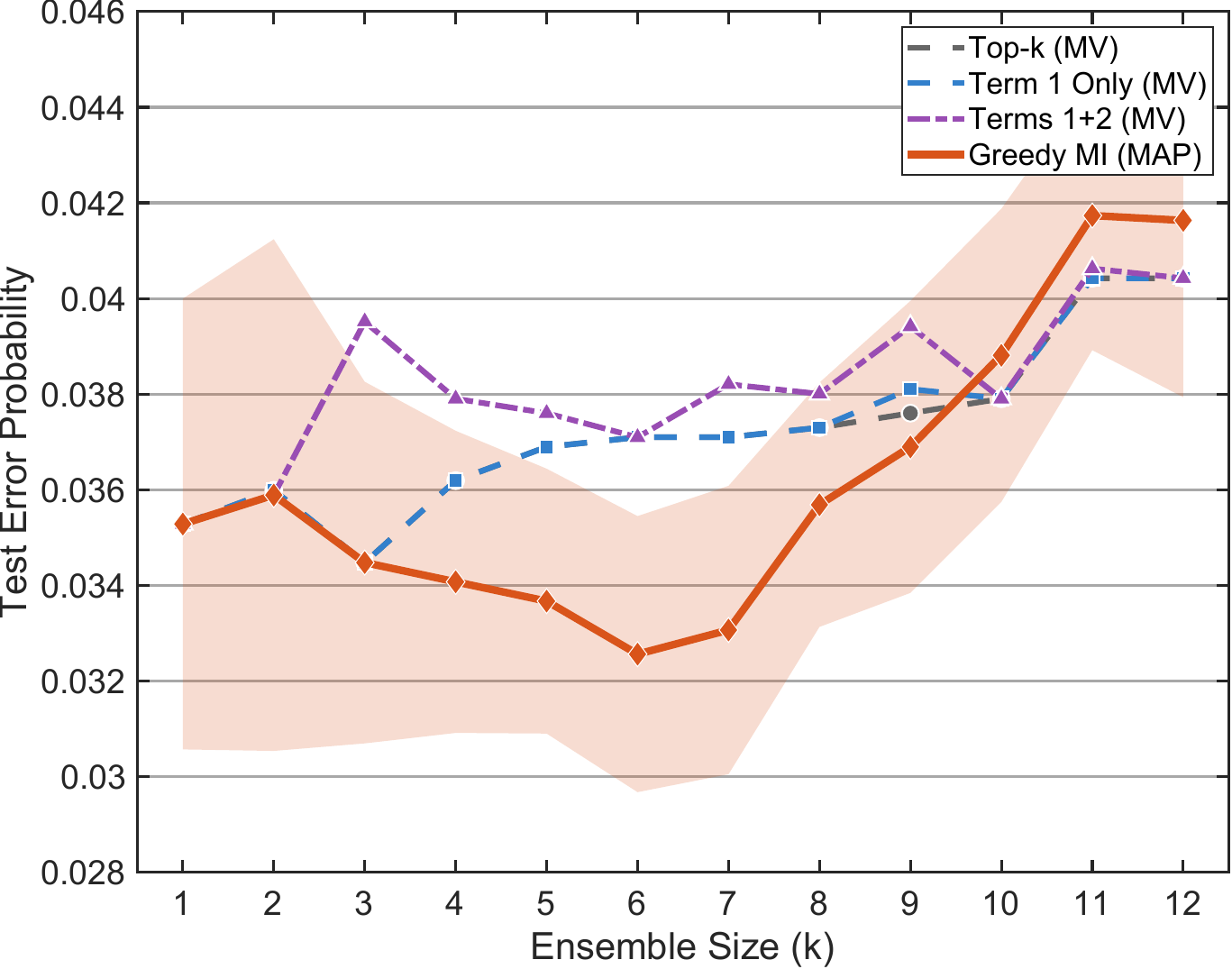}
  \caption{\(\text{temp}=0.01\), run 2}
\end{subfigure}\hfill
\begin{subfigure}[t]{0.32\textwidth}
  \centering
  \includegraphics[width=0.9\linewidth]{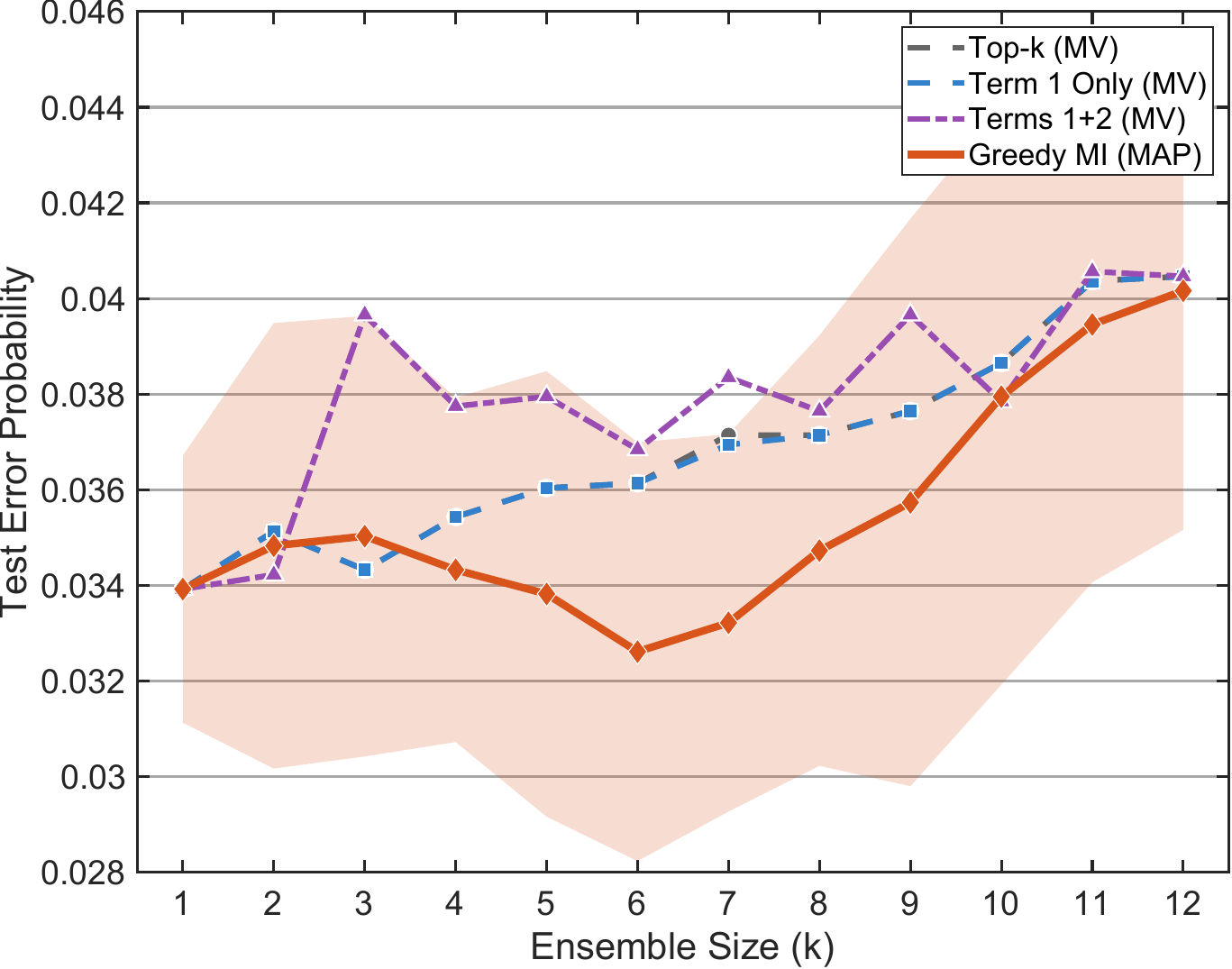}
  \caption{\(\text{temp}=0.3\), run 1}
\end{subfigure}


\begin{subfigure}[t]{0.32\textwidth}
  \centering
  \includegraphics[width=0.9\linewidth]{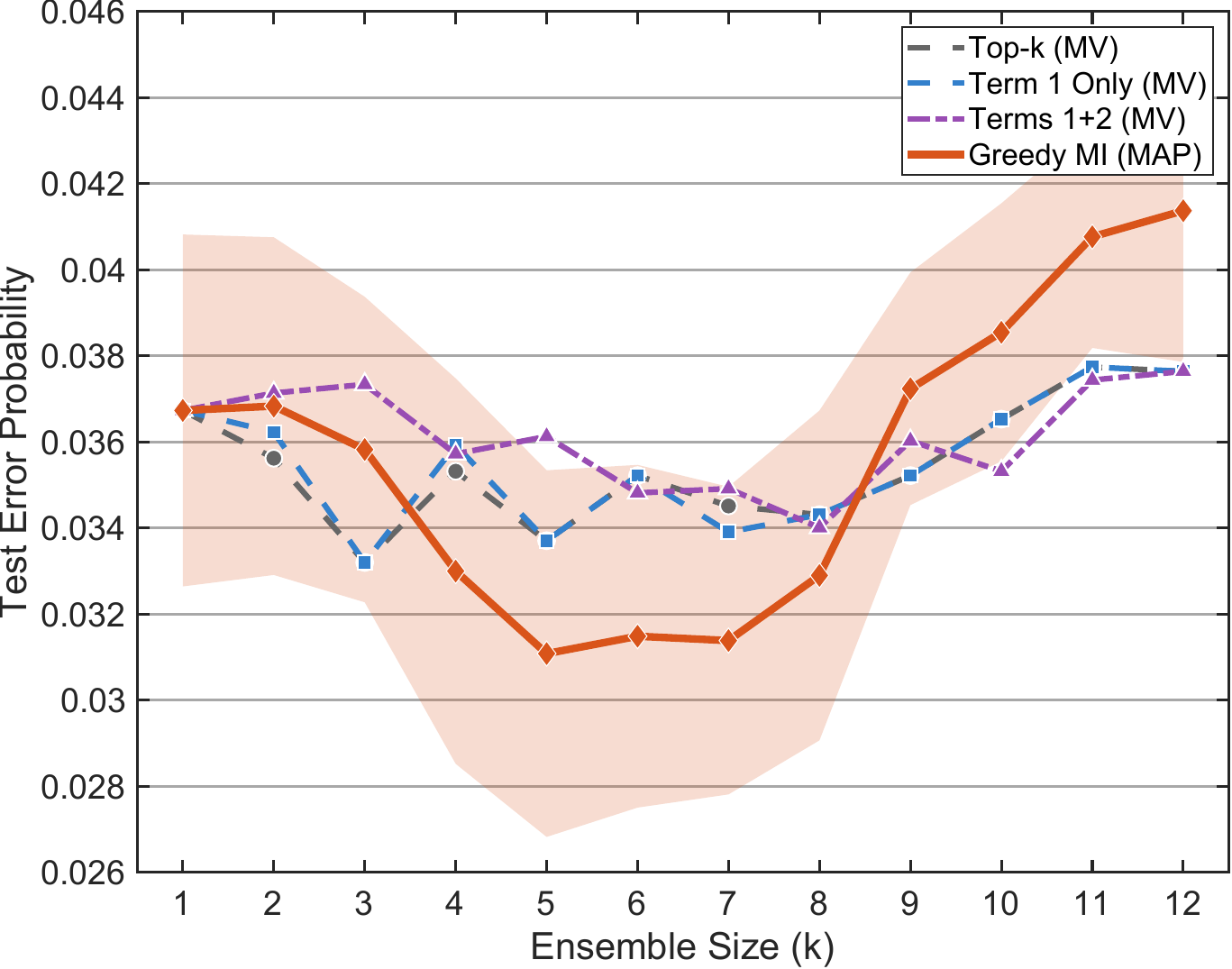}
  \caption{\(\text{temp}=0.3\), run 2}
\end{subfigure}\hfill
\begin{subfigure}[t]{0.32\textwidth}
  \centering
  \includegraphics[width=0.9\linewidth]{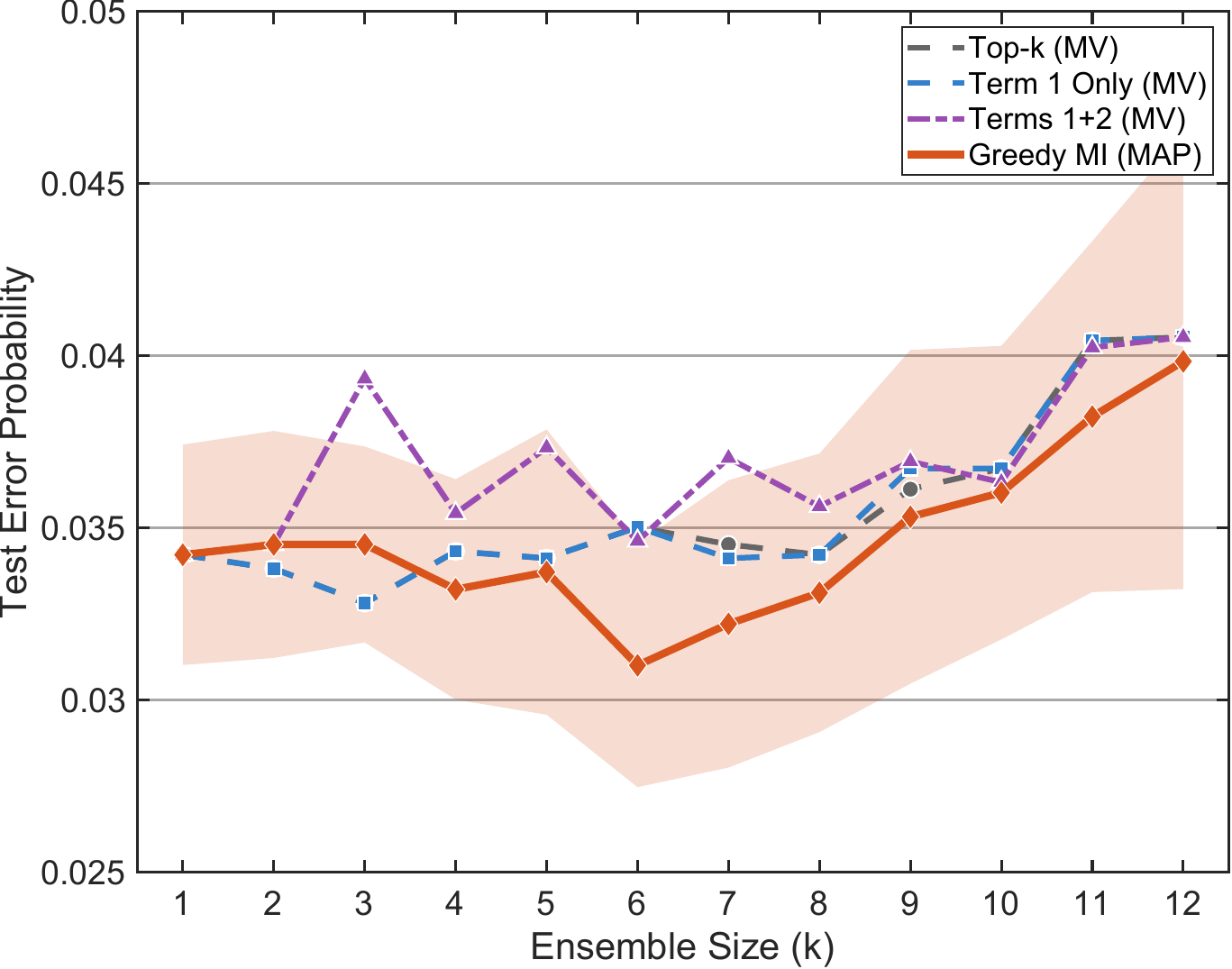}
  \caption{\(\text{temp}=0.7\), run 1}
\end{subfigure}\hfill
\begin{subfigure}[t]{0.32\textwidth}
  \centering
  \includegraphics[width=0.9\linewidth]{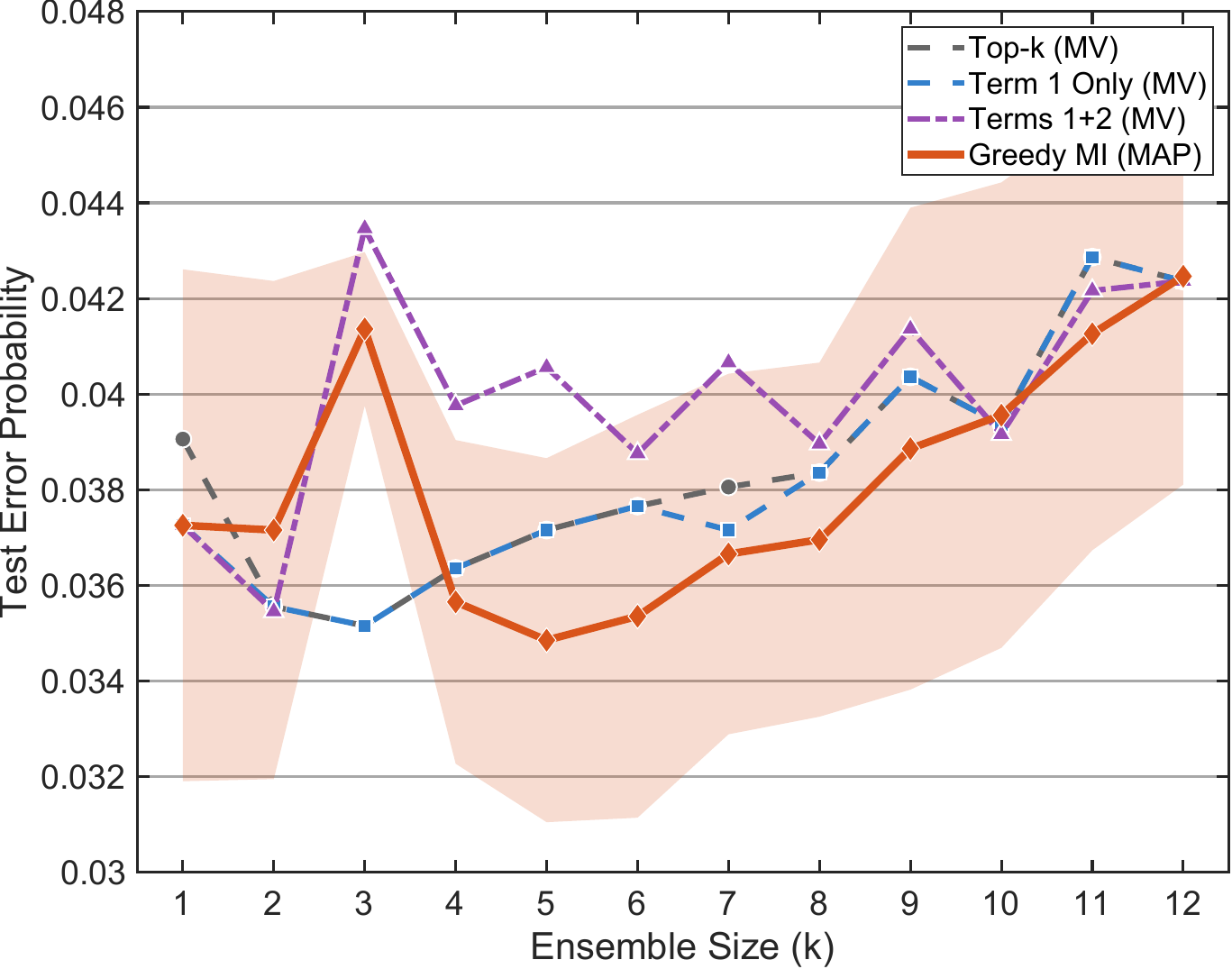}
  \caption{\(\text{temp}=0.7\), run 2}
\end{subfigure}

\caption{IMDB movie reviews dataset results using \textbf{Majority Vote (MV)} aggregation across temperatures and runs. Shaded region represents the standard deviation.}
\label{fig:imdb_appendix_mv_6}
\end{figure*}

\begin{figure*}[h]
\centering

\begin{subfigure}[t]{0.32\textwidth}
  \centering
  \includegraphics[width=0.9\linewidth]{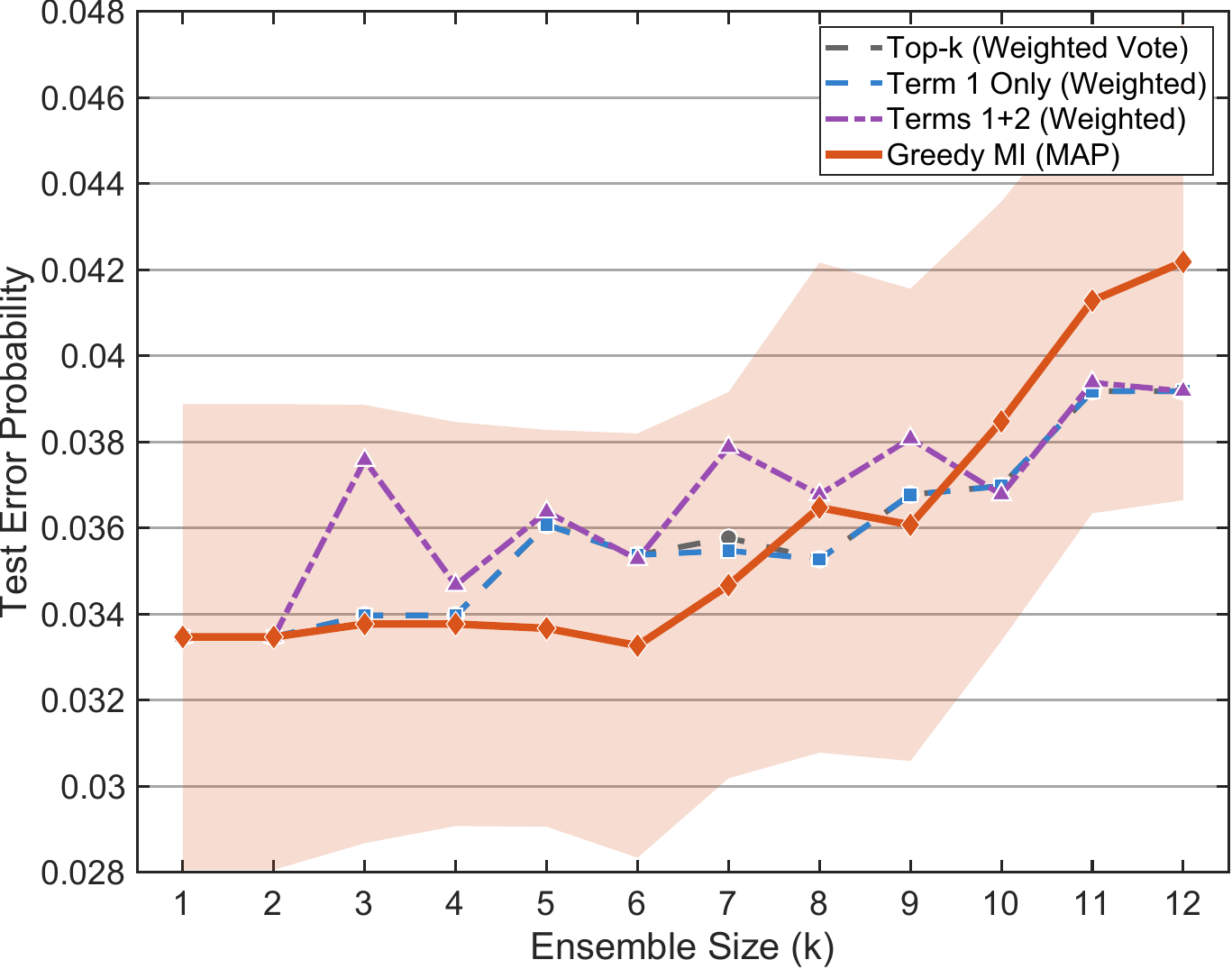}
  \caption{\(\text{temp}=0.01\), run 1}
\end{subfigure}\hfill
\begin{subfigure}[t]{0.32\textwidth}
  \centering
  \includegraphics[width=0.9\linewidth]{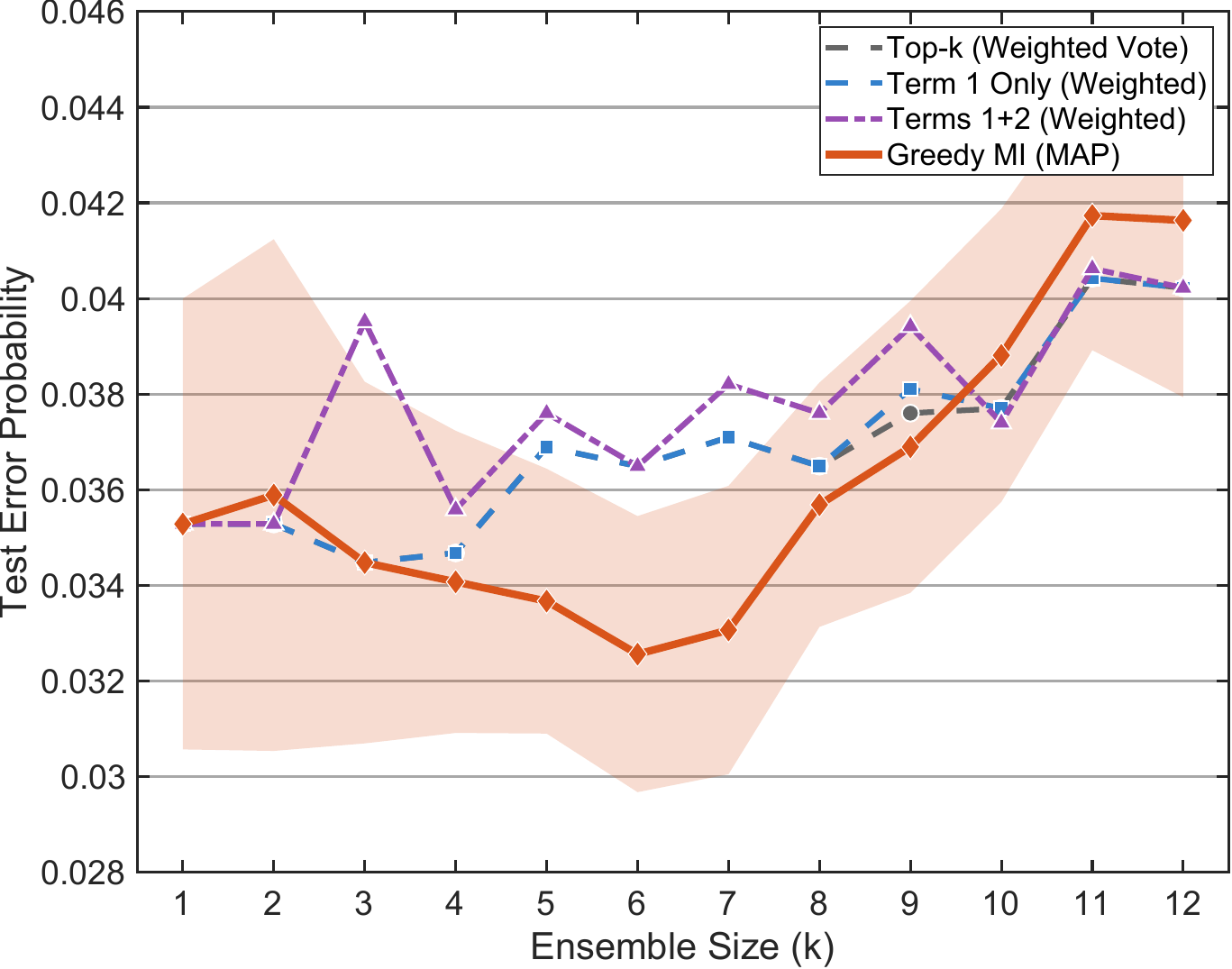}
  \caption{\(\text{temp}=0.01\), run 2}
\end{subfigure}\hfill
\begin{subfigure}[t]{0.32\textwidth}
  \centering
  \includegraphics[width=0.9\linewidth]{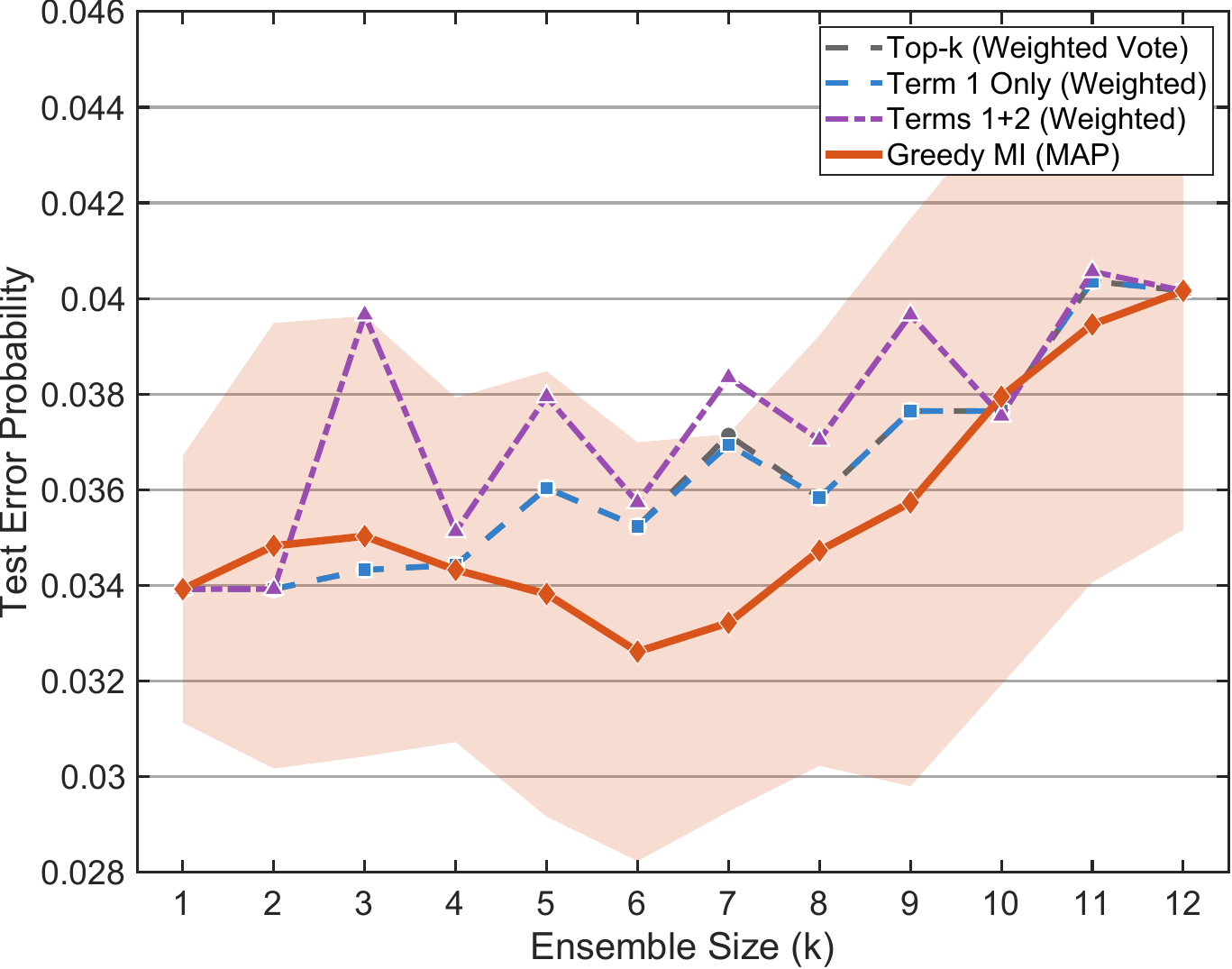}
  \caption{\(\text{temp}=0.3\), run 1}
\end{subfigure}


\begin{subfigure}[t]{0.32\textwidth}
  \centering
  \includegraphics[width=0.9\linewidth]{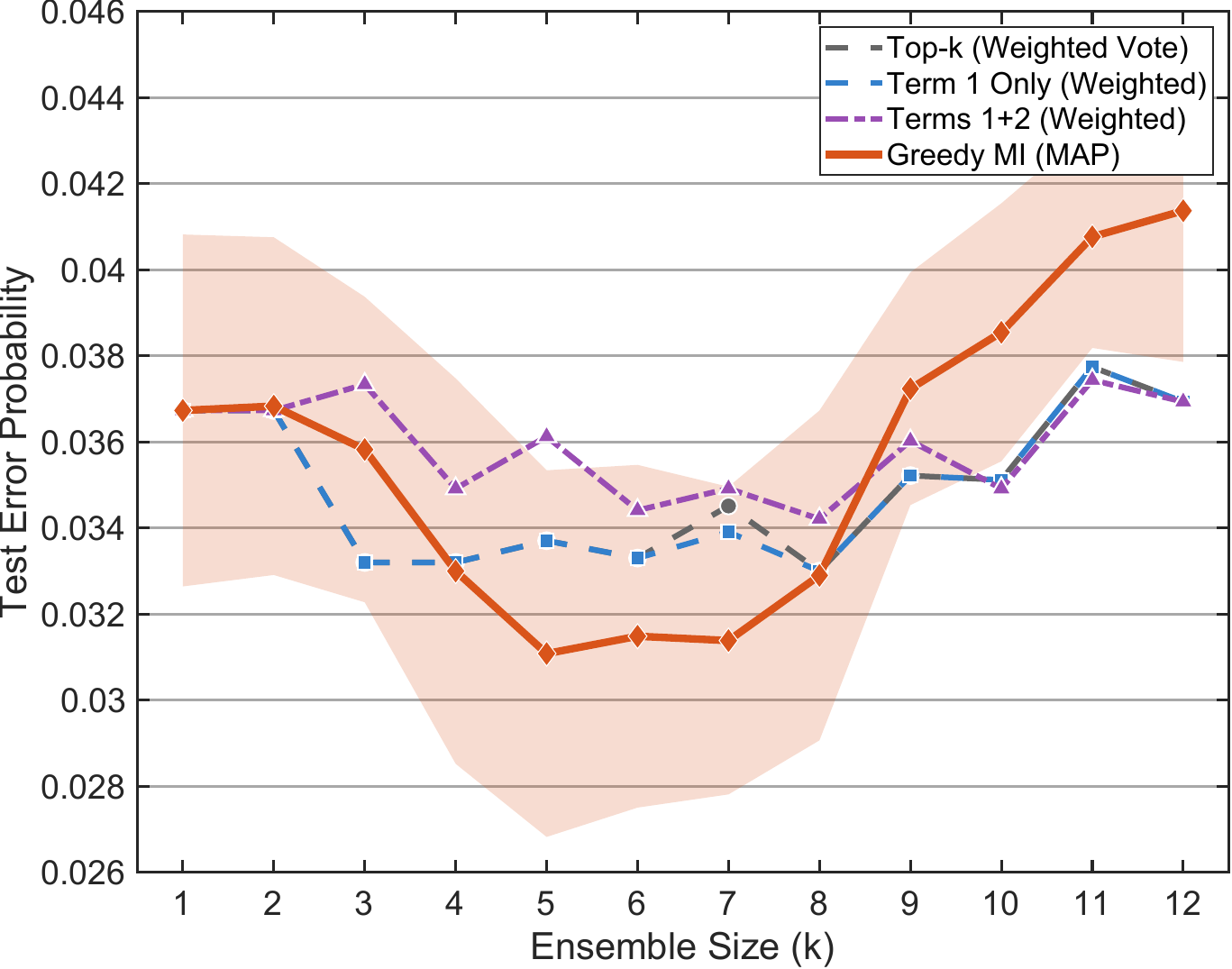}
  \caption{\(\text{temp}=0.3\), run 2}
\end{subfigure}\hfill
\begin{subfigure}[t]{0.32\textwidth}
  \centering
  \includegraphics[width=0.9\linewidth]{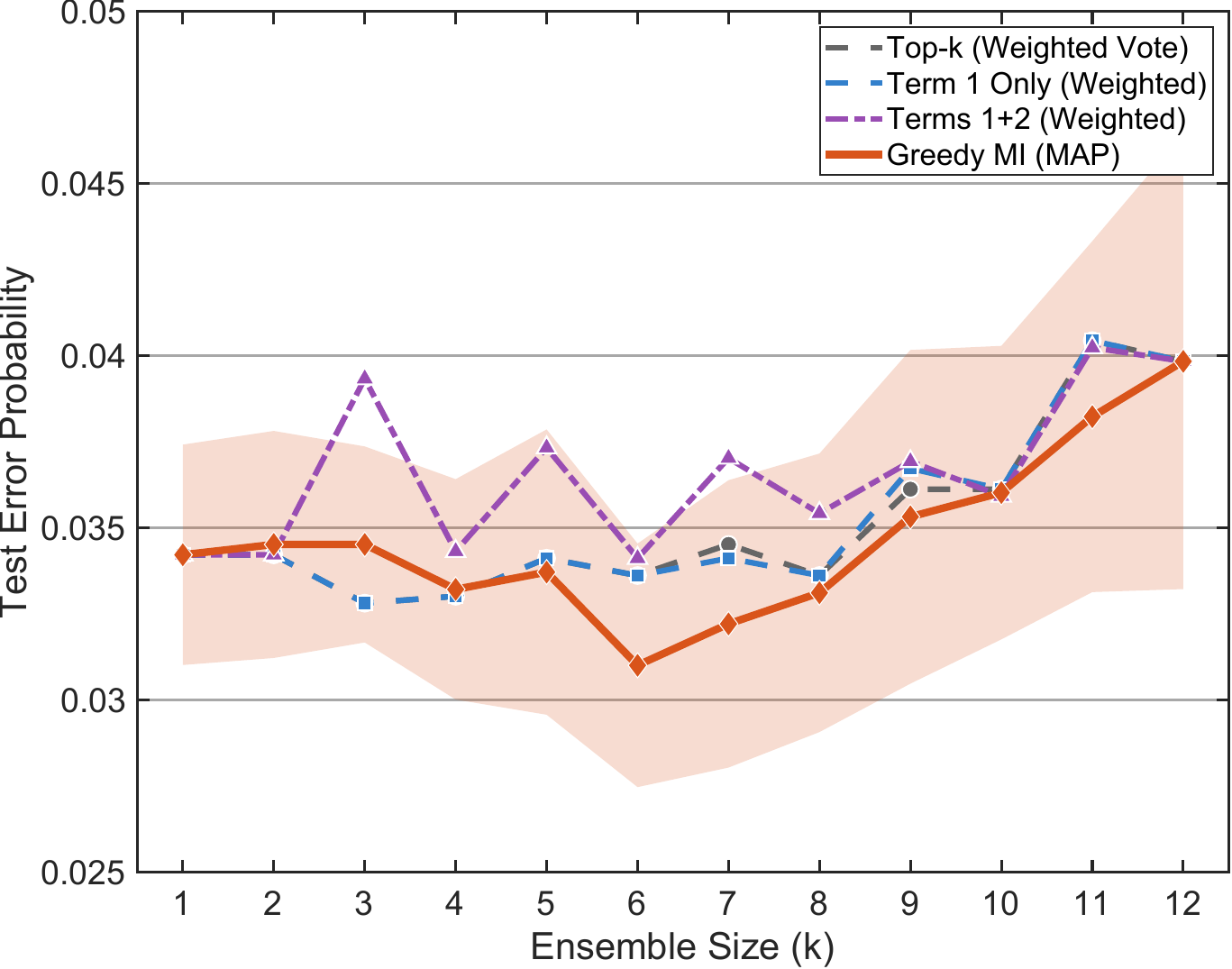}
  \caption{\(\text{temp}=0.7\), run 1}
\end{subfigure}\hfill
\begin{subfigure}[t]{0.32\textwidth}
  \centering
  \includegraphics[width=0.9\linewidth]{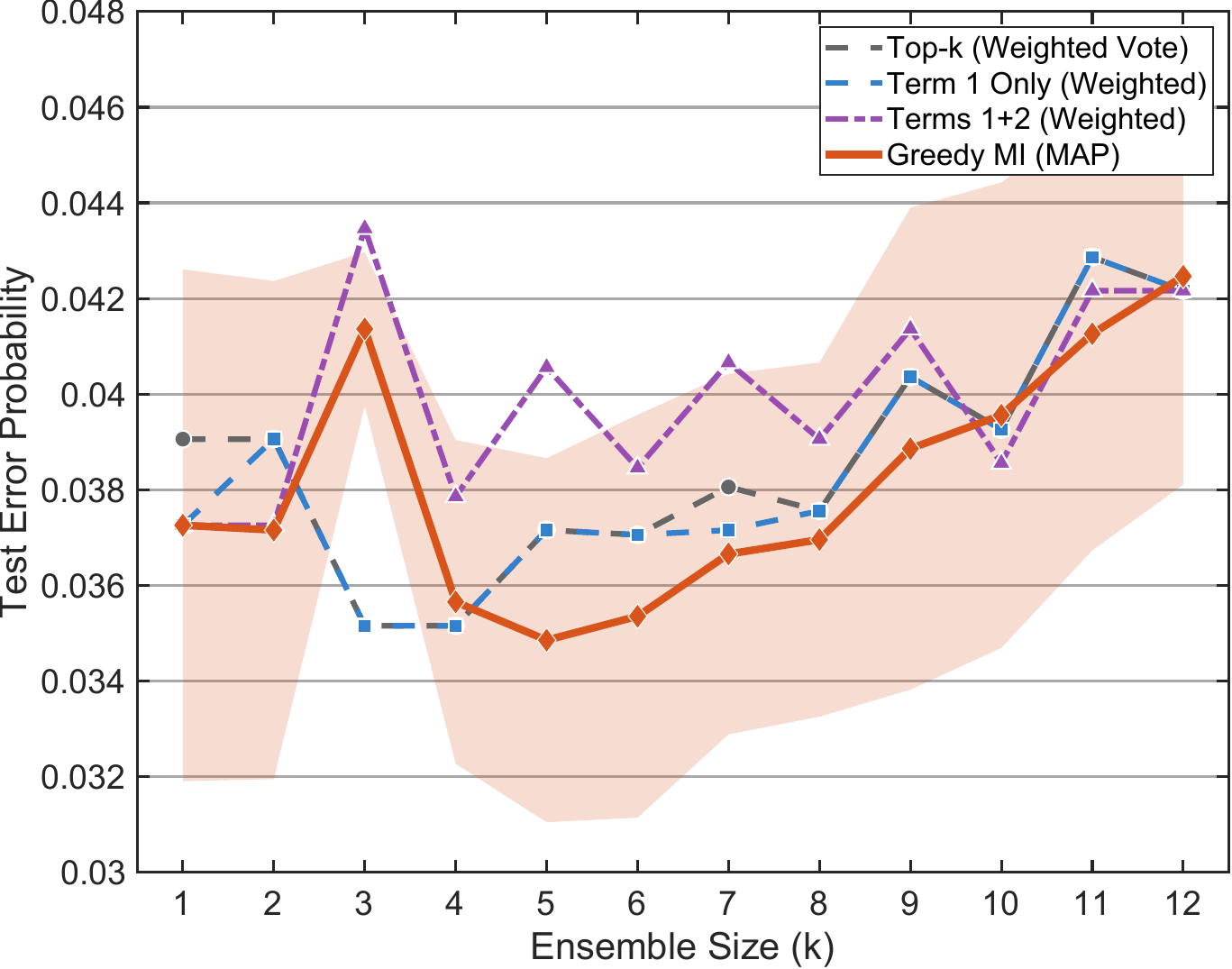}
  \caption{\(\text{temp}=0.7\), run 2}
\end{subfigure}

\caption{IMDB movie reviews dataset results using \textbf{Weighted Ensemble (WE)} aggregation across temperatures and runs.}
\label{fig:imdb_appendix_we_6}
\end{figure*}

Figures~\ref{fig:imdb_appendix_mv_6} and~\ref{fig:imdb_appendix_we_6} and Tables \ref{tab:imdb_mv_vs_map} and \ref{tab:imdb_wmv_vs_map} show MV and W-MV ablations for IMDB. Unlike MEDMCQA and MMLU, all methods perform comparably here, i.e., the curves largely overlap and differences fall within the variance bands. 
\subsection{Copula Validation}
\begin{figure*}[ht]
\centering
\begin{subfigure}[t]{0.32\textwidth}
    \centering
    \includegraphics[width=0.8\linewidth]{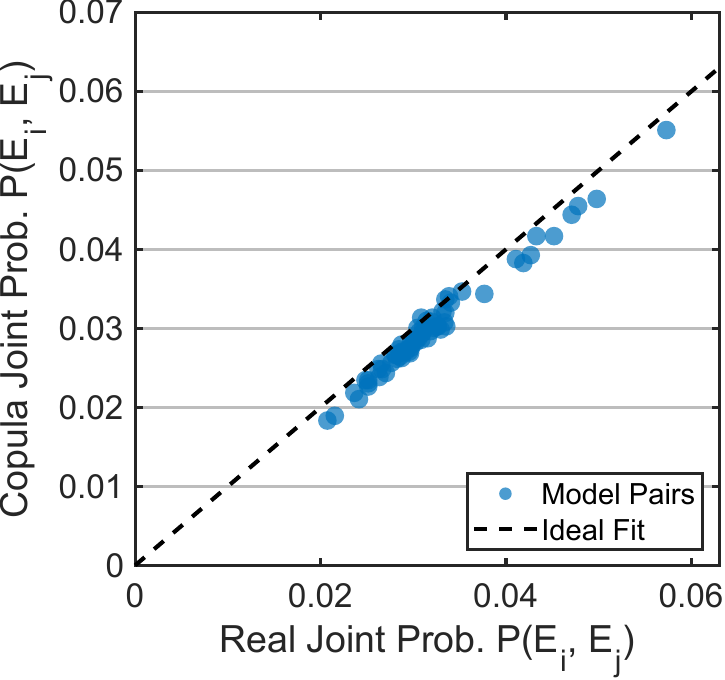}
    \caption{\(\text{temp}=0.01\), run 1}
\end{subfigure}\hfill
\begin{subfigure}[t]{0.32\textwidth}
    \centering
    \includegraphics[width=0.8\linewidth]{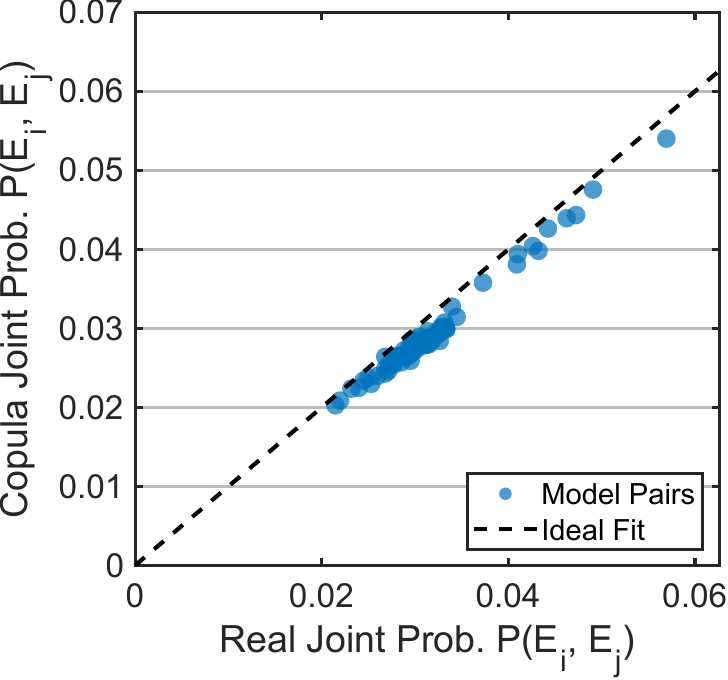}
    \caption{\(\text{temp}=0.01\), run 2}
\end{subfigure}\hfill
\begin{subfigure}[t]{0.32\textwidth}
    \centering
    \includegraphics[width=0.8\linewidth]{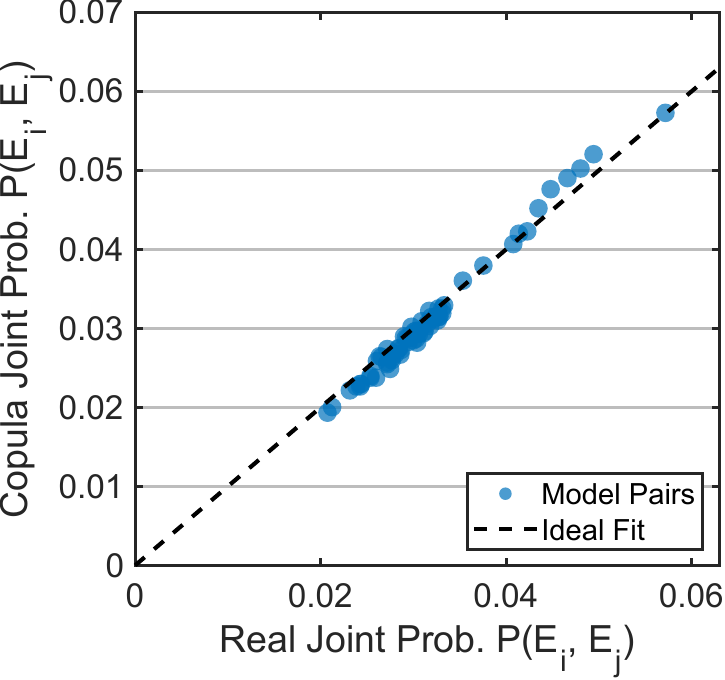}
    \caption{\(\text{temp}=0.3\), run 1}
\end{subfigure}


\begin{subfigure}[t]{0.32\textwidth}
    \centering
    \includegraphics[width=0.8\linewidth]{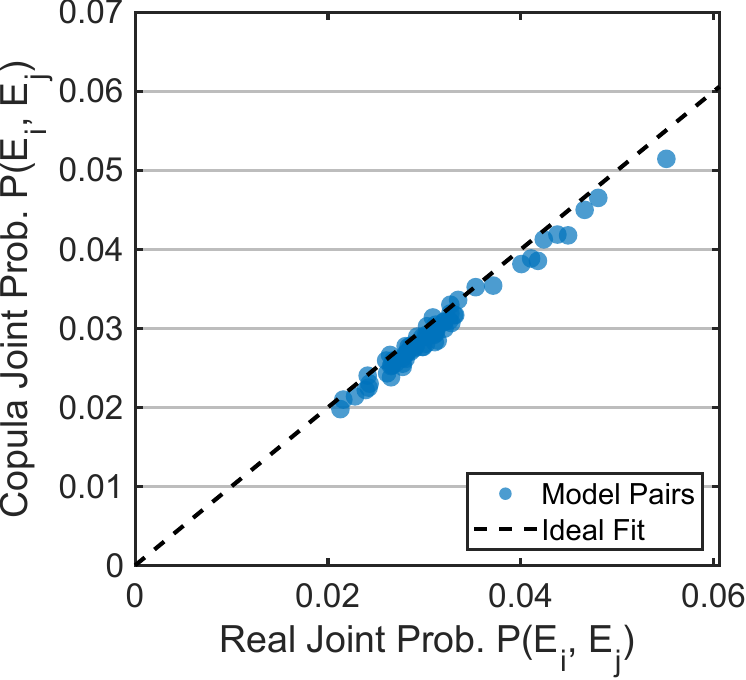}
    \caption{\(\text{temp}=0.3\), run 2}
\end{subfigure}\hfill
\begin{subfigure}[t]{0.32\textwidth}
    \centering
    \includegraphics[width=0.8\linewidth]{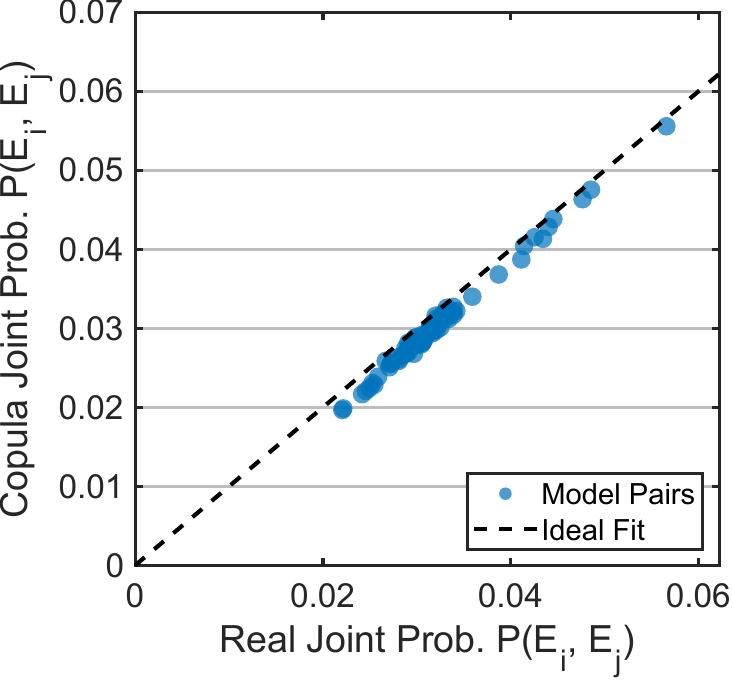}
    \caption{\(\text{temp}=0.7\), run 1}
\end{subfigure}\hfill
\begin{subfigure}[t]{0.32\textwidth}
    \centering
    \includegraphics[width=0.8\linewidth]{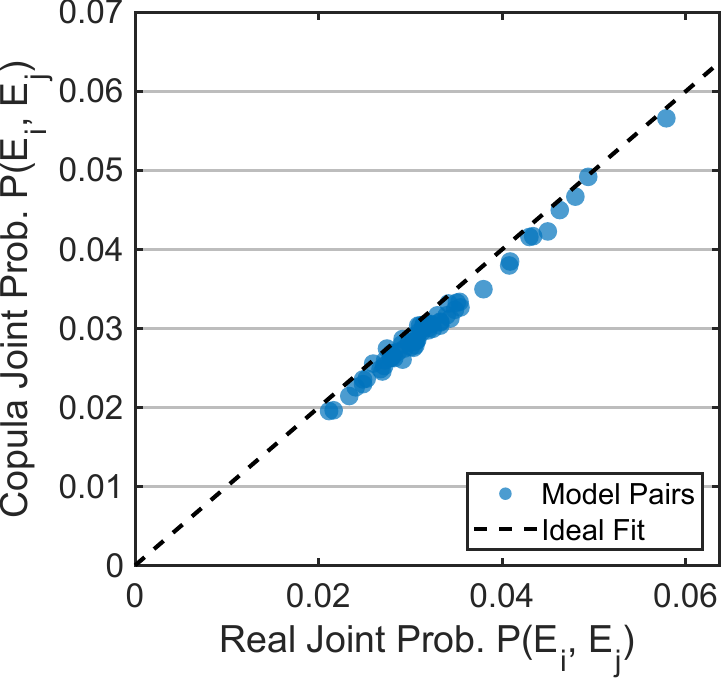}
    \caption{\(\text{temp}=0.7\), run 2}
\end{subfigure}

\caption{Gaussian-copula validation on IMDB movie reviews datset (scatter plots): pairwise structural fit diagnostics across temperatures and runs.}
\label{fig:app_imdb_copula_scatter_6up}
\end{figure*}

\begin{figure*}[ht]
\centering
\begin{subfigure}[t]{0.32\textwidth}
    \centering
    \includegraphics[width=0.8\linewidth]{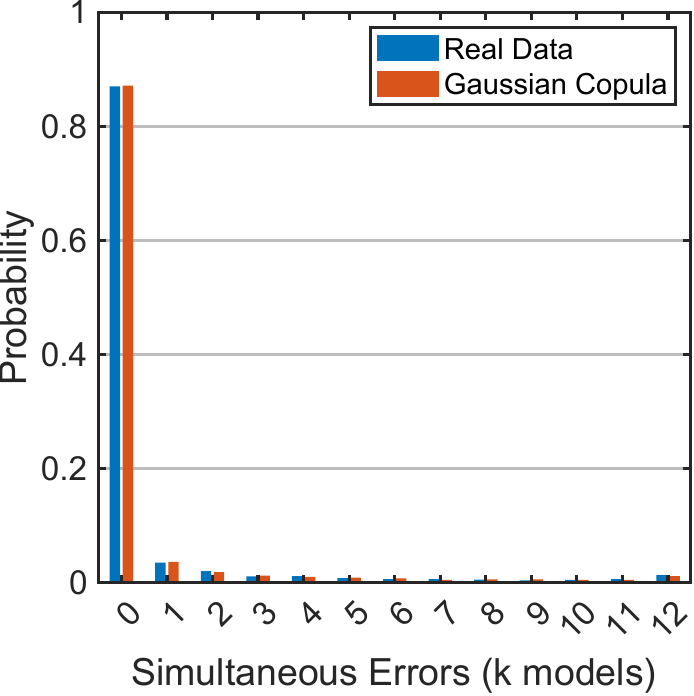}
    \caption{\(\text{temp}=0.01\), run 1}
\end{subfigure}\hfill
\begin{subfigure}[t]{0.32\textwidth}
    \centering
    \includegraphics[width=0.8\linewidth]{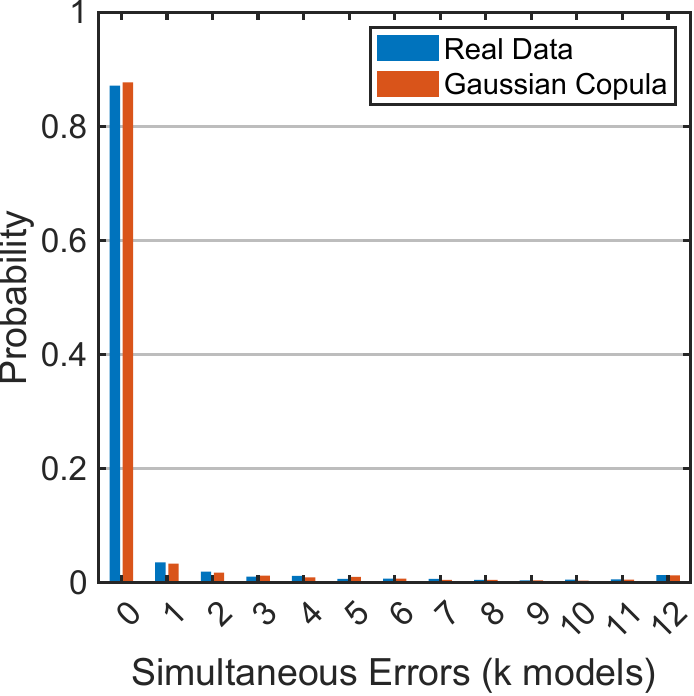}
    \caption{\(\text{temp}=0.01\), run 2}
\end{subfigure}\hfill
\begin{subfigure}[t]{0.32\textwidth}
    \centering
    \includegraphics[width=0.8\linewidth]{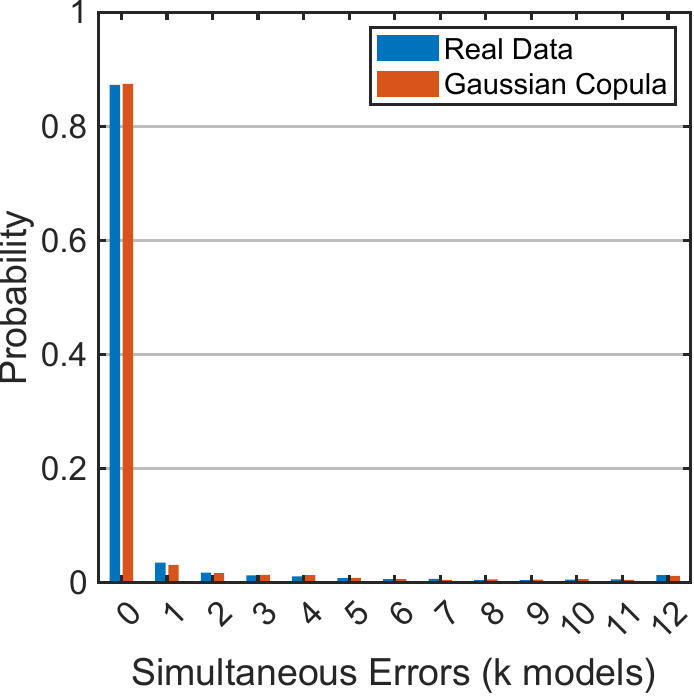}
    \caption{\(\text{temp}=0.3\), run 1}
\end{subfigure}


\begin{subfigure}[t]{0.32\textwidth}
    \centering
    \includegraphics[width=0.8\linewidth]{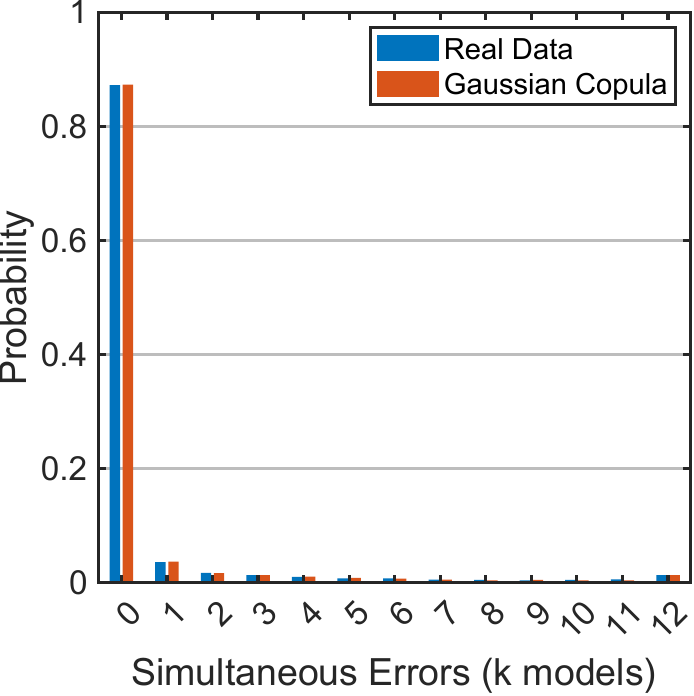}
    \caption{\(\text{temp}=0.3\), run 2}
\end{subfigure}\hfill
\begin{subfigure}[t]{0.32\textwidth}
    \centering
    \includegraphics[width=0.8\linewidth]{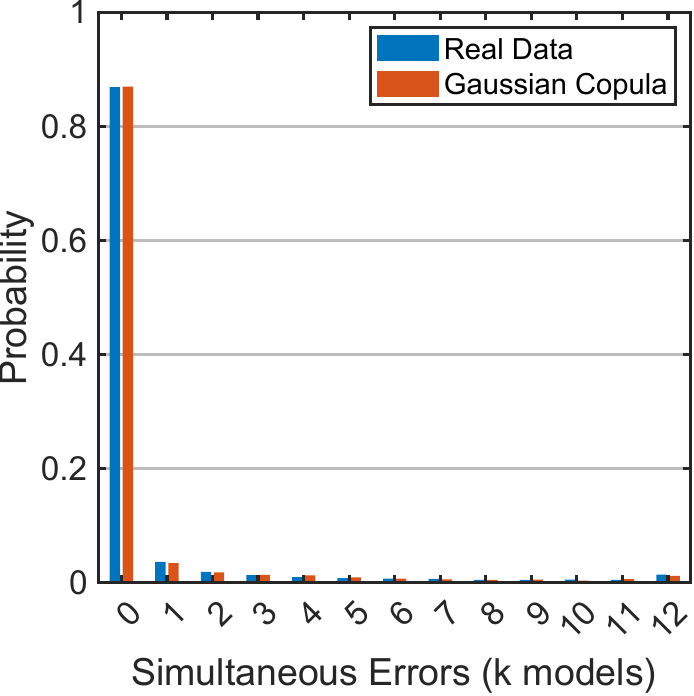}
    \caption{\(\text{temp}=0.7\), run 1}
\end{subfigure}\hfill
\begin{subfigure}[t]{0.32\textwidth}
    \centering
    \includegraphics[width=0.8\linewidth]{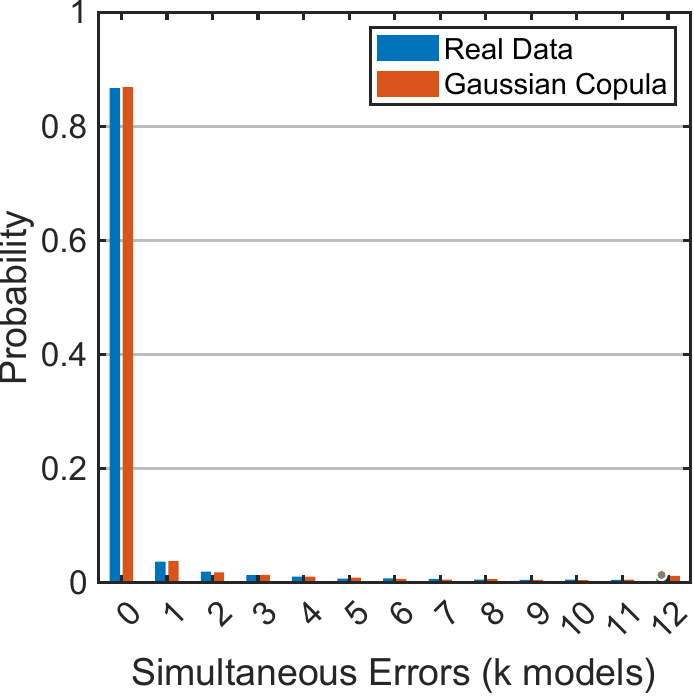}
    \caption{\(\text{temp}=0.7\), run 2}
\end{subfigure}

\caption{Gaussian-copula validation on IMDB movie reviews dataset (simultaneous-error / tail plots): empirical vs copula-predicted probability of \(k\)-way simultaneous errors across temperatures and runs.}
\label{fig:app_imdb_copula_heavy_6u}
\end{figure*}
Figure~\ref{fig:app_imdb_copula_scatter_6up} shows the pairwise fit remains tight along the diagonal across all settings. Figure~\ref{fig:app_imdb_copula_heavy_6u} reveals that the error distribution is heavily concentrated at $k=0$ (reflecting the high model accuracies) with the copula accurately capturing this structure. The near-absence of simultaneous errors ($k \geq 2$) further explains why ensemble selection yields limited gains on this dataset.
\clearpage

\section{Supplementary Tables}
\label{app:tables}
\begin{table}[ht]
\centering
\caption{Test error probability on MEDMCQA averaged over 30 evaluations (3 temperatures $\times$ 2 runs $\times$ 5 folds). Bold indicates the best method for each $k$.}
\label{tab:medmcqa_results}
\small
\begin{tabular}{ccccc}
\toprule
$k$ & Top-$k$ (Acc) & Term 1 (Rel) & Terms 1+2 (mRMR) & Greedy MI (Ours) \\
\midrule
1  & $0.171 \pm 0.008$ & $0.171 \pm 0.008$ & $0.171 \pm 0.008$ & $0.171 \pm 0.008$ \\
2  & $0.172 \pm 0.008$ & $0.172 \pm 0.008$ & $0.171 \pm 0.008$ & $0.171 \pm 0.008$ \\
3  & $0.174 \pm 0.008$ & $0.174 \pm 0.008$ & $0.171 \pm 0.008$ & $\mathbf{0.164 \pm 0.008}$ \\
4  & $0.173 \pm 0.009$ & $0.173 \pm 0.009$ & $0.171 \pm 0.008$ & $\mathbf{0.163 \pm 0.008}$ \\
5  & $0.170 \pm 0.008$ & $0.172 \pm 0.008$ & $0.173 \pm 0.008$ & $\mathbf{0.163 \pm 0.008}$ \\
6  & $0.171 \pm 0.007$ & $0.171 \pm 0.007$ & $0.174 \pm 0.008$ & $\mathbf{0.164 \pm 0.008}$ \\
7  & $0.170 \pm 0.007$ & $0.170 \pm 0.007$ & $0.175 \pm 0.008$ & $\mathbf{0.167 \pm 0.008}$ \\
8  & $0.172 \pm 0.008$ & $0.172 \pm 0.008$ & $0.176 \pm 0.009$ & $\mathbf{0.170 \pm 0.007}$ \\
9  & $0.174 \pm 0.007$ & $\mathbf{0.173 \pm 0.008}$ & $0.177 \pm 0.009$ & $0.174 \pm 0.008$ \\
10 & $\mathbf{0.177 \pm 0.008}$ & $\mathbf{0.177 \pm 0.008}$ & $0.177 \pm 0.009$ & $0.181 \pm 0.009$ \\
11 & $0.182 \pm 0.009$ & $\mathbf{0.182 \pm 0.008}$ & $0.184 \pm 0.008$ & $0.183 \pm 0.008$ \\
12 & $0.188 \pm 0.009$ & $0.188 \pm 0.009$ & $0.188 \pm 0.009$ & $0.188 \pm 0.009$ \\
\bottomrule
\end{tabular}
\end{table}

\begin{table}[ht]
\centering
\caption{Test error probability on MEDMCQA with majority voting (MV) aggregation for baselines versus MAP aggregation for Greedy MI. Results averaged over 30 evaluations. Bold indicates the best method for each $k$.}
\label{tab:medmcqa_mv_vs_map}
\small
\begin{tabular}{ccccc}
\toprule
$k$ & Top-$k$ (MV) & Term 1 (MV) & Terms 1+2 (MV) & Greedy MI (MAP) \\
\midrule
1  & $0.171 \pm 0.008$ & $0.171 \pm 0.008$ & $0.171 \pm 0.008$ & $0.171 \pm 0.008$ \\
2  & $0.181 \pm 0.009$ & $0.181 \pm 0.009$ & $0.264 \pm 0.011$ & $\mathbf{0.171 \pm 0.008}$ \\
3  & $0.172 \pm 0.008$ & $0.172 \pm 0.008$ & $0.231 \pm 0.008$ & $\mathbf{0.164 \pm 0.008}$ \\
4  & $0.181 \pm 0.009$ & $0.179 \pm 0.009$ & $0.221 \pm 0.010$ & $\mathbf{0.163 \pm 0.008}$ \\
5  & $0.170 \pm 0.008$ & $0.173 \pm 0.008$ & $0.220 \pm 0.008$ & $\mathbf{0.163 \pm 0.008}$ \\
6  & $0.170 \pm 0.007$ & $0.170 \pm 0.007$ & $0.204 \pm 0.011$ & $\mathbf{0.164 \pm 0.008}$ \\
7  & $0.172 \pm 0.009$ & $0.172 \pm 0.009$ & $0.214 \pm 0.009$ & $\mathbf{0.167 \pm 0.008}$ \\
8  & $0.173 \pm 0.008$ & $0.173 \pm 0.008$ & $0.202 \pm 0.007$ & $\mathbf{0.170 \pm 0.007}$ \\
9  & $0.181 \pm 0.008$ & $0.187 \pm 0.008$ & $0.207 \pm 0.011$ & $\mathbf{0.174 \pm 0.008}$ \\
10 & $0.183 \pm 0.008$ & $0.183 \pm 0.008$ & $0.188 \pm 0.008$ & $\mathbf{0.181 \pm 0.009}$ \\
11 & $0.194 \pm 0.008$ & $0.195 \pm 0.009$ & $0.189 \pm 0.010$ & $\mathbf{0.183 \pm 0.008}$ \\
12 & $0.189 \pm 0.008$ & $0.189 \pm 0.008$ & $0.189 \pm 0.008$ & $\mathbf{0.188 \pm 0.009}$ \\
\bottomrule
\end{tabular}
\end{table}
\begin{table}[ht]
\centering
\caption{Test error probability on MEDMCQA with weighted majority voting (W-MV) aggregation for baselines versus MAP aggregation for Greedy MI. Results averaged over 30 evaluations. Bold indicates the best method for each $k$.}
\label{tab:medmcqa_wmv_vs_map}
\small
\begin{tabular}{ccccc}
\toprule
$k$ & Top-$k$ (W-MV) & Term 1 (W-MV) & Terms 1+2 (W-MV) & Greedy MI (MAP) \\
\midrule
1  & $0.171 \pm 0.008$ & $0.171 \pm 0.008$ & $0.171 \pm 0.008$ & $0.171 \pm 0.008$ \\
2  & $0.171 \pm 0.008$ & $0.171 \pm 0.008$ & $0.171 \pm 0.008$ & $0.171 \pm 0.008$ \\
3  & $0.172 \pm 0.008$ & $0.172 \pm 0.008$ & $0.173 \pm 0.012$ & $\mathbf{0.164 \pm 0.008}$ \\
4  & $0.172 \pm 0.008$ & $0.172 \pm 0.008$ & $0.204 \pm 0.009$ & $\mathbf{0.163 \pm 0.008}$ \\
5  & $0.170 \pm 0.008$ & $0.173 \pm 0.008$ & $0.190 \pm 0.012$ & $\mathbf{0.163 \pm 0.008}$ \\
6  & $0.170 \pm 0.009$ & $0.170 \pm 0.009$ & $0.197 \pm 0.010$ & $\mathbf{0.164 \pm 0.008}$ \\
7  & $0.172 \pm 0.009$ & $0.172 \pm 0.009$ & $0.189 \pm 0.009$ & $\mathbf{0.167 \pm 0.008}$ \\
8  & $0.172 \pm 0.009$ & $0.172 \pm 0.009$ & $0.197 \pm 0.008$ & $\mathbf{0.170 \pm 0.007}$ \\
9  & $\mathbf{0.172 \pm 0.009}$ & $0.174 \pm 0.009$ & $0.191 \pm 0.011$ & $0.174 \pm 0.008$ \\
10 & $0.181 \pm 0.009$ & $0.181 \pm 0.009$ & $\mathbf{0.179 \pm 0.008}$ & $0.181 \pm 0.009$ \\
11 & $0.180 \pm 0.008$ & $0.181 \pm 0.009$ & $\mathbf{0.177 \pm 0.009}$ & $0.183 \pm 0.008$ \\
12 & $\mathbf{0.182 \pm 0.009}$ & $\mathbf{0.182 \pm 0.009}$ & $\mathbf{0.182 \pm 0.009}$ & $0.188 \pm 0.009$ \\
\bottomrule
\end{tabular}
\end{table}

\begin{table}[ht]
\centering
\caption{Model pool for MEDMCQA: 12 LLMs with individual accuracies and Gaussian-copula parameters. Results averaged over 6 runs (3 temperatures $\times$ 2 runs). Models sorted by accuracy.}
\label{tab:medmcqa_models}
\small
\begin{tabular}{lccc}
\toprule
\textbf{Model} & \textbf{Accuracy} & \textbf{Threshold $\tau_j$} & \textbf{Avg.\ $\rho_j$} \\
\midrule
OpenAI/GPT-5-chat          & $83.1\% \pm 0.1\%$ & $-0.96$ & $0.59$ \\
OpenAI/GPT-4.1             & $82.2\% \pm 0.1\%$ & $-0.92$ & $0.53$ \\
OpenAI/GPT-4o              & $80.4\% \pm 0.3\%$ & $-0.86$ & $0.65$ \\
OpenAI/GPT-4.1-mini        & $77.7\% \pm 0.1\%$ & $-0.76$ & $0.46$ \\
Qwen/Qwen3-235b            & $77.4\% \pm 0.3\%$ & $-0.75$ & $0.63$ \\
Moonshotai/kimi-k2         & $76.3\% \pm 0.6\%$ & $-0.72$ & $0.63$ \\
Google/Gemini-2.5-flash    & $72.9\% \pm 0.4\%$ & $-0.61$ & $0.58$ \\
MistralAI/Mistral-3.2-24b  & $72.2\% \pm 0.3\%$ & $-0.59$ & $0.61$ \\
Anthropic/Claude-3.5-haiku & $70.7\% \pm 0.2\%$ & $-0.54$ & $0.52$ \\
MistralAI/Mistral-3.1-24b  & $69.7\% \pm 0.1\%$ & $-0.52$ & $0.58$ \\
OpenAI/GPT-4.1-nano        & $66.8\% \pm 0.3\%$ & $-0.43$ & $0.51$ \\
Meta-LLaMA/LLaMA-3.1-8b    & $66.0\% \pm 1.5\%$ & $-0.41$ & $0.34$ \\
\midrule
\textbf{Average} & $74.6\%$ & --- & $\bar{\rho} = 0.55$ \\
\bottomrule
\end{tabular}
\end{table}

\begin{figure}[ht]
\label{fig:medmcqa_corr}
  \vskip 0.2in
  \begin{center}
    \centerline{\includegraphics[width=0.7\columnwidth]{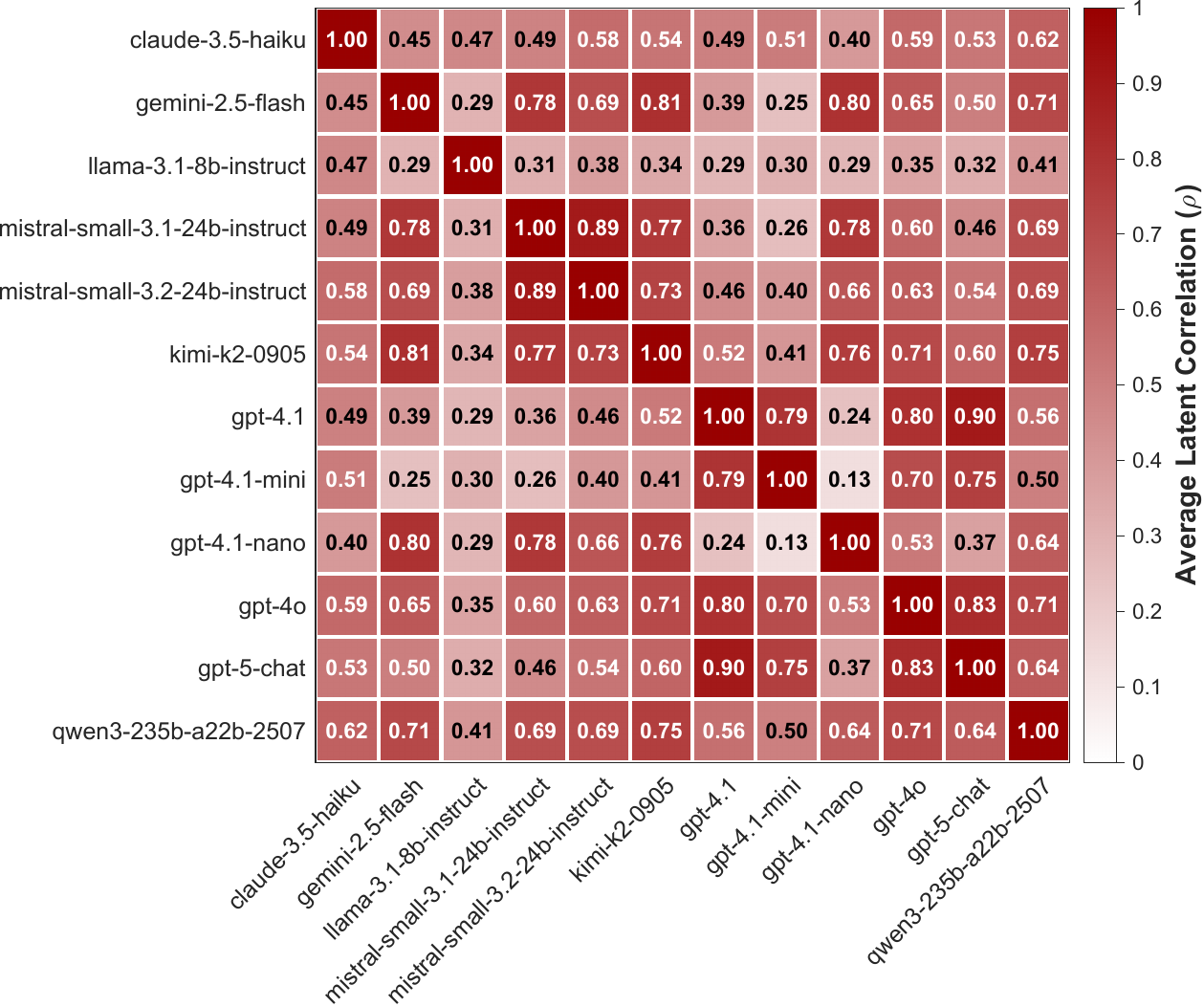}}
    \caption{
Averaged corr. matrix for MEDMCQA for 3 temperature settings and 2 runs.}
    \label{medmcqa_corr}
  \end{center}
\end{figure}

\begin{table}[ht]
\centering
\caption{Test error probability on MMLU averaged over 30 evaluations (3 temperatures $\times$ 2 runs $\times$ 5 folds). Bold indicates the best method for each $k$.}
\label{tab:mmlu_results}
\small
\begin{tabular}{ccccc}
\toprule
$k$ & Top-$k$ (Acc) & Term 1 (Rel) & Terms 1+2 (mRMR) & Greedy MI (Ours) \\
\midrule
1  & $0.151 \pm 0.007$ & $0.151 \pm 0.007$ & $0.151 \pm 0.007$ & $0.151 \pm 0.007$ \\
2  & $0.152 \pm 0.007$ & $0.152 \pm 0.007$ & $0.151 \pm 0.007$ & $0.151 \pm 0.008$ \\
3  & $0.152 \pm 0.007$ & $0.152 \pm 0.008$ & $0.151 \pm 0.007$ & $\mathbf{0.148 \pm 0.008}$ \\
4  & $0.150 \pm 0.005$ & $0.150 \pm 0.005$ & $0.150 \pm 0.007$ & $\mathbf{0.144 \pm 0.007}$ \\
5  & $0.149 \pm 0.007$ & $0.149 \pm 0.007$ & $0.149 \pm 0.008$ & $\mathbf{0.142 \pm 0.008}$ \\
6  & $0.149 \pm 0.007$ & $0.149 \pm 0.007$ & $0.149 \pm 0.008$ & $\mathbf{0.141 \pm 0.007}$ \\
7  & $0.150 \pm 0.007$ & $0.149 \pm 0.007$ & $0.151 \pm 0.008$ & $\mathbf{0.143 \pm 0.007}$ \\
8  & $0.151 \pm 0.006$ & $0.151 \pm 0.006$ & $0.150 \pm 0.007$ & $\mathbf{0.148 \pm 0.007}$ \\
9  & $0.153 \pm 0.007$ & $0.152 \pm 0.007$ & $\mathbf{0.151 \pm 0.008}$ & $\mathbf{0.151 \pm 0.008}$ \\
10 & $0.156 \pm 0.007$ & $0.156 \pm 0.007$ & $0.155 \pm 0.008$ & $\mathbf{0.155 \pm 0.007}$ \\
11 & $\mathbf{0.158 \pm 0.007}$ & $\mathbf{0.158 \pm 0.007}$ & $0.159 \pm 0.008$ & $\mathbf{0.158 \pm 0.007}$ \\
12 & $\mathbf{0.164 \pm 0.008}$ & $\mathbf{0.164 \pm 0.008}$ & $0.166 \pm 0.007$ & $\mathbf{0.164 \pm 0.008}$ \\
13 & $0.171 \pm 0.008$ & $0.171 \pm 0.008$ & $0.171 \pm 0.008$ & $0.171 \pm 0.008$ \\
\bottomrule
\end{tabular}
\end{table}

\begin{table}[ht]
\centering
\caption{Test error probability on MMLU with majority voting (MV) aggregation for baselines versus MAP aggregation for Greedy MI. Results averaged over 30 evaluations. Bold indicates the best method for each $k$.}
\label{tab:mmlu_mv_vs_map}
\small
\begin{tabular}{ccccc}
\toprule
$k$ & Top-$k$ (MV) & Term 1 (MV) & Terms 1+2 (MV) & Greedy MI (MAP) \\
\midrule
1  & $0.151 \pm 0.007$ & $0.151 \pm 0.007$ & $0.151 \pm 0.007$ & $0.151 \pm 0.007$ \\
2  & $0.156 \pm 0.007$ & $0.156 \pm 0.007$ & $0.246 \pm 0.022$ & $\mathbf{0.151 \pm 0.008}$ \\
3  & $0.150 \pm 0.007$ & $0.150 \pm 0.007$ & $0.219 \pm 0.026$ & $\mathbf{0.148 \pm 0.008}$ \\
4  & $0.150 \pm 0.007$ & $0.150 \pm 0.007$ & $0.183 \pm 0.013$ & $\mathbf{0.144 \pm 0.007}$ \\
5  & $0.147 \pm 0.007$ & $0.146 \pm 0.007$ & $0.200 \pm 0.012$ & $\mathbf{0.142 \pm 0.008}$ \\
6  & $0.145 \pm 0.007$ & $0.145 \pm 0.007$ & $0.175 \pm 0.012$ & $\mathbf{0.141 \pm 0.007}$ \\
7  & $0.149 \pm 0.007$ & $0.150 \pm 0.007$ & $0.183 \pm 0.015$ & $\mathbf{0.143 \pm 0.007}$ \\
8  & $0.150 \pm 0.006$ & $0.150 \pm 0.006$ & $0.170 \pm 0.009$ & $\mathbf{0.148 \pm 0.007}$ \\
9  & $0.155 \pm 0.005$ & $0.155 \pm 0.006$ & $0.176 \pm 0.009$ & $\mathbf{0.151 \pm 0.008}$ \\
10 & $0.155 \pm 0.005$ & $0.155 \pm 0.005$ & $0.167 \pm 0.008$ & $0.155 \pm 0.007$ \\
11 & $0.165 \pm 0.007$ & $0.165 \pm 0.007$ & $0.172 \pm 0.009$ & $\mathbf{0.158 \pm 0.007}$ \\
12 & $0.164 \pm 0.008$ & $0.164 \pm 0.008$ & $0.164 \pm 0.008$ & $0.164 \pm 0.008$ \\
13 & $0.174 \pm 0.008$ & $0.174 \pm 0.008$ & $0.174 \pm 0.008$ & $\mathbf{0.171 \pm 0.008}$ \\
\bottomrule
\end{tabular}
\end{table}

\begin{table}[ht]
\centering
\caption{Test error probability on MMLU with weighted majority voting (W-MV) aggregation for baselines versus MAP aggregation for Greedy MI. Results averaged over 30 evaluations. Bold indicates the best method for each $k$.}
\label{tab:mmlu_wmv_vs_map}
\small
\begin{tabular}{ccccc}
\toprule
$k$ & Top-$k$ (W-MV) & Term 1 (W-MV) & Terms 1+2 (W-MV) & Greedy MI (MAP) \\
\midrule
1  & $0.151 \pm 0.007$ & $0.151 \pm 0.007$ & $0.151 \pm 0.007$ & $0.151 \pm 0.007$ \\
2  & $0.151 \pm 0.007$ & $0.151 \pm 0.007$ & $0.151 \pm 0.007$ & $0.151 \pm 0.008$ \\
3  & $0.150 \pm 0.007$ & $0.150 \pm 0.007$ & $0.158 \pm 0.015$ & $\mathbf{0.148 \pm 0.008}$ \\
4  & $0.149 \pm 0.007$ & $0.149 \pm 0.007$ & $0.174 \pm 0.009$ & $\mathbf{0.144 \pm 0.007}$ \\
5  & $0.147 \pm 0.007$ & $0.146 \pm 0.007$ & $0.174 \pm 0.011$ & $\mathbf{0.142 \pm 0.008}$ \\
6  & $0.146 \pm 0.006$ & $0.146 \pm 0.007$ & $0.167 \pm 0.009$ & $\mathbf{0.141 \pm 0.007}$ \\
7  & $0.149 \pm 0.007$ & $0.150 \pm 0.007$ & $0.164 \pm 0.012$ & $\mathbf{0.143 \pm 0.007}$ \\
8  & $0.149 \pm 0.006$ & $0.149 \pm 0.006$ & $0.164 \pm 0.009$ & $\mathbf{0.148 \pm 0.007}$ \\
9  & $0.154 \pm 0.005$ & $0.154 \pm 0.006$ & $0.162 \pm 0.008$ & $\mathbf{0.151 \pm 0.008}$ \\
10 & $\mathbf{0.153 \pm 0.005}$ & $0.154 \pm 0.006$ & $0.162 \pm 0.008$ & $0.155 \pm 0.007$ \\
11 & $\mathbf{0.156 \pm 0.007}$ & $\mathbf{0.156 \pm 0.007}$ & $0.158 \pm 0.008$ & $0.158 \pm 0.007$ \\
12 & $0.161 \pm 0.008$ & $0.161 \pm 0.008$ & $\mathbf{0.158 \pm 0.007}$ & $0.164 \pm 0.008$ \\
13 & $\mathbf{0.161 \pm 0.008}$ & $\mathbf{0.161 \pm 0.008}$ & $\mathbf{0.161 \pm 0.008}$ & $0.171 \pm 0.008$ \\
\bottomrule
\end{tabular}
\end{table}

\begin{figure}[ht]
  \vskip 0.2in
  \begin{center}
    \centerline{\includegraphics[width=0.7\columnwidth]{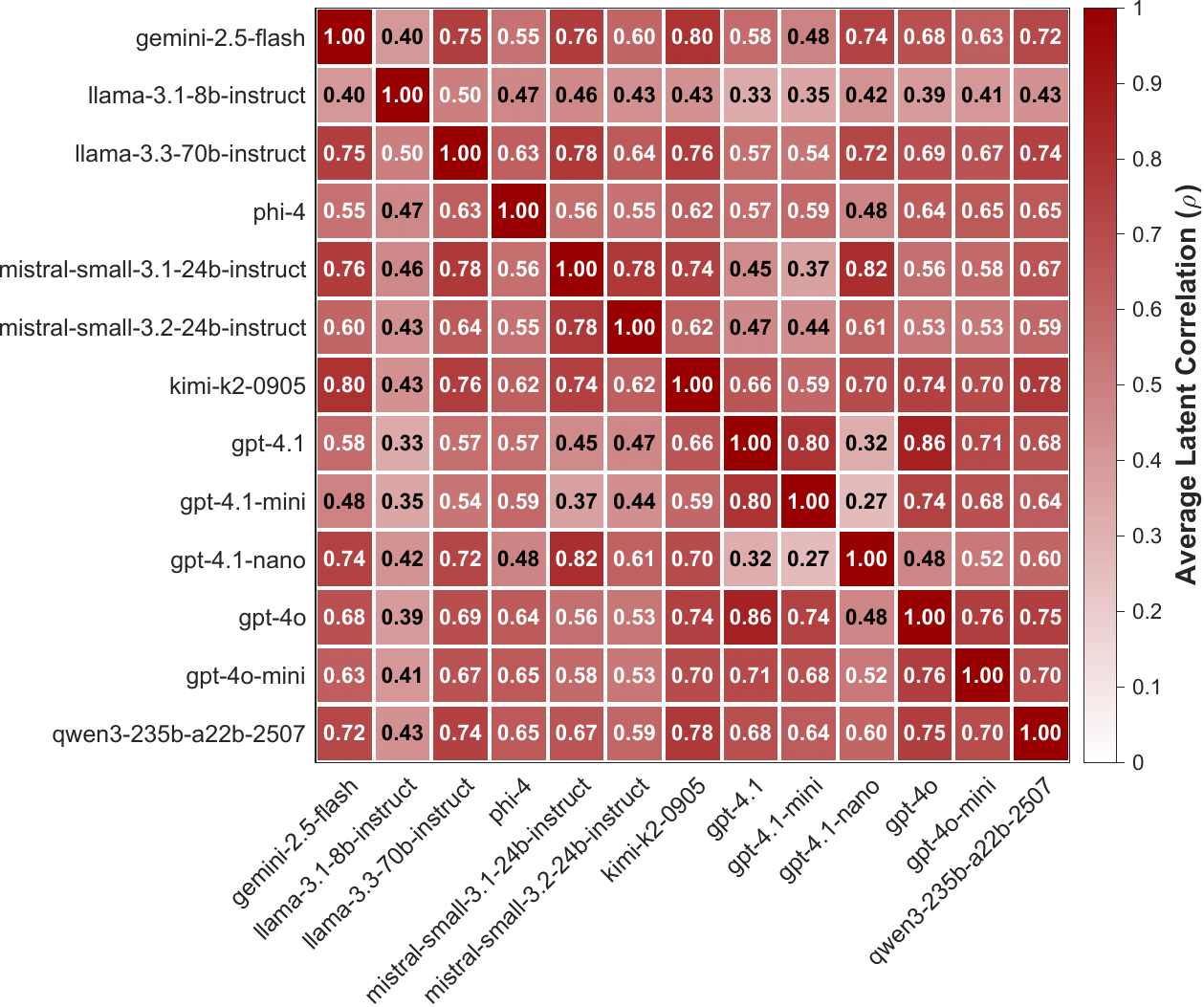}}
    \caption{
Averaged corr. matrix for MMLU for 3 temperature settings and 2 runs. }
    \label{mmlu_corr}
  \end{center}
\end{figure}

\begin{table}[ht]
\centering
\caption{Model pool for MMLU: 13 LLMs with individual accuracies and Gaussian-copula parameters. Results averaged over 6 runs (3 temperatures $\times$ 2 runs). Models sorted by accuracy.}
\label{tab:mmlu_models}
\small
\begin{tabular}{lccc}
\toprule
\textbf{Model} & \textbf{Accuracy} & \textbf{Threshold $\tau_j$} & \textbf{Avg.\ $\rho_j$} \\
\midrule
OpenAI/GPT-4.1             & $85.0\% \pm 0.2\%$ & $-1.04$ & $0.58$ \\
OpenAI/GPT-4o              & $84.0\% \pm 0.3\%$ & $-0.99$ & $0.65$ \\
Qwen/Qwen3-235b            & $82.2\% \pm 0.2\%$ & $-0.92$ & $0.66$ \\
Moonshotai/kimi-k2         & $82.1\% \pm 0.8\%$ & $-0.92$ & $0.68$ \\
OpenAI/GPT-4.1-mini        & $81.6\% \pm 0.2\%$ & $-0.90$ & $0.54$ \\
Google/Gemini-2.5-flash    & $79.8\% \pm 0.4\%$ & $-0.83$ & $0.64$ \\
OpenAI/GPT-4o-mini         & $78.3\% \pm 0.2\%$ & $-0.78$ & $0.63$ \\
Meta-LLaMA/LLaMA-3.3-70b   & $77.2\% \pm 0.1\%$ & $-0.75$ & $0.67$ \\
microsoft/phi-4            & $74.1\% \pm 0.5\%$ & $-0.65$ & $0.58$ \\
MistralAI/Mistral-3.2-24b  & $73.2\% \pm 0.5\%$ & $-0.62$ & $0.57$ \\
MistralAI/Mistral-3.1-24b  & $71.8\% \pm 0.1\%$ & $-0.58$ & $0.63$ \\
OpenAI/GPT-4.1-nano        & $69.2\% \pm 0.2\%$ & $-0.50$ & $0.56$ \\
Meta-LLaMA/LLaMA-3.1-8b    & $65.9\% \pm 1.3\%$ & $-0.41$ & $0.42$ \\
\midrule
\textbf{Average} & $77.3\%$ & --- & $\bar{\rho} = 0.60$ \\
\bottomrule
\end{tabular}
\end{table}

\begin{table}[ht]
\centering
\caption{Test error probability on IMDB movie reviews averaged over 30 evaluations (3 temperatures $\times$ 2 runs $\times$ 5 splits). Bold indicates the best method for each $k$.}
\label{tab:imdb_results}
\small
\begin{tabular}{ccccc}
\toprule
$k$ & Top-$k$ (Acc) & Term 1 (Rel) & Terms 1+2 (mRMR) & Greedy MI (Ours) \\
\midrule
1  & $0.0354 \pm 0.0043$ & $0.0351 \pm 0.0042$ & $0.0351 \pm 0.0042$ & $0.0351 \pm 0.0042$ \\
2  & $0.0363 \pm 0.0044$ & $0.0362 \pm 0.0044$ & $0.0351 \pm 0.0042$ & $0.0354 \pm 0.0045$ \\
3  & $\mathbf{0.0339 \pm 0.0036}$ & $\mathbf{0.0339 \pm 0.0036}$ & $0.0347 \pm 0.0040$ & $0.0358 \pm 0.0043$ \\
4  & $0.0360 \pm 0.0041$ & $0.0360 \pm 0.0041$ & $0.0345 \pm 0.0041$ & $\mathbf{0.0340 \pm 0.0036}$ \\
5  & $0.0377 \pm 0.0042$ & $0.0377 \pm 0.0042$ & $0.0341 \pm 0.0039$ & $\mathbf{0.0335 \pm 0.0039}$ \\
6  & $0.0387 \pm 0.0038$ & $0.0387 \pm 0.0038$ & $0.0344 \pm 0.0040$ & $\mathbf{0.0327 \pm 0.0039}$ \\
7  & $0.0397 \pm 0.0038$ & $0.0392 \pm 0.0041$ & $0.0349 \pm 0.0044$ & $\mathbf{0.0335 \pm 0.0039}$ \\
8  & $0.0398 \pm 0.0040$ & $0.0398 \pm 0.0040$ & $0.0360 \pm 0.0047$ & $\mathbf{0.0350 \pm 0.0041}$ \\
9  & $0.0392 \pm 0.0041$ & $0.0392 \pm 0.0039$ & $0.0376 \pm 0.0039$ & $\mathbf{0.0367 \pm 0.0044}$ \\
10 & $0.0403 \pm 0.0042$ & $0.0403 \pm 0.0042$ & $0.0391 \pm 0.0041$ & $\mathbf{0.0382 \pm 0.0043}$ \\
11 & $0.0411 \pm 0.0044$ & $0.0411 \pm 0.0044$ & $\mathbf{0.0404 \pm 0.0043}$ & $0.0405 \pm 0.0042$ \\
12 & $0.0413 \pm 0.0046$ & $0.0413 \pm 0.0046$ & $0.0413 \pm 0.0046$ & $0.0413 \pm 0.0046$ \\
\bottomrule
\end{tabular}
\end{table}

\begin{table}[ht]
\centering
\caption{Test error probability on IMDB movie reviews with majority voting (MV) aggregation for baselines versus MAP aggregation for Greedy MI. Results averaged over 30 evaluations. Bold indicates the best method for each $k$.}
\label{tab:imdb_mv_vs_map}
\small
\begin{tabular}{ccccc}
\toprule
$k$ & Top-$k$ (MV) & Term 1 (MV) & Terms 1+2 (MV) & Greedy MI (MAP) \\
\midrule
1  & $0.0354 \pm 0.0043$ & $0.0351 \pm 0.0042$ & $0.0351 \pm 0.0042$ & $0.0351 \pm 0.0042$ \\
2  & $0.0351 \pm 0.0038$ & $0.0352 \pm 0.0040$ & $0.0352 \pm 0.0042$ & $0.0354 \pm 0.0045$ \\
3  & $\mathbf{0.0340 \pm 0.0037}$ & $\mathbf{0.0340 \pm 0.0037}$ & $0.0395 \pm 0.0043$ & $0.0358 \pm 0.0043$ \\
4  & $0.0356 \pm 0.0041$ & $0.0357 \pm 0.0040$ & $0.0373 \pm 0.0040$ & $\mathbf{0.0340 \pm 0.0036}$ \\
5  & $0.0357 \pm 0.0039$ & $0.0357 \pm 0.0039$ & $0.0377 \pm 0.0038$ & $\mathbf{0.0335 \pm 0.0039}$ \\
6  & $0.0362 \pm 0.0036$ & $0.0362 \pm 0.0036$ & $0.0364 \pm 0.0039$ & $\mathbf{0.0327 \pm 0.0039}$ \\
7  & $0.0362 \pm 0.0040$ & $0.0358 \pm 0.0041$ & $0.0378 \pm 0.0039$ & $\mathbf{0.0335 \pm 0.0039}$ \\
8  & $0.0362 \pm 0.0041$ & $0.0362 \pm 0.0041$ & $0.0369 \pm 0.0040$ & $\mathbf{0.0350 \pm 0.0041}$ \\
9  & $0.0373 \pm 0.0039$ & $0.0375 \pm 0.0038$ & $0.0386 \pm 0.0038$ & $\mathbf{0.0367 \pm 0.0044}$ \\
10 & $0.0377 \pm 0.0039$ & $0.0377 \pm 0.0039$ & $\mathbf{0.0373 \pm 0.0038}$ & $0.0382 \pm 0.0043$ \\
11 & $0.0402 \pm 0.0041$ & $0.0402 \pm 0.0041$ & $\mathbf{0.0401 \pm 0.0039}$ & $0.0405 \pm 0.0042$ \\
12 & $\mathbf{0.0401 \pm 0.0040}$ & $\mathbf{0.0401 \pm 0.0040}$ & $\mathbf{0.0401 \pm 0.0040}$ & $0.0413 \pm 0.0046$ \\
\bottomrule
\end{tabular}
\end{table}

\begin{table}[ht]
\centering
\caption{Test error probability on IMDB movie reviews with weighted majority voting (W-MV) aggregation for baselines versus MAP aggregation for Greedy MI. Results averaged over 30 evaluations. Bold indicates the best method for each $k$.}
\label{tab:imdb_wmv_vs_map}
\small
\begin{tabular}{ccccc}
\toprule
$k$ & Top-$k$ (W-MV) & Term 1 (W-MV) & Terms 1+2 (W-MV) & Greedy MI (MAP) \\
\midrule
1  & $0.0354 \pm 0.0043$ & $0.0351 \pm 0.0042$ & $0.0351 \pm 0.004$ & $0.0351 \pm 0.0042$ \\
2  & $0.0354 \pm 0.0043$ & $0.0354 \pm 0.0043$ & $0.0351 \pm 0.0042$ & $0.0354 \pm 0.0045$ \\
3  & $\mathbf{0.0340 \pm 0.0037}$ & $\mathbf{0.0340 \pm 0.0037}$ & $0.0395 \pm 0.0043$ & $0.0358 \pm 0.0043$ \\
4  & $0.0341 \pm 0.0036$ & $0.0341 \pm 0.0036$ & $0.0354 \pm 0.0039$ & $\mathbf{0.0340 \pm 0.0036}$ \\
5  & $0.0357 \pm 0.0039$ & $0.0357 \pm 0.0039$ & $0.0377 \pm 0.0038$ & $\mathbf{0.0335 \pm 0.0039}$ \\
6  & $0.0352 \pm 0.0039$ & $0.0352 \pm 0.0039$ & $0.0357 \pm 0.0037$ & $\mathbf{0.0327 \pm 0.0039}$ \\
7  & $0.0362 \pm 0.0040$ & $0.0358 \pm 0.0041$ & $0.0378 \pm 0.0039$ & $\mathbf{0.0335 \pm 0.0039}$ \\
8  & $0.0353 \pm 0.0039$ & $0.0353 \pm 0.0039$ & $0.0367 \pm 0.0041$ & $\mathbf{0.0350 \pm 0.0041}$ \\
9  & $0.0373 \pm 0.0039$ & $0.0375 \pm 0.0038$ & $0.0386 \pm 0.0038$ & $\mathbf{0.0367 \pm 0.0044}$ \\
10 & $0.0371 \pm 0.0038$ & $0.0371 \pm 0.0038$ & $\mathbf{0.0369 \pm 0.0037}$ & $0.0382 \pm 0.0043$ \\
11 & $0.0402 \pm 0.0041$ & $0.0402 \pm 0.0041$ & $\mathbf{0.0401 \pm 0.0039}$ & $0.0405 \pm 0.0042$ \\
12 & $\mathbf{0.0397 \pm 0.0039}$ & $\mathbf{0.0397 \pm 0.0039}$ & $\mathbf{0.0397 \pm 0.0039}$ & $0.0413 \pm 0.0046$ \\
\bottomrule
\end{tabular}
\end{table}

\begin{figure}[ht]
\label{fig:imdb_corr}
  \begin{center}    \centerline{\includegraphics[width=0.65\columnwidth]{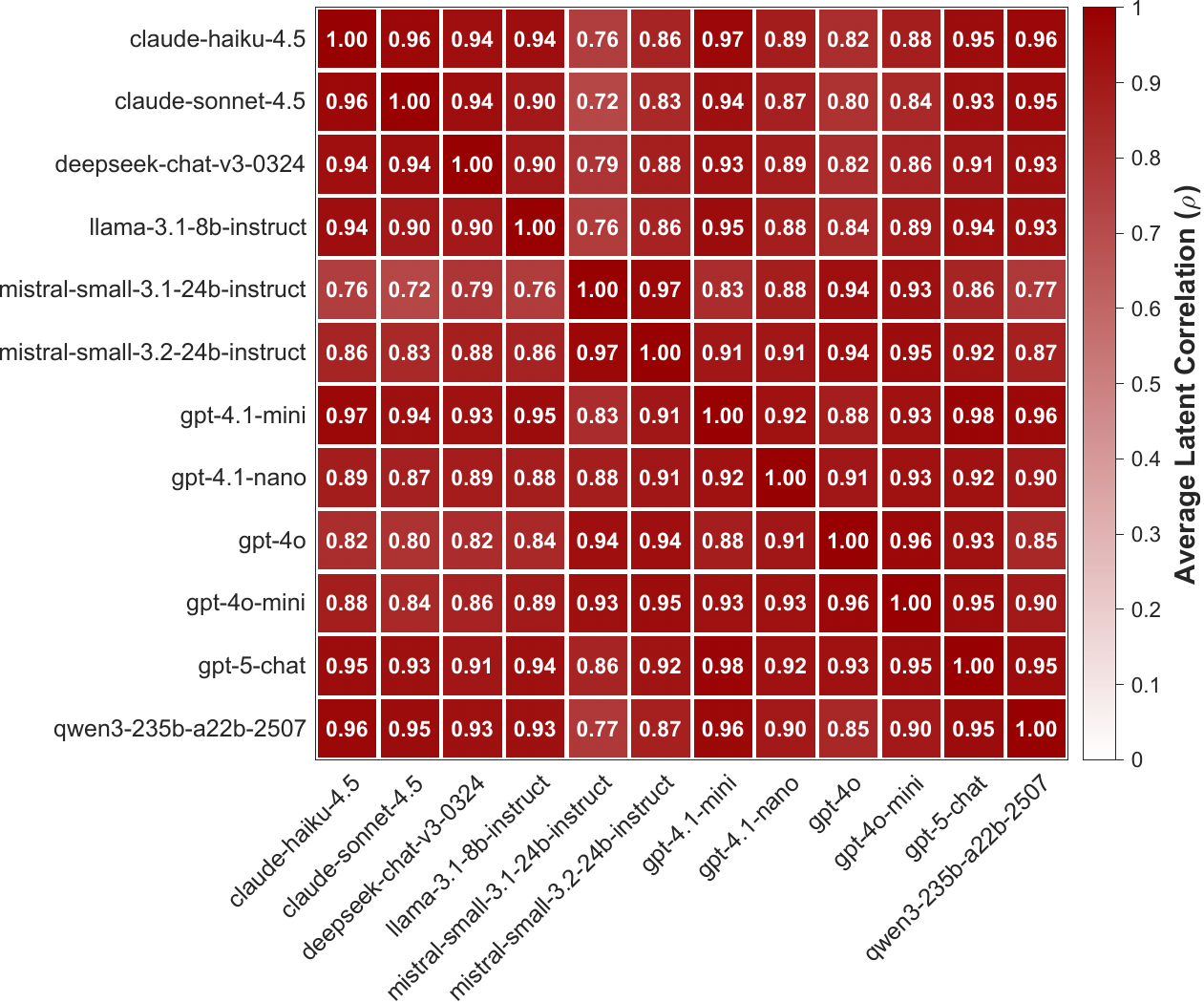}}
    \caption{
Averaged corr. matrix for IMDB movie reviews for 3 temperature settings and 2 runs.}
    \label{imdb_corr}
  \end{center}
\end{figure}

\begin{table}[ht]
\centering
\caption{Model pool for IMDB movie reviews dataset: 12 LLMs with individual accuracies and Gaussian-copula parameters. Results averaged over 6 runs (3 temperatures $\times$ 2 runs). Models sorted by accuracy.}
\label{tab:models_imdb}
\small
\begin{tabular}{lccc}
\toprule
\textbf{Model} & \textbf{Accuracy} & \textbf{Threshold $\tau_j$} & \textbf{Avg.\ $\rho_j$} \\
\midrule
Anthropic/Claude-sonnet-4.5    & $96.39\% \pm 0.05\%$ & $-1.798$ & $0.881$ \\
OpenAI/GPT-5-chat              & $96.26\% \pm 0.07\%$ & $-1.782$ & $0.930$ \\
Anthropic/Claude-haiku-4.5     & $96.10\% \pm 0.08\%$ & $-1.762$ & $0.902$ \\
OpenAI/GPT-4.1-mini            & $95.85\% \pm 0.08\%$ & $-1.733$ & $0.926$ \\
Qwen/Qwen3-235b                & $95.67\% \pm 0.09\%$ & $-1.714$ & $0.907$ \\
deepseek/deepseek-chat-v3      & $95.37\% \pm 0.10\%$ & $-1.682$ & $0.891$ \\
Meta-LLaMA/LLaMA-3.1-8b        & $94.79\% \pm 0.17\%$ & $-1.625$ & $0.890$ \\
OpenAI/GPT-4o                  & $94.49\% \pm 0.05\%$ & $-1.597$ & $0.880$ \\
OpenAI/GPT-4.1-nano            & $94.19\% \pm 0.07\%$ & $-1.571$ & $0.900$ \\
OpenAI/GPT-4o-mini             & $94.00\% \pm 0.07\%$ & $-1.555$ & $0.909$ \\
MistralAI/Mistral-small-3.2-24b & $93.60\% \pm 0.18\%$ & $-1.522$ & $0.900$ \\
MistralAI/Mistral-small-3.1-24b & $91.81\% \pm 0.06\%$ & $-1.393$ & $0.84$ \\
\midrule
\textbf{Average} & $94.54\%$ & --- & $\bar{\rho} = 0.90$ \\
\bottomrule
\end{tabular}
\end{table}

\end{document}